\documentclass[journal]{IEEEtran}

\usepackage{amssymb}
\usepackage{supertabular}

\usepackage{tikz}

\usepackage{pifont}

\usepackage{bbm}
\usepackage{rotating}
\usepackage{dsfont}

\usepackage{bm}
\usepackage{makecell}
\usepackage{array}
\usepackage{eucal}
\usepackage{graphicx}
\usepackage{epstopdf}
\usepackage{amsfonts,amsthm}
\usepackage{setspace}
\usepackage{cite}
\usepackage[colorlinks,linkcolor=blue,anchorcolor=blue,citecolor=blue]{hyperref}
\usepackage{subfig}
\usepackage{subfloat}
\usepackage{multirow,booktabs}

\usepackage{amsmath,mathrsfs}
\usepackage[misc]{ifsym}

\newcommand{\scr}[1]{\mathcal{#1}}%???

\usepackage{amsmath,graphicx}
\usepackage{epstopdf}
\usepackage{amsfonts,amsthm}
\usepackage{algorithmic}
\usepackage[ruled,linesnumbered,boxed]{algorithm2e}
\usepackage{setspace}
\usepackage{cite}
\usepackage[colorlinks,linkcolor=blue,anchorcolor=blue,citecolor=blue]{hyperref}
\usepackage{subfig}
\usepackage{subfloat}
\usepackage{multirow,booktabs}

\usepackage{amsmath}
\usepackage{amssymb}
\usepackage[misc]{ifsym}
\usepackage{url}

\usepackage{color,soul}
\definecolor{myred}{RGB}{255,0,0}
\definecolor{myblue}{RGB}{0,0,255}
\usepackage{threeparttable}

\usepackage{hyperref}

\usepackage{caption}

\hypersetup{hypertex=true,
            colorlinks=true,
            linkcolor=red,
            anchorcolor=green,
            citecolor=blue}

\newtheorem{proposition}{Proposition}[section]
\newtheorem{Definition}{Definition}[section]
\newtheorem{Theorem}{Theorem}[section]
\newtheorem{Remark}{Remark}[section]
\newtheorem{Lemma}{Lemma}[section]

\newtheorem{property}{Property}[section]

\ifCLASSINFOpdf
\else
\fi

\begin{document}

\title{Low-Rank Tensor Completion Based on Fractional Regularization with Ky Fan $p$-$k$ Norm}
%%%%%%%%%%%%%%%%%%%%%%%%%%%%%%%%%%%%%%%%%%%%%%%%%%%%%%%%%%%%

\vspace{-1.75cm}
\author{
Shan~Fan,
~Feng~Zhang,
~Jianjun~Wang,~\IEEEmembership{Member,~IEEE,} 
~Xi-Le~Zhao,~\IEEEmembership{Senior Member,~IEEE,}
and
~Tingwen~Huang,~\IEEEmembership{Fellow,~IEEE}
%%%%%%%%%%%%%%%%%%%%%%%%%%%%%%%%%%%%%%%%%%%%%%%%%%%%%%%%%%%%%%%%%%%%%%%%%%%%%%%%%%%%%%%%%
\vspace{-0.75cm}
\thanks{
%This work was supported in part by the National Key Research and Development Program of China under Grant 2023YFA1008502; in part by Fundamental Research Funds for the Central Universities under Grant SWU-KR25013; in part by National Natural Science Foundation of China under Grant 12101512; and in part by the Graduate Research and Innovation Project of Southwest University under Grant SWUS26072.
This work was supported in part by the National Key Research and Development Program of China under Grant 2023YFA1008502; in part by the Graduate Research and Innovation Project of Southwest University under Grant SWUS26072; in part by Fundamental Research Funds for the Central Universities under Grant SWU-KR25013; and in part by National Natural Science Foundation of China under Grant 12101512.
 (Corresponding author: Feng Zhang.)
}

\thanks{Shan Fan, Feng Zhang, and Jianjun Wang are
	with the School of Mathematics and Statistics, Southwest University, Chongqing 400715, China (e-mail:
	13890745311@163.com,  zfmath@swu.edu.cn, wjj@swu.edu.cn).} %
\thanks{
	Xi-Le Zhao is  with the School of Mathematical Sciences/Research Center for Image and Vision Computing,
	University of Electronic Science and Technology of China, Chengdu 611731, China (e-mail: xlzhao122003@163.com).
}
%\thanks{
%	Tingwen Huang is with the Department of Mathematics, Texas A\&M 
%	University at Qatar, Doha, Qatar (e-mail: huangtingwen@suat-sz.edu.cn;
%	huangtw2024@163.com).
\thanks{
	Tingwen Huang is with the Faculty of Computer Science and Control Engineering, Shenzhen University of Advanced Technology, Shenzhen 518055,
	China (e-mail: huangtingwen2013@gmail.com).
}
\vspace{-0.5cm}
}
%%%%??
%\markboth{Journal of \LaTeX\ Class Files,~Vol.~, No.~, XX~XXXX}%
%  \vspace{-0.5cm}
%%%%%%%%%%%%%%%%%
\maketitle
\begin{abstract}
This paper addresses low-rank tensor completion (LRTC) by proposing a novel nonconvex surrogate, namely the ratio of the tensor nuclear norm to the tensor Ky Fan $p$-$k$ norm (TNPK), to accurately approximate the tensor tubal rank.
%The TNPK possesses appealing properties including scale invariance, truncation capability, and parameter flexibility, and it degenerates into the ratio of the tensor nuclear norm to the tensor Ky Fan $k$ norm (TNK) or the ratio of the tensor nuclear norm to the tensor Frobenius norm (TNF) with specific parameter settings of $p$ and $k$.
%The TNPK possesses appealing properties, including scale invariance, parameter flexibility, low-rank promotion, and truncation capability.
The TNPK possesses appealing properties, including scale invariance, parameter flexibility,
%low-rank promotion,
and the existence of closed-form solutions under specific choices of $p$ and $k$.
With specific parameter settings of $p$ and $k$, it reduces to the ratio of the tensor nuclear norm to the tensor Ky Fan $k$ norm (TNK) or the ratio of the tensor nuclear norm to the tensor Frobenius norm (TNF).
We construct a LRTC model and, under the tensor null space property (NSP), prove that low-rank tensors are local minimizers of the proposed model.
%We construct a low-rank tensor completion model and theoretically verify its rank-promoting capability via the tensor null space property (NSP), proving that low-rank tensors are local minimizers of the model.
Moreover, we derive the proximal operator of the Ky Fan $p$-$k$ inverse-norm and further develop an efficient alternating direction method of multipliers (ADMM) algorithm with guaranteed subsequential convergence under mild conditions. Extensive experiments on synthetic and real-world datasets validate the superior performance of our method against state-of-the-art competitors.
\end{abstract}
\vspace{-0.1cm}
\IEEEpeerreviewmaketitle

\vspace{-0.203cm}
\section{\textbf{Introduction}}
\IEEEPARstart{H}{igh-dimensional} data, such as color images, multispectral images, and magnetic resonance imaging (MRI) images, naturally admit tensor representations. As higher-order extensions of matrices, tensors have been widely used in computer vision \cite{lin2024tensor,jiang2023nonnegative}, neuroscience \cite{beckmann2005tensorial}, machine learning \cite{xu2024dual,chen2023tensor}, and image processing \cite{marquez2020compressive,lin2022robust}. However, tensor data are often incomplete or corrupted during acquisition, storage, and transmission, which severely degrades the performance of subsequent analysis and applications. To address this issue, tensor completion (TC) has attracted considerable attention in recent years \cite{liu2012tensor}. Since TC is typically an ill-posed and ill-conditioned inverse problem, appropriate prior information is essential for reliable recovery. Among various priors, low-rank structure has proven to be one of the most effective and widely adopted assumptions \cite{wang2021generalized}. Motivated by this observation, this paper focuses on low-rank tensor completion (LRTC), whose goal is to recover an underlying low-rank tensor from partially observed or corrupted measurements. The mathematical formulation is given as follows:
\begin{equation}
	\min_\mathcal{X} \mathrm{rank}(\mathcal{X}) \quad  \text{s.t.} \quad \mathcal{P} _\Omega(\mathcal{X} )=\mathcal{P}_\Omega(\mathcal{M} ) ,
	\label{1-1}
\end{equation}
where $\mathrm{rank}(\cdot) $ denotes the rank function, $ \mathcal{M} $ represents the observed incomplete tensor, $ \mathcal{X} $ is the target tensor to be recovered, $ \Omega $ denotes the set of observed indices, and $ \mathcal{P}_{\Omega} $ is the projection operator.

The rank regularization in \eqref{1-1} encodes the intrinsic low-rank prior of real-world tensor data. Unlike matrices, however, tensor rank admits no unified definition. Existing tensor ranks are typically tied to distinct tensor decomposition frameworks, including the CANDECOMP/PARAFAC (CP) rank \cite{liu2012tensor}, Tucker rank \cite{kolda2009tensor}, Tensor Train (TT) rank \cite{oseledets2011tensor}, and Tensor Ring (TR) rank \cite{zhao2016tensor}. Nevertheless, CP rank is generally NP-hard to compute, Tucker rank depends on matricization operations that may destroy intrinsic tensor structures, while TT and TR ranks contain multiple coupled rank parameters, which complicates optimization and model selection. In contrast, the tensor singular value decomposition (t-SVD) framework defines tubal rank \cite{kilmer2013third} directly in the transform domain. It better preserves multidimensional correlations and possesses desirable properties analogous to matrix SVD, such as the Eckart-Young-like theorem for optimal low-rank tensor approximation. Owing to these merits, methods based on tubal rank have achieved remarkable performance in diverse tensor recovery tasks.%可引用文献

Despite the appealing properties of tubal rank, directly minimizing it remains computationally intractable owing to its inherent discontinuity and nonconvexity. Consequently, various surrogate functions have been developed for tubal-rank-based tensor recovery. Among them, the tensor nuclear norm (TNN) \cite{lu2019tensor}, defined as the convex envelope of the tensor average rank within the unit ball of the tensor spectral norm, is one of the most widely adopted relaxations. The resulting optimization model is formulated as
\[	\min_\mathcal{X} \left \| \mathcal{X} \right \|_*  \quad  \text{s.t.} \quad \mathcal{P} _\Omega(\mathcal{X} )=\mathcal{P}_\Omega(\mathcal{M} ), \]
where $\|\cdot\|_*$ denotes the tensor nuclear norm.
Based on the t-SVD, the TNN essentially sums up all the singular values of $\mathcal{X}$, and thus serves as an $\ell_1$ type convex surrogate that penalizes every singular value equally.
%Although TNN-based methods yield promising performance, TNN imposes uniform shrinkage to all singular values and thus tends to over-penalize dominant singular components, leading to biased low-rank estimation.
Although TNN-based methods achieve promising performance, this equal treatment over-penalizes the dominant singular values and thus leads to biased low-rank estimation.
To alleviate this issue, numerous nonconvex surrogates have been proposed, including the partial sum of tensor nuclear norm (PSTNN) \cite{jiang2020multi}, the weighted t-TNN (W-t-TNN) \cite{mu2020weighted}, and the iterative reweighted t-TNN (IR-t-TNN) \cite{wang2021generalized}, among others \cite{xu2019laplace,chen2020robust}. By introducing adaptive penalization across singular values, these methods offer a more accurate approximation of the tubal rank and deliver superior recovery performance, as corroborated by the experimental results in Section \ref{sec6}. More recently, growing attention has been directed toward more flexible composite surrogate formulations, specifically difference-based and fractional-based variants. Compared with conventional single-norm regularization, these composite surrogates afford greater flexibility for approximating the tubal rank and preserving the intrinsic low-rank structure of tensor data.

Difference-based regularizers mitigate this bias by subtracting a correction term that reduces the over-penalization of dominant components. The matrix nuclear-minus-Frobenius penalty, $\left \| \cdot   \right \| _*-\alpha \left \| \cdot \right \| _F$ \cite{wang2021scalable}, adapted from the $\ell_1-\alpha\ell_2$ minimization in sparse recovery \cite{lou2015computing,yin2015minimization}, together with its tensor counterparts \cite{liu2024convex}, can yield more accurate rank approximations than single-norm penalties.
%However, a difference of norms is only positively homogeneous of degree one and thus cannot reproduce the scale invariance intrinsic to the rank function, so it may fail to recover the low-rank structure when the data scale is unknown or varies across instances.
%Nevertheless, difference-based regularizers are positively homogeneous of degree one, unlike the rank function, which is positively homogeneous of degree zero and therefore scale invariant. Consequently, their effective penalization scales proportionally with the data magnitude, making such surrogates scale dependent and potentially less reliable for low-rank recovery when the data scale is unknown or heterogeneous across instances.
Nevertheless, difference-based regularizers are positively homogeneous of degree one, in contrast to the degree-zero, scale-invariant rank function. As a result, their penalization is scale dependent, which may weaken low-rank recovery when the data scale is unknown or heterogeneous.
By contrast, fractional-based preserves this scale invariance. %Just as
Similar to $\ell_1/\ell_2$ which inherits the scale invariance of $\ell_0$ to promote sparsity \cite{yin2014ratio,rahimi2019scale,wang2020accelerated,zeng2021analysis,tao2022minimization},
%the nuclear-norm-to-Frobenius-norm ratio promotes a low-rank spectrum because the nuclear and Frobenius norms correspond to the $\ell_1$ and $\ell_2$ norms of the singular-value vector, respectively.
the nuclear-norm-to-Frobenius-norm ratio promotes a low-rank spectrum by applying the same ratio principle to the singular-value vector, whose $\ell_1$ and $\ell_2$ norms correspond to the nuclear and Frobenius norms, respectively.
%Owing to this ratio structure, fractional-type surrogates are less sensitive to the absolute data scale and offer a more scale-consistent approximation to rank than difference-based regularizers.
Owing to this ratio structure, fractional-based surrogates are less scale-sensitive and better aligned with the scale-invariant nature of rank.
Gao et al. \cite{gao2024low} first introduced the nuclear-to-Frobenius norm ratio at the matrix level, and Zheng et al. \cite{zheng2024scale,zheng2025tensor} later extended it to tensors under the t-SVD framework as the tensor nuclear-to-Frobenius norm ratio, a nonconvex surrogate for the tensor rank.
Despite this scale-consistent formulation, existing fractional-based surrogates rely on a fixed global aggregation of the singular-value spectrum, limiting their ability to adapt to different spectral distributions and to selectively promote the underlying low-rank structure.
%still aggregate the entire singular-value spectrum through a fixed global ratio, which limits their ability to selectively distinguish dominant components from the spectral tail. 
%Nevertheless, these fractional-type still aggregate the entire singular-value spectrum in a fixed manner and provide no parameter 
%to control how many dominant singular values are emphasized in the approximation.
%Consequently, they cannot explicitly encode a prescribed target rank or adapt the penalization to the underlying singular-value distribution.
Consequently, they may
%provide insufficient flexibility to capture spectrum-dependent low-rank patterns.
lack sufficient flexibility to capture spectrum-dependent low-rank patterns.

These limitations, however, are not inherent to ratio constructions themselves,
%but arise from the fixed full-spectrum norms used in existing surrogates.
but stem from the use of fixed full-spectrum norm aggregations in existing surrogates.
At the matrix level, Doan and Vavasis \cite{doan2016finding,doan2022low} introduced a more flexible ratio based on the Ky Fan 2-$k$ norm. For $k \le \min \{m,n \}$, this norm is defined as $\|\mathbf{X}\|_{k,2}=\left(\sum_{i=1}^{k}\sigma_i^2\right)^{1/2},$
with dual norm $\lVert \mathbf{X} \rVert_{k,2}^{\star}$. They satisfy $\lVert \mathbf{X} \rVert_{k,2}
\le
\lVert \mathbf{X} \rVert_F
\le
\lVert \mathbf{X} \rVert_{k,2}^{\star},$
with equality throughout if and only if $\operatorname{rank}(\mathbf{X})\le k$. Hence, the ratio $\lVert \mathbf{X}\rVert_{k,2}^{\star}/\lVert \mathbf{X}\rVert_F$ attains its minimum exactly on matrices of rank at most $k$,
%allowing $k$ to serve as an explicit mechanism for controlling the number of dominant singular values involved in the approximation.
allowing $k$ to tune the spectral selectivity of the approximation and thereby adapt the surrogate to different low-rank structures.
%the $\ell_2$ norm of the vector formed by the $k$ largest singular values of $\mathbf{X}$, namely $\lVert \mathbf{X} \rVert_{k,2}=\left(\sum_{i=1}^{k}\sigma_i^2\right)^{1/2}$, together with its dual norm $\lVert \mathbf{X} \rVert_{k,2}^{\star}$. These norms satisfy 
%$\lVert \mathbf{X} \rVert_{k,2}
%\le
%\lVert \mathbf{X} \rVert_F
%\le
%\lVert \mathbf{X} \rVert_{k,2}^{\star}$, with equality throughout if and only if $\operatorname{rank}(\mathbf{X})\le k$.
%Consequently, the ratio $\lVert \mathbf{X}\rVert_{k,2}^{\star}/\lVert \mathbf{X}\rVert_F$ attains its minimum exactly on matrices whose rank does not exceed $k$, so that $k$ serves as a tunable mechanism for controlling how many dominant singular values the surrogate emphasizes, rather than fixing the rank in advance. 
When $k=1$, the dual Ky Fan 2-$k$ norm reduces to the nuclear norm, and the ratio recovers the conventional N/F ratio.
%Doan and Vavasis further argued that, for rank-$k$ approximation, the exponent $p=2$ in the Ky Fan $p$-$k$ norm is generally preferable to $p=1$, the latter corresponding to the ordinary Ky Fan $k$ norm.
More generally, the Ky Fan $p$-$k$ framework suggests that not only the number of retained singular values, but also the way they are aggregated, affects the quality of rank approximation.
%Hence the number of retained singular values is only one factor governing how well a ratio surrogate approximates the rank. The exponent $p$ supplies a further degree of freedom controlling how the dominant singular values are aggregated, and both $k$ and $p$ can be tuned to the singular-value distribution of the data.

Motivated by these observations, we propose, under the t-SVD framework, the ratio of the tensor nuclear norm to the tensor Ky Fan $p$-$k$ norm (TNPK) as a new nonconvex fractional regularizer for tubal-rank approximation.
The purpose of TNPK is to provide a more flexible and accurate surrogate for the tubal rank in
LRTC %low-rank tensor recovery
tasks.
Specifically, TNPK normalizes the tensor nuclear norm by the tensor Ky Fan $p$-$k$ norm, where the denominator aggregates the $k$ largest tensor singular values under the $\ell_p$ metric.
%As illustrated in Fig.~\ref{fig:000}, this fractional-based construction endows the proposed regularizer with four desirable properties: scale invariance,
This fractional-based construction endows the proposed regularizer with several desirable properties, including scale invariance,
parameter flexibility,
%low-rank promotion by concentrating spectral energy on dominant components,
and
the existence of closed-form solutions for LRTC under specific choices of $p$ and $k$.
%truncation capabilit induced by the corresponding inverse singular-value proximal operator.
Moreover, the proposed model is developed within a general transform-based t-product framework. In this paper, we consider three representative invertible linear transforms, namely the discrete Fourier transform (DFT), the discrete cosine transform (DCT), and random orthogonal matrices (ROM).
The main contributions of this paper are summarized as follows:

\begin{itemize}
	\item
%\textbf{1) A new nonconvex fractional regularizer.}
We introduce TNPK as a new nonconvex surrogate for the tensor tubal rank. It is scale-invariant and, through the parameters $p$ and $k$, offers a flexible approximation to the tubal rank. Moreover, under suitable choices of $p$ and $k$, TNPK degenerates to the ratio of the tensor nuclear norm to the tensor Ky Fan $k$ norm (TNK) or to the ratio of the tensor nuclear norm to the tensor Frobenius norm (TNF), thereby recovering several existing fractional regularizers as special cases.
	\item
%\textbf{2) Theoretical guarantee under invertible linear transforms.} 
We formulate a TNPK-based LRTC model under a general t-product framework induced by invertible linear transforms. Via the tensor null space property (NSP), we prove that the underlying low-rank tensor is a local minimizer of the proposed model.
	\item
%\textbf{3) Closed-form proximal operator and efficient algorithm.}
%We derive a closed-form proximal operator for TNPK that automatically performs truncation at index $k$, and build on it to develop an efficient ADMM algorithm whose subsequence convergence is established under mild conditions.
%\textbf{4) Extensive empirical validation.}
%Extensive experiments on synthetic and real-world datasets demonstrate that the proposed method achieves superior recovery performance compared with state-of-the-art baselines.
%\textbf{3) Efficient ADMM Algorithm and Empirical Validation.}
We derive the proximal operator associated with the TNPK inverse-norm and develop an efficient ADMM algorithm for the TNPK-regularized LRTC model, in which all variables admit closed-form updates. We further establish the subsequence convergence of the proposed algorithm under mild conditions. Experiments on synthetic and real-world datasets demonstrate the superior recovery performance of the proposed method over state-of-the-art baselines.
\end{itemize}

The outline of this article is structured as follows. Section \ref{sec2} reviews the related work. Section \ref{sec3} introduces the notations and preliminaries used throughout the paper. Section \ref{sec4} presents the rationale behind the proposed TNPK regularization. Section \ref{sec5} develops the TNPK-based LRTC model. Section \ref{sec6} reports the experimental results on synthetic and real-world datasets. Finally, Section \ref{sec7} concludes this article.

%\vspace{-0.35cm}
%\vspace{-0.5048cm}
\section{\textbf{Related Work}}
\label{sec2}
%To underscore our contributions, this section reviews some
%representative works about .....
We review representative works on difference-based and fractional-based regularization for sparse and low-rank recovery.

\vspace{-0.15cm}
\subsection{Difference-based regularization}
The difference-of-norms idea originates in sparse recovery, where the gap between the $\ell_1$ and $\ell_2$ norms was found to promote sparsity more effectively than the $\ell_1$ norm in highly coherent dictionaries \cite{lou2015computing} and compressed sensing \cite{yin2015minimization}, and was later refined by truncated variants \cite{ma2017truncated}. Wang et al. \cite{wang2021scalable} carried this principle to low-rank matrix recovery by noting that the nuclear and Frobenius norms are exactly the $\ell_1$ and $\ell_2$ norms of the singular-value vector, so that the nuclear-minus-Frobenius regularizer applies the same difference-of-norms shrinkage to the spectrum. The construction has since been extended to multi-channel color image denoising \cite{shan2023multi}, truncated low-rank matrix minimization \cite{guo2023low}, tensor robust principal component analysis \cite{liu2024convex}, and quaternion-based color image reconstruction \cite{guo2025quaternion}.
Across these variants, however, the resulting penalty remains positively homogeneous of degree one and therefore scales with the magnitude of the spectrum. Its corrective effect is consequently governed not only by the relative distribution of singular values, but also by their absolute scale, which is inconsistent with the degree-zero scale invariance of the rank function. This residual scale dependence weakens the correction of dominant-component over-penalization and can make recovery sensitive to data scaling, especially when the scale is unknown or varies across blocks, channels, or tensor modes. To overcome this limitation, this paper adopts a fractional, ratio-based regularizer that restores scale invariance and decouples low-rank promotion from data magnitude.

\vspace{-0.1cm}
\subsection{Fractional-based regularization}

Fractional-based regularization offers a scale-invariant nonconvex strategy for sparse and low-rank recovery, with roots in nonnegative sparse coding and the study of sparsity measures \cite{hoyer2002non,hurley2009comparing}. Its earliest and most representative instance is the $\ell_1/\ell_2$ ratio \cite{yin2014ratio}, which has since been refined through scale-invariant recovery theory, accelerated algorithms, and convergence analysis \cite{rahimi2019scale,wang2020accelerated,zeng2021analysis,tao2022minimization}. A more flexible variant, the $\ell_1/s_k$ ratio, introduces an explicit parameter $k$ in the denominator that focuses the surrogate on the $k$ dominant components, and has been studied within a general structured fractional-programming framework via proximal algorithms \cite{li2022proximal,zhou2025equivalent,zhou2026min}.

The same ratio principle extends to low-rank recovery through the singular-value spectrum. As the matrix analogue of $\ell_1/\ell_2$, the nuclear-norm-over-Frobenius-norm (N/F) ratio was introduced for low-rank matrix recovery \cite{gao2024low} and later extended to tensor and quaternion settings \cite{zheng2024scale,zheng2025tensor,guo2025quaternionover}. These N/F ratios, however, aggregate the entire singular-value spectrum in a fixed manner and provide no parameter to control how many dominant singular values are emphasized. The Ky Fan $2$-$k$ norm removes this restriction through an explicit rank-related parameter. At the matrix level, Doan and Vavasis \cite{doan2016finding,doan2022low} used this norm and its dual for low-rank clustering and matrix recovery, with ADMM-based solvers demonstrating the practicality of such models \cite{wang2020ky}. The induced ratio is minimized exactly on matrices of rank at most $k$, so that $k$ directly governs the dominant singular-value components retained in the approximation.

\vspace{-0.05cm}
%\vspace{-0.5048cm}
\section{\textbf{Notations and Preliminaries}}
\label{sec3}
This section presents the notations and preliminaries that will be used throughout the paper. Tensors are denoted by calligraphic letters, e.g., $ \mathcal{A} $; %boldface
matrices by boldface uppercase letters, e.g., $\mathbf{A} $; vectors by boldface lowercase letters, e.g., $\mathbf{a}$; and scalars by lowercase letters. In particular, the identity matrix is denoted by $\mathbf{I}$. We denote the fields of real and complex numbers by $\mathbb{R}$ and $\mathbb{C}$, respectively. For a third-order tensor $ \mathcal{A} \in \mathbb{C}^{n_1 \times n_2 \times n_3}$, the $(i,j,k)$-th entry is denoted by $ \mathcal{A}_{ijk}$ or $a_{ijk}$. The $i$-th horizontal, lateral, and frontal slices are represented by $ \mathcal{A}(i,:,:),   \mathcal{A}(:,i,:)$, and $ \mathcal{A}(:,:,i)$, respectively; we also write $ \mathcal{A}^{(i)}$ to denote the $i$-th frontal slice. The zero tensor is denoted by $\mathcal{O}$, and $\mathcal{N}(F)$ denotes the null space of a linear operator $F$. The notation $[n]$ stands for the index set $\{1,\ldots,n\}$, and $|S|$ denotes the cardinality of a set $S$.
For matrices $\mathbf{A}$ and $\mathbf{B}$, their inner product is defined by
$\langle \mathbf{A}, \mathbf{B} \rangle := \mathrm{Tr}(\mathbf{A}^* \mathbf{B})$, where $\mathbf{A}^*$ denotes the conjugate transpose and $\mathrm{Tr}(\cdot)$ denotes the trace. For tensors $\mathcal{A}, \mathcal{B} \in \mathbb{C}^{n_1 \times n_2 \times n_3}$, the inner product is defined as $\langle \mathcal{A}, \mathcal{B} \rangle = \sum_{l=1}^{n_3} \langle \mathbf{A}^{(l)}, \mathbf{B}^{(l)} \rangle$. The complex conjugate of $\mathcal{A}$ is denoted by $\mathrm{conj}(\mathcal{A})$.
%The tensor Frobenius norm is defined as $\left \| \mathcal{A}  \right \| _F=\sqrt{\sum _{ijk}\left | a_{ijk} \right |^2 } $.
%!!!!????????fanshudingyi

%A more general definition of the t-product based on an arbitrary invertible linear transform $L$ is given in \cite{kernfeld2015tensor}.
%In \cite{lu2019low}, it is shown that such a general t-product can be implemented by applying an invertible linear transform to each tube fiber $\mathcal{A}(i,j,:)$, carrying out frontal-slice-wise matrix multiplications in the transform domain, and subsequently applying the inverse transform to recover the result.
%Let $L:\mathbb{R}^{n_1\times n_2\times n_3}
%\rightarrow
%\mathbb{R}^{n_1\times n_2\times n_3}$ be an invertible linear transform acting along the third mode (i.e., on each tube fiber), with inverse $L^{-1}$. For a tensor $\mathcal A$, its transform-domain representation is
%\[
%\bar{\mathcal A}
%= L(\mathcal A)
%= \mathcal A \times_3 L,
%\qquad
%\mathcal A
%= L^{-1}(\bar{\mathcal A})
%= \bar{\mathcal A} \times_3 L^{-1},
%\]
%where $\times_3$ denotes the mode-3 product. Define the block diagonal operator
%\[
%\bar{\mathbf{A}}=\mathrm{bdiag}(\bar{\mathcal A})
%=
%\begin{bmatrix}
%	\bar{\mathbf{A}}^{(1)}&  &  & \\
%	& \bar{\mathbf{A}}^{(2)} &  & \\
%	&  & \ddots  & \\
%	&  &  &\bar{\mathbf{A}}^{(n_3)}
%\end{bmatrix},
%\]
%whose inverse operator $\mathrm{unbdiag}(\cdot)$ reconstructs the
%tensor from the block diagonal matrix.

Following \cite{kernfeld2015tensor,lu2019low}, we consider a general transform-based t-product induced by an arbitrary invertible linear transform $L$ acting along the third mode. Specifically, let
$L:\mathbb{R}^{n_1\times n_2\times n_3}
\rightarrow
\mathbb{R}^{n_1\times n_2\times n_3}$
be an invertible linear transform applied to each tube fiber $\mathcal{A}(i,j,:)$, with inverse $L^{-1}$. In this framework, the t-product can be implemented by transforming tensors into the $L$-domain, performing frontal-slice-wise matrix multiplications, and then applying the inverse transform. In this paper, we consider three representative choices of $L$, namely the discrete Fourier transform (DFT), the discrete cosine transform (DCT), and random orthogonal matrices (ROM).
For a tensor $\mathcal A$, its transform-domain representation is defined as
\[\bar{\mathcal A}
= L(\mathcal A)
= \mathcal A \times_3 L,
\quad
\mathcal A
= L^{-1}(\bar{\mathcal A})
= \bar{\mathcal A} \times_3 L^{-1},\]
where $\times_3$ denotes the mode-3 product. Define the block diagonal operator
\[
\bar{\mathbf{A}}=\mathrm{bdiag}(\bar{\mathcal A})
=
\begin{bmatrix}
	\bar{\mathbf{A}}^{(1)}&  &  & \\
	& \bar{\mathbf{A}}^{(2)} &  & \\
	&  & \ddots  & \\
	&  &  &\bar{\mathbf{A}}^{(n_3)}
\end{bmatrix},
\]
whose inverse operator $\mathrm{unbdiag}(\cdot)$ reconstructs the tensor from the block diagonal matrix.

\begin{Definition}[t-product \cite{lu2019low}]
	Let $L$ be an arbitrary invertible linear transform, and let $\bar{\mathcal A}=L(\mathcal A)$ and $\bar{\mathcal B}=L(\mathcal B)$.
	The $L$-based t-product of tensors $\mathcal{A} \in \mathbb{R}^{n_1 \times n \times n_3} $ and $\mathcal{B} \in \mathbb{R}^{n \times n_2 \times n_3}$ is defined by \[
	\mathcal A *_L \mathcal B
	=
	L^{-1}\!\Big(
	\mathrm{unbdiag}\big(
	\mathrm{bdiag}(\bar{\mathcal A}) \times
	\mathrm{bdiag}(\bar{\mathcal B})
	\big)
	\Big),
	\]
	resulting in a tensor of size $n_1 \times n_2 \times n_3$. 
\end{Definition}

\begin{Definition}[tensor transpose \cite{lu2019low}]
	Let $L$ be an arbitrary invertible linear transform,
	for $\mathcal{A} \in \mathbb{R}^{n_1 \times n_2 \times n_3}$, the tensor transpose $ \mathcal{A}^{*}$ is defined by
	\[
	\big(L(\mathcal{A}^{*}))^{(i)} = \big(L(\mathcal{A})^{(i)}\big)^{*},
	\quad i = 1,\ldots,n_3 ,
	\]
	it corresponds to applying the conjugate transpose to each frontal slice in the transform domain.
\end{Definition}

\begin{Definition}[f-diagonal tensor \cite{lu2019low}]
	Let $L$ be an arbitrary invertible linear transform, and let $ \mathcal{A} \in \mathbb{R}^{n_1 \times n_2 \times n_3}$. If each frontal slice of $ L(\mathcal{A}) $ is a diagonal matrix, then $\mathcal{A}$ is called an f-diagonal tensor.
\end{Definition}

\begin{Definition}[identity tensor \cite{lu2019low}]
	Let $L$ be an arbitrary invertible linear transform. A tensor $ \mathcal{I} \in \mathbb{R}^{n \times n \times n_3}$ is called an identity tensor if each frontal slice of $L(\mathcal{I})$ is an $n \times n$ identity matrix.
\end{Definition}

\begin{Definition}[orthogonal tensor \cite{lu2019low}]
	Let $L$ be an arbitrary invertible linear transform. A tensor $ \mathcal{Q} \in \mathbb{R}^{n \times n \times n_3}$ is called an orthogonal tensor if $\mathcal{Q}^{*} *_L \mathcal{Q}= \mathcal{Q} *_L \mathcal{Q}^{*}=\mathcal{I},$ where $\mathcal{I}$ denotes the identity tensor.
\end{Definition}

\begin{Theorem}[t-SVD \cite{lu2019low}]
	Let $L$ be an arbitrary invertible linear transform, and let $ \mathcal{A} \in \mathbb{R}^{n_1 \times n_2 \times n_3} $. Then $ \mathcal{A} $ admits the decomposition 
	\[	\mathcal{A}=\mathcal{U} *_L \mathcal{S} *_L \mathcal{V}^{*},\]
	where $\mathcal{U} \in \mathbb{R}^{n_1 \times n_1 \times n_3},  \mathcal{V} \in \mathbb{R}^{n_2 \times n_2 \times n_3}$ are orthogonal tensors, and $\mathcal{S} \in \mathbb{R}^{n_1 \times n_2 \times n_3}$ is an f-diagonal tensor.
\end{Theorem}

\begin{Definition}[tensor singular values \cite{lu2019low}]
	Let $\mathcal{A} =\mathcal{U} *_L \mathcal{S} *_L \mathcal{V}^{*}$ be the t-SVD of $\mathcal A \in \mathbb{R}^{n_1 \times n_1 \times n_3}$ under the invertible linear transform $L$. Since $\mathcal{S}$ is f-diagonal, its diagonal tubes $\mathcal S(j,j,:), j=1,\ldots, \min\{n_1,n_2\},$ are called the singular tubes of $\mathcal A$. Let $\bar{\mathcal A} = L(\mathcal A)$ and denote by $\sigma_j(\bar{\mathcal A}^{(i)})$ the $j$-th largest singular value of the $i$-th frontal slice $\bar{\mathcal A}^{(i)}$. The scalar tensor singular value associated
	with the $j$-th singular tube is defined as \[
	\sigma_j(\mathcal{A}) =\frac{1}{\ell} \sum_{i=1}^{n_3} \sigma_j\!\left(\bar{\mathcal{A}}^{(i)}\right),
	\quad
	j=1,\ldots,\min\{n_1,n_2\}.
	\]
\end{Definition}

\begin{Definition}[tensor tubal rank \cite{lu2019low}]
	The tensor tubal rank of a tensor $\mathcal{A}$ is defined as the number of nonzero singular tubes of $\mathcal{S}$ in its t-SVD decomposition. In other words,
	\[\mathrm{rank}_t(\mathcal{A})=\# \left \{ i,\mathcal{S}(i,i,:)\ne 0 \right \} .\]
\end{Definition}

\begin{Definition}[tensor spectral norm and nuclear norm \cite{lu2019low}]
	Let $L$ be an arbitrary invertible linear transform, and let $\mathcal{A} \in \mathbb{R}^{n_1 \times n_2 \times n_3}$. Let $\sigma_{ij}(\mathcal{A})$ denote the $j$-th singular value of $\bar{\mathbf{A}}^{(i)}$. The tensor spectral norm of $\mathcal{A}$ is defined as
	\[||\mathcal{A} ||:=\max_i||\bar{\mathbf{A}}^{(i)}||_2=\max_{i,j}\sigma _{ij}(\mathcal{A}).\]
	
	If $L$ satisfies $L^*L=LL^*=\ell I$, the tensor nuclear norm of $\mathcal{A}$ is defined as
	\[||\mathcal{A}||_{*}
	:= \frac{1}{\ell } ||\bar{\mathbf{A}}^{(i)}||_*=\frac{1}{\ell }\sum_{i=1}^{n_3} \sum_{j=1}^{\min\{n_1,n_2\}}
	\sigma_{ij}(\mathcal{A}).\]
\end{Definition}

\begin{Definition}[tensor Frobenius norm]
	Let $L$ be an arbitrary invertible linear transform, and let
	$\mathcal{A} \in \mathbb{R}^{n_1 \times n_2 \times n_3}$ with entries $a_{ijk}$.
	The tensor Frobenius norm of $\mathcal{A}$ is defined as
	\[
	\|\mathcal{A}\|_F := \sqrt{\sum_{i,j,k} |a_{ijk}|^2}.
	\]
	If $L$ satisfies $L^*L = LL^* = \ell I$, then, letting
	$\bar{\mathcal{A}} = L(\mathcal{A})$ and $\sigma_{ij}(\mathcal{A})$ denote the
	$j$-th singular value of $\bar{\mathbf{A}}^{(i)}$, the Frobenius norm admits the
	equivalent expressions
	\[
	\|\mathcal{A}\|_F^2
	= \frac{1}{\ell}\|\bar{\mathbf{A}}\|_F^2
	= \frac{1}{\ell}\sum_{i=1}^{n_3}\sum_{j=1}^{\min\{n_1,n_2\}}
	\sigma_{ij}^2(\mathcal{A}).
	\]
\end{Definition}

Building on the classical Ky Fan $k$ norm, Horn et al. \cite{horn2012matrix} introduce a generalization, the Ky Fan $p$-$k$ norm:
\[||\mathbf{A} ||_{p,k}=(\sum_{i=1}^{k} \sigma _i^p(\mathbf{A} ))^{\frac{1}{p} },\]
where $\sigma_1\ge \sigma_2\ge \cdots \ge \sigma_k \ge0,$ are the $k$ largest singular values of $\mathbf{A} \in\mathbb{R}^{m\times n}$, and the
parameters satisfy $k \le  \min\{m,n\}$ and $1\le p<\infty.$ One can readily
verify that $\|\cdot\|_{p,k}$ is a unitarily invariant norm.
As shown by Tanaka et al. \cite{tanaka2014positive}, this family recovers several
classical norms as special cases:
for $p=1$ it reduces to the classical Ky Fan
$k$ norm; for $k=\min\{m,n\}$ it reduces to the Schatten $p$ norm; and for
$k=\min\{m,n\}$ together with $p=2$ it reduces to the Frobenius norm.
%In particular, we focus on the following cases: when $p=1$, it reduces to the classical Ky Fan $k$ norm; when $k=rank(A)$, it reduces to the Schatten $p$ norm; and when $p=2$ and $k=rank(A)$, it is equivalent to the Frobenius norm.
%???k?????r,?????\min(n_1,n_2)?
Motivated by the matrix Ky Fan $p$-$k$ norm, we extend this notion to the tensor setting and define the tensor Ky Fan $p$-$k$ norm as follows.
\begin{Definition}[tensor Ky Fan $p$-$k$ norm]
	For a tensor $\mathcal{A} \in \mathbb{R}^{n_1 \times n_2 \times n_3}$, the tensor Ky Fan $p$-$k$ norm is defined as :
	\[||\mathcal{A} ||_{p,k}=(\sum_{i=1}^{k} \sigma _i^p(\mathcal{A} ))^{\frac{1}{p} },\]
	where $\sigma_i(\mathcal{A})$ denotes the $i$-th largest singular value of $\mathcal{A}$. The parameters satisfy $1 \le k \le \min\{n_1,n_2\}, 1\le p<\infty$. 
	When $p=1$, the norm reduces to the Ky Fan $k$ norm, i.e., the sum of the largest k singular values.
	When $p= \infty $, we have:
	\[||\mathcal{A} ||_{\infty ,k}=\max_{1\le i\le k}\sigma _i(\mathcal{A} )=\sigma _1(\mathcal{A} ).\]
\end{Definition}

%\begin{lemma}[norm relations]%%?????????????
%	For any tensor $\mathcal{X} \neq \mathcal{O} $, the following inequalities hold:
%	\begin{equation}
	%		||\mathcal{X}||_{(k)}\le ||\mathcal{X}||_{*} \le \sqrt{n}||\mathcal{X}||_{F},
	%	\end{equation}
%	and
%	\begin{equation}
	%		\frac{1}{\sqrt{\ell }}||\mathcal{X}||_{F} \le ||\mathcal{X}||_{*} \le
	%		\sqrt{r}||\mathcal{X}||_{F}^2,
	%	\end{equation}
%	where $r=\operatorname{rank}_t(\mathcal{X}), n=\min \left \{n_1,n_2  \right \}$.
%\end{lemma}

\begin{Theorem}[spectral reduction for Ky Fan $p$-$k$ inverse-norm minimization]
	Let $L$ be an arbitrary invertible linear transform satisfying $L^*L=LL^*=\ell I$ for some constant $\ell>0$, and let $1 \le k \le \min\{n_1,n_2\}, 1\le p<\infty$, and $r=\min\{n_1,n_2\}$.
	Consider the following optimization problem:
	\begin{equation}
		\min_{\mathcal{X}\in \mathbb{R}^{n_1\times n_2\times n_3}  }\frac{\lambda  }{\left \| \mathcal{X}  \right \| _{p,k}} +\frac{1}{2}\left \| \mathcal{X}-\mathcal{Y}   \right \| ^2_F ,
		\label{eq:01}
	\end{equation}
	where $\lambda>0$ and $\mathcal{Y} \in \mathbb{R}^{n_1 \times n_2 \times n_3}$ is a given tensor. Let the t-SVD of  $\mathcal{Y}$ be
	\[
	 \mathcal Y = \mathcal U_\mathcal{Y} *_L \mathcal S_\mathcal{Y} *_L \mathcal V_\mathcal{Y}^* , 
	\]
	with singular values $\sigma _1(\mathcal{Y})\ge \sigma _2 (\mathcal{Y})\ge \cdots \ge \sigma_r(\mathcal{Y}) \ge 0 $.
	Then problem \eqref{eq:01} admits an optimal solution of the form
	\[	\mathcal X^* = \mathcal U_\mathcal{Y} *_L \mathcal S_\mathcal{X}^* *_L \mathcal V_\mathcal{Y}^*,
	\]
	where $\mathcal S_\mathcal{X}^*$ is a f-diagonal tensor whose diagonal elements $\sigma_1^{*}\ge \sigma_2^{*}\ge \cdots \ge \sigma_r^{*}\ge 0$ are determined by the following singular value optimization problem:
	\begin{equation}
		\min_{\sigma_1 \ge \sigma_2 \ge \cdots \ge \sigma_r \ge 0}\ \frac{\lambda}{\left(\sum_{i=1}^{k} \sigma_i^{p}\right)^{1/p}} +\frac{1}{2}\sum_{i=1}^{r}\bigl(\sigma_i-\sigma_i(\mathcal{Y})\bigr)^2 .
		\label{eq:022}
	\end{equation}
	\label{theorem:2}
\end{Theorem}

%\textbf{$\mathit{ Proof.}$}
%%In the matrix case, Lu \cite{canyi2017structured} utilized the von Neumann trace inequality to construct an equivalent reformulation, transforming the original problem into an optimization problem in the singular value space. Since this problem is isomorphic in structure to \eqref{eq:01}, its derivation can be directly transferred to the framework of this paper. On this basis, we further extend the method to the tensor case, thereby completing the proof of the theorem.
%In the matrix case, Lu \cite{canyi2017structured} used the von Neumann trace inequality to derive an equivalent reformulation of the original problem as an optimization problem over the singular values. Since the tensor problem considered in \eqref{eq:01} admits an analogous spectral representation under the t-SVD framework, the same argument can be adapted to the tensor setting. Specifically, by using the orthogonal invariance of the Frobenius norm and the tensor version of the von Neumann trace inequality, the optimization problem can be reduced to the corresponding singular value optimization problem, which completes the proof of the theorem.

\vspace{-0.5cm}
\begin{proof}
	Let the t-SVD of $\mathcal X$ be $\mathcal X = \mathcal U_\mathcal{X} *_L \mathcal S_\mathcal{X} *_L \mathcal V_\mathcal{X}^*$, with singular values $\sigma_1(\mathcal{X})\ge \sigma_2(\mathcal{X})\ge \cdots \ge \sigma_r(\mathcal{X}) \ge 0$. Since the Ky Fan $p$-$k$ inverse-norm depends only on the singular values of $\mathcal X$, we have
	\[\frac{\lambda}{\left\| \mathcal{X} \right\|_{p,k}}= \frac{\lambda}{\left(\sum_{i=1}^{k}\sigma_i(\mathcal{X})^{p}\right)^{1/p}} . \]
	For the data-fidelity term, we expand
	\[
	\|\mathcal{X}-\mathcal{Y}\|_F^2 = \|\mathcal{X}\|_F^2 -2\langle \mathcal{X},\mathcal{Y}\rangle + \|\mathcal{Y}\|_F^2 .
	\]
	By the orthogonal invariance of the Frobenius norm under the t-product, both $\|\mathcal{X}\|_F^2=\sum_{i=1}^{r}\sigma_i(\mathcal{X})^2$ and $\|\mathcal{Y}\|_F^2=\sum_{i=1}^{r}\sigma_i(\mathcal{Y})^2$ depend only on the singular values. Moreover, the tensor von Neumann trace inequality \cite{chretien2015neumann} gives
	\[ \langle \mathcal X,\mathcal Y\rangle \leq \sum_{i=1}^{\min\{n_1,n_2\}}\sigma _i(\mathcal X)\sigma_{i}(\mathcal{Y}) , \]
	with equality attained when $\mathcal{X}$ and $\mathcal{Y}$ share the same left and right singular tensors, that is, when $\mathcal{X} = \mathcal{U}_\mathcal{Y} *_L \mathcal{S}_\mathcal{X} *_L \mathcal{V}_\mathcal{Y}^{*}$.
	%The equality is attained when $\mathcal X$ shares the same left and right singular tensors as $\mathcal Y$, that is, when
	%\[ \mathcal X = \mathcal U_y *_L \mathcal S_x *_L \mathcal V_y^* . \]
	Combining the above relations yields
	\[
	\|\mathcal{X}-\mathcal{Y}\|_F^2 \ \ge\ \sum_{i=1}^{r}\bigl(\sigma_i(\mathcal{X})-\sigma_i(\mathcal{Y})\bigr)^2 .
	\]
	
	Consequently, for any prescribed singular values $\sigma_1(\mathcal{X})\ge \sigma_2(\mathcal{X})\ge \cdots \ge \sigma_r(\mathcal{X}) \ge 0$, the objective function in \eqref{eq:01} is minimized, over all tensors with these singular values, by aligning the left and right singular tensors of $\mathcal{X}$ with those of $\mathcal{Y}$.
	 The tensor optimization problem \eqref{eq:01} is therefore reduced to the singular value optimization problem \eqref{eq:022}. Since the objective of \eqref{eq:022} is continuous on its feasible set, tends to $+\infty$ as $(\sigma_1,\ldots,\sigma_r)\to \mathbf{0}$, and is coercive, problem \eqref{eq:022} attains a minimizer $\sigma^{*}=(\sigma_1^{*},\sigma_2^{*},\ldots,\sigma_r^{*})$. Let $\mathcal{S}_x^{*}$ be the corresponding f-diagonal tensor. Then
	\[
	\mathcal{X}^{*} = \mathcal{U}_\mathcal{Y} *_L \mathcal{S}_\mathcal{X}^{*} *_L \mathcal{V}_\mathcal{Y}^{*}
	\]
	is itself a valid t-SVD, so that $\sigma_i(\mathcal{X}^{*})=\sigma_i^{*}$ for all $i$, and $\mathcal{X}^{*}$ attains the lower bound established above. Hence $\mathcal{X}^{*}$ is an optimal solution to \eqref{eq:01}, which completes the proof.
\end{proof}

\vspace{-0.2cm}
%\vspace{-0.35cm}
%Randomized Techniques Based High-Order Tensor Approximation
%\section{\textbf{rationales of the TNPK model}}
\section{\textbf{Fractional Regularization with Ky Fan $p$-$k$ Norm }}
\label{sec4}
To accurately approximate the tensor tubal rank while overcoming the limitations of existing fractional regularizers, we propose a novel fractional-based surrogate, termed the tensor nuclear norm to Ky Fan $p$-$k$ norm (TNPK). It is defined as
%\begin{equation}
%	\left \| \scr{X} \right \|_{\mathrm{TNPK} } =\frac{\left \| \mathcal{X}  \right \| _*}{\left \| \mathcal{X}  \right \|_{p,k}}.
%	\label{eq:3.1}
%\end{equation}
\[	\left \| \scr{X} \right \|_{\mathrm{TNPK} } =\frac{\left \| \mathcal{X}  \right \| _*}{\left \| \mathcal{X}  \right \|_{p,k}}.\]
From a singular-value perspective, the TNN treats all singular values as a vector and imposes an $\ell_1$ norm penalty. The TNF regularizer \cite{zheng2024scale,zheng2025tensor} further adopts an $\ell_1/\ell_2$ ratio formulation, which effectively promotes global sparsity of the singular-value vector.
In contrast, the proposed TNPK regularizer uses the Ky Fan $p$-$k$ norm of the singular values as the denominator, namely, the $\ell_p$ norm of the $k$ largest singular values.
By excluding small tail singular values from the denominator, TNPK reduces their influence and thus enhances robustness to noise. Moreover, the two parameters $p$ and $k$ provide flexible control for adapting to different low-rank structures.
Here, we formulate the general tensor recovery problem with the proposed TNPK regularizer as follows:
\begin{equation}
	\min_{\mathcal{X}} \ ||\mathcal{X}||_{\mathrm{TNPK}}
	\quad \text{s.t.} \quad F(\mathcal{X}) = \mathcal{T},
	\label{eq:3.3}
\end{equation}
where $F(\cdot)$ is a linear operator and $\mathcal{T}$ denotes the corresponding measurements. We further assume that $\mathcal{T} \neq \mathcal{O}$ to avoid the trivial solution $\mathcal{X} = \mathcal{O}$. 
In the following, we investigate several theoretical properties of the proposed TNPK regularization.
%Due to the special structure of the fractional-based formulation, the TNPK regularizer possesses several important properties, including scale invariance, unitary invariance, and parameter flexibility.
Benefiting from its fractional-ratio structure, TNPK exhibits several desirable properties, including scale invariance, unitary invariance, and parameter flexibility.

\begin{proposition}[scale invariance]
	Let $\mathcal{A} \in \mathbb{R}^{n_1 \times n_2 \times n_3}$. Then for any nonzero scalar $c$,
	\[
	||\mathcal{A}||_{\mathrm{TNPK}} = ||c\mathcal{A}||_{\mathrm{TNPK}} .
	\]
\end{proposition}

\begin{proof}
	By the definitions of the tensor nuclear norm and the tensor Ky Fan $p$-$k$ norm, for any scalar $c \neq 0$, we have
	\[
	||c\mathcal{A}||_{*}=|c| \cdot ||\mathcal{A}||_{*}, \quad ||c\mathcal{A}||_{p,k}=|c| \cdot ||\mathcal{A}||_{p,k} .
	\]
	Therefore, \[
	\|c\mathcal{A}\|_{\mathrm{TNPK}}
	=
	\frac{\|c\mathcal{A}\|_{*}}{\|c\mathcal{A}\|_{p,k}}
	=
	\frac{|c|\|\mathcal{A}\|_{*}}{|c|\|\mathcal{A}\|_{p,k}}
	=
	\|\mathcal{A}\|_{\mathrm{TNPK}}.
	\]
	Hence, the TNPK is scale invariant, which completes the proof.
\end{proof}

\begin{figure}[!htbp]
	\vspace{-0.2cm}
	\renewcommand{\arraystretch}{0.0}
	\setlength\tabcolsep{1pt}
	\centering
	\begin{tabular}{c c }
		\centering
		\includegraphics[width=1.71in]{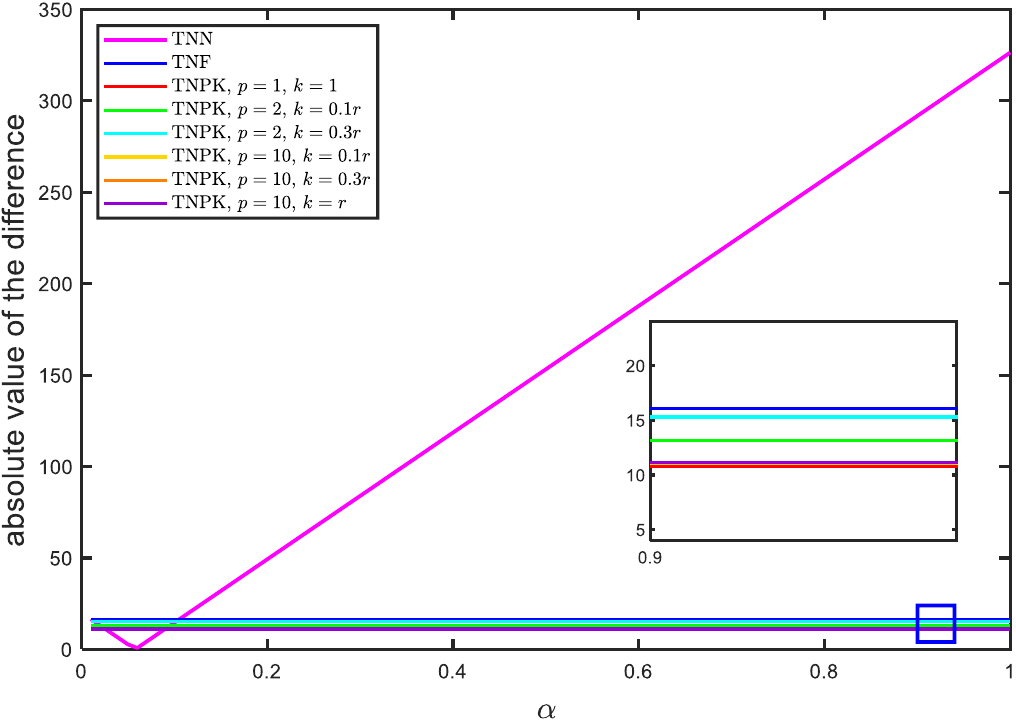}
		&
		\includegraphics[width=1.71in]{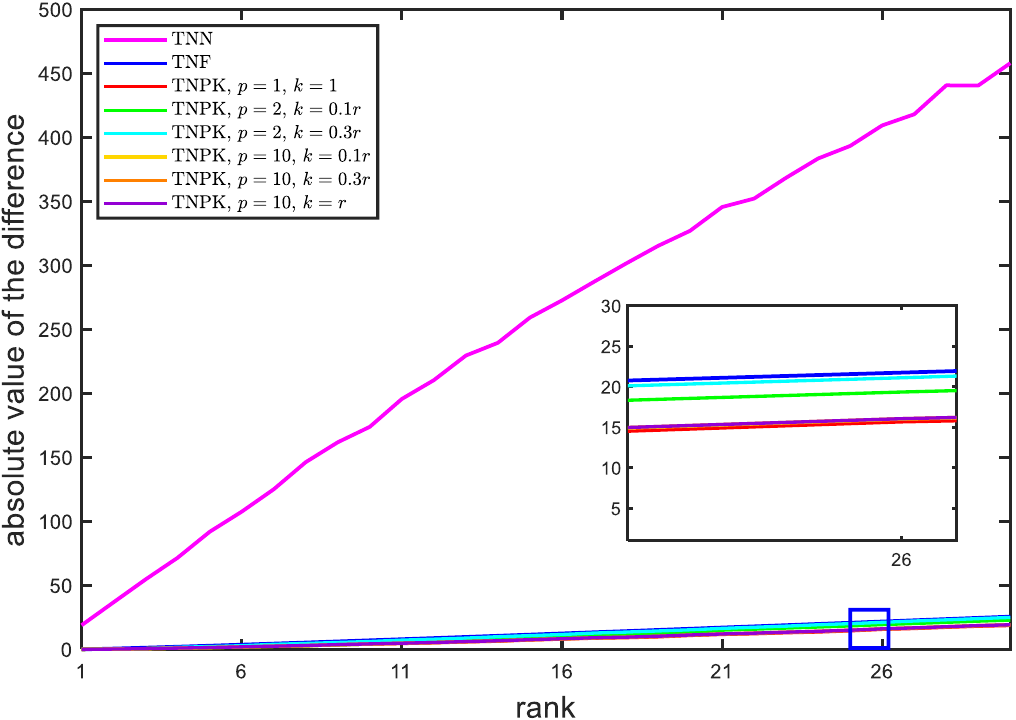}
		\\[+1mm]
		(a) & (b)
	\end{tabular}
	%\vspace{-0.15cm}
	%\caption{}
	\caption{Comparison of the approximation ability of different norm-based measures to the tensor tubal rank: (a) When the tubal rank is fixed as $r = 20$, we vary the parameter $\alpha$ and compare the following metrics: $\left | \left \| \mathcal{X}  \right \| _*-\mathrm{rank} (\mathcal{X} ) \right | $, $\left | \frac{\left \| \mathcal{X}  \right \| _*}{\left \| \mathcal{X}  \right \| _F} -\mathrm{rank} (\mathcal{X} ) \right |$ and $\left | \frac{\left \| \mathcal{X}  \right \| _*}{\left \| \mathcal{X}  \right \| _{p,k}} -\mathrm{rank} (\mathcal{X} ) \right |$. (b) When $\alpha = 1$ is fixed, we vary the tubal rank $r$ and compare the same three metrics under different tubal rank settings.}
	%
	%\vspace{-0.76cm}
	\label{fig:0} % \label{fixed-rank-recovery}
\end{figure}

\vspace{-0.2cm}
We select the TNN and the TNF \cite{zheng2024scale,zheng2025tensor} as comparison methods. We randomly generate a tensor $\mathcal{X} \in \mathbb{R}^{n_1 \times n_2 \times n_3}$ with tubal rank $r$ as follows:
\[\mathcal{X} = \alpha \mathcal{P} *_L \mathcal{Q}^T,\]
where $\mathcal{P} \in \mathbb{R}^{n \times r \times n_3}, \mathcal{Q} \in \mathbb{R}^{r \times n \times n_3}$, and $r \ll n$, and $\alpha > 0$ is a scaling constant. 
Next, we compute the tensor singular values and present the following three metrics in Figure \ref{fig:0}:
\[\left | \left \| \mathcal{X}  \right \| _*-\mathrm{rank} (\mathcal{X} ) \right | ,\left | \frac{\left \| \mathcal{X}  \right \| _*}{\left \| \mathcal{X}  \right \| _F} -\mathrm{rank} (\mathcal{X} ) \right |,\left | \frac{\left \| \mathcal{X}  \right \| _*}{\left \| \mathcal{X}  \right \| _{p,k}} -\mathrm{rank} (\mathcal{X} ) \right |.\]
The experimental settings are $n_1 = n_2 = 40$ and $n_3 = 20$.
In Figure \ref{fig:0} (a), we fix the tubal rank as $r = 20$ and vary $\alpha \in \left \{ 0.01, 0.02, \ldots, 0.99, 1 \right \} $. The results show that, under suitable parameter settings, the TNPK-based metric achieves the smallest absolute deviation from the rank and remains stable with respect to changes in $\alpha$. This indicates that TNPK is insensitive to the magnitude of tensor entries and can effectively approximate the tubal rank function.
In Figure \ref{fig:0} (b), we fix $\alpha = 1$ and vary the tubal rank $r \in \left \{ 1,2,\ldots,29,30 \right \} $ to investigate the effect of rank. The results demonstrate that the TNPK-based metric achieves the best approximation performance when the tensor has a low tubal rank. Moreover, even as the tubal rank increases,
%the proposed method remains competitive and maintains strong performance.
TNPK remains competitive and maintains strong approximation performance.

%\begin{figure}[!htbp] 
%	\centering
%	\includegraphics[width=0.5\textwidth]{fig/scale_invariant} % ???????? .png/.jpg
%	\caption{Comparison of the approximation ability of different norm-based measures to the tensor rank: (a) When the tubal rank is fixed as $r = 20$, we vary the parameter $\alpha$ and compare the following metrics: $\left | \left \| \mathcal{X}  \right \| _*-\mathrm{rank} (\mathcal{X} ) \right | $, $\left | \frac{\left \| \mathcal{X}  \right \| _*}{\left \| \mathcal{X}  \right \| _F} -\mathrm{rank} (\mathcal{X} ) \right |$ and $\left | \frac{\left \| \mathcal{X}  \right \| _*}{\left \| \mathcal{X}  \right \| _{p,k}} -\mathrm{rank} (\mathcal{X} ) \right |$. (b) When $\alpha = 1$ is fixed, we vary the rank $r$ and compare the same three metrics under different rank settings.
%	} 
%	\label{fig:0} 
%\end{figure}

\begin{proposition}[unitary invariance]
	Given an arbitrary invertible linear transform $L$, and let $\mathcal{A} \in \mathbb{R}^{n_1 \times n_2 \times n_3}$.
	For any orthogonal tensors $\mathcal{P} \in \mathbb{R}^{n_1 \times n_1 \times n_3}$ and $\mathcal{Q} \in \mathbb{R}^{n_2 \times n_2 \times n_3} $, we have
	\[\begin{aligned}
		||\mathcal{A}||_{\mathrm{TNPK}} &=||\mathcal{P} \ast_L \mathcal{A}||_{\mathrm{TNPK}}=||\mathcal{A} \ast_L \mathcal{Q}^{\ast}||_{\mathrm{TNPK}}\\
		&=||\mathcal{P} \ast_L \mathcal{A} \ast_L \mathcal{Q}^{\ast}||_{\mathrm{TNPK}}.
	\end{aligned}
	\]
\end{proposition}

\begin{property}[parameter flexibility]
	Let $\mathcal{X} \in \mathbb{R}^{n_1 \times n_2 \times n_3}$, the proposed TNPK regularizer provides a flexible fractional-based framework
	through the parameters $p$ and $k$, where
	$1 \leq k \leq \min(n_1,n_2)$ and $1 \leq p < \infty$.
	By choosing different values of $p$ and $k$, the denominator of the TNPK regularizer reduce to several special cases.
	\begin{itemize}
		\item
		\textbf{(I)}: When $p=1$, the denominator of TNPK degenerates to the Ky Fan $k$ norm:\[
		\|\mathcal{X}\|_{1,k}
		=
		\sum_{i=1}^{k}\sigma_i(\mathcal{X})
		=
		\|\mathcal{X}\|_{(k)}.
		\]
		Accordingly, the TNPK regularizer reduces to
		\[
		\|\mathcal{X}\|_{\mathrm{TNK}}
		=
		\frac{\|\mathcal{X}\|_*}{\|\mathcal{X}\|_{(k)}},
		\]
		which we refer to as the TNK regularizer.
		
		\item
		\textbf{(II)}: When $p = \infty$, the denominator of TNPK degenerates to the largest
		singular value:
		\[
		\|\mathcal{X}\|_{\infty,k}
		=
		\max_{1 \leq i \leq k}\sigma_i(\mathcal{X})
		=
		\sigma_1(\mathcal{X}).
		\]
		In this case the denominator no longer depends on $k$ and coincides with the
		Ky Fan $1$ norm.
		Hence, the TNPK regularizer coincides with the TNK regularizer with $k=1$.
		
		\item 
		\textbf{(III)}: When $p=2$ and $k=\min\{n_1,n_2\}$, the denominator of TNPK degenerates to the
		Frobenius norm:
		\[
		\|\mathcal{X}\|_{2,\min\{n_1,n_2\}}
		=
		\left(
		\sum_{i=1}^{\min(n_1,n_2)}
		\sigma_i^2(\mathcal{X})
		\right)^{1/2}
		=
		\|\mathcal{X}\|_F.
		\]
		Accordingly, the TNPK regularizer reduces to
		\[
		\|\mathcal{X}\|_{\mathrm{TNF}}
		=
		\frac{\|\mathcal{X}\|_*}{\|\mathcal{X}\|_F}.
		\]
		When the transform $L$ is specified as DFT, this special case coincides with the DFT-based TNF regularizer in \cite{zheng2024scale,zheng2025tensor}. 
		In contrast to that regularizer, which is defined only under the DFT, the proposed TNPK regularizer is formulated under a general invertible transform $L$.
	\end{itemize}
\label{property3}
\end{property}

As shown above, the TNPK regularizer encompasses two representative special cases:
the TNK and the TNF. Since the TNF regularizer has already been systematically studied by
Zheng et al. \cite{zheng2024scale,zheng2025tensor}, this paper does not revisit it in
detail and instead focuses on the TNK regularizer.

%Tran and Webster \cite{tran2019class} extended the classical null space property (NSP) from sparse vector recovery to the analysis of a class of nonconvex surrogate functions under symmetry, separability, and concavity assumptions. Subsequently, Rahimi et al. \cite{rahimi2019scale} proposed a stronger NSP-type condition to handle the nonseparable $\ell_1/\ell_2$ formulation.
%In the tensor setting, Zheng et al. \cite{zheng2024scale} further generalized the NSP framework to nonseparable functionals.
%Inspired by their analysis, particularly the NSP framework developed in \cite{zheng2024scale}, this work incorporates analogous structural conditions into the theoretical study of the TNK regularization model.
Tran and Webster \cite{tran2019class} extended the classical null space property (NSP) from sparse vector recovery to a class of nonconvex surrogates under symmetry, separability, and concavity assumptions. Rahimi et al. \cite{rahimi2019scale} subsequently proposed a stronger NSP condition for the nonseparable $\ell_1/\ell_2$ formulation, and Zheng et al. \cite{zheng2024scale} extended the NSP framework to nonseparable tensor functionals, including the TNF regularizer. Although these studies provide the analytical foundation for our work, their results cannot be directly applied to the TNK regularizer. The denominator of TNK is the Ky Fan $k$ norm, which is a nonsmooth and nonseparable function determined only by the $k$ largest singular values. By contrast, the Frobenius norm in the denominator of TNF depends on the entire singular spectrum. The truncated denominator of TNK requires conditions beyond the scope of existing NSP results. We therefore establish a tailored NSP condition under which local optimality can be guaranteed for the TNK regularization problem.

Let $\mathcal{X} = \mathcal{U}_\mathcal{X} *_L \mathcal{S}_\mathcal{X} *_L \mathcal{V}_\mathcal{X}^{*} \in \mathbb{R}^{n_1 \times n_2 \times n_3}$ be a nonzero tensor with $\mathrm{rank}_t(\mathcal{X}) = r$, and let $n := \min\{n_1, n_2\}$.
For any tensor
$\mathcal{W}\in\mathbb{R}^{n_1\times n_2\times n_3}$ with the t-SVD, $\mathcal{W}
= \mathcal{U}_{\mathcal{W}}*_L \mathcal{S}_{\mathcal{W}}*_L \mathcal{V}_{\mathcal{W}}^{*},$
where $1\leq k\leq r\leq n$, we define
%\[
%\mathcal{W}_k
%:=
%\mathcal{U}_{\mathcal{W}}(:,1\!:\!k,:)
%*_L
%\mathcal{S}_{\mathcal{W}}(1\!:\!k,1\!:\!k,:)
%*_L
%\mathcal{V}_{\mathcal{W}}(:,1\!:\!k,:)^{*},
%\]
%\[
%\begin{aligned}
%	\mathcal{W}_{[k+1:r]}
%	: = &\mathcal{U}_{\mathcal{W}}(:,k+1\!:\!r,:)
%	*_L
%	\mathcal{S}_{\mathcal{W}}(k+1\!:\!r,k+1\!:\!r,:)\\
%	&*_L
%	\mathcal{V}_{\mathcal{W}}(:,k+1\!:\!r,:)^{*},
%\end{aligned}
%\]
%\[
%\begin{aligned}
%	\mathcal{W}_{[r+1:]}
%	:=&
%	\mathcal{U}_{\mathcal{W}}(:,r+1\!:\!n,:)
%	*_L
%	\mathcal{S}_{\mathcal{W}}(r+1\!:\!n,r+1\!:\!n,:)\\
%	&*_L
%	\mathcal{V}_{\mathcal{W}}(:,r+1\!:\!n,:)^{*}.
%\end{aligned}
%\]
\[
\begin{aligned}
\mathcal{W}_k
:=&
\mathcal{U}_{\mathcal{W}}(:,1\!:\!k,:)
*_L
\mathcal{S}_{\mathcal{W}}(1\!:\!k,1\!:\!k,:)
*_L
\mathcal{V}_{\mathcal{W}}(:,1\!:\!k,:)^{*},\\
	\mathcal{W}_{[k+1:r]}
: = &\mathcal{U}_{\mathcal{W}}(:,k+1\!:\!r,:)
*_L
\mathcal{S}_{\mathcal{W}}(k+1\!:\!r,k+1\!:\!r,:)\\
&*_L
\mathcal{V}_{\mathcal{W}}(:,k+1\!:\!r,:)^{*},\\
\mathcal{W}_{[r+1:]}
:=&
\mathcal{U}_{\mathcal{W}}(:,r+1\!:\!n,:)
*_L
\mathcal{S}_{\mathcal{W}}(r+1\!:\!n,r+1\!:\!n,:)\\
&*_L
\mathcal{V}_{\mathcal{W}}(:,r+1\!:\!n,:)^{*}.
\end{aligned}
\]
%Thus, $\mathcal{W}_k$, $\mathcal{W}_{[k+1:r]}$, and
%$\mathcal{W}_{[r+1:]}$ collect the components associated with the first
%$k$ singular values, the singular values indexed from $k+1$ to $r$, and
%the remaining singular values indexed from $r+1$ to $n$, respectively.
These subtensors collect the components associated with the first $k$ singular values, the singular values indexed from $k+1$ to $r$, and the remaining singular values indexed from $r+1$ to $n$, respectively.
%Accordingly, the TNN of the first two subtensors are given by
%\[||\mathcal{W}_k||_\ast =\sum_{j=1}^{k}\sigma _j(\mathcal{W}) ,||\mathcal{W}_{[k+1:r]}||_\ast =\sum_{j=k+1}^{r}\sigma _j(\mathcal{W}).\]
%Using this decomposition, we next formulate a tensor null space property tailored to the TNK regularizer and establish the corresponding local optimality result in Theorem~\ref{theorem:nsp}. For convenience, we denote the complementary component of $\mathcal{W}_k$ by \[\mathcal{W}_{k^c}
%:=
%\mathcal{W}_{[k+1:]}=
%\mathcal{W}_{[k+1:r]}+
%\mathcal{W}_{[r+1:]}.\]
%Thus, $\mathcal{W}=
%\mathcal{W}_k+\mathcal{W}_{k^c}.$ 
For convenience, we further define \[\mathcal{W}_r=\mathcal{W}_k+\mathcal{W}_{[k+1:r]},\quad \mathcal{W}_{r^c}=\mathcal{W}_{[r+1:]}.\] Thus, \[\begin{aligned}
	\mathcal{W} = \mathcal{W}_r+\mathcal{W}_{r^c} = \mathcal{W}_k+\mathcal{W}_{[k+1:r]}+\mathcal{W}_{[r+1:]}.
\end{aligned}\]
Moreover,
%\[||\mathcal{W}_k||_\ast =\sum_{j=1}^{k}\sigma _j(\mathcal{W})=:||\mathcal{W}_k||_{(k)} ,\]
%\[||\mathcal{W}_{[k+1:r]}||_\ast =\sum_{j=k+1}^{r}\sigma _j(\mathcal{W}),\]
%\[|| \mathcal{W}_{r^c}||_\ast=\sum_{j=r+1}^{r}\sigma _j(\mathcal{W}).\]
\[\begin{aligned}
	||\mathcal{W}_k||_\ast &=\sum_{j=1}^{k}\sigma _j(\mathcal{W})=:||\mathcal{W}_k||_{(k)} ,\\
	||\mathcal{W}_{[k+1:r]}||_\ast& =\sum_{j=k+1}^{r}\sigma _j(\mathcal{W}),\\
	|| \mathcal{W}_{r^c}||_\ast&=\sum_{j=r+1}^{n}\sigma _j(\mathcal{W}).
\end{aligned}\]
%Using this decomposition, we formulate an NSP-type condition suitable for the local optimality analysis of the TNK regularization problem.
Based on this decomposition, we introduce an NSP-type condition of order $r$ and derive the estimates required for the local optimality analysis of the TNK regularization problem.

\begin{Definition}[NSP-type condition \cite{zheng2024scale}]
	Let $r$ be a positive integer and let $s\in(0,1)$. A linear operator $F$ is said to satisfy an NSP-type condition of order $r$ with constant $s$ if
	\begin{equation}
		||\mathcal{W}_r||_\ast\le s||\mathcal{W}_{r^c}||_\ast
		\label{eq:nsp-r}
	\end{equation}
	holds for every $\mathcal{W} \in \mathcal{N}(F)\setminus{\left \{ \mathcal{O} \right \} }$.
%	for all $\mathcal{W} \in \mathcal{N}(F)\setminus{\left \{ \mathcal{O} \right \} }$,
%	then $F$ is said to satisfy an NSP-type condition of order $k$ with respect to the Ky Fan $k$ norm, with constant $s \in (0,1)$.
	\label{TNK-NSP}
\end{Definition}

The following lemma provides the estimates required for the TNK regularizer.
%Furthermore, by relating our analysis to the Ky Fan $k$ norm, we obtain the following lemma.
\begin{Lemma}
	Suppose that the linear operator $F$ satisfies the NSP-type condition of order $r$ with constant $s\in(0,1)$. Then, for every $\mathcal{W} \in \mathcal{N}(F)\setminus{\left \{ \mathcal{O} \right \} }$
	and every integer $k\leq r$, we have
	\[||\mathcal{W}||_{(k)}=||\mathcal{W}_k||_\ast \le ||\mathcal{W}_r||_\ast \le s||\mathcal{W}_{r^c}||_\ast,\]
	and
	\[||\mathcal{W}_{r^c}||_\ast-||\mathcal{W}_r||_\ast \ge (1-s) ||\mathcal{W}_{r^c}||_\ast.\]
\end{Lemma}

\begin{proof}
	Since $k\leq r$, we have $||\mathcal{W}_k||_\ast \le ||\mathcal{W}_r||_\ast.$ Combining this inequality with \eqref{eq:nsp-r} gives
	\[||\mathcal{W}||_{(k)}=||\mathcal{W}_k||_\ast \le ||\mathcal{W}_r||_\ast \le s||\mathcal{W}_{r^c}||_\ast.\]
	Moreover, it follows from \eqref{eq:nsp-r} that
	\[\begin{aligned}
		||\mathcal{W}_{r^c}||_\ast-||\mathcal{W}_r||_\ast &\ge ||\mathcal{W}_{r^c}||-s||\mathcal{W}_{r^c}||\\
		&= (1-s) ||\mathcal{W}_{r^c}||_\ast.
	\end{aligned}\]
	This completes the proof.
\end{proof}

\begin{Theorem}
	Let
	$\mathcal{X} \in \mathbb{R}^{n_1 \times n_2 \times n_3}\setminus{\left \{ \mathcal{O} \right \}}$
	%$\mathcal{X} \in \mathbb{R}^{n_1 \times n_2 \times n_3}$
	be a tensor with $\mathrm{rank}_t(\mathcal{X}) = r, 1\leq k\leq r,$ and satisfying $F(\mathcal{X})=\mathcal{T}$.
	Suppose that the linear operator $F$ satisfies an NSP-type condition of order $r$ with constant $s<\frac{||\mathcal{X} ||_{(k)}}{||\mathcal{X} ||_{\ast }+||\mathcal{X} ||_{(k)}} $ (see Definition \ref{TNK-NSP}). Then $\mathcal{X}$ is a local minimizer of problem \eqref{eq:3.3}.
	More precisely, there exists a constant $t_*>0$ such that
	\[||\mathcal{X}+\mathcal{H}||_\mathrm{TNK} \ge ||\mathcal{X} ||_\mathrm{TNK} \]
	holds for every $\mathcal{H} \in \mathcal{N}(F)\setminus{\left \{ \mathcal{O} \right \} }$ satisfying $||\mathcal{X} ||_F \le  t_\ast$.
	\label{theorem:nsp}
\end{Theorem}

%Motivated by the analysis in \cite{zheng2024scale}, and taking into account the role of the Ky Fan $k$ norm, we develop a corresponding extension. The proof of Theorem \ref{theorem:nsp} is given in Appendix A.
Theorem \ref{theorem:nsp} establishes a local optimality guarantee for the TNK regularization problem by explicitly accounting for the distinct roles of the tubal-rank parameter $r$ and the truncation parameter $k$ in the Ky Fan $k$ norm denominator. The proof is provided in Appendix \ref{appA}.

\vspace{-0.5cm}
\section{\textbf{TNPK-based low-rank tensor completion}}%low-rank tensor completion
\label{sec5}
%\vspace{-0.5cm}
In this section, we consider using TNPK to solve the LRTC problem. Given the observed tensor $\scr{M}\in \mathbb{R}^{n_1\times n_2\times n_3}$ on the observation index set $\Omega $, where the unknown entries take the value of zero. The tensor completion problem is to fill in the missing entries from the incomplete observed tensor $\scr{M}$ in order to recover the underlying complete tensor. In this work, we address the LRTC problem based on the proposed TNPK:
\begin{equation}
	\min_{\mathcal{X} } \left \| \scr{X} \right \|_{\mathrm{TNPK} }  \quad \text{s.t.} \quad \mathcal{P} _\Omega(\mathcal{X}  - \mathcal{M} ) = \mathcal{O},
	\label{eq:1}
\end{equation}
where $\scr{P}_\Omega (\cdot )$ is a projection operator in a way that $\scr{P}_\Omega (\scr{X} - \scr{M} ) = \scr{O}$ forces the entries of $\scr{X}$ agree with $\scr{M}$ on $\Omega$ and sets the other elements of $\scr{X}$ to zero.
We reformulate problem \eqref{eq:1} as an unconstrained optimization problem by incorporating the constraint into the objective function through an indicator function. Specifically, the indicator function is defined as follows:
%Following the method in \cite{zheng2024scale}, we introduce the same indicator function to transform problem \eqref{eq:1} into an unconstrained optimization problem. The indicator function is defined as follows:
\[I_\Phi(\mathcal{E}) := 
\begin{cases} 
	0 & \text{if } \mathcal{E} \in \Phi, \\
	\infty & \text{otherwise},
\end{cases}\]
where $\Phi := \{ \scr{E} \in \mathbb{R}^{n_1\times n_2\times n_3},\scr{P} _\Omega (\scr{E} - \scr{M} ) = \scr{O} \}.$ Then, \eqref{eq:1} can be expressed as the following unconstrained optimization: 
\begin{equation}
	\min_{\mathcal{X} } \frac{\left \| \mathcal{X}  \right \| _*}{\left \| \mathcal{X}  \right \|_{p,k}} + I_\Phi(\mathcal{X}).
	\label{eq:2}
\end{equation}

\vspace{-0.5cm}
\subsection{Proximal Operator of the Ky Fan $p$-$k$ Inverse-Norm}
To solve problem \eqref{eq:1}, we employ the standard Alternating Direction Method of Multipliers (ADMM) \cite{lu2017unified}. The key computational challenge in each ADMM iteration lies in evaluating the proximal operator of the inverse-norm associated with the tensor Ky Fan $p$-$k$ norm:
%For problem \eqref{eq:1}, the standard Alternating Direction Method of Multipliers (ADMM) \cite{lu2017unified} can be applied. The key step is to compute the proximal operator of the inverse-norm associated with the tensor Ky Fan $p$-$k$ norm:
\begin{equation}
	\min_{\mathcal{A}\in \mathbb{R}^{n_1\times n_2\times n_3}} \frac{\lambda }{\left \|\mathcal{A}\right \| _{p,k}} + \frac{1}{2} \left \|\mathcal{A} - \mathcal{B}\right \| _F^2.
	\label{problem1}
\end{equation}

According to Theorem \ref{theorem:2}, the proximal operator of the TNPK inverse-norm can be transformed into the following scalar optimization problem:
\begin{equation}
	\min_{s_1\ge  \cdots \ge s_r \ge 0}\left \{ \frac{\lambda }{S_{p,k}(s)}+\frac{1}{2}\sum_{i=1}^{r} (s_i-\sigma _i)  ^2 \right \},
	\label{scalar}
\end{equation}
where $S_{p,k}(s)$ is defined as
\[	S_{p,k}(s)=\left\{\begin{matrix}
	\left ( \sum_{i=1}^{k} s_i^p \right )^{1/p} ,& 1\le p \le \infty \\
	s_1,& p = \infty .
\end{matrix}\right.\]
Let $\mathcal{B}\in \mathbb{R}^{n_1\times n_2\times n_3}$ admit the following t-SVD:
\[
\mathcal{B}=\mathcal{U} _\mathcal{B}*_L \mathcal{S}_\mathcal{B} *_L \mathcal{V}_\mathcal{B}^*,
\]
with singular values $\sigma _1(\mathcal{B} )\ge \sigma _2(\mathcal{B} )\ge \dots \ge \sigma_r (\mathcal{B} )\ge 0$.
To address the above optimization problem, we introduce a TNPK inverse-norm proximal operator. This operator extends the classical singular value thresholding (SVT) operator to the nonconvex inverse-norm setting and is defined as follows:
%We propose a TNPK inverse-norm proximal operator, which generalizes the classical singular value thresholding (SVT) to a nonconvex inverse-norm framework. We then define the inverse-norm proximal operator as follows.
\begin{equation}
	\mathcal{D}_\lambda ^{(p,k)}(\mathcal{B} )=\mathcal{U}_\mathcal{B} *_L \mathcal{S}_\lambda ^{(p,k)}*_L \mathcal{V}_\mathcal{B}^*,
	\label{singular_value_thresholding_operator}
\end{equation}
where $\mathcal{S}_\lambda ^{(p,k)}$ is an f-diagonal tensor whose diagonal entries $s_1\ge s_2\ge \dots \ge s_r \ge 0$ are the solution to the scalar optimization problem \eqref{scalar}.

According to the KKT conditions, the singular values $s_i$ satisfy:
\[s_i=\sigma _i(\mathcal{B}),\quad i>k,\]
while for the leading $k$ singular values,
\begin{equation}
	s_i-\sigma _i(\mathcal{B})=\frac{\lambda s_i^{p-1}}{S_{p,k}^{p+1}}, \quad i\le k.
	\label{key eq} 
\end{equation}
%Below we present several important cases admitting closed-form solutions.
%Building upon the special cases identified in Property \ref{property3}, we next derive closed-form expressions for the corresponding proximal operators.
For the special cases identified in Property \ref{property3}, the corresponding proximal operators can be characterized explicitly, as shown below.

\begin{itemize}
	\item
	\textbf{(I)}:
	In this case, the denominator in the inverse-norm term reduces to the Ky Fan $k$ norm, namely,
	$S_{p,k}(s) = \sum_{i=1}^{k} s_i$.
	Therefore, \eqref{key eq} simplifies to: $s_i=\sigma _i(\mathcal{B})+\alpha$, where $\alpha =\frac{\lambda }{S_k^2} ,S_k=\sum_{i=1}^{k}s_i $.
	Let $\Sigma _k=\sum_{i=1}^{k}\sigma _i (\mathcal{B})$. Then $S_k$ satisfies the cubic equation: \[S_k^3-\Sigma_kS_k^2-k\lambda =0,\]
	This equation admits a unique positive real root, which can be obtained using Cardano’s formula: \[S_k=\frac{\Sigma_k}{3} +\sqrt[3]{-\frac{q}{2} +\sqrt{\Delta} } +\sqrt[3]{-\frac{q}{2} -\sqrt{\Delta} }, \]
	where 
	\[q=-k\lambda -\frac{2\Sigma_k^3}{27} ,\quad \Delta=\frac{k^2\lambda ^2}{4}+ \frac{k\lambda \Sigma_k^3}{27} .\]
	%Finally, the singular values are given by:
	Consequently, the updated singular values are given by
	\begin{equation}
		s_i=\begin{cases}
			\sigma _i(\mathcal{B})+\frac{\lambda}{S_k^2} & ,i\le k \\
			\sigma _i(\mathcal{B})&,i>k.
		\end{cases}
		\label{solution1}
	\end{equation}
	
	\item
	\textbf{(II)}:
	Here $S_{p,k}(s) = s_1$, so only the largest singular value is affected
	while all others remain unchanged.
	This is the special case of \textbf{Case (I)} with $k = 1$,
	and the closed-form solution follows directly from \eqref{solution1}.
%	In this case, $S_{p,k}=s_1 $, meaning that only the largest singular value is modified, while all other singular values remain unchanged. This case can be regarded as a special case of \textbf{Case I} with $k=1$, and the result can be derived accordingly.
	
	\item
	\textbf{(III)}:
	In this case, the denominator in the inverse-norm term reduces to the Frobenius norm,
	$S_{p,k}(s) = \bigl(\sum_{i=1}^{r} s_i^2\bigr)^{1/2}$. From the KKT conditions, it can be derived that the singular values satisfy a proportional relationship $s_i=\gamma \sigma _i(\mathcal{B})$. Substituting this into the norm definition yields $F=\gamma \sigma _F$, where $F=\left \| \mathcal{A} \right \| _F$ and $\sigma _F=\left \| \mathcal{B} \right \| _F$. Combining with $\gamma =\frac{1}{1-\lambda/F^3 } $, we obtain:
	\[\frac{F}{\sigma _F}=\frac{1}{1-\lambda /F^3 }  .\]
	Rearranging gives a cubic equation in $F$:
	\[F^3-\sigma _FF^2-\lambda =0.\]
	This equation admits a unique positive real root, which can be obtained via Cardano’s formula:
	\[F=\sqrt[3]{\frac{\lambda }{2} +\sqrt{\left ( \frac{\lambda }{2}  \right ) ^2-\frac{\sigma _F^6}{27}  }}
	+\sqrt[3]{\frac{\lambda }{2} -\sqrt{\left ( \frac{\lambda }{2}  \right ) ^2-\frac{\sigma _F^6}{27}  }} .\]
	Finally, the singular values are given by:
	\begin{equation}
		s_i=\gamma \sigma _i(\mathcal{B})=\frac{F}{\sigma _F} \sigma _i(\mathcal{B}),\quad \forall i\le r.
		\label{solution2}
	\end{equation}
\end{itemize}

Based on the above analysis, the overall formulation of the proximal operator of the Ky Fan $p$-$k$ inverse-norm can be summarized as follows.
\begin{Theorem}
	Let $\mathcal{B} \in \mathbb{R}^{n_1 \times n_2 \times n_3}$, and let $\lambda>0$. Assume that the parameter pair $(p,k)$ belongs to one of the following admissible cases:
	\[\begin{aligned}
		(p,k)&\in\{(1,k):1\le k\le \min\{n_1,n_2\}\} \\
		&\cup \{(\infty ,k):1\le k\le \min\{n_1,n_2\}\} \cup \{(2,\min\{n_1,n_2\})\}.
	\end{aligned}\]
	Then the Ky Fan $p$-$k$ inverse-norm singular-value operator defined in
	%For any $\lambda > 0$, the tensor Ky Fan $p$-$k$ inverse-norm value operator 
	\eqref{singular_value_thresholding_operator} satisfies:
	\begin{equation}
		\mathcal{D}_\lambda ^{(p,k)}(\mathcal{B} )=
		\arg \min_{\mathcal{A}} \frac{\lambda }{\left \|\mathcal{A}\right \| _{p,k}} + \frac{1}{2} \left \|\mathcal{A} - \mathcal{B}\right \| _F^2.
	\end{equation}
	Thus, under the above admissible choices of $(p,k)$, $\mathcal{D}_\lambda ^{(p,k)}(\mathcal{B} )$ provides a proximal operator associated with the tensor Ky Fan $p$-$k$ inverse-norm regularizer $\left \|\cdot \right \| _{p,k}^{-1}$. 
	%This operator is exactly the proximal operator of the Inverse-norm $\left \|\cdot \right \| _{p,k}^{-1}$. 
	For parameter choices beyond these admissible cases, a closed-form solution is generally unavailable, and iterative numerical methods can be employed to compute the corresponding proximal update.
	\label{operator}
\end{Theorem}

\begin{proof}
	According to Theorem \ref{theorem:2}, the solution to problem \eqref{problem1} can be reduced to the scalar optimization problem \eqref{scalar} via singular value decoupling. Furthermore, by applying the KKT conditions, the update rules for the singular values can be derived. For the three special cases, the closed-form solutions are given by \eqref{solution1} and \eqref{solution2}, respectively. Therefore, the operator defined in \eqref{singular_value_thresholding_operator} yields the optimal solution to the original problem.
\end{proof}

\vspace{-0.25cm}
%
%Theorem \ref{operator} provides the closed-form solution of the TNPK inverse-norm proximal operator. For the general case where no closed-form solution exists, iterative methods can be employed for numerical computation. Algorithm \ref{alg:tnpk_svt} presents the detailed procedure for updating tubal singular values within the t-SVD framework.
The proximal operator established in Theorem \ref{operator} enables the
closed-form update of the corresponding tensor variable in the subsequent
LRTC algorithm. Specifically, by applying this operator to the tubal singular
values under the t-SVD framework, the Ky Fan $p$-$k$ inverse-norm regularized
subproblem can be solved efficiently. Algorithm~\ref{alg:tnpk_svt} presents
the detailed procedure for updating the tubal singular values.

\begin{algorithm}[!htbp]
	\setstretch{0.3}
	\caption{ Ky Fan $p$-$k$ Inverse-Norm Proximal Operator.}
	\label{alg:tnpk_svt}
	\KwIn{
		Tensor $\mathcal{B} \in \mathbb{R}^{n_1 \times n_2 \times n_3}$,
		$\lambda > 0$,
		$p \in [1,\infty]$,
		$k \in [1,r]$,
		where $r = \min\{n_1, n_2\}$.
	}
	\textbf{Initialize:}
	Compute the t-SVD: $\mathcal{B}=\mathcal{U} _\mathcal{B}*_L \mathcal{S}_\mathcal{B} *_L \mathcal{V}_\mathcal{B}^*$
	and extract singular values $\sigma _1(\mathcal{B} )\ge \sigma _2(\mathcal{B} )\ge \dots \ge \sigma_r (\mathcal{B} )$\;
	Set $s_i = \sigma_i(\mathcal{B} )$ for $i = 1, \dots, r$\;
	
	\eIf{$p = 1$ \textbf{or}
		$p = \infty$ }
	{
		Obtain $s_i$ from the closed-form solution \eqref{solution1}\;
	}
	{
		\eIf{$p = 2$ \textnormal{and} $k =\min\{n_1, n_2\}$}
		{
			Obtain $s_i$ from the closed-form solution \eqref{solution2}\;
		}
		{
			General case: no closed-form solution
			
%			\tcp{General case: no closed-form solution}
			\For{$i = 1$ \KwTo $k$}
			{
				Solve \eqref{key eq} iteratively to obtain $s_i$\;
			}
			\For{$i = k+1$ \KwTo $r$}
			{
				Set $s_i = \sigma_i(\mathcal{B} )$\;
			}
		}
	}
	
	Construct the f-diagonal tensor $\mathcal{S}$ with diagonal entries $s_i$\;
	
	\KwOut{
		$\mathcal{D}_\lambda ^{(p,k)}(\mathcal{B} )=\mathcal{U}_\mathcal{B} *_L \mathcal{S}_\lambda ^{(p,k)}*_L \mathcal{V}_\mathcal{B}^*$
		as defined in \eqref{singular_value_thresholding_operator}.
	}
\end{algorithm}

\vspace{-0.72cm}
\subsection{Optimization algorithm}
To solve \eqref{eq:2} using ADMM \cite{lu2017unified}, we introduce the auxiliary variable $\scr{H}$ and rewrite \eqref{eq:2} into an equivalent form
\begin{equation}
	\min_{\mathcal{X}, \mathcal{H}} \frac{\left \| \mathcal{X} \right \| _*}{ \left \|  \mathcal{H} \right \|_{p,k}}+ I_\Phi(\mathcal{X}) \quad \text{s.t.} \quad \mathcal{X} = \mathcal{H},
	\label{eq:3}
\end{equation}
then the augmented Lagrangian function of \eqref{eq:3} is given by
\begin{equation}
	T_{\text{TC}}\left(\mathcal{X}, \mathcal{H},  \mathcal{C}\right) =\frac{\left \| \mathcal{X}  \right \| _*}{\left \| \mathcal{H}  \right \|_{p,k}}
	+I_\Phi (\mathcal{X} ) +\frac{\mu _1}{2}\left \| \mathcal{X}-\mathcal{H} \right \|_F ^2 +\left \langle \mathcal{C} , \mathcal{X}-\mathcal{H}\right \rangle ,
	\label{eq:4}
\end{equation}
where $\mathcal{C}$ is a Lagrange multipliers and $\mu_1 > 0$ is a penalty parameters. The ADMM iteration update scheme for \eqref{eq:4} is as follows:
\begin{equation}
	\begin{cases}
		\mathcal{X}^{(t+1)} = \arg\min_{\mathcal{X}}	T_{\text{TC}}\big(\mathcal{X}, \mathcal{H}^{(t)}, \mathcal{C}^{(t)}\big),\\[2mm]
		\mathcal{H}^{(t+1)} = \arg\min_{\mathcal{H}} 	T_{\text{TC}}\big(\mathcal{X}^{(t+1)}, \mathcal{H}, \mathcal{C}^{(t)}\big),\\[1mm]
		\mathcal{C}^{(t+1)} = \mathcal{C}^{(t)} + \mu_1 \big(\mathcal{X}^{(t+1)} - \mathcal{H}^{(t+1)}\big),
		\label{eq:44}
	\end{cases}
\end{equation}
where $t$ is the number of iterations. Next, we give the specific solutions of the $\scr{X}$-subproblem and the $\scr{H}$-subproblem.

\subsubsection{\textbf{Solving the $\scr{X}$-subproblem}}
%\noindent \textbf{(1) Solving the $\scr{X}$-subproblem:}
The $\scr{X}$-subproblem is equivalent to:
\begin{equation}
	\min_{\mathcal{X}} \frac{\left \| \mathcal{X}  \right \| _*}{\left \| \mathcal{H}^{(t)}  \right \|_{p,k}}
	+I_\Phi (\mathcal{X} ) +\frac{\mu _1}{2}\left \| \mathcal{X}-\mathcal{H}^{(t)} +\frac{\mathcal{C}^{(t)} }{\mu _1} \right \|_F ^2.
	\label{eq:5}
\end{equation}
At this point, \eqref{eq:5} can be viewed as a constrained TNN regularized tensor completion problem. To handle this, an auxiliary variable $\scr{Y}$ is further introduced to enforce the constraint $I_\Phi (\mathcal{X} )$, and rewrite \eqref{eq:5} into an equivalent form
\begin{equation}
	\begin{aligned}
		&\min_{\mathcal{X}, \mathcal{Y}} \frac{\left \| \mathcal{X} \right \| _*}{ \left \|  \mathcal{H}^{(t)} \right \|_{p,k}}+ I_\Phi(\mathcal{Y})\\
		&+\frac{\mu _1}{2}\left \| \mathcal{X}-\mathcal{H}^{(t)} +\frac{\mathcal{C}^{(t)} }{\mu _1} \right \|_F ^2 \quad\text{s.t.} \quad \mathcal{X} = \mathcal{Y}.
	\end{aligned}
	\label{eq:6}
\end{equation}
Hence we can apply ADMM to solve the new $\scr{X}$-subproblem \eqref{eq:6}, and the augmented Lagrangian of \eqref{eq:6} is given by
%\begin{equation}
%	L^{(k)}_{\text{inner}}\left(\mathcal{X}, \mathcal{Y},  \mathcal{B}\right) =\frac{\left \| \mathcal{X}  \right \| _*}{\left \| \mathcal{H}  \right \|_{p,k}}
%	+I_\Phi (\mathcal{Y} ) +\frac{\mu _1}{2}\left \| \mathcal{X}-\mathcal{H}^{(k)}+\frac{\mathcal{A}^{(k)} }{\mu _1} \right \|_F ^2 
%	+\frac{\mu _2}{2}\left \| \mathcal{X}-\mathcal{Y}+\frac{\mathcal{B} }{\mu _2} \right \|_F ^2 
%\end{equation}
\[\begin{aligned}
	T^{(t)}_{\text{inner}}\left(\mathcal{X}, \mathcal{Y},  \mathcal{N}\right) &=\frac{\left \| \mathcal{X}  \right \| _*}{\left \| \mathcal{H}  \right \|_{p,k}}+I_\Phi (\mathcal{Y} )+\frac{\mu _1}{2}\left \| \mathcal{X}-\mathcal{H}^{(t)}+\frac{\mathcal{C}^{(t)} }{\mu _1} \right \|_F ^2 \\
	&+\frac{\mu _2}{2}\left \| \mathcal{X}-\mathcal{Y}+\frac{\mathcal{N} }{\mu _2} \right \|_F ^2 ,
\end{aligned}\]
where $\mathcal{N}$ is a Lagrange multiplier and $\mu_2> 0$ is a penalty parameters. 
The ADMM framework leads the following iterations:
\[\begin{cases}
	\mathcal{X}_{j+1} = \arg\min_{\mathcal{X}} T_{\text{inner}}^{(t)}(\mathcal{X}, \mathcal{Y}_j, \mathcal{N}_j),\\[2mm]
	\mathcal{Y}_{j+1} = \arg\min_{\mathcal{Y}} T_{\text{inner}}^{(t)}(\mathcal{X}_{j+1}, \mathcal{Y}, \mathcal{N}_j),\\[1mm]
	\mathcal{N}_{j+1} = \mathcal{N}_j + \mu_2 (\mathcal{X}_{j+1} - \mathcal{Y}_{j+1}).
\end{cases}\]
Here, the subscript $j$ corresponds to the inner-loop iteration, while the superscript $t$ corresponds to the outer-loop iteration in \eqref{eq:44}.

First, we update $\mathcal{X}_{j+1}$, with the corresponding subproblem given by:
\begin{equation}
	\begin{aligned}
		\mathcal{X}_{j+1} &=\arg \min_{\mathcal{X}} \frac{\left \| \mathcal{X}  \right \| _*}{(\mu_1 +\mu_2)\left \| \mathcal{H}^{(t)}  \right \|_{p,k}}\\
		&+\frac{1}{2}\left \| \mathcal{X}-\frac{1}{\mu _1+\mu _2} (\mu _1\mathcal{H}^{(t)}+\mu _2\mathcal{Y}_j )+\frac{\mathcal{C}^{(t)}+\mathcal{N}_j}{\mu _1+\mu _2}  \right \|_F^{2}.
	\end{aligned}
	\label{eq:7}
\end{equation}
For the subproblem of $\mathcal{X}_{j+1}$, a closed-form solution can be obtained via the tensor singular value thresholding (t-SVT) \cite{lu2019tensor}. To make this paper self-contained, we present the theorem from \cite{lu2019tensor}, followed by the formula to update $\mathcal{X}_{j+1}$.
\begin{Theorem}[t-SVT \cite{lu2019tensor}]
	Given a scalar $\tau > 0$, consider the following minimization problem:
	\[\min_{\mathcal{X} \in \mathbb{R}^{n_1 \times n_2 \times n_3}} \tau \|\mathcal{X}\|_* + \frac{1}{2} \|\mathcal{X} - \mathcal{Z}\|_F^2.\]
	Let $\mathcal{Z} = \mathcal{U}_\mathcal{Z} *_L \mathcal{S}_\mathcal{Z} *_L \mathcal{V}_\mathcal{Z}^*$. A closed-form solution for $\mathcal{X}$ exists and is given by
	\[
	\mathcal{X} = \mathcal{D}_\tau(\mathcal{Z}),
	\]
	where $\mathcal{D}_\tau(\mathcal{Z}) = \mathcal{U}_\mathcal{Z} *_L \mathcal{S}_\tau *_L \mathcal{V}_\mathcal{Z}^*$ is the tensor singular value thresholding (t-SVT) operator, and $\mathcal{S}_\tau$ is an $n_1 \times n_2 \times n_3$ tensor that satisfies
	\[
	\bar{\mathcal{S}}_\tau = \max\{\bar{\mathcal{S}} - \tau, 0\}.
	\]
	\label{theorem:1}
\end{Theorem}

Combining Theorem \ref{theorem:1} and \eqref{eq:7}, we have
\[\mathcal{X}_{j+1}=\mathcal{D}_ \tau (\mathcal{Z} _ j)\]
Where $\tau  =\frac{1}{(\mu_1 +\mu_2)\left \| \mathcal{H}^{(t)}  \right \|_{p,k}}$
and $\mathcal{Z} _ j = \frac{1}{\mu_1 +\mu_2} [\mu_1(\mathcal{H}^{(t)}-\frac{\mathcal{C}^{(t)} }{\mu_1} )+\mu_2(\mathcal{Y}_j-\frac{\mathcal{N}_j }{\mu_2} )]$.

Next, we update $\mathcal{Y}_{j+1}$, with the corresponding subproblem given by:
\[\mathcal{Y} _{j+1}=\mathrm {arg}\min_{\mathcal{Y} } \left \{ I_\Phi (\mathcal{Y} )+\frac{\mu_2}{2}\left \| \mathcal{Y} -(\mathcal{X}_{j+1}+\frac{\mathcal{N}_j }{\mu_2} )  \right \|^2_F   \right \}.  \]
Here, we follow the projection method in \cite{zheng2024scale} and project $\mathcal{X}_{j+1}+\frac{\mathcal{N}_j }{\mu_2}$ onto the set $\Omega^C$, i.e.,
\begin{equation}
	\begin{cases}
		\mathcal{P}_\Omega (\mathcal{Y}_{j+1} ) =\mathcal{P}_\Omega (\mathcal{M} )\\
		\mathcal{P}_{\Omega^C} (\mathcal{Y}_{j+1} ) =\mathcal{P}_{\Omega^C}(\mathcal{X}_{j+1} + \frac{\mathcal{N}_j}{\mu_2}),
	\end{cases}
	\label{eq:8}
\end{equation}
where $\Omega^C$ denotes the complementary set of $\Omega$. \eqref{eq:8} can be rewritten as
\[\mathcal{Y}_{j+1} = \mathcal{P}_\Omega(\mathcal{M}) + \mathcal{P}_{\Omega^C} \left( \mathcal{X}_{j+1} + \frac{\mathcal{N}_j}{\mu_2} \right).\]

\subsubsection{\textbf{Solving the $\scr{H}$-subproblem}}
The $\scr{H}$-subproblem is equivalent to:
\begin{equation}
	\mathcal{H}^{(t+1)} = \arg \min_{\mathcal{H}} \frac{\left \| \mathcal{X}^{(t+1)}  \right \|_* }{\left \|\mathcal{H}\right \| _{p,k}} + \frac{\mu_1}{2} \left \|\mathcal{H} - (\mathcal{X}^{(t+1)}+\frac{\mathcal{C} ^{(t)}}{\mu_1}  )\right \| _F^2,
	\label{eq:9}
\end{equation}
Let $\rho ^{(t+1)}=\left \| \mathcal{X}^{(t+1)}  \right \|_* ,\mathcal{K} ^{(t)}=\mathcal{X}^{(t+1)}+\frac{\mathcal{C} ^{(t)}}{\mu_1} $. Then, the minimization subproblem \eqref{eq:9} can be simplified as:
\begin{equation}
	\mathcal{H}^{(t+1)} = \arg \min_{\mathcal{H}} \frac{\rho^{(t+1)}}{\left \|\mathcal{H}\right \| _{p,k}} + \frac{\mu_1}{2} \left \|\mathcal{H} - \mathcal{K}^{(t)}\right \| _F^2.
	\label{eq:10}
\end{equation}

For the subproblem of $\mathcal{H}^{(t+1)}$, a closed-form solution can be obtained via the Ky Fan $p$-$k$ inverse-norm proximal operator. It is worth noting that we only consider the cases where closed-form solutions exist. Combining Theorem \ref{operator} and \eqref{eq:10}, we have:
\begin{equation}
	\mathcal{H}^{(t+1)}=\mathcal{D}_\lambda ^{(p,k)}(\mathcal{K}^{(t)} ),
	\label{al_key}
\end{equation}
where $\lambda =\frac{\rho^{(t+1)}}{\mu_1} $.

In this paper, we focus on two special cases that admit closed-form proximal
operators: the \textbf{TNK} and the \textbf{TNF}. The TNF case is closely
related to the tensor nuclear-to-Frobenius norm ratio regularization proposed
by Zheng et al.~\cite{zheng2024scale,zheng2025tensor}, while differing from it
in both modeling framework and algorithmic realization. Specifically, Zheng
et al. \cite{zheng2024scale,zheng2025tensor} develop their formulation under the DFT-based t-SVD framework and
perform a global scaling of the entire tensor. In contrast, our formulation is
established under a more general invertible linear transform, and the resulting
method operates in the transformed spectral domain by updating the tubal
singular values via the derived proximal operator.

We summarize the ADMM scheme in Algorithm \ref{alg:tnpk_admm} for solving the TNPK-based tensor
completion problem \eqref{eq:1}.

\begin{algorithm}[!htbp]
	\setstretch{0.3}
	\caption{TNPK-based Tensor Completion via ADMM.}
	\label{alg:tnpk_admm}
	\KwIn{
		Observed tensor $\mathcal{M}$ on index set $\Omega$,
		parameters $\mu_1,\ \mu_2$,
		$\mathrm{tMax}$,\ $\mathrm{jMax}$,\ $\epsilon$.
	}
	\textbf{Initialize:}
	Compute the TNN-regularized LRTC solution as $\mathcal{X}^{(0)}$\;
	Set $\mathcal{H}^{(0)} = \mathcal{X}^{(0)}$,\
	 $t = 1$\;
	
	\While{$t \le \mathrm{tMax}$ \textnormal{and not converged}}
	{
		Set $j = 1$\;
		\While{$j \le \mathrm{jMax}$}
		{
			Update $\mathcal{X}$:
			\[
			\begin{aligned}
				&\mathcal{X}_{j+1} = \\
				&\mathcal{D}_{\tau^{(t)}}\!\left(
				\frac{\mu_1\!\left(\mathcal{H}^{(t)} - \tfrac{1}{\mu_1}\mathcal{C}^{(t)}\right)
					+\mu_2\!\left(\mathcal{Y}_j       - \tfrac{1}{\mu_2}\mathcal{N}_j\right)}
				{\mu_1+\mu_2}
				\right)\;
			\end{aligned}
			\]
			Update $\mathcal{Y}$:
			\[
			\mathcal{Y}_{j+1} = \mathcal{P}_{\Omega}(\mathcal{M})
			+ \mathcal{P}_{\Omega^c}\!\left(\mathcal{X}_{j+1}
			+ \tfrac{1}{\mu_2}\mathcal{N}_j\right)\;
			\]
			Update multiplier $\mathcal{N}$:
			\[
			\mathcal{N}_{j+1} = \mathcal{N}_j
			+ \mu_2\!\left(\mathcal{X}_{j+1} - \mathcal{Y}_{j+1}\right)\;
			\]
			$j = j + 1$\;
		}
		Set $\mathcal{X}^{(t+1)} = \mathcal{X}_{j}$,\quad
		$\mathcal{Y}^{(t+1)} = \mathcal{Y}_{j}$\;
		
		Update $\mathcal{H}^{(t+1)}$ by solving \eqref{al_key}\;
		
		Update multiplier $\mathcal{C}$:
		\[
		\mathcal{C}^{(t+1)} = \mathcal{C}^{(t)}
		+ \mu_1\!\left(\mathcal{X}^{(t+1)} - \mathcal{H}^{(t+1)}\right)\;
		\]
		Update proximal parameter:
		\[\tau^{(t+1)} = \frac{1}{(\mu_1+\mu_2)
			\left\|\mathcal{H}^{(t+1)}\right\|_{p,k}}\;
		\]
		Check convergence:
		$\left\|\mathcal{X}^{(t+1)}-\mathcal{X}^{(t)}\right\|_\infty \le \epsilon$
		and
		$\left\|\mathcal{C}^{(t+1)}-\mathcal{C}^{(t)}\right\|_\infty \le \epsilon$\;
		
		$t = t+1$\;
	}
	\KwOut{Recovered tensor $\hat{\mathcal{X}} = \mathcal{X}^{(t)}$.}
\end{algorithm}

\vspace{-0.6cm}
\subsection{\textbf{Complexity}}
In this section, we analyze the time complexity of Algorithm \ref{alg:tnpk_admm}. The main computational burden comes from the updates of $\mathcal{X}$ and $\mathcal{H}$, both of which involve the tensor singular value thresholding procedure.
Let $C_L(n_3)$ denote the computational cost of applying the linear transform $L$ to a single length-$n_3$ tube. Since the tensor contains $n_1n_2$ tubes along the third mode, the total cost of applying $L$ to all tubes is $\mathcal{O}(n_1 n_2 C_L(n_3))$.
The computational complexities of different transforms are summarized in Table \ref{Complexity}. 
In particular, for the DFT and DCT, fast algorithms lead to $C_L(n_3)=\mathcal{O}(n_3\log n_3)$, whereas the ROM is implemented as a dense matrix-vector multiplication and therefore requires $C_L(n_3)=\mathcal{O}(n_3^2)$.
For the $\mathcal{X}$-update, the dominant operation in each inner iteration is the t-SVT. This operation consists of two main steps: first, applying the transform along the third mode; second, computing the skinny SVD of each frontal slice in the transformed domain. Since there are $n_3$ frontal slices, each of size $n_1\times n_2$, the total cost of the SVD step is $\mathcal{O}(n_1 n_2 n_3\min\{n_1,n_2\})$.
Therefore, the computational cost of one t-SVT can be written as $C_{\mathrm{tSVT}}=\mathcal{O}(n_1 n_2 C_L(n_3)+n_1 n_2 n_3\min\{n_1,n_2\})$.
The $\mathcal{H}$-update evaluates the Ky Fan $p$-$k$ inverse-norm proximal operator $\mathcal{D}_\lambda^{(p,k)}$ on $\mathcal{K}^{(t)}$.
This update also requires the same transform and frontal-slice SVD operations as the t-SVT. Hence, its dominant computational cost is also $C_{\mathrm{tSVT}}$.
The additional operations, such as selecting and sorting the leading singular values, are of lower order compared with the SVD step.
The remaining updates, including the $\mathcal{Y}$-update, the multiplier updates, and the convergence check, are mainly element-wise operations. Their total cost is $\mathcal{O}(n_1 n_2 n_3)$, which is negligible compared with the transform and SVD costs.
In each outer iteration, Algorithm \ref{alg:tnpk_admm} performs $\mathrm{jMax}$ inner $\mathcal{X}$-updates and one $\mathcal{H}$-update. Equivalently, each outer iteration requires $(\mathrm{jMax}+1)$ t-SVT-type computations. Thus, over $\mathrm{tMax}$ outer iterations, the total time complexity is $\mathcal{O}\big(\mathrm{tMax}(\mathrm{jMax}+1)\,C_{\mathrm{tSVT}}\big)$, where $C_{\mathrm{tSVT}}=\mathcal{O}(n_1 n_2 C_L(n_3)+n_1 n_2 n_3\min\{n_1,n_2\})$.
Accordingly, the DFT and DCT have the same asymptotic complexity due to their fast implementations, while the dense ROM becomes more computationally expensive when $n_3$ is large.

\begin{table*}[!htbp]
	\centering
	\footnotesize 
	\caption{Computational complexity of Algorithm \ref{alg:tnpk_admm} under different transforms}
	\renewcommand{\arraystretch}{0.1}
	\setlength\tabcolsep{3.8pt}
	\begin{tabular}{l c c c}
		\toprule
		Transform & $\mathcal{X}$ Update Complexity & $\mathcal{H}$ Update Complexity & Total Algorithm Complexity \\
		\midrule \midrule
		DFT &
		\makecell{$\mathcal{O}(n_1 n_2 n_3 \log n_3$ \\ $+\, n_1 n_2 n_3 \min\{n_1, n_2\})$} &
		\makecell{$\mathcal{O}(n_1 n_2 n_3 \log n_3$ \\ $+\, n_1 n_2 n_3 \min\{n_1, n_2\})$} &
		\makecell{$\mathcal{O}\big(\mathrm{tMax}(\mathrm{jMax}+1)$ \\ $\times (n_1 n_2 n_3 \log n_3 + n_1 n_2 n_3 \min\{n_1, n_2\})\big)$} \\
		\midrule
		DCT &
		\makecell{$\mathcal{O}(n_1 n_2 n_3 \log n_3$ \\ $+\, n_1 n_2 n_3 \min\{n_1, n_2\})$} &
		\makecell{$\mathcal{O}(n_1 n_2 n_3 \log n_3$ \\ $+\, n_1 n_2 n_3 \min\{n_1, n_2\})$} &
		\makecell{$\mathcal{O}\big(\mathrm{tMax}(\mathrm{jMax}+1)$ \\ $\times (n_1 n_2 n_3 \log n_3 + n_1 n_2 n_3 \min\{n_1, n_2\})\big)$} \\
		\midrule
		ROM &
		\makecell{$\mathcal{O}(n_1 n_2 n_3^2$ \\ $+\, n_1 n_2 n_3 \min\{n_1, n_2\})$} &
		\makecell{$\mathcal{O}(n_1 n_2 n_3^2$ \\ $+\, n_1 n_2 n_3 \min\{n_1, n_2\})$} &
		\makecell{$\mathcal{O}\big(\mathrm{tMax}(\mathrm{jMax}+1)$ \\ $\times (n_1 n_2 n_3^2 + n_1 n_2 n_3 \min\{n_1, n_2\})\big)$} \\
		\bottomrule
	\end{tabular}
	\label{Complexity}
\end{table*}
%\begin{table}[!htbp]
%	\centering
%	\scriptsize
%	\caption{Computational complexity of Algorithm \ref{alg:tnpk_admm} under different transforms}
%	\renewcommand{\arraystretch}{0.1}
%	\setlength\tabcolsep{0.5pt}
%	\begin{tabular}{l c c c}
%		\toprule
%		Type & DFT & DCT & ROM \\
%		\midrule
%		$\mathcal{X}$ Update & 
%		\makecell[l]{$\mathcal{O}(n_1 n_2 n_3^2$ \\ $+ n_1 n_2 n_3 m)$} &
%		\makecell[l]{$\mathcal{O}(n_1 n_2 n_3 \log n_3$ \\ $+ n_1 n_2 n_3 m)$} &
%		\makecell[l]{$\mathcal{O}(n_1 n_2 n_3 \log n_3$ \\ $+ n_1 n_2 n_3 r)$} \\
%		\midrule
%		$\mathcal{H}$ Update & 
%		$\mathcal{O}(n_1 n_2 n_3 m)$ &
%		$\mathcal{O}(n_1 n_2 n_3 m)$ &
%		$\mathcal{O}(n_1 n_2 n_3 r)$ \\
%		\midrule
%		Total & 
%		\makecell[l]{$\mathcal{O}(\mathrm{jMax} \cdot n_1 n_2 n_3^2$ \\ $+ (\mathrm{jMax}+1) n_1 n_2 n_3 m)$} &
%		\makecell[l]{$\mathcal{O}(\mathrm{jMax} \cdot n_1 n_2 n_3 \log n_3$ \\ $+ (\mathrm{jMax}+1) n_1 n_2 n_3 m)$} &
%		\makecell[l]{$\mathcal{O}(\mathrm{jMax} \cdot n_1 n_2 n_3 \log n_3$ \\ $+ (\mathrm{jMax}+1) n_1 n_2 n_3 r)$} \\
%		\bottomrule
%	\end{tabular}
%	\label{Complexity}
%\end{table}

\vspace{-0.64cm}
\subsection{\textbf{Convergence analysis}}
This section is devoted to the convergence analysis of the LRTC algorithm \ref{alg:tnpk_admm} for the TNK.
Our analysis proceeds in two stages. In the first stage, we show that the augmented Lagrangian function satisfies a sufficient descent property and that the sequence generated by \eqref{eq:44} is 
bounded; these two properties together yield subsequential convergence, namely, 
the sequence admits at least one subsequence converging to a stationary point.
In the second stage, we further strengthen this result to global convergence of 
the entire sequence by invoking the Kurdyka-\L ojasiewicz (KL) property. To this 
end, we first introduce the following two assumptions, under which we then 
establish this subsequential convergence.

\textbf{C1:} The sequence $\left \{ \mathcal{X}^{(t)}\right \}$ generated by \eqref{eq:44} is bounded, and hence the resulting sequence of nuclear norms is
also bounded, which is denoted as $\mathrm{sup}_k\left \{ ||\mathcal{X}^{(t)} ||_\ast  \right \}  \le M$.

\textbf{C2:} The sequence of Ky Fan $k$ norms admits a uniform positive lower bound; that is, there exists a constant $\delta>0$ such that $||\mathcal{H}^{(t)}||_{(k)} \ge \delta$, $\forall t$.

\begin{Lemma}
	Under assumptions C1-C2, the sequence $\left \{ \mathcal{X}^{(t)}, \mathcal{H}^{(t)}, \mathcal{C}^{(t)} \right \} $ generated by \eqref{eq:44} satisfies
	\[\begin{aligned}
		||\mathcal{C}^{(t+1)} - \mathcal{C}^{(t)}||_F^2 &\le \frac{2kn}{\delta ^4}||\mathcal{X}^{(t+1)} - \mathcal{X}^{(t)}||_F^2\\
		&+\frac{2M^2C_k^2}{\delta ^6}||\mathcal{H}^{(t+1)} - \mathcal{H}^{(t)}||_F^2,
	\end{aligned}\]
	where $n:= \min \left \{ n_1, n_2 \right \} $, $C_k = L_k \delta + 2k$, $L_k = \frac{2\sqrt{2k} }{\gamma }+1 $, and $M$ and $\delta$ are the constants defined in C1 and C2, respectively.%%CK LK取值确定为正确的
	\label{4lemma1}
\end{Lemma}

\begin{Lemma}[sufficient descent]
	Under assumptions C1-C2 and for a sufficiently large parameter $\mu_1> 0$, the sequence $\left \{ \mathcal{X}^{(t)}, \mathcal{H}^{(t)}, \mathcal{C}^{(t)} \right \} $ generated by \eqref{eq:44} satisfies
	\[\begin{aligned}
		T_{\mathrm{TC}}\bigl(\mathcal{X}^{(t+1)}, \mathcal{H}^{(t+1)}, \mathcal{C}^{(t+1)}\bigr)&\le T_{\mathrm{TC}}\bigl(\mathcal{X}^{(t)}, \mathcal{H}^{(t)}, \mathcal{C}^{(t)}\bigr)\\
		&-c_1||\mathcal{X}^{(t+1)} - \mathcal{X}^{(t)}||_F^2\\
		&-c_2||\mathcal{H}^{(t+1)} - \mathcal{H}^{(t)}|_F^2,
	\end{aligned}\]
	where $c_1, c_2 > 0$ are positive constants.
	\label{4lemma2}
\end{Lemma}

\begin{Lemma}[subgradient bound]
	Let $\left \{ \mathcal{X}^{(t)}, \mathcal{H}^{(t)}, \mathcal{C}^{(t)} \right \} $ denote the sequence generated by \eqref{eq:44}. Then, there exists a tensor
	$\mathcal{W}^{(t+1)} \in \partial T_{\mathrm{TC}}(\mathcal{X}^{(t+1)}, \mathcal{H}^{(t+1)}, \mathcal{C}^{(t+1)})$ and constants $\kappa_1, \kappa_2 > 0$ such that
	\begin{equation}
		||\mathcal{W}^{(t+1)}||_F^2 \le \kappa_1 ||\mathcal{X}^{(t+1)} - \mathcal{X}^{(t)}||_F^2+\kappa_2  ||\mathcal{H}^{(t+1)} - \mathcal{H}^{(t)}||_F^2.
	\end{equation}
	\label{4lemma3}
\end{Lemma}

\vspace{-0.35cm}
\begin{Theorem}[subsequential convergence]
	Under assumptions C1-C2 and for a sufficiently large parameter $\mu_1>0$, the sequence $\left \{ \mathcal{X}^{(t)}, \mathcal{H}^{(t)}, \mathcal{C}^{(t)} \right \} $ generated by \eqref{eq:44} satisfies the following properties:
	
	(I) Boundedness: The sequences $\left \{ \mathcal{H}^{(t)}\right \}$ and $\left \{\mathcal{C}^{(t)} \right \}$ are bounded.
	
	(II) Vanishing iterates: 
	$||\mathcal{X}^{(t+1)} - \mathcal{X}^{(t)}||_F \to 0, 
	||\mathcal{H}^{(t+1)} - \mathcal{H}^{(t)}||_F \to 0, 
	||\mathcal{C}^{(t+1)} - \mathcal{C}^{(t)}||_F \to 0 $ as $t \to \infty$.
	
	(III) Subsequential convergence to a critical point: There exists a convergent subsequence
	$\left \{ \mathcal{X}^{(t)}, \mathcal{H}^{(t)}, \mathcal{C}^{(t)} \right \}  \to \left \{\mathcal{X}^*, \mathcal{H}^*, \mathcal{C}^*\right \}$ such that $\mathcal{O}  \in \partial T_{\mathrm{TC}}(\mathcal{X}^*, \mathcal{H}^*, \mathcal{C}^*)$.
	\label{th1}
\end{Theorem}

\begin{Remark}
	Since the proposed TNK regularizer is noncoercive, the boundedness assumption C1 is required for the convergence analysis. In fact, if the observation set $\Omega$ does not cover the entire domain, the optimal solution may be unbounded, making this assumption necessary. Moreover, the uniform positive lower bound $\delta>0$ in assumption C2 is crucial in Lemmas \ref{4lemma1}, \ref{4lemma2}, and subsequent related lemmas, as it ensures that denominators do not vanish and that the constants in the subgradient bound remain finite. Although both assumptions are relatively strong from a theoretical perspective, they can be readily verified in practice through numerical experiments.
\end{Remark}

Building on the subsequential convergence properties established in Theorem  \ref{th1}, we verify that the augmented Lagrangian function $T_{\mathrm{TC}}$ satisfies the KL property, which allows us to establish global convergence.

\begin{Definition}[KL property \cite{attouch2010proximal}]
	Let $h:\mathbb{R}^n \to (-\infty,+\infty]$ be a proper closed function. The function $h$ is said to satisfy the KL property at a point $\hat{\mathbf{x} }\in \mathrm{dom}\partial h$ if there exist a constant $\nu \in (0,\infty]$, a neighborhood $U$ of $\hat{\mathbf{x} }$, and a continuous concave function $\phi:[0,\nu)\to [0,\infty)$ with $\phi(0)=0$ such that:
	
	(a) $\phi$ is continuously differentiable on $(0,\nu)$ with $\phi'>0$;
	
	(b) for every $\mathbf{x}  \in U$ satisfying
	\[
	h(\hat{\mathbf{x} }) < h(\mathbf{x} ) < h(\hat{\mathbf{x} }) + \nu,
	\]
	it holds that
	\[
	\phi'\bigl(h(\mathbf{x} )-h(\hat{\mathbf{x} })\bigr)\mathrm{dist}\bigl(0,\partial h(\mathbf{x} )\bigr) \ge 1,
	\]
	where $\mathrm{dist}(\mathbf{x} ,C)$ denotes the Euclidean distance from a point $\mathbf{x} $ to a closed set $C$, with the convention that $\mathrm{dist}(\mathbf{x} ,\emptyset) := +\infty$.
\end{Definition}

To establish global convergence, we consider a modified augmented Lagrangian function $T_\delta $.
\begin{equation}
	\begin{aligned}
		T_\delta(\mathcal{X},\mathcal{H},\mathcal{C})
		&= \frac{||\mathcal{X}||_\ast }{||\mathcal{H}||_{(k)}}  + I_\Phi(\mathcal{X})+\frac{\mu_1}{2} ||\mathcal{X}-\mathcal{H}||_F^2\\
		&+\langle \mathcal{C}, \mathcal{X}-\mathcal{H} \rangle
		+ I_{||\mathcal{H}||_{(k)} \ge \delta}(\mathcal{H}).
	\end{aligned}
	\label{eq:444}
\end{equation}
Under assumption C2, we have $T_\delta = T_{\mathrm{TC}}$. We next show that $T_\delta$ satisfies the KL property, which implies that $T_{\mathrm{TC}}$ also satisfies the KL property under the same assumption.

\begin{Lemma}
	The modified augmented Lagrangian function \eqref{eq:444} satisfies the KL property.
	\label{4lemma4}
\end{Lemma}

\begin{Theorem}[global convergence]
	Under assumptions C1-C2 and for a sufficiently large parameter $\mu_1 > 0$, the sequence $\left \{ \mathcal{X}^{(t)}, \mathcal{H}^{(t)}, \mathcal{C}^{(t)} \right \} $ generated by \eqref{eq:44} converges to a stationary point of \eqref{eq:4}.
	\label{th2}
\end{Theorem}

The proofs of Lemmas \ref{4lemma1}, \ref{4lemma2}, \ref{4lemma3}, and \ref{4lemma4}, as well as Theorems \ref{th1} and \ref{th2}, are provided in Appendix \ref{appB}.

\section{\textbf{Experiments}}
\label{sec6}
In this section, extensive experiments on both synthetic and real-world data are conducted to evaluate the performance of the proposed TNPK regularization method. The results show that the proposed method outperforms existing state-of-the-art methods for the LRTC problem. All the experiments are implemented using MATLAB (R2021a) on the Windows 10 platform with Intel Core i5-1035G1 1.00 GHz and 8 GB of RAM.

\vspace{-0.2cm}
\subsection{Synthetic data}
We generate a ground-truth low rank tensor by the t-product $\mathcal{X}_{\text{GT}}=\mathcal{P}*_L \mathcal{Q}^T $, where $\mathcal{P}\in \mathbb{R}^{n\times r\times n_3} $ and $\mathcal{Q}\in \mathbb{R}^{r\times n\times n_3} $ with $r\ll n$. The entries of tensors $\mathcal{P}$ and $\mathcal{Q}$ are independently drawn from a standard Gaussian distribution, while the locations of the observed index set $\Omega $ are sampled uniformly at random. 
%We compare the proposed TNK regularization method with several state-of-the-art approaches, including TNN \cite{lu2018exact}, PSTNN \cite{jiang2020multi}, W-t-TNN \cite{mu2020weighted}, IR-t-TNN \cite{wang2021generalized} and TNF \cite{zheng2024scale}. In Algorithm \ref{alg:tnf_admm}, the parameter $\epsilon$ is fixed at $10^{-10}$, and the parameters $\mu _1$ and $\mu_2$ are gradually increased following a continuation strategy \cite{lu2019tensor}. The other methods are implemented using the MATLAB codes released by the original authors with default parameter settings.%%要改
For the invertible linear transform $L$ used in the t-product, we consider three linear transformations: (1) Discrete Fourier
Transform (DFT); (2) Discrete Cosine Transform (DCT) ; (3) Random Orthogonal Matrix (ROM). Among them, for the DFT case, $l=n_3$; for the DCT and ROM cases, $l=1$.

\subsubsection{\textbf{Effect of the Number of Inner Iterations (jMax)}}
We start by discussing the effect of the maximum number of inner iterations (jMax) on the performance of LRTC. For this purpose, we consider a third-order tensor of size $100 \times 100 \times 30$, with the tubal rank set to 20 and a sampling rate of 0.5. Two representative TNPK models are selected for comparison: \textbf{TNK $(p=1, k=18)$} and \textbf{TNF $(p=2, k=20)$}.
Meanwhile, the maximum number of outer iterations (kMax) is fixed at 800.
Next, we evaluate the recovery performance and convergence behavior of the three invertible linear transforms under different values of jMax.
\begin{figure}[!htbp]
	\renewcommand{\arraystretch}{0.4}
	\setlength\tabcolsep{1pt}
	\centering
	\begin{tabular}{ccc }
		\centering
		\includegraphics[width=1.155in]{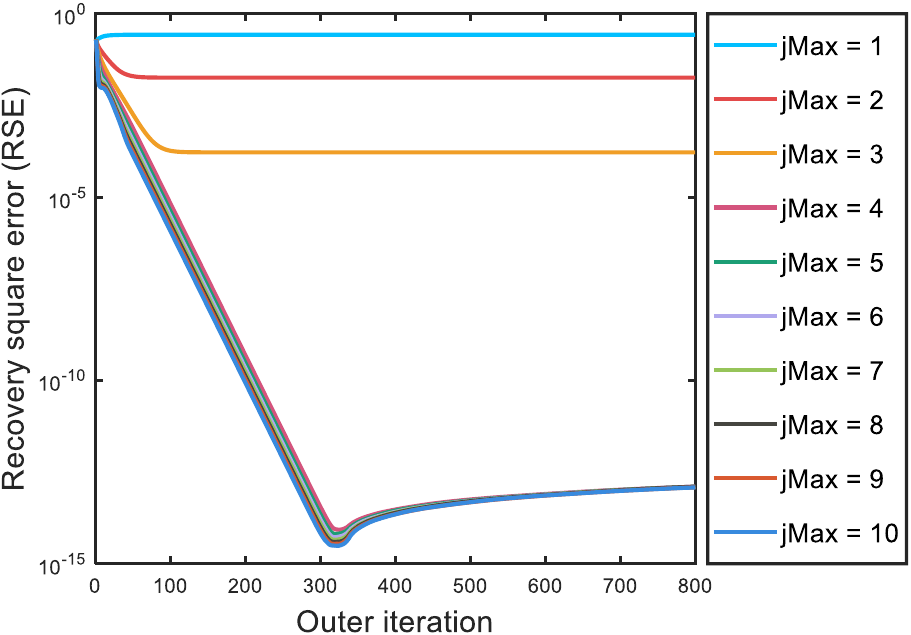}&
		\includegraphics[width=1.155in]{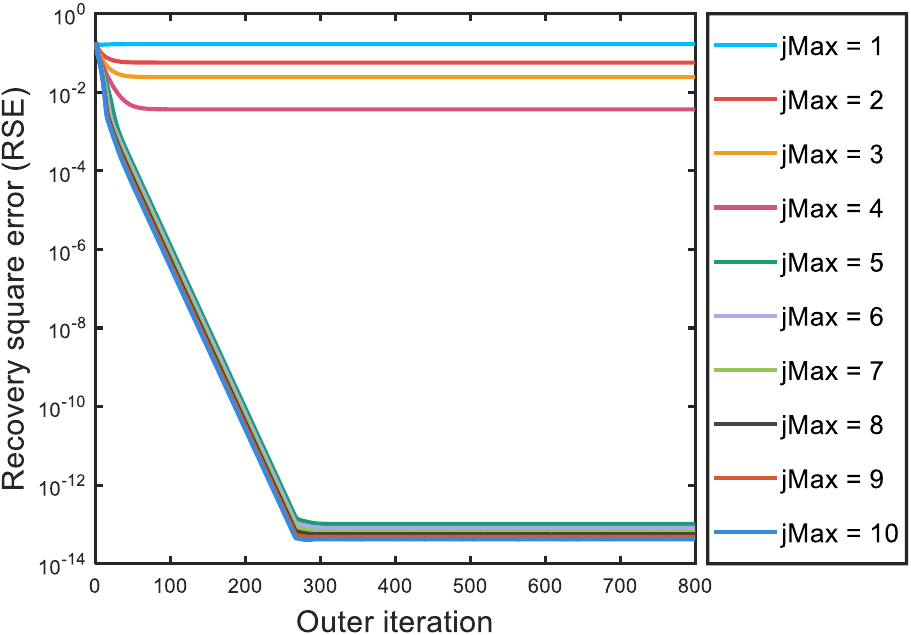}&
		\includegraphics[width=1.155in]{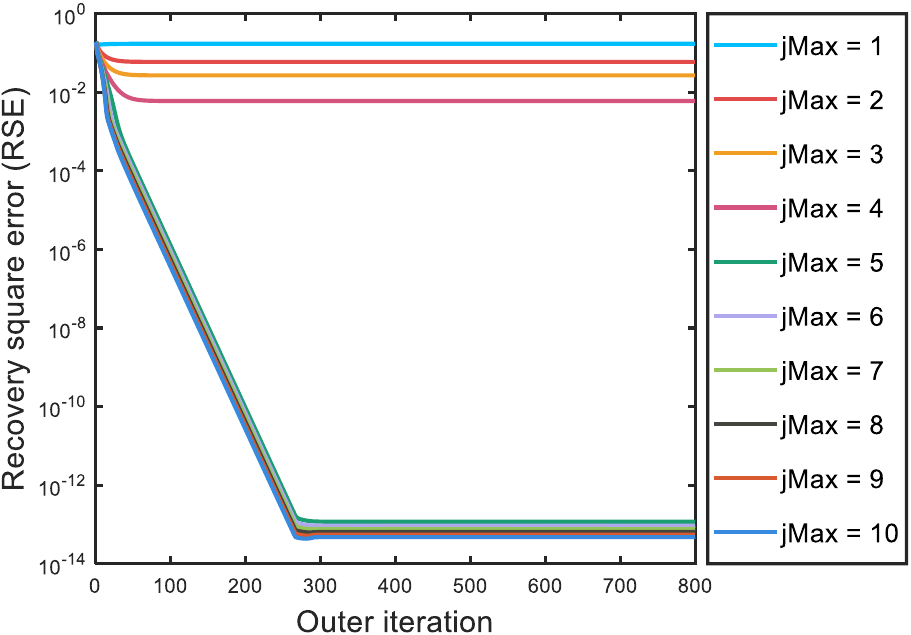}
		\\[+1mm]
		\footnotesize{(a)}  &
		\footnotesize{(b)}  & \footnotesize{(c)}
		\\[+1mm]
		\includegraphics[width=1.155in]{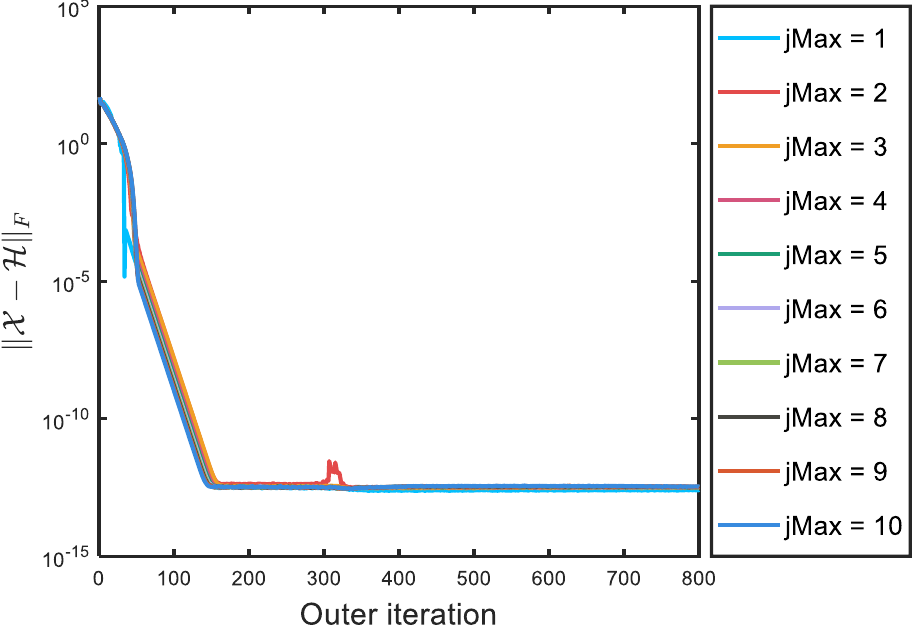}&
		\includegraphics[width=1.155in]{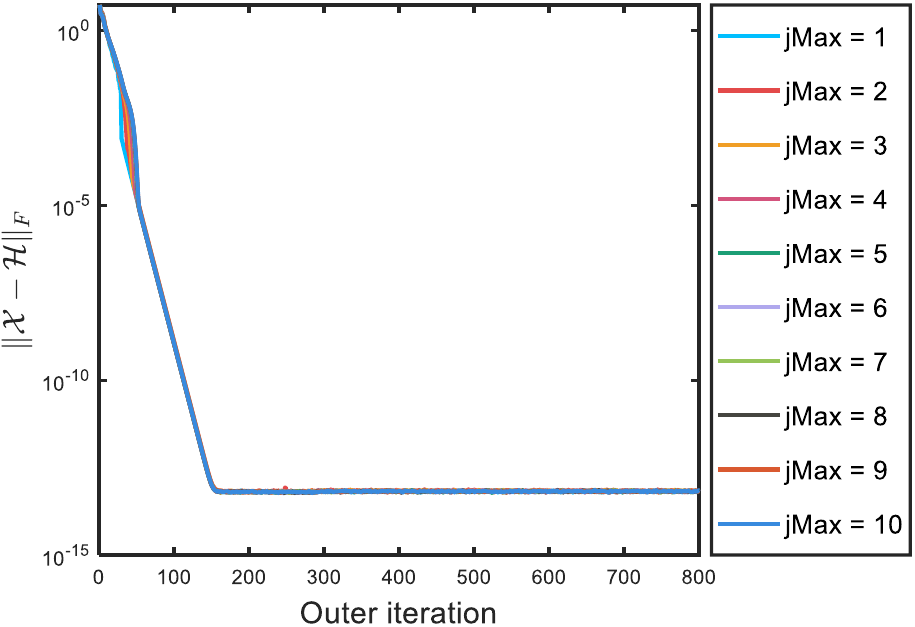}&
		\includegraphics[width=1.155in]{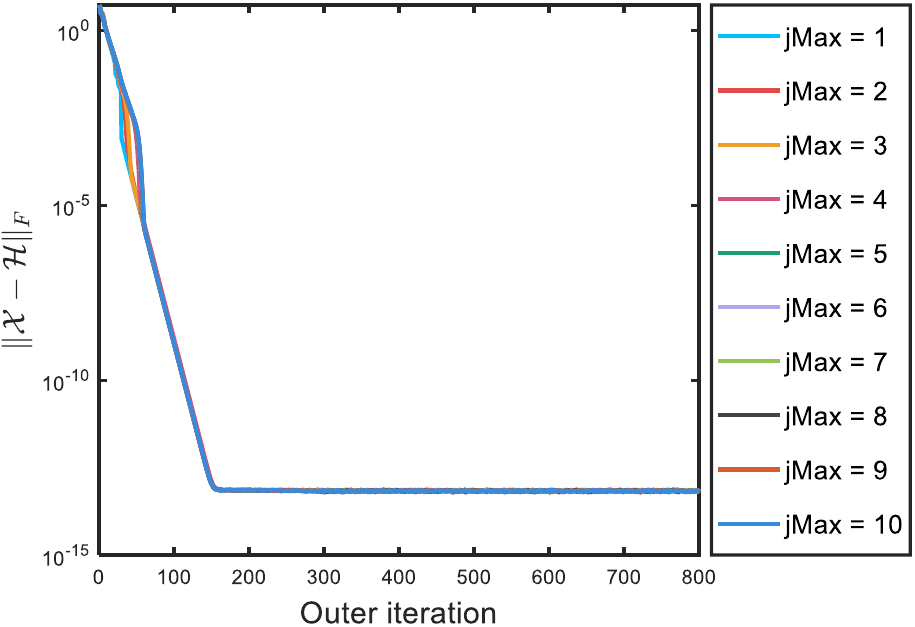}
		\\[+1mm]
		\footnotesize{(d) }  &
		\footnotesize{(f)}  & \footnotesize{(e)}
		\\[+1mm]
		\includegraphics[width=1.155in]{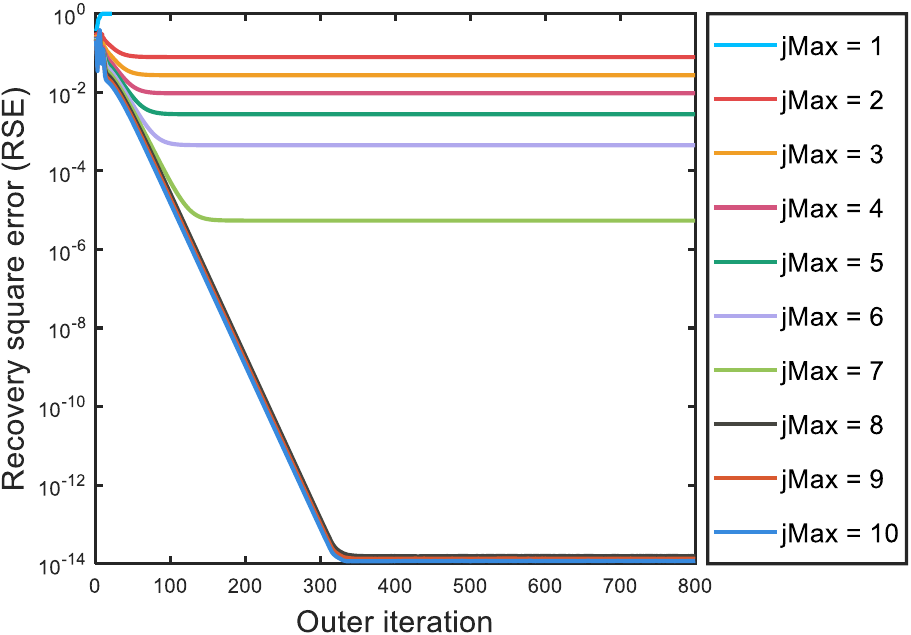}&
		\includegraphics[width=1.155in]{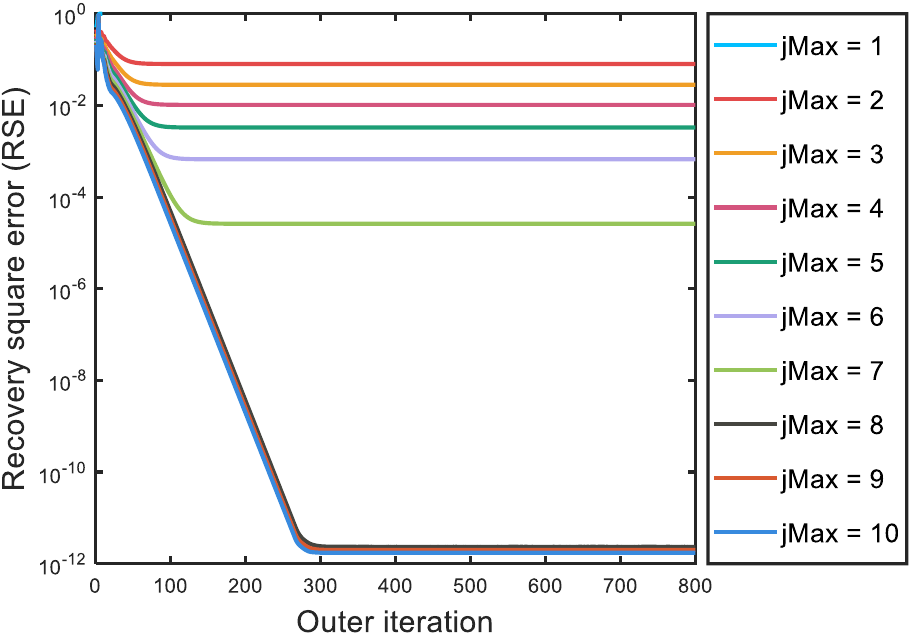}&
		\includegraphics[width=1.155in]{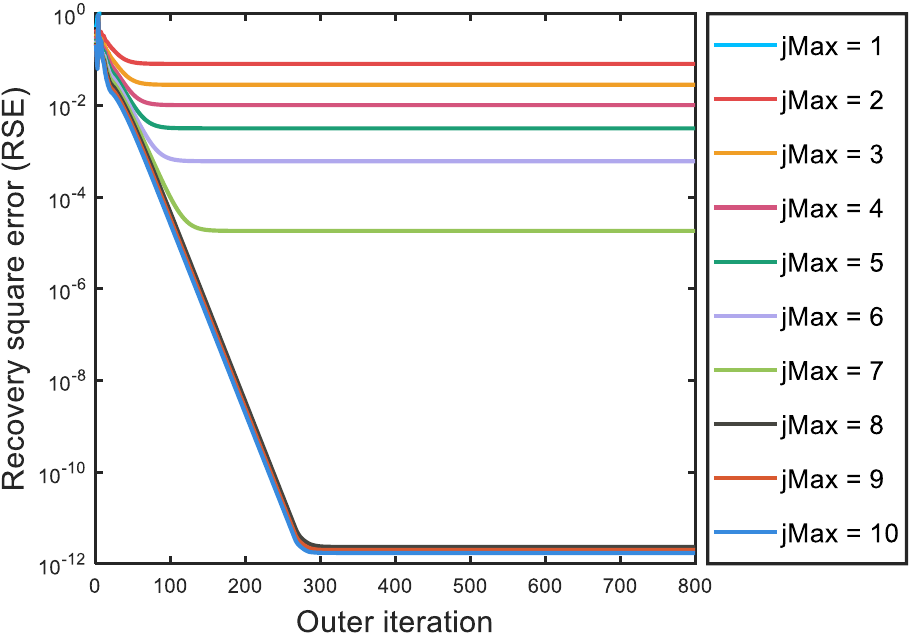}
		\\[+1mm]
		\footnotesize{(g)}  &
		\footnotesize{(h)}  & \footnotesize{(i)} 
		\\[+1mm]
		\includegraphics[width=1.155in]{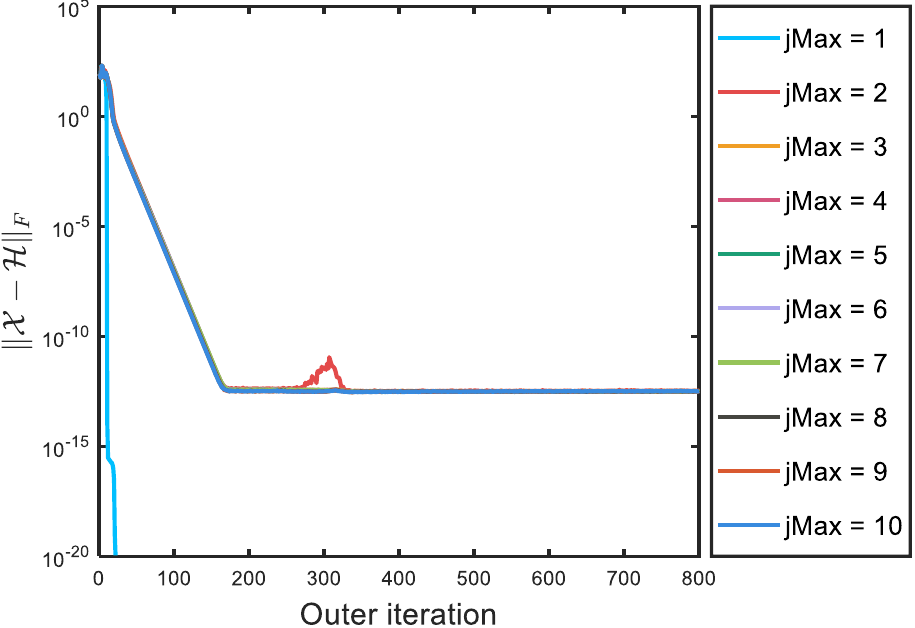}&
		\includegraphics[width=1.155in]{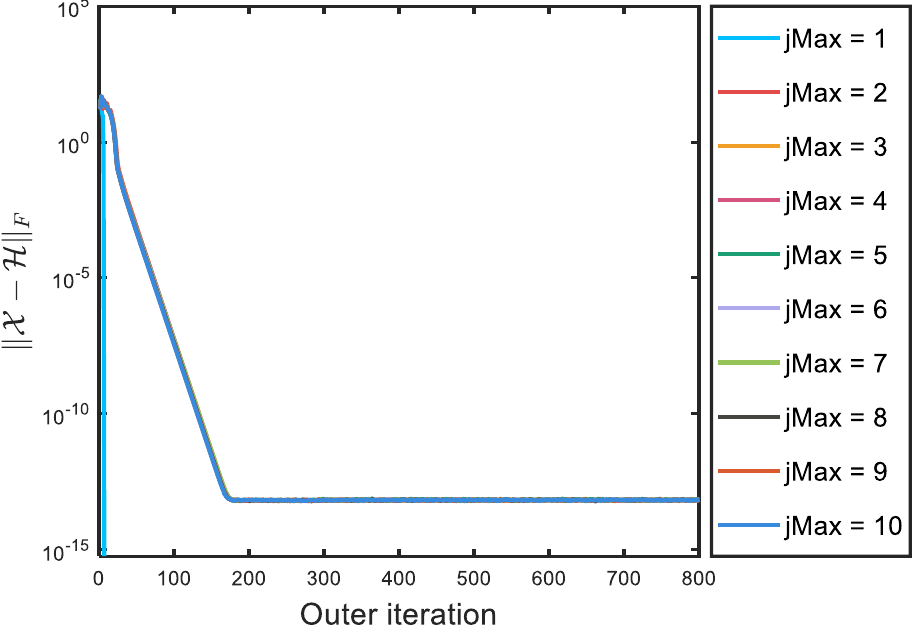}&
		\includegraphics[width=1.155in]{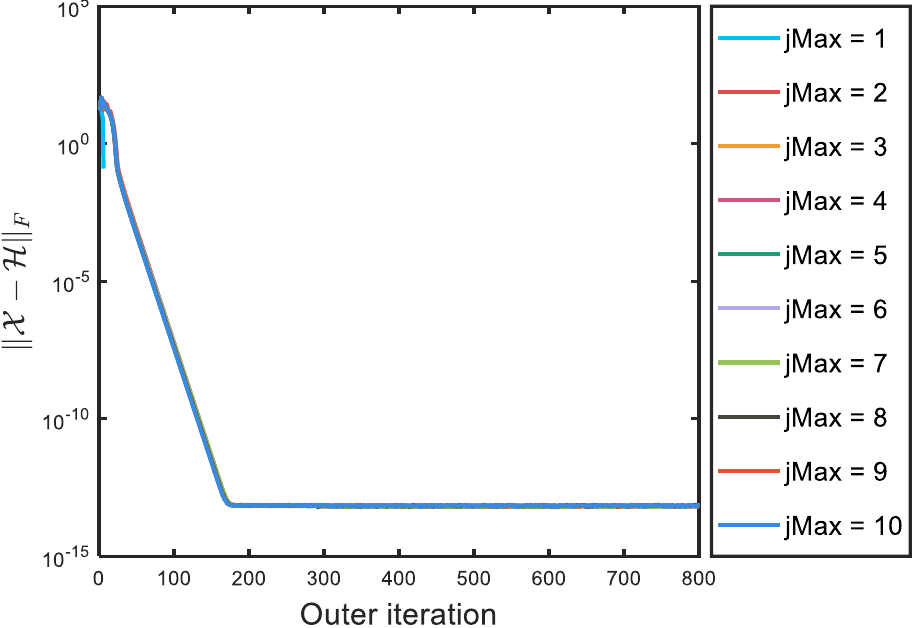}
		\\[+1mm]
		\footnotesize{(j) }  &
		\footnotesize{(k)}  & \footnotesize{(l)}
	\end{tabular}
	%\vspace{-0.15cm}
	\caption{Comparison of results for different maximum numbers of inner loops (jMax) on synthetic data: From left to right, the results correspond to the three transforms (DFT, DCT, and ROM). Panels (a)-(c) and (d)-(f) show, for the TNK model, the relative squared error (RSE) curves measuring the difference between $\mathcal{X}$ and the ground truth tensor $\mathcal{X}_{\text{GT}}$, and the curves of the difference between $\mathcal{X}$ and the auxiliary tensor $\mathcal{H}$ across iterations, respectively; panels (g)-(i) and (j)-(l) correspondingly show the RSE curves and iteration-wise difference curves for the TNF model.
	} 
	\vspace{-0.2cm}
	\label{fig:1} 
\end{figure}
For the DFT, the parameters are set to $\mu_1 = 10^{-4}$ and $\mu_2 = 10^{-3}$, while for the DCT and ROM, the parameters are set to $\mu_1 = 10^{-2}$ and $\mu_2 = 10^{-1}$.
The parameter jMax is taken as jMax $=[ 1,2,\ldots,10]$.
The curves of the relative square error (RSE) of the recovered tensor $\hat{\mathcal{X}}$ with respect to the ground truth tensor $\mathcal{X}_\text{GT}$ under different transforms, defined as
\[\mathrm{RSE} =\frac{||\hat{\mathcal{X}} - \mathcal{X}_\mathrm{GT} ||^2_F}{||\mathcal{X}_\mathrm{GT}||^2_F},\]
as well as the curves of the difference between the recovered tensor and its auxiliary variable $\mathcal{H}$ across iterations (including both inner and outer loops), are shown in Figure \ref{fig:1}. 
It can be observed that, in the TNK model, when $ \text{jMax} \le 4 $, the relative squared error decreases slowly, and the convergence performance is noticeably worse than under other parameter settings; whereas when $ \text{jMax} \ge 5 $, all curves essentially coincide after approximately 300 iterations. For the TNF model, when $ \text{jMax} \le 7 $, the relative squared error also decreases slowly, and the convergence performance is clearly poorer; whereas when $ \text{jMax} \ge 8 $, all curves similarly converge to the same trend after about 300 iterations. Considering that increasing the number of inner iterations significantly raises the computational cost, to achieve a balance between computational efficiency and convergence performance, the subsequent experiments will adopt $ \text{jMax} = 5 $ for the TNK model and $ \text{jMax} = 8 $ for the TNF model.

%\begin{figure*}[!htbp] % h=here, t=top, b=bottom, p=page
%	\centering
%	\includegraphics[width=0.8\textwidth]{fig/jMax} % 图片路径，不用加 .png/.jpg
%	\caption{Comparison of results for different maximum numbers of inner loops (jMax) on synthetic data: From left to right, the results correspond to the three transforms (DFT, DCT, and ROM). Panels (a)-(c) and (d)-(f) show, for the TNK model, the relative squared error (RSE) curves measuring the difference between $\mathcal{X}$ and the ground truth tensor $\mathcal{X}_{\text{GT}}$, and the curves of the difference between $\mathcal{X}$ and the auxiliary tensor $\mathcal{H}$ across iterations, respectively; panels (g)-(i) and (j)-(l) correspondingly show the RSE curves and iteration-wise difference curves for the TNF model.
%	} 
%	\label{fig:1} 
%\end{figure*}
\begin{figure}[!htbp]
	\renewcommand{\arraystretch}{0.5}
	\setlength\tabcolsep{1pt}
	\centering
	\begin{tabular}{ccc }
		\centering
		\includegraphics[width=1.125in]{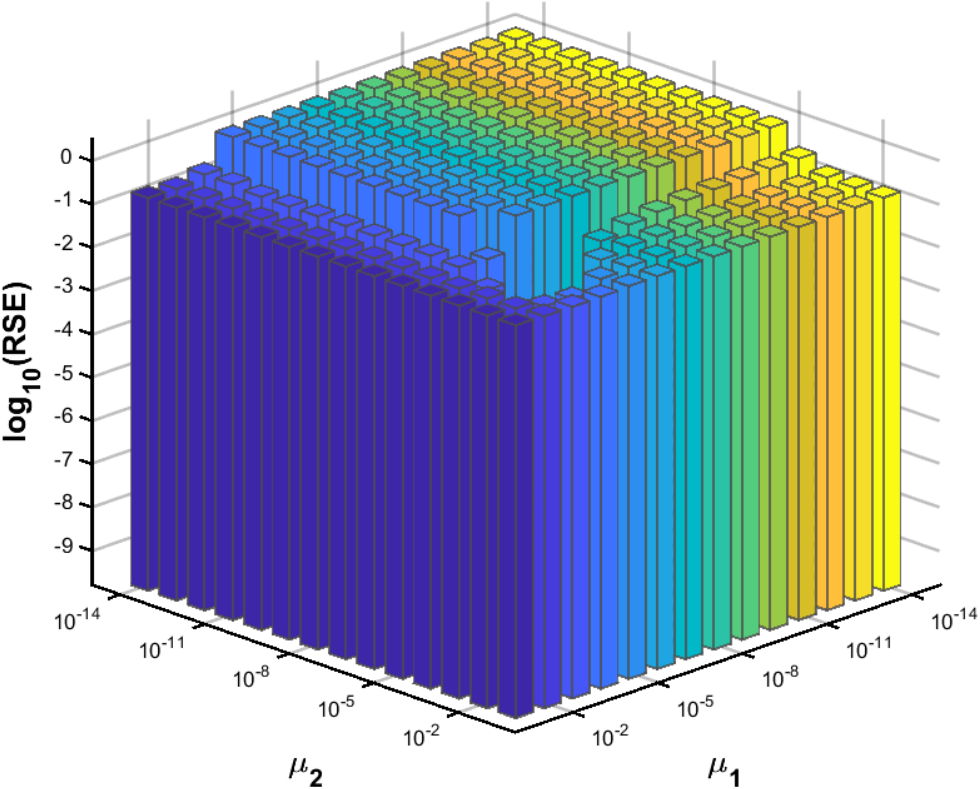}&
		\includegraphics[width=1.125in]{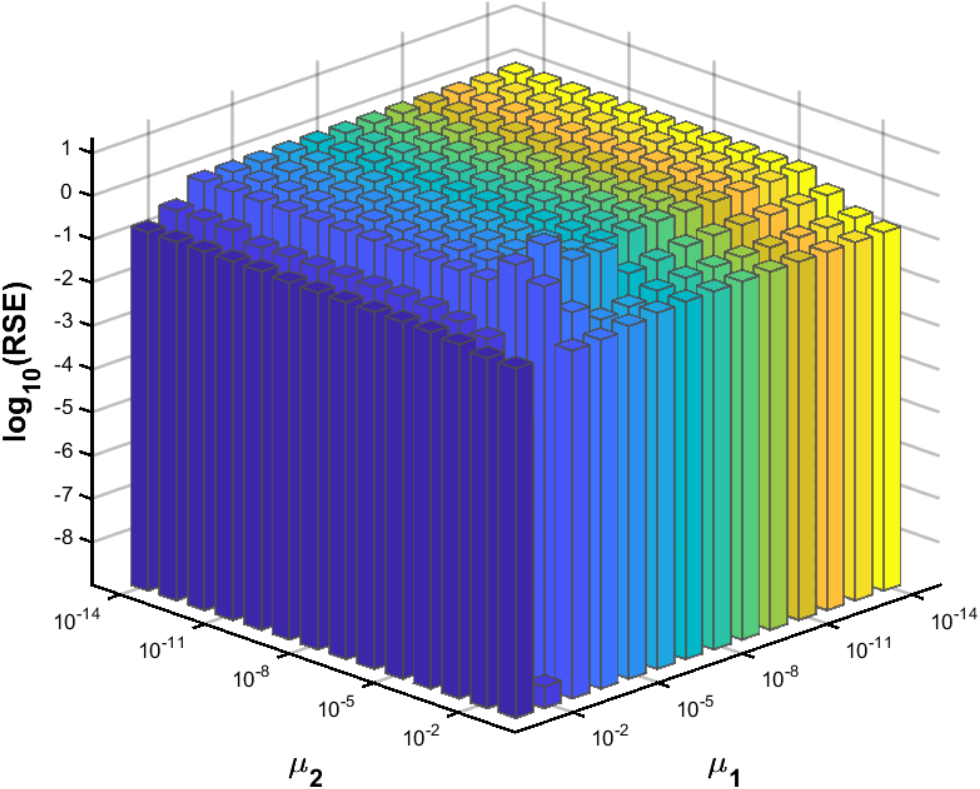}&
		\includegraphics[width=1.125in]{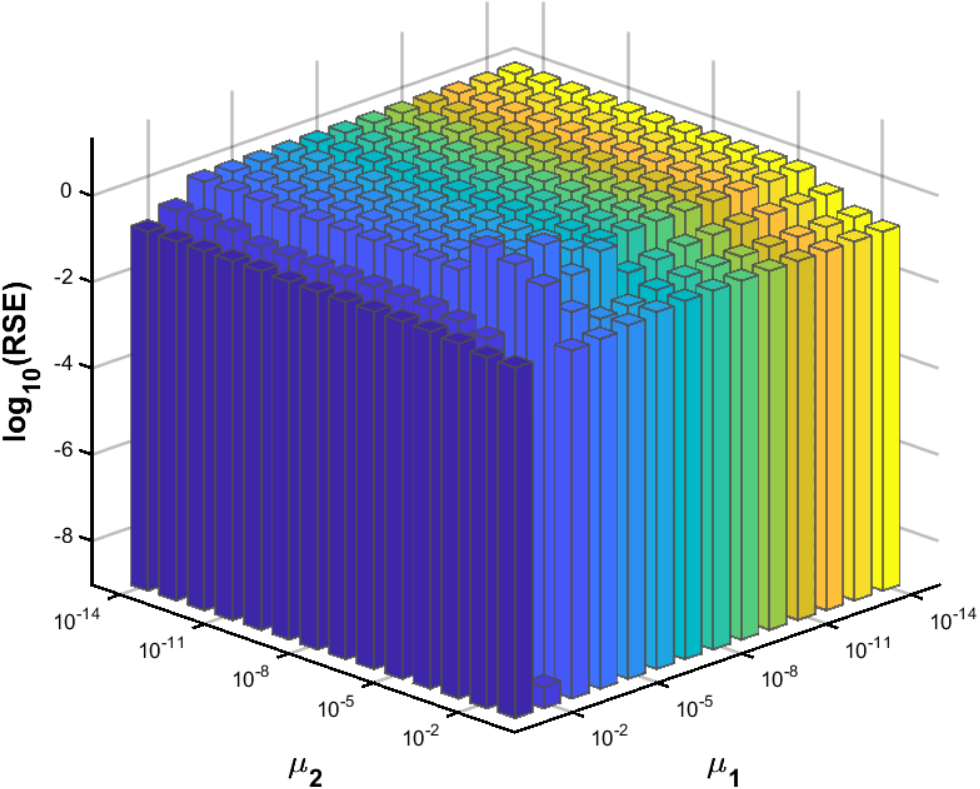}
		\\[+1mm]
		\footnotesize{(a)}  &
		\footnotesize{(b)}  & \footnotesize{(c)}
		\\[+1mm]
		\includegraphics[width=1.125in]{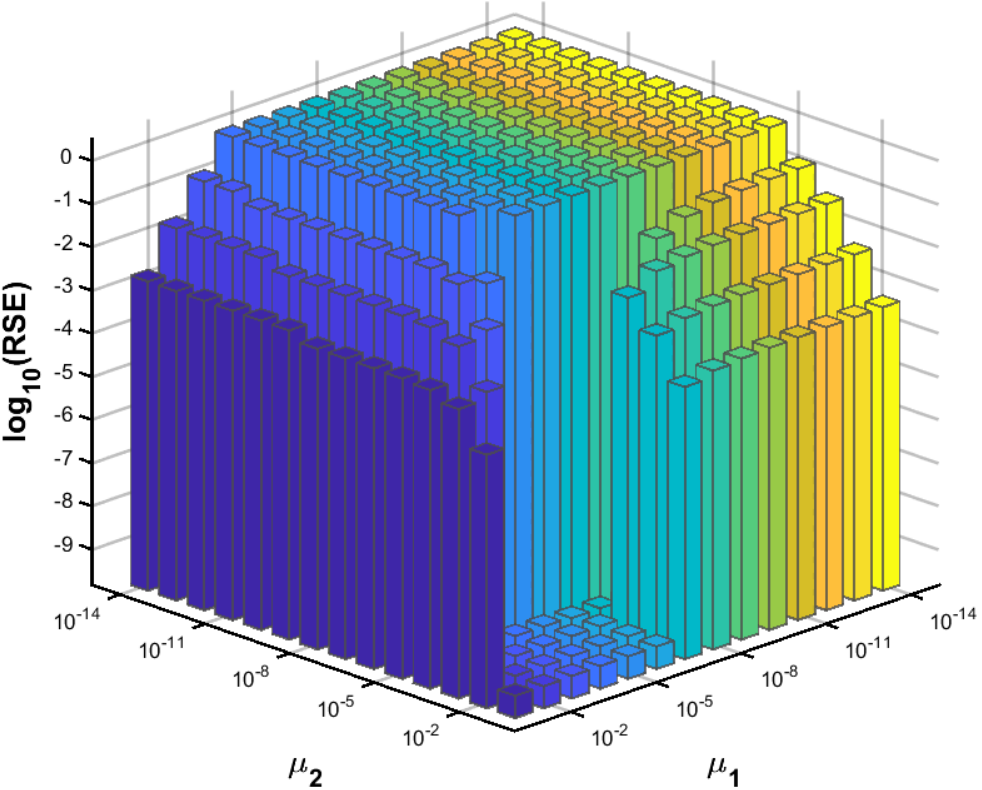}&
		\includegraphics[width=1.125in]{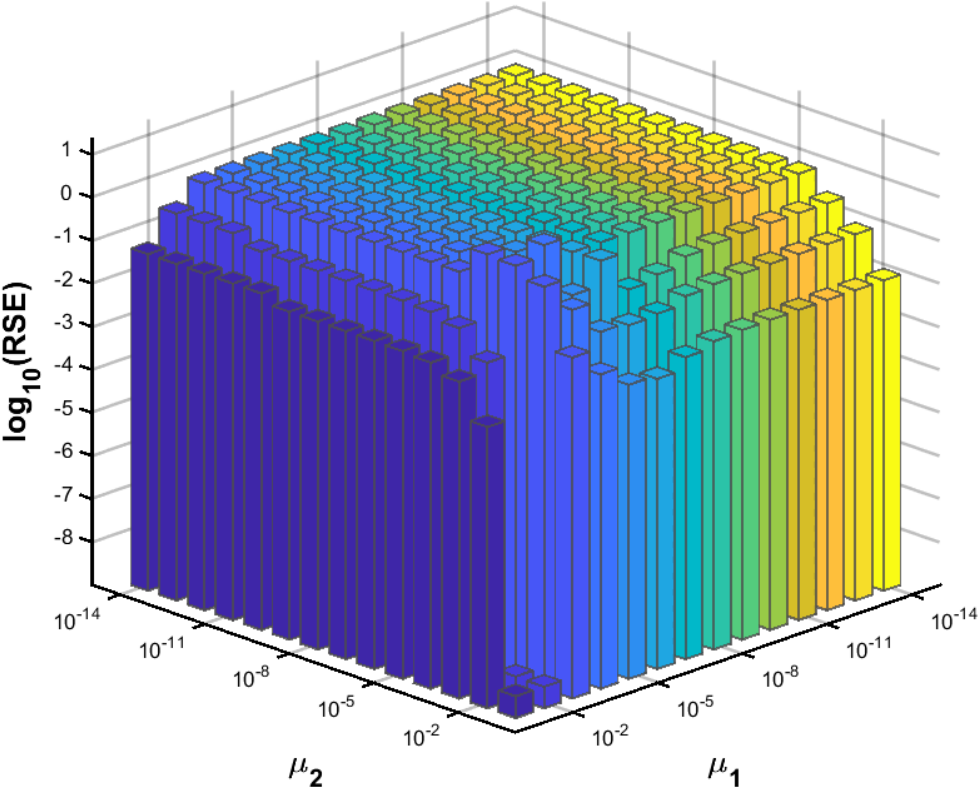}&
		\includegraphics[width=1.125in]{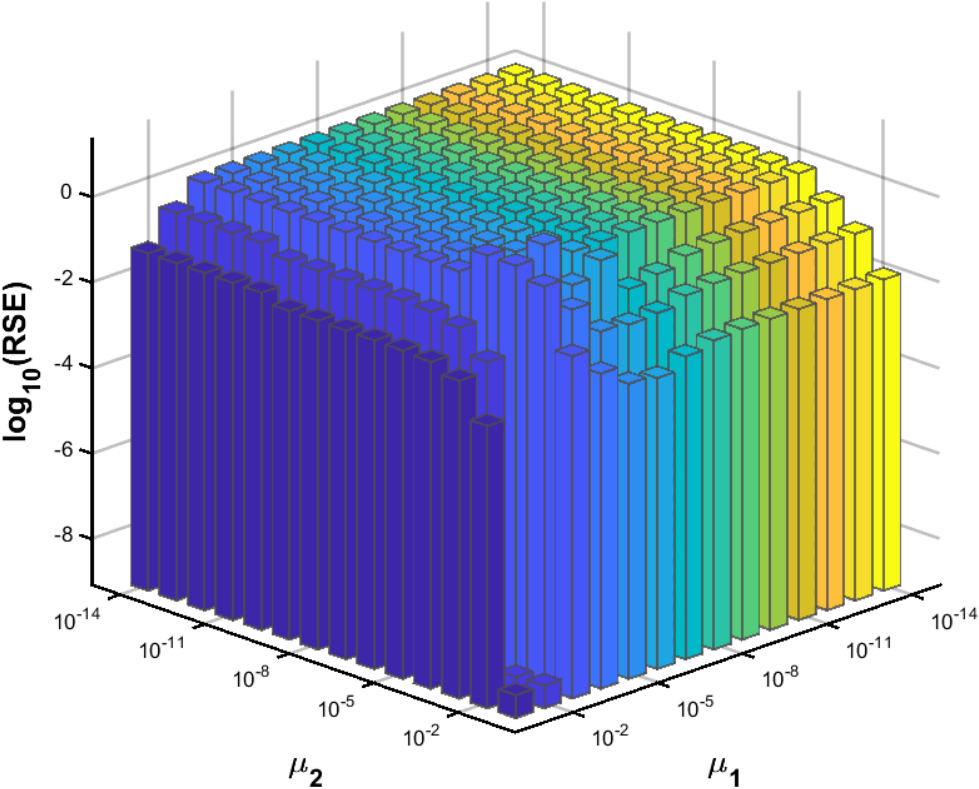}
		\\[+1mm]
		\footnotesize{(d) }  &
		\footnotesize{(f)}  & \footnotesize{(e)}
		\\[+1mm]
		\includegraphics[width=1.125in]{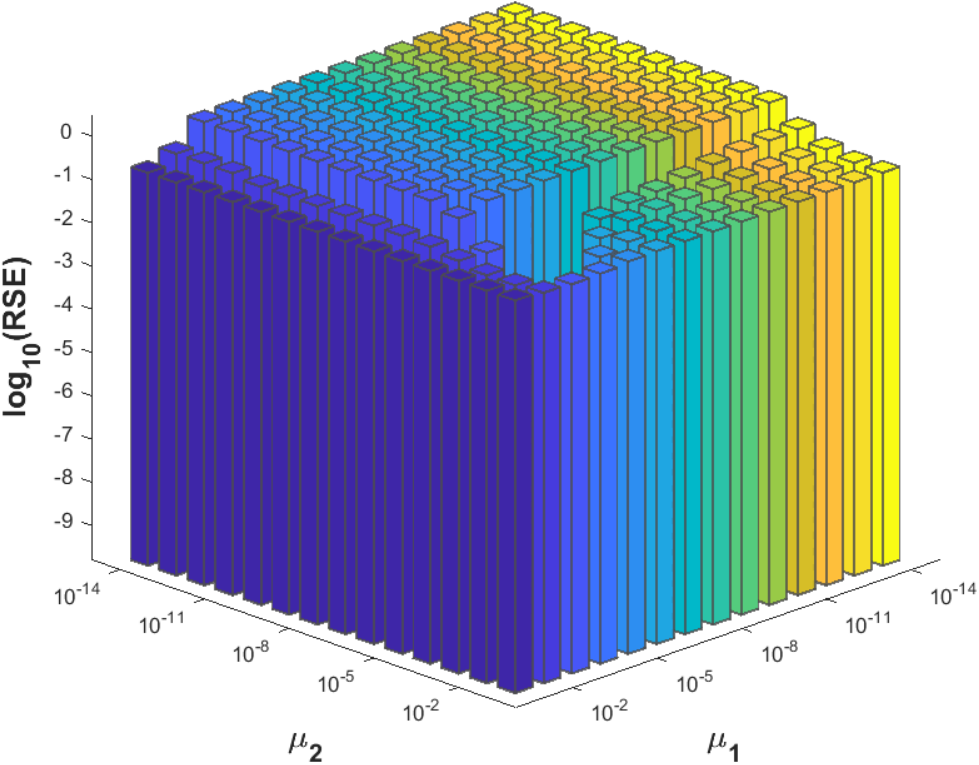}&
		\includegraphics[width=1.125in]{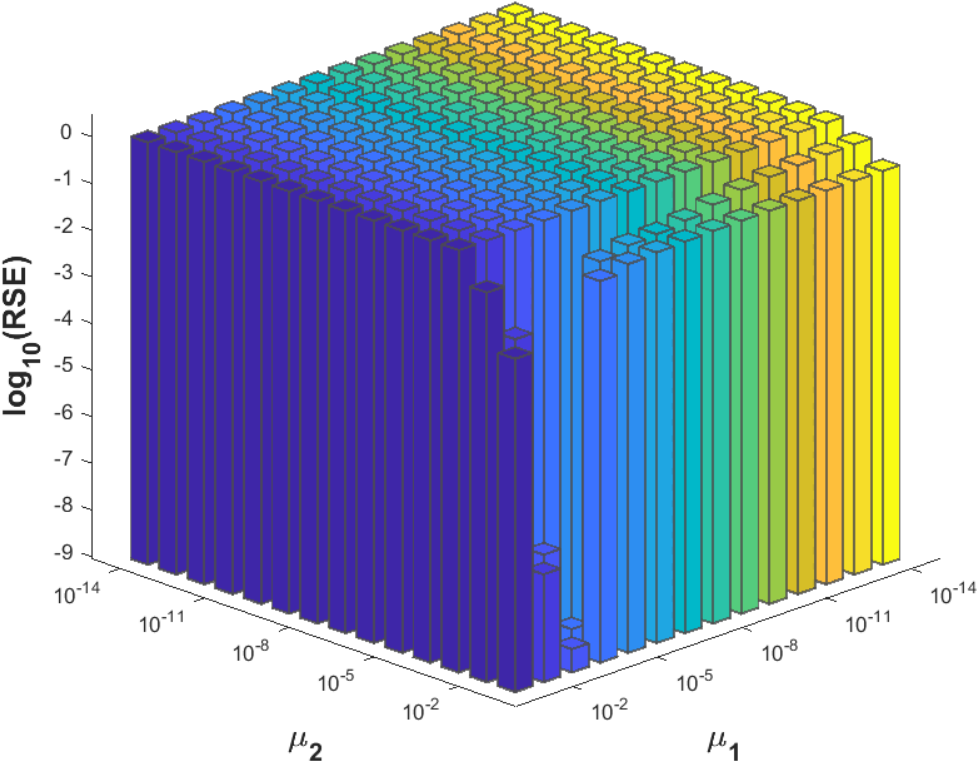}&
		\includegraphics[width=1.125in]{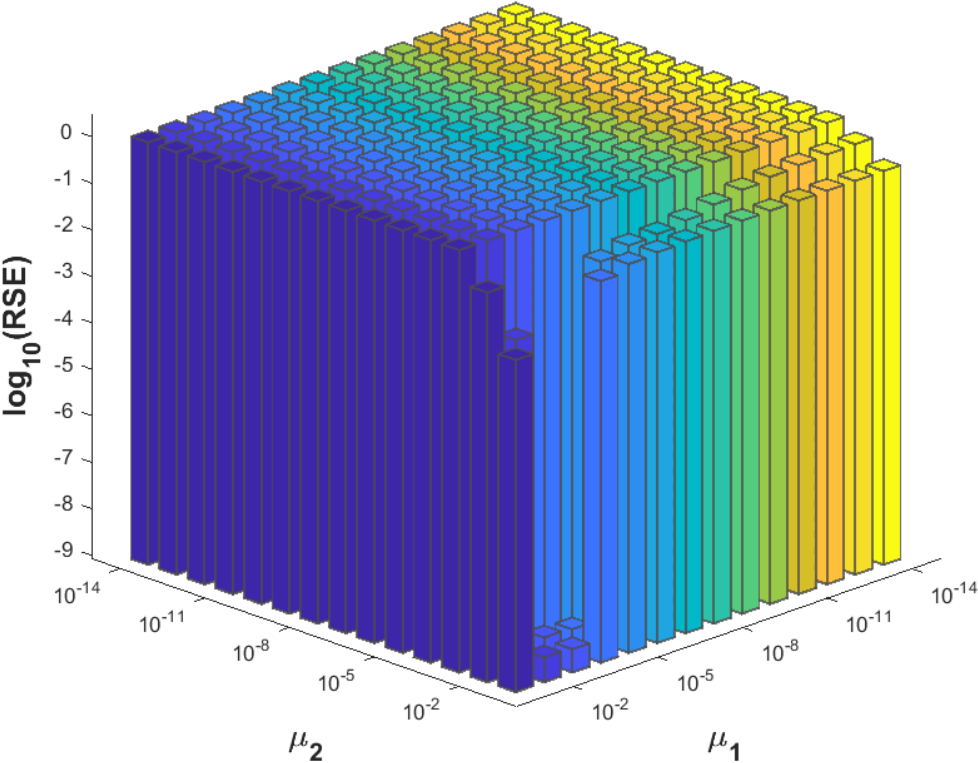}
		\\[+1mm]
		\footnotesize{(g)}  &
		\footnotesize{(h)}  & \footnotesize{(i)} 
		\\[+1mm]
		\includegraphics[width=1.125in]{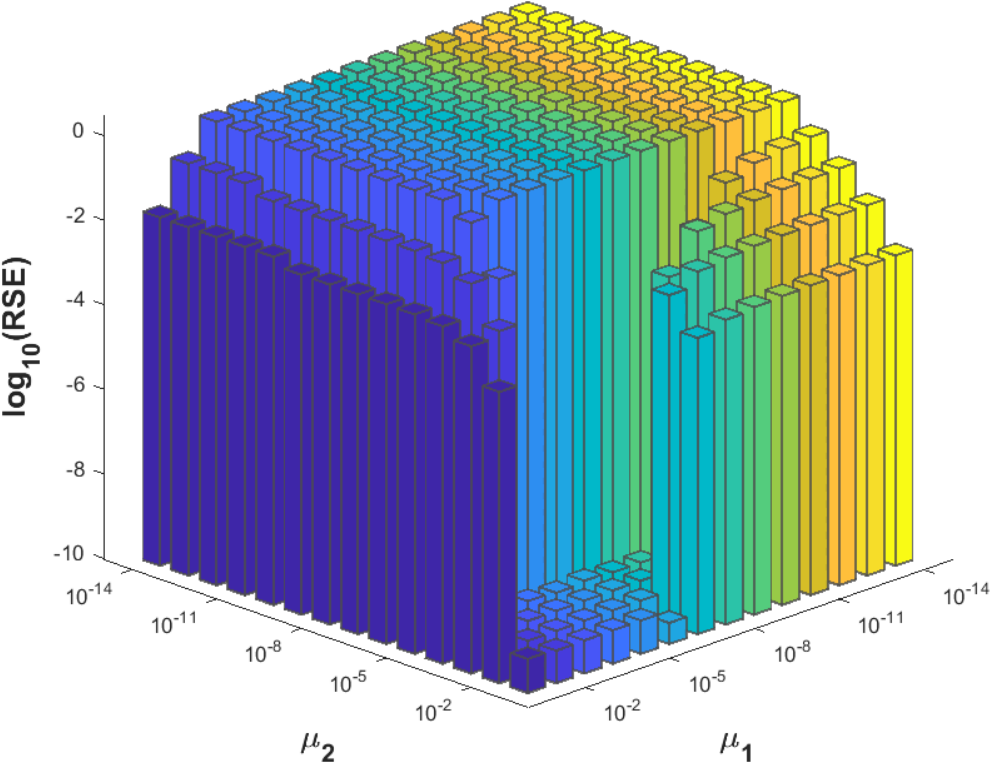}&
		\includegraphics[width=1.125in]{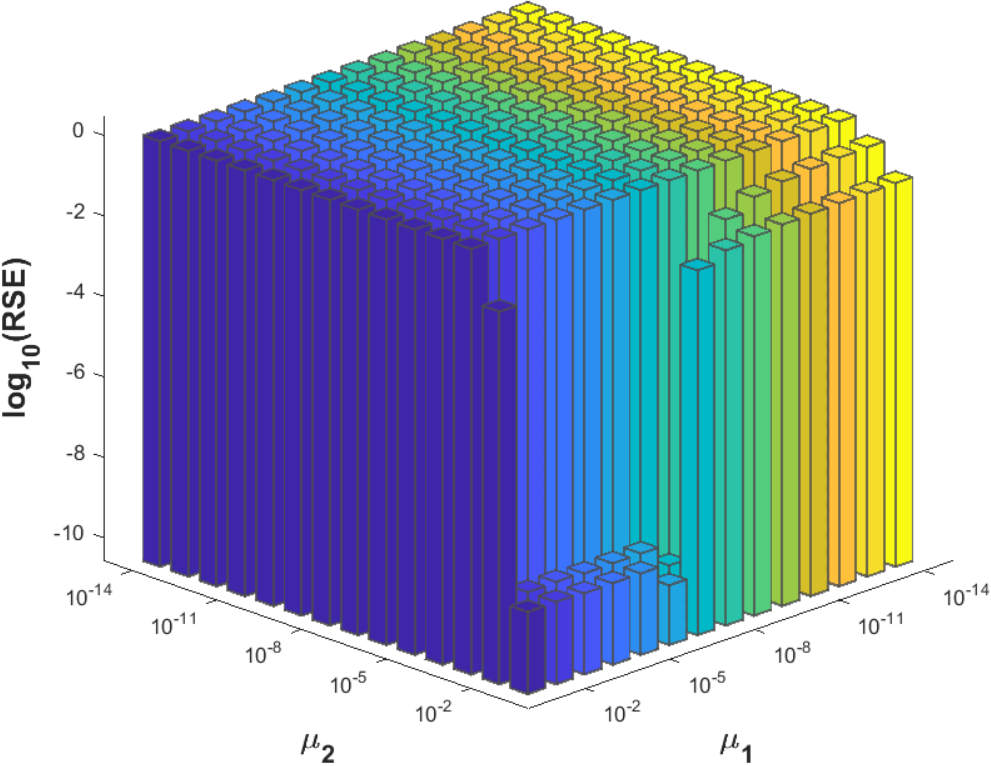}&
		\includegraphics[width=1.125in]{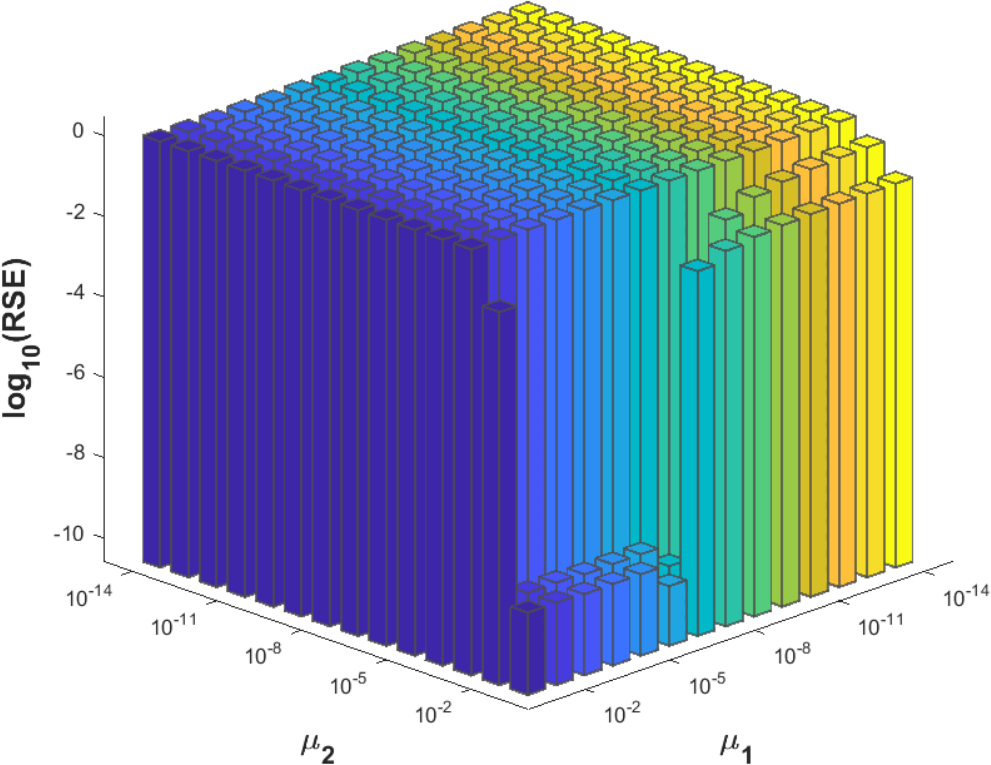}
		\\[+1mm]
		\footnotesize{(j) }  &
		\footnotesize{(k)}  & \footnotesize{(l)}
	\end{tabular}
	%\vspace{-0.15cm}
	\caption{Parameter Sensitivity Analysis of TNK (p=1, k=18) and TNF (p=2, k=20) under Different Transforms and Sampling Rates: (a)-(c) TNK, SR=0.5; (d)-(f) TNK, SR=0.6; (g)-(i) TNF, SR=0.5; (j)-(l) TNF, SR=0.6. Each subfigure shows RSE curves for DFT, DCT, and ROM transforms.
	} 
	\vspace{-0.2cm}
	\label{fig:2} 
\end{figure}
\subsubsection{\textbf{Parameter Sensitivity Analysis $\mathbf{(\mu_1 ,\mu_2)} $}}
We investigated the sensitivity of the parameters $\mu_1$ and $\mu_2$ in Algorithm \ref{alg:tnpk_admm} as well as the selection of their optimal combination.
In this study, a third-order tensor of size $100 \times 100 \times 30$ with a tubal rank of 20 was used, and two cases were tested: \textbf{TNK $(p=1, k=18)$} and \textbf{TNF $(p=2, k=20)$}. Two sampling rates, 0.5 and 0.6, were considered. For the parameter exploration, $(\mu_1, \mu_2)$ were varied within the range $(10^i, 10^j)$, where $i$ and $j$ ranged from -14 to -2. The RSE results under different parameter settings are shown in Figure \ref{fig:2}.
When the sampling rate was 0.5, the optimal parameter setting for DFT was $(10^{-4}, 10^{-3})$, while for both DCT and ROM it was $(10^{-2}, 10^{-1})$. As the sampling rate increased, the algorithm became less sensitive to parameter variations. Considering the completion performance across different sampling rates, we ultimately adopted $(10^{-4}, 10^{-3})$ for DFT and $(10^{-2}, 10^{-1})$ for both DCT and ROM to achieve the best overall performance.

%\begin{figure*}[!htbp] % h=here, t=top, b=bottom, p=page
%	\centering
%	\includegraphics[width=0.8\textwidth]{fig/sensitive} % 图片路径，不用加 .png/.jpg
%	\caption{Parameter Sensitivity Analysis of \textbf{TNK $(p=1, k=18)$} under Different Transforms and Sampling Rates: (a)-(c) RSE curves of DFT, DCT, and ROM at a sampling rate of 0.5; (d)-(f) RSE curves of DFT, DCT, and ROM at a sampling rate of 0.6.} 
%	\label{fig:2} 
%\end{figure*}

%\begin{figure*}[!htbp] % h=here, t=top, b=bottom, p=page
%	\centering
%	\includegraphics[width=0.8\textwidth]{fig/sensitive2} % 图片路径，不用加 .png/.jpg
%	\caption{Parameter Sensitivity Analysis of \textbf{TNF $(p=2, k=20)$} under Different Transforms and Sampling Rates: (a)-(c) RSE curves of DFT, DCT, and ROM at a sampling rate of 0.5; (d)-(f) RSE curves of DFT, DCT, and ROM at a sampling rate of 0.6.} 
%	\label{fig:2_2} 
%\end{figure*}

\subsubsection{\textbf{Effect of Parameter $k$ on the Performance of the TNK Model}}
In the TNK model, the parameter $k$ in the Ky Fan $k$ norm can take any value, but different choices significantly affect the final tensor completion performance. To investigate this effect, we compared the algorithm’s performance under different values of $k$. The experiments were conducted on a third-order tensor of size $100 \times 100 \times 30$ with a tubal rank of 20. Figure \ref{fig:3} presents the RSE, recovered tubal rank, and running time for different values of $k$ under the three transforms. The results indicate that not all values of $k$ yield satisfactory outcomes, and the final performance varies with $k$. Therefore, the choice of parameter $k$ should be determined according to the specific application scenario.
% \begin{figure*}[!htbp] % h=here, t=top, b=bottom, p=page
%	\centering
%	\includegraphics[width=0.8\textwidth]{fig/k_different} % 图片路径，不用加 .png/.jpg
%	\caption{Influence of the parameter $k$ in the Ky Fan $k$ norm on recovery performance: from top to bottom, the results under the three invertible linear transforms DFT, DCT, and ROM are shown for sampling rates of 0.5 and 0.6, respectively.
%	} 
%	\label{fig:3} 
%\end{figure*}
\begin{figure}[!htbp]
	\renewcommand{\arraystretch}{0.5}
	\setlength\tabcolsep{2pt}
	\centering
	\begin{tabular}{ccc }
		\centering
		\includegraphics[width=1.1in]{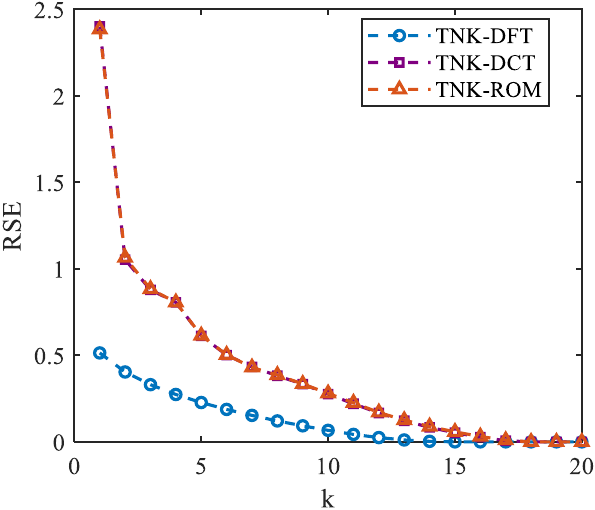}&
		\includegraphics[width=1.1in]{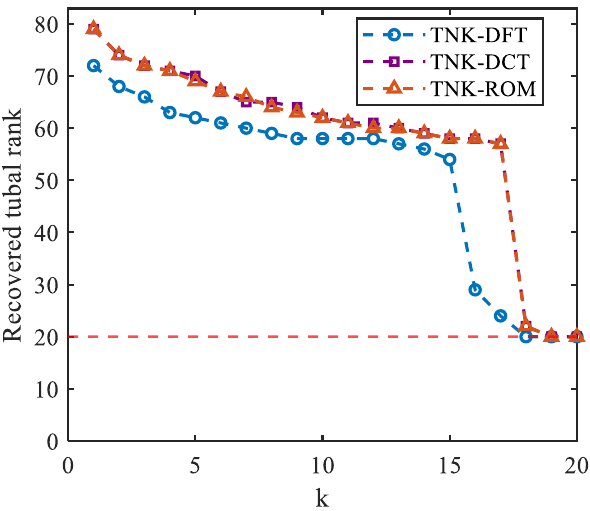}&
		\includegraphics[width=1.1in]{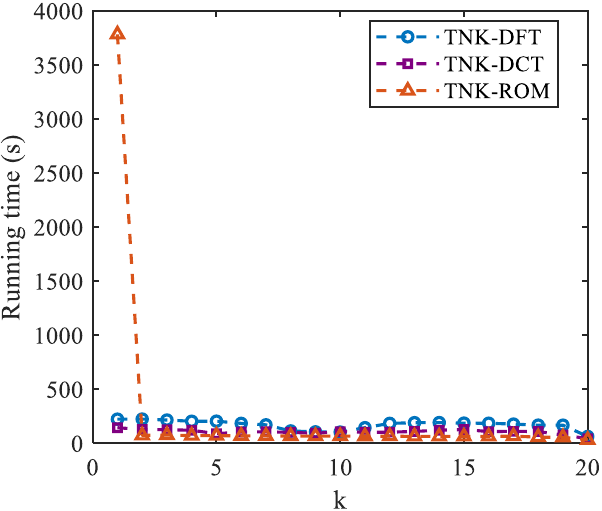}
		\\[+1mm]
		\footnotesize{(a)}  &
		\footnotesize{(b)}  & \footnotesize{(c)}
		\\[+1mm]
		\includegraphics[width=1.1in]{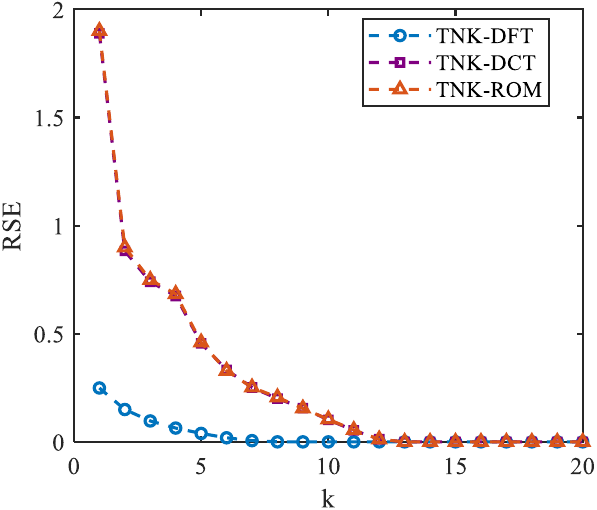}&
		\includegraphics[width=1.1in]{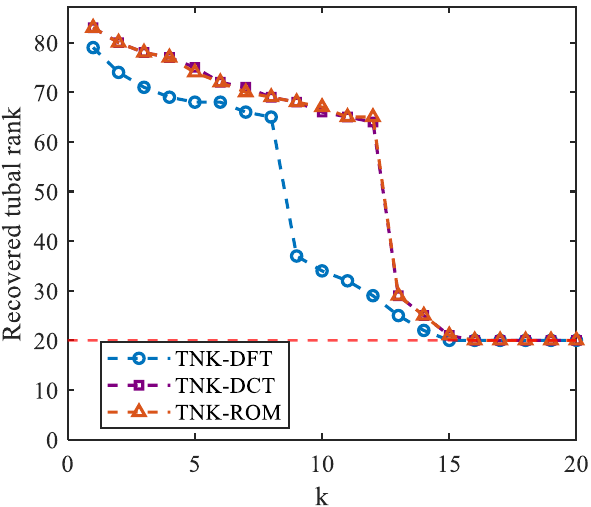}&
		\includegraphics[width=1.1in]{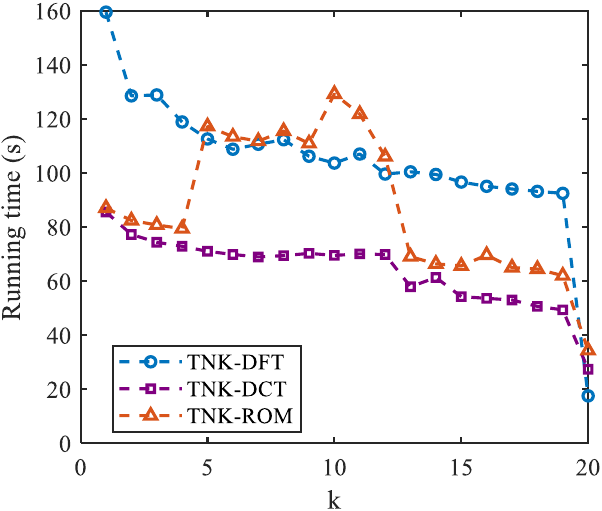}
		\\[+1mm]
		\footnotesize{(d) }  &
		\footnotesize{(f)}  & \footnotesize{(e)}
	\end{tabular}
	%\vspace{-0.15cm}
	\caption{Influence of the parameter $k$ in the Ky Fan $k$ norm on recovery performance: from top to bottom, the results under the three invertible linear transforms DFT, DCT, and ROM are shown for sampling rates of 0.5 and 0.6, respectively.
	} 
	%\vspace{-0.2cm}
	\label{fig:3} 
\end{figure}

\subsubsection{\textbf{Phase Transition Diagrams for Different Tubal Ranks and Sampling Rates}}
To verify the robustness of the proposed TNPK-based tensor completion method, we conduct a series of experiments on synthetic data to investigate the effects of different combinations of the tubal rank $r$ and the sampling rate $p$ on recovery performance, and compare the experimental results with the TNN-based tensor completion method \cite{lu2018exact}.
Specifically, the tubal rank is set to $r = [1, 2, \ldots, 20] $, and the sampling rate is set to $p = [0.01, 0.01, \ldots, 0.99]$. A third-order tensor of size $40 \times 40 \times 20$ is generated using the same settings as in the previous experiments. For each parameter pair $(r, p)$, we independently perform 10 trials. A trial is declared successful if the recovered tensor $\hat{\mathcal{X}}$ satisfies $\mathrm{RSE} \le 10^{-3}.$
The performance of the proposed method is then evaluated by counting the number of successful recoveries.
We consider the \textbf{TNK $(p=1,k=40)$} model.
Since the tensors considered in this experiment are of size $40 \times 40 \times 20$, $k=40$ includes all possible singular components, and hence the TNPK regularizer reduces to one for any nonzero tensor. Therefore, in this limiting case, the regularizer itself no longer explicitly promotes low tubal rank. Nevertheless, the algorithm still achieves satisfactory recovery performance in the experiments. This behavior can be attributed primarily to the implicit regularization induced by the iterative optimization procedure, including its initialization, variable-splitting strategy, and finite-iteration stopping rule, together with the intrinsic low-rank structure of the synthetic data and the information provided by the observed entries.
Moreover, the selected parameter $k$ depends on the tubal rank, and Figure \ref{fig:3} shows that the performance improves as $k$ approaches the tubal rank, we uniformly set $k=40$ for all tubal ranks $r$.

Figure \ref{fig:r_p} illustrates the distributions of successful recovery for the proposed method and the competing approaches under different parameter settings in the tensor completion task.
The white regions indicate complete successful recovery, whereas the black regions correspond to complete recovery failure; a larger white region implies stronger recovery performance. For ease of comparison, a white dashed line is drawn along the diagonal. It can be observed that, regardless of the inverse transform used, the TNK-regularized method consistently yields a significantly larger white region than the TNN method \cite{lu2018exact}. Moreover, similar recovery behaviors are observed under the DFT, DCT, and ROM transforms. These results indicate that the proposed method exhibits strong effectiveness and robustness with respect to different choices of linear inverse transforms, and its performance advantage becomes more pronounced under low sampling rates or high tubal rank settings.
\begin{figure}[!htbp]
	\renewcommand{\arraystretch}{0.5}
	\setlength\tabcolsep{3pt}
	\centering
	\begin{tabular}{cc }
		\centering
		\includegraphics[width=1.6in]{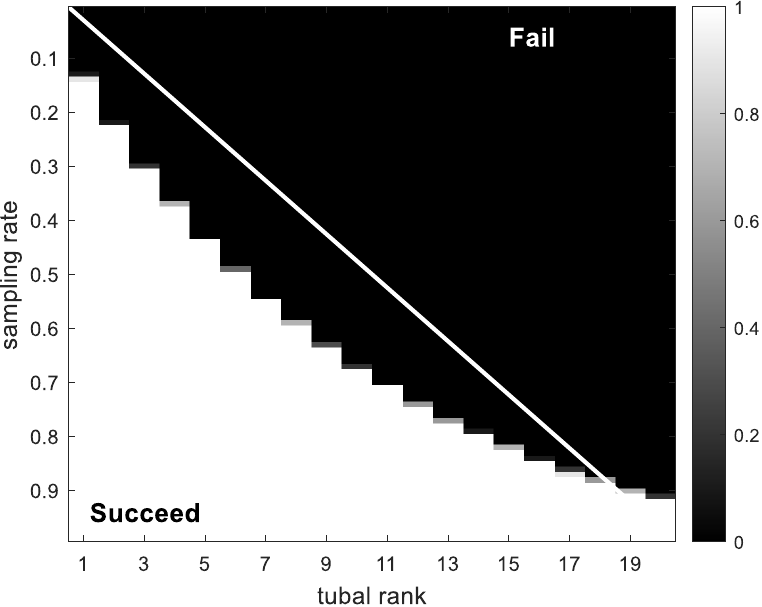}&
		\includegraphics[width=1.6in]{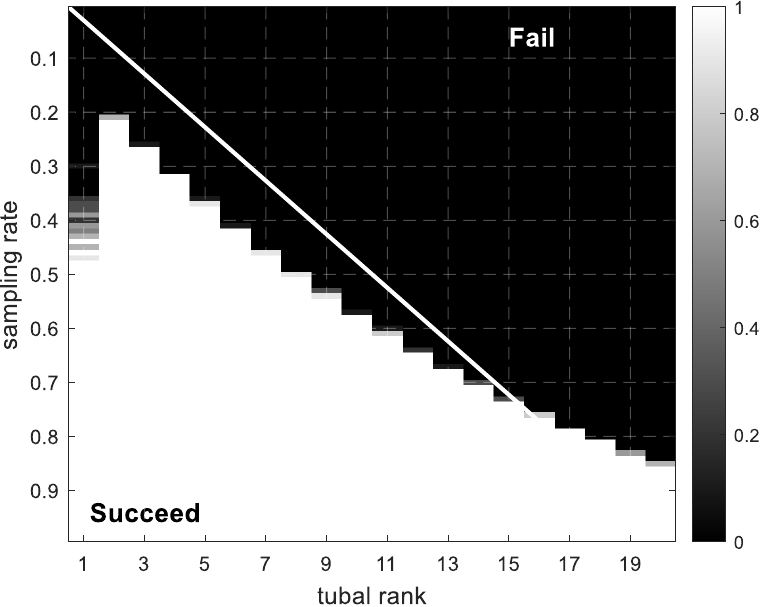}
		\\[+1mm]
		\footnotesize{(a)}  &
		\footnotesize{(b)}  
		\\[+1mm]
		\includegraphics[width=1.6in]{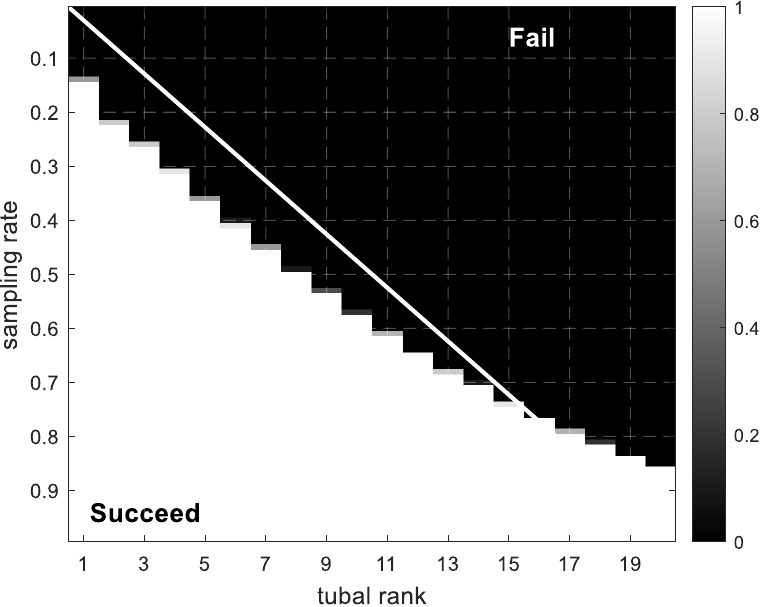}&
		\includegraphics[width=1.6in]{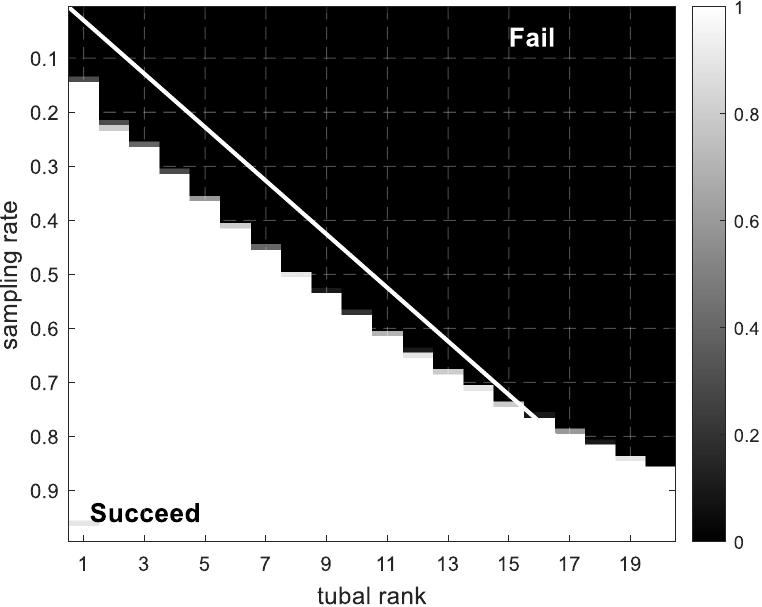}
		\\[+1mm]
		\footnotesize{(c) }  &
		\footnotesize{(d)}  
	\end{tabular}
	%\vspace{-0.15cm}
	\caption{Comparison of success rates for the tensor completion problem on synthetic data: (a)TNN, (b) TNK-DFT, (c) TNK-DCT, and (d) TNK-ROM. The x-axis represents the sampling rate, ranging from 0.01 to 1, and the y-axis represents the rank $r$, which ranges from 1 to 20. Each cell shows the percentage of successful recoveries.
	} 
	%\vspace{-0.2cm}
	\label{fig:r_p} 
\end{figure}

% \begin{figure}[!htbp] % h=here, t=top, b=bottom, p=page
%	\centering
%	\includegraphics[width=0.5\textwidth]{fig/r_p} % 图片路径，不用加 .png/.jpg
%	\caption{Comparison of success rates for the tensor completion problem on synthetic data: (a)TNN, (b) TNK-DFT, (c) TNK-DCT, and (d) TNK-ROM. The x-axis represents the sampling rate, ranging from 0.01 to 1, and the y-axis represents the rank $r$, which ranges from 1 to 20. Each cell shows the percentage of successful recoveries.
%	} 
%	\label{fig:r_p} 
%\end{figure}

\subsection{Real-world data experiment}
In this section, we apply the proposed method to the recovery of color images, hyperspectral images, and videos in order to demonstrate its effectiveness and superiority. It is worth noting that our method, similar to TNN \cite{lu2018exact}, exhibits a certain degree of directional dependence. Specifically, when different degradation models are chosen, the recovery performance is to some extent affected by the parameter $k$. When $n_1$ and $n_2$ are large, the admissible range of $k$ becomes significantly broader, and different choices within this range may lead to substantial variations in reconstruction quality. Moreover, in large-scale settings, the computational cost increases accordingly. For a color image of size $h \times w $, it can be represented as a tensor $\mathcal{M} \in \mathbb{R}^{h \times 3 \times w} $, where the lateral slices correspond to the three color channels. Similarly, for grayscale videos or hyperspectral images, it is generally preferable to arrange the smaller modes in the first two dimensions. Such an arrangement is more conducive to capturing the underlying low-rank structure and can further improve computational efficiency.

\textbf{Datasets:} In this section, three types of real-world tensor data are employed to validate the superiority and effectiveness of the proposed method compared with other algorithms. Specifically, \textbf{ Type I} consists of 56 color images randomly selected from the Berkeley Segmentation Dataset \cite{martin2001database}, with sizes of $481 \times 321 \times 3$ or $321 \times 491 \times 3$.
\textbf{Type II} contains five multispectral images (MSIs) randomly selected from the CAVE dataset \footnote{https://www.cs.columbia.edu/CAVE/databases/multispectral/}
, namely balloons, toy, cloth, feathers, and flowers. Each scene contains 31 spectral bands, and the selected MSI data have a size of $512 \times 512 \times 31$. \textbf{Type III} includes four grayscale videos selected from the YUV Video Database\footnote{http://trace.kom.aau.dk/yuv/index.html}, namely Akiyo, Carphone, Foreman, and Mother\_daughter. Due to computational limitations, only the first 50 frames of each video are used in the experiments. The selected videos are in QCIF format, where each frame has a size of $144 \times 176$, forming a third-order tensor of size $144 \times 176 \times 50$.

\textbf{Experiment Settings:} Since this paper only considers the three cases that admit closed-form solutions, in this section we select two regularization models, TNK and TNF, and apply them to several real-world datasets, comparing their performance with the following LRTC methods: TNN \cite{lu2018exact}, PSTNN \cite{jiang2020multi}, W-t-TNN \cite{mu2020weighted}, IR-t-TNN \cite{wang2021generalized}, and TNF \cite{zheng2024scale}.
In the experiments, three types of degradation are considered, including random missing entries (only 30\% or 10\% of the pixels are observed, where the missing locations are generated via uniform random sampling), grid masks, and text masks.
The recovery performance is evaluated using three metrics, namely peak signal-to-noise ratio (PSNR), structural similarity index (SSIM), and feature similarity index (FSIM), where higher values indicate better reconstruction quality. Let $\mathcal{X}_{\mathrm{GT} } ,\hat {\mathcal{X}} \in\mathbb{R}^{n_1\times n_2\times n_3}$ denote the original tensor and its recovered version, respectively. The PSNR is defined as \[\mathrm{PSNR}=10\log_{10}
\left(
\frac{n_1n_2n_3\|\mathcal{X}_{\mathrm{GT} }\|_\infty^2}
{\|\mathcal{X}_{\mathrm{GT} }-\hat{\mathcal{X}}\|_F^2}
\right).\]
The SSIM between $\mathcal{X}_{\mathrm{GT} }$ and $\hat{\mathcal{X}}$ is defined as \[\mathrm{SSIM}(\mathcal{X}_{\mathrm{GT} },\hat{\mathcal{X}})=\frac{
	(2\mu_{\mathcal{X}_{\mathrm{GT} }}\mu_{\hat{\mathcal{X}}}+C_1)(2\sigma_{{\mathcal{X}_{\mathrm{GT} }\hat{\mathcal{X}}}}+C_2)
}{
	(\mu_{\mathcal{X}_{\mathrm{GT} }}^2+\mu_{\hat{\mathcal{X}}}^2+C_1)(\sigma_{\mathcal{X}_{\mathrm{GT} }}^2+\sigma_{\hat{\mathcal{X}}}^2+C_2)
},\]
where $\mu_{\mathcal{X}_{\mathrm{GT}}}$ and $\mu_{\hat{\mathcal{X}}}$ are the mean values, $\sigma_{\mathcal{X}_{\mathrm{GT} }}^2$ and $\sigma_{\hat{\mathcal{X}}}^2$ are the variances, $\sigma_{{\mathcal{X}_{\mathrm{GT} }\hat{\mathcal{X}}}}$ is the covariance between $\mathcal{X}_{\mathrm{GT} }$ and $\hat {\mathcal{X}} $, and $C_1,C_2$ are small positive constants used to avoid numerical instability. The FSIM is computed based on phase congruency and gradient magnitude. Specifically, for each spatial location $p$, the local similarity is given by $S_L(p)=S_{PC}(p)S_G(p),$ where \[S_{PC}(p)=\frac{2PC_{\mathcal{X}_{\mathrm{GT}}}(p)PC_{\hat{\mathcal{X}}}(p)+T_1}
{PC_{\mathcal{X}_{\mathrm{GT}}}^2(p)+PC_{\hat{\mathcal{X}}}^2(p)+T_1},
\]
\[S_G(p)=\frac{2G_{\mathcal{X}_{\mathrm{GT}}}(p)G_{\hat{\mathcal{X}}}(p)+T_2}
{G_{\mathcal{X}_{\mathrm{GT}}}^2(p)+G_{\hat{\mathcal{X}}}^2(p)+T_2}.\]
Here, $PC_{\mathcal{X}_{\mathrm{GT}}}(p)$ and $PC_{\hat{\mathcal{X}}}(p)$ denote the phase congruency values, $G_{\mathcal{X}_{\mathrm{GT}}}(p)$ and $G_{\hat{\mathcal{X}}}(p)$ denote the gradient magnitudes, and $T_1,T_2$ are positive constants. The overall FSIM is then defined as \[\mathrm{FSIM}(\mathcal{X}_{\mathrm{GT}},\hat{\mathcal{X}})=\frac{
	\sum_{p\in\Omega} S_L(p)PC_m(p)
}{
	\sum_{p\in\Omega} PC_m(p)
},\]
\[
PC_m(p)=\max\{PC_{\mathcal{X}_{\mathrm{GT}}}(p),PC_{\hat{\mathcal{X}}}(p)\},\]
where $\Omega$ denotes the set of all spatial locations.

\begin{figure}[!htbp]
	\centering
	\includegraphics[width=0.48\textwidth]{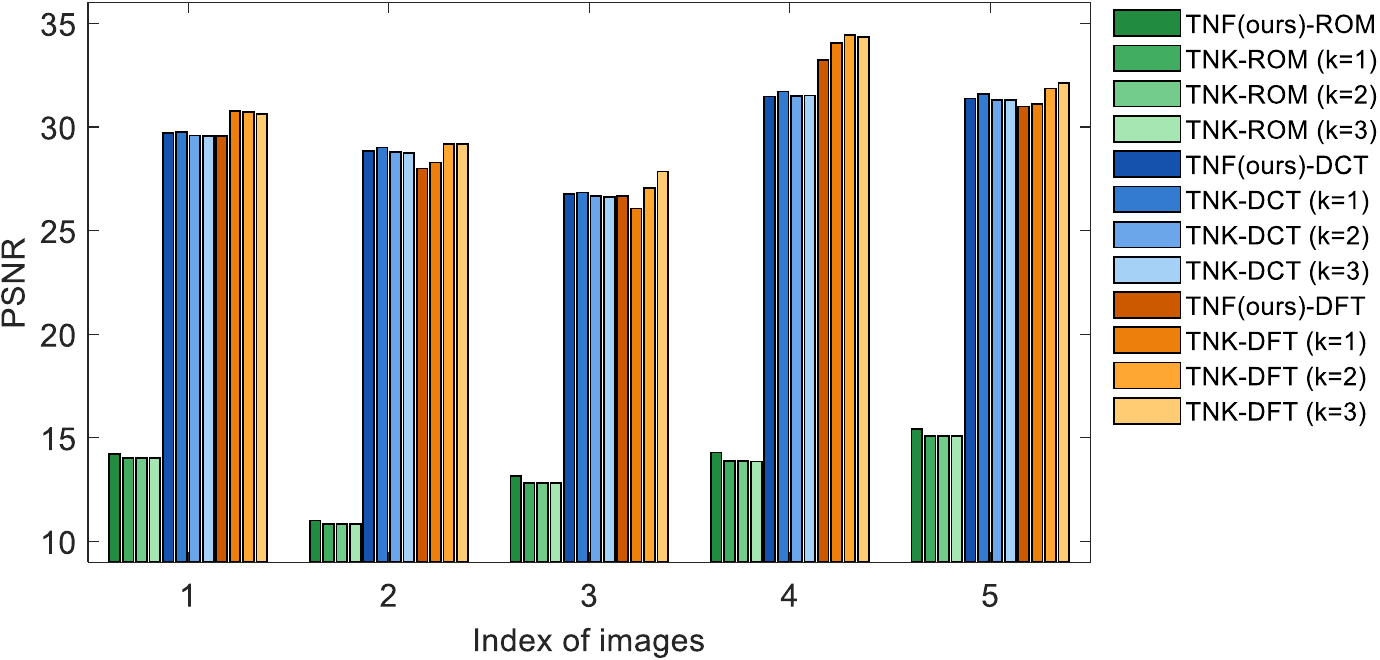}
	\caption{PSNR values of different models in color image restoration under various invertible linear transforms (SR = 0.3).}
	\label{fig:L}
\end{figure}
\subsubsection{\textbf{Color Image Inpainting}}
First, we investigated the impact of three different invertible linear transformations on the model’s performance. Five color images (size $481 \times 321 \times 3$) were randomly selected, and completion experiments were conducted with a sampling rate of 0.3. Figure \ref{fig:L} presents the quantitative PSNR results corresponding to the different transformations.
It can be seen that ROM exhibits relatively poor performance, suggesting that employing a random orthogonal matrix as the linear transformation $L$ may lead to suboptimal recovery and may not be suitable for image analysis tasks. In contrast, DFT achieves superior performance, and hence it was adopted as the preferred transformation in the subsequent experiments.

Next, we evaluated the model on a larger set of 50 color images (size $481 \times 321 \times 3$) with randomly missing pixels, using only the best-performing DFT transformation.
\begin{figure}[!htbp]
	\centering
	\includegraphics[width=0.45\textwidth]{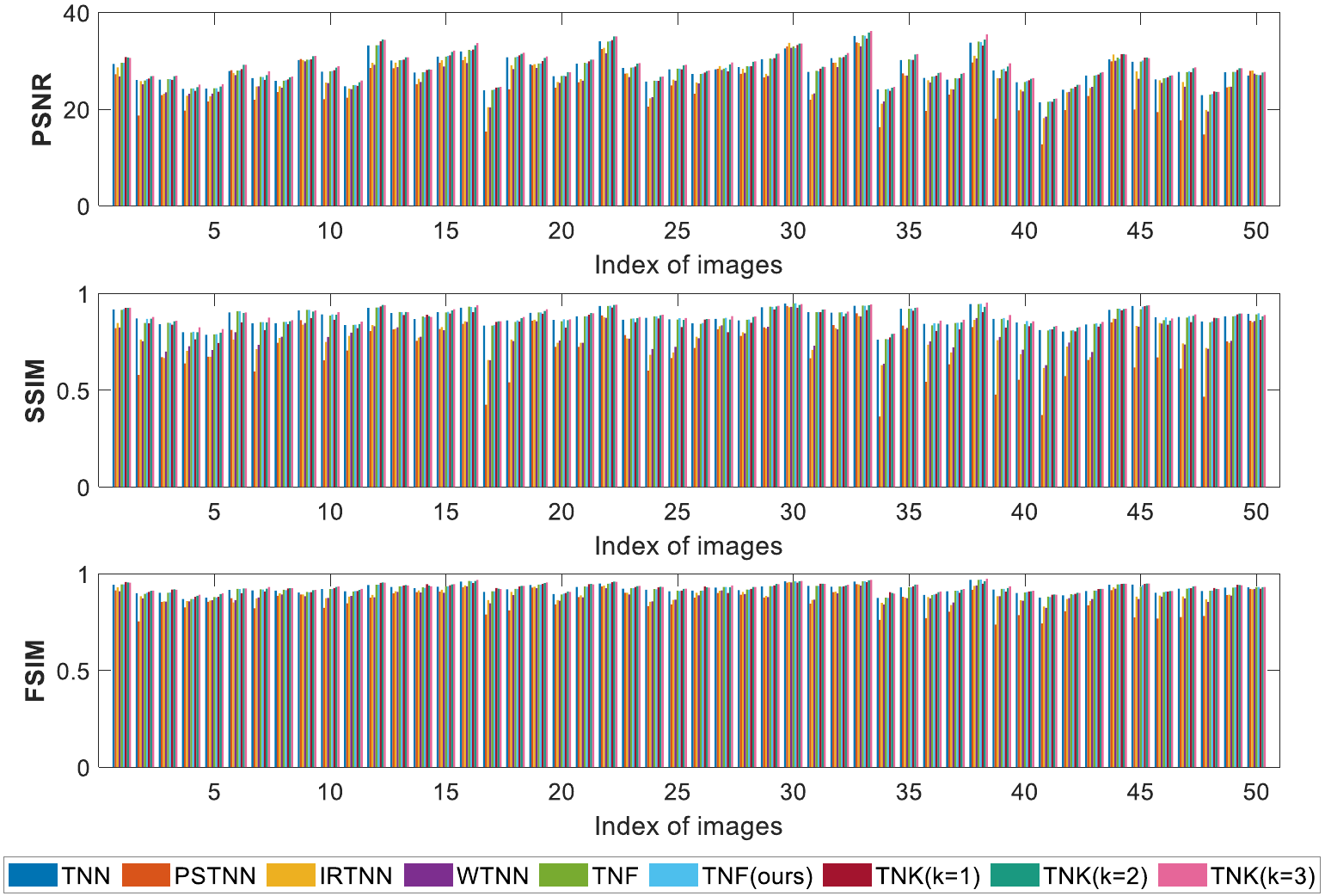}
	\caption{Comparison of evaluation metric values of various methods in color image completion at 30\% sampling rate.}
	\label{fig:50}
\end{figure}
Figure \ref{fig:50} shows the quantitative results in terms of PSNR, SSIM, and FSIM for all test images.
\begin{figure*}[!htbp]
	\renewcommand{\arraystretch}{0.45}
	\setlength\tabcolsep{0.43pt}
	\centering
	\begin{tabular}{ccc  ccc ccc c }%cc ccc  ccc c cc
		\centering
		\includegraphics[width=0.7in]{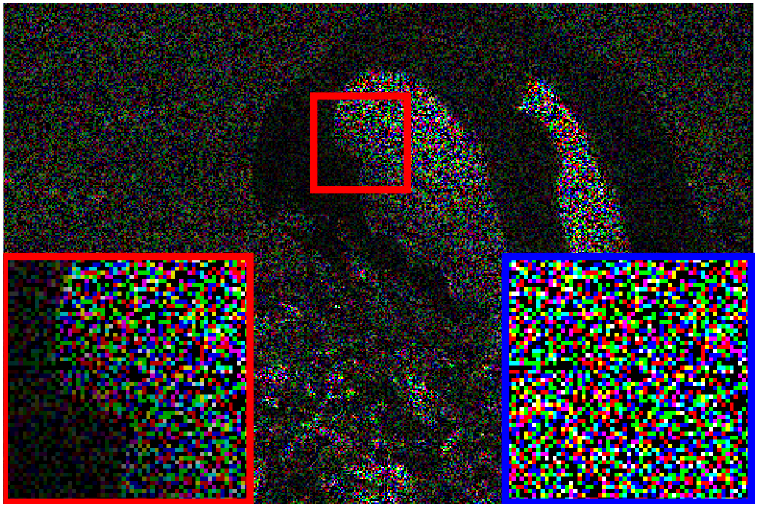} &
		\includegraphics[width=0.7in]{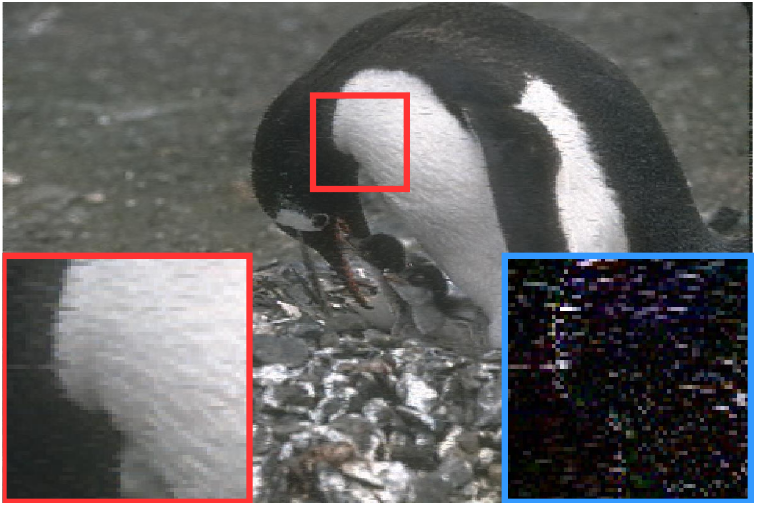} &
		\includegraphics[width=0.7in]{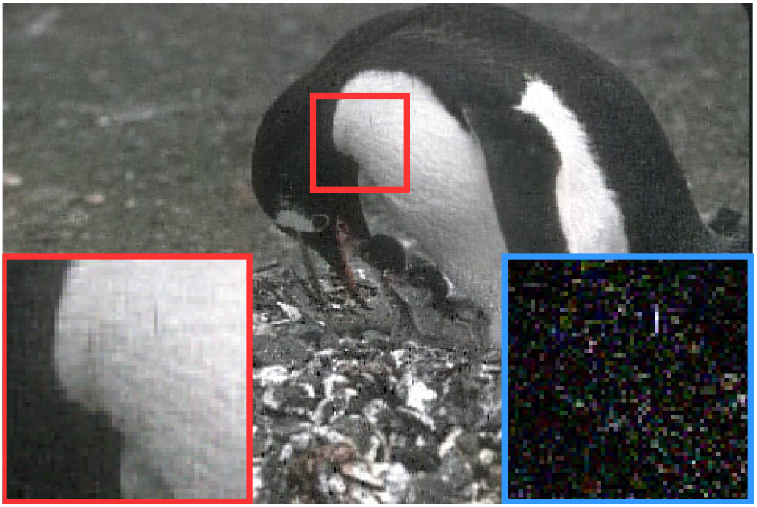} &
		\includegraphics[width=0.7in]{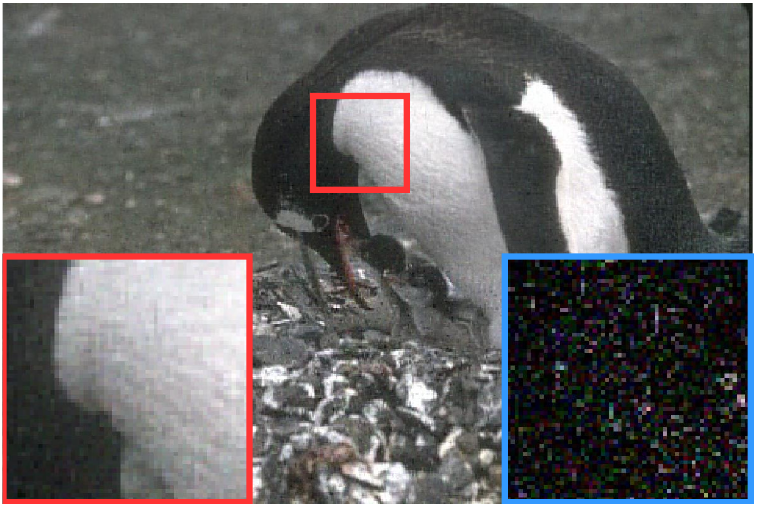} &
		\includegraphics[width=0.7in]{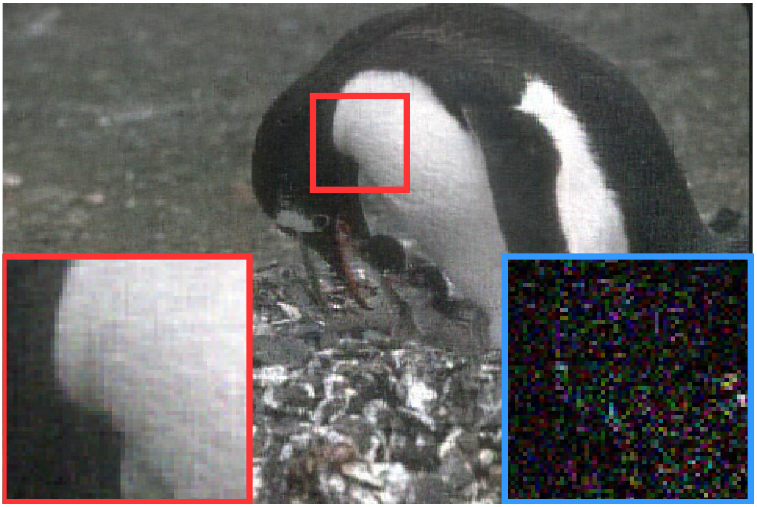} &
		\includegraphics[width=0.7in]{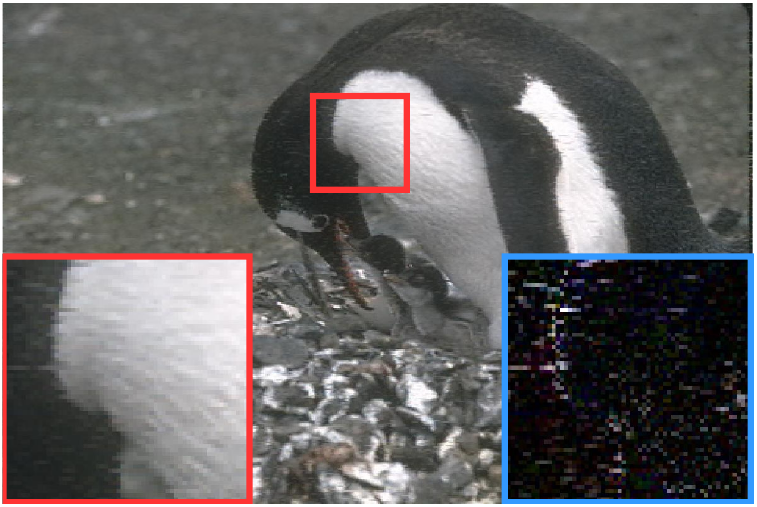} &
		\includegraphics[width=0.7in]{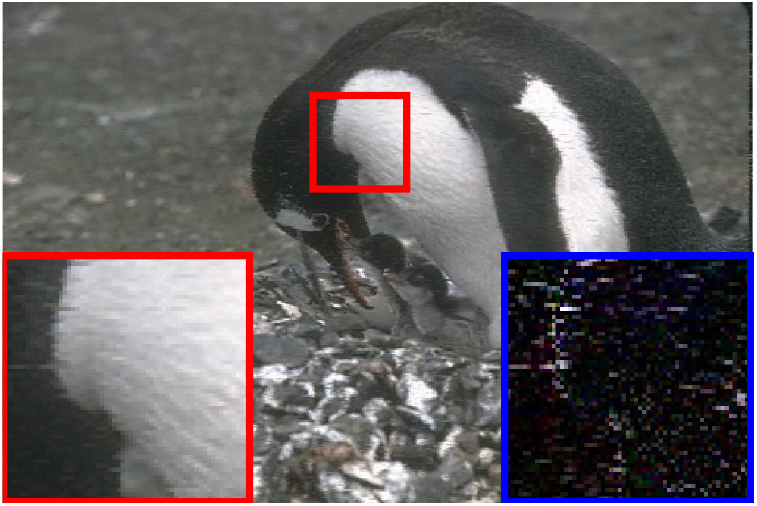} &
		\includegraphics[width=0.7in]{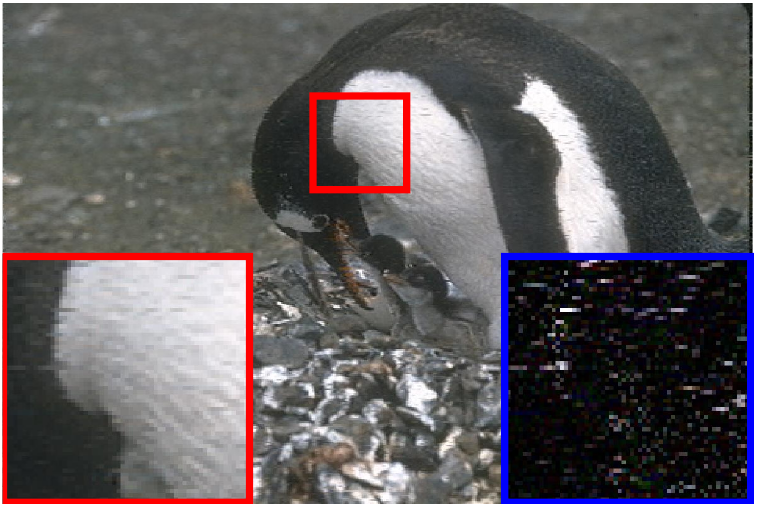} &
		\includegraphics[width=0.7in]{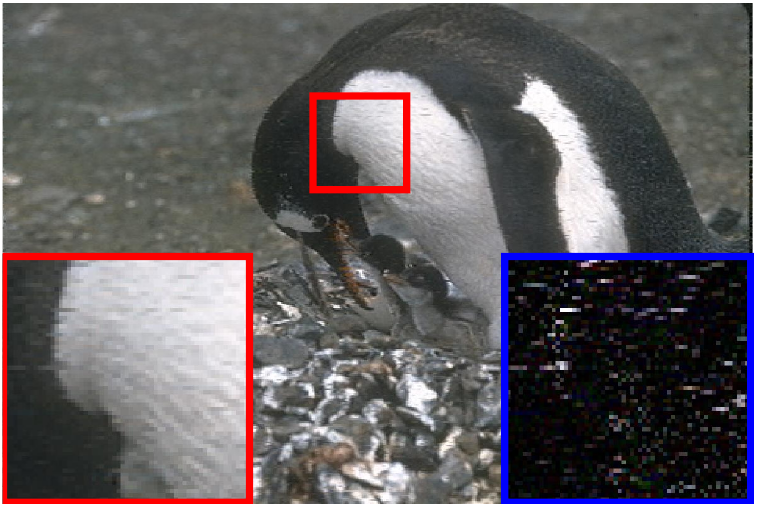} &
		\includegraphics[width=0.7in]{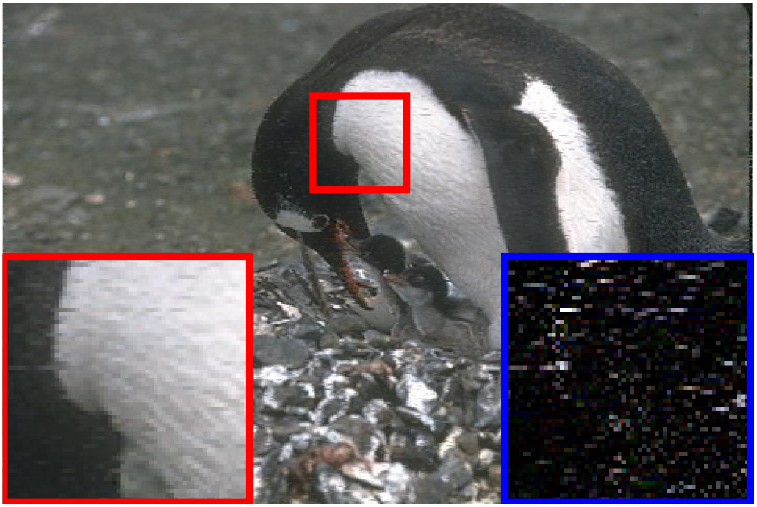} \\
		
		\includegraphics[width=0.7in]{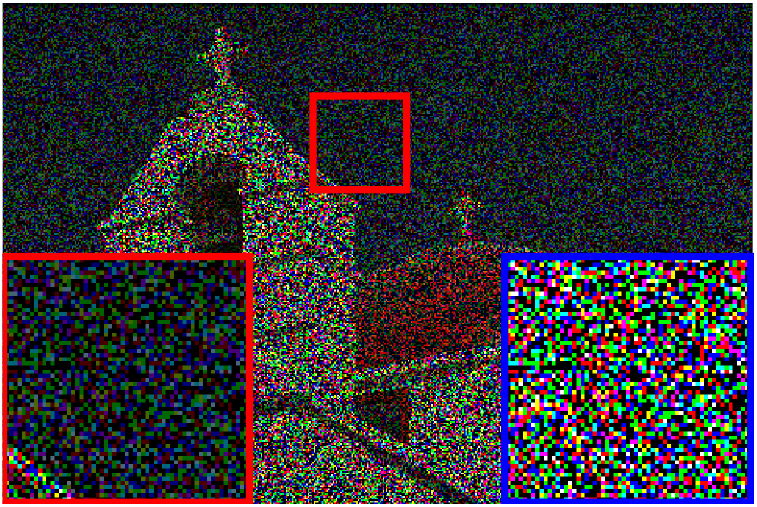} &
		\includegraphics[width=0.7in]{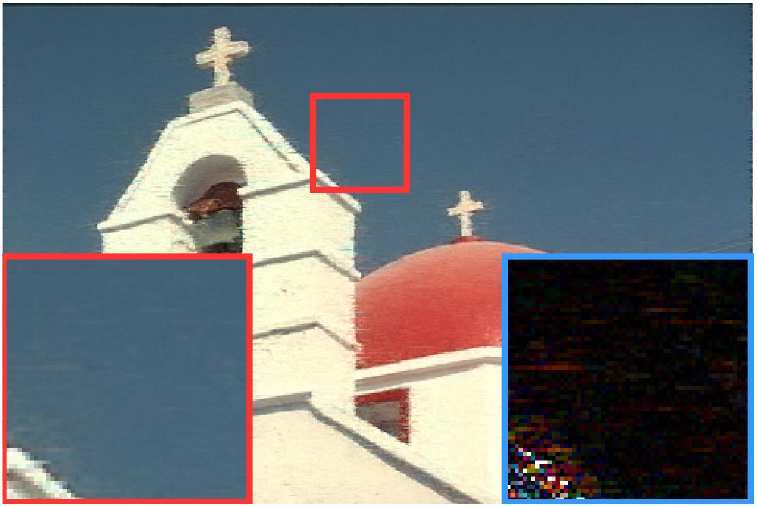} &
		\includegraphics[width=0.7in]{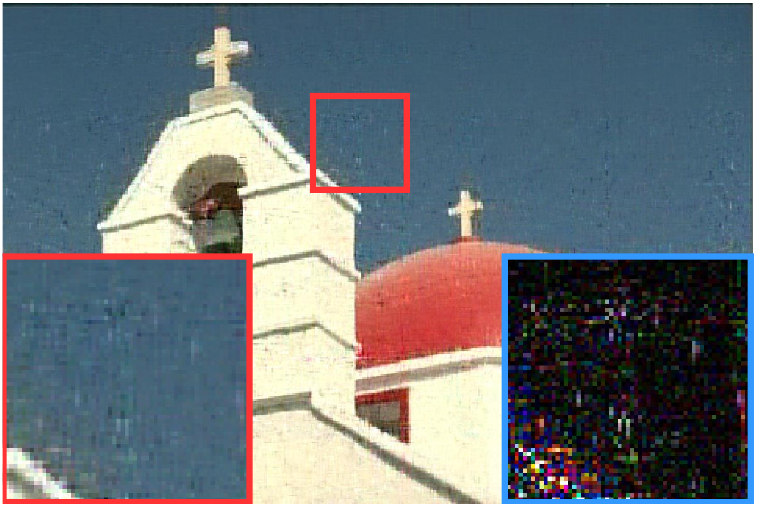} &
		\includegraphics[width=0.7in]{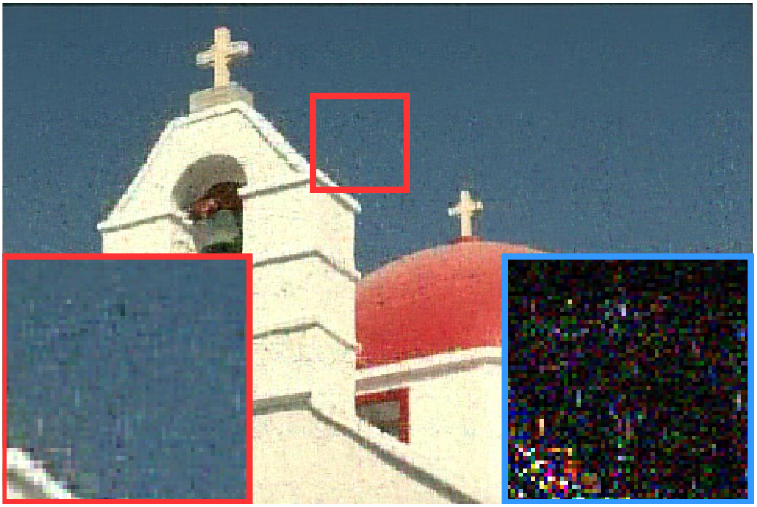} &
		\includegraphics[width=0.7in]{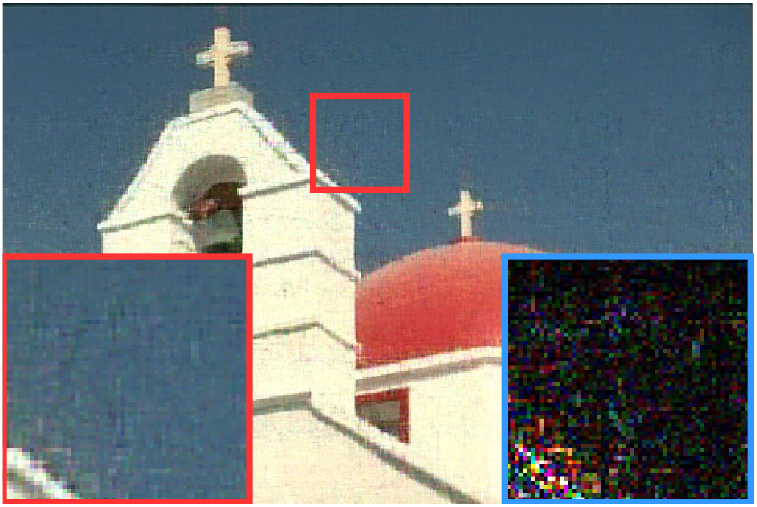} &
		\includegraphics[width=0.7in]{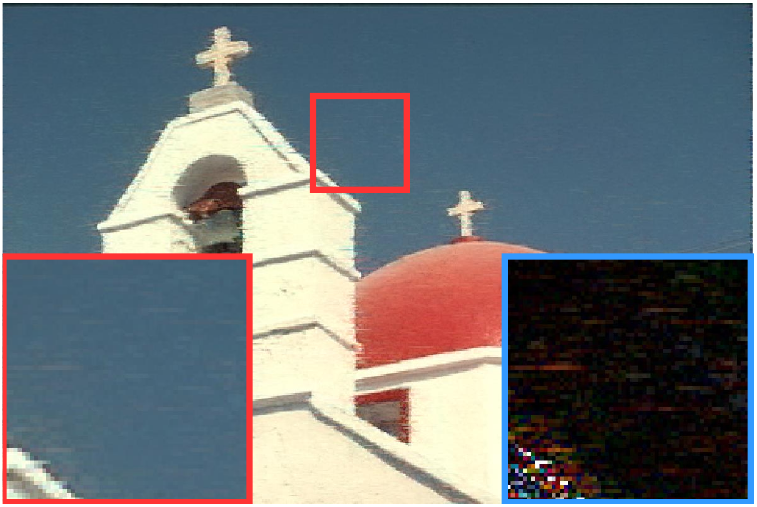} &
		\includegraphics[width=0.7in]{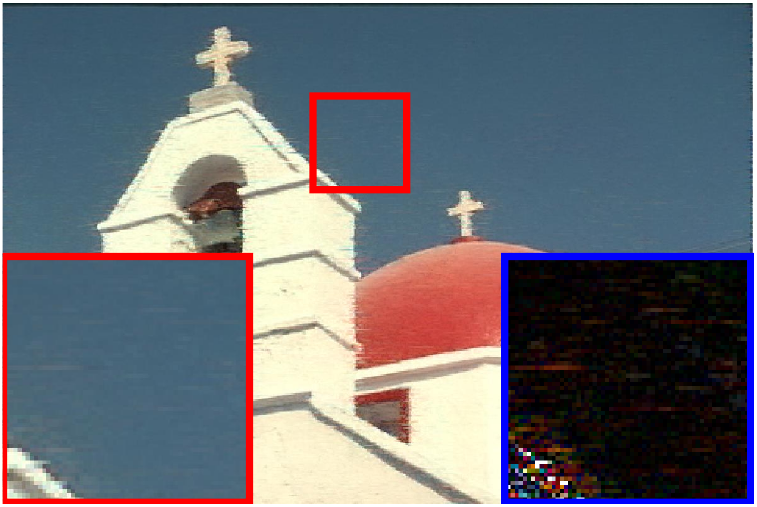} &
		\includegraphics[width=0.7in]{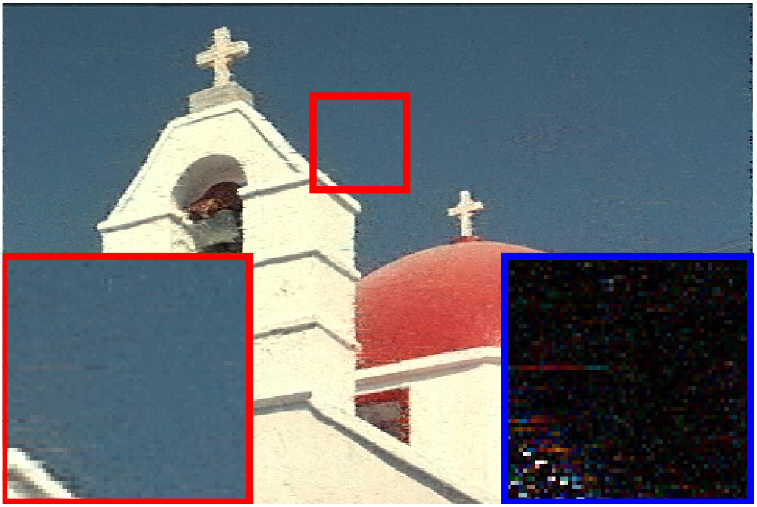} &
		\includegraphics[width=0.7in]{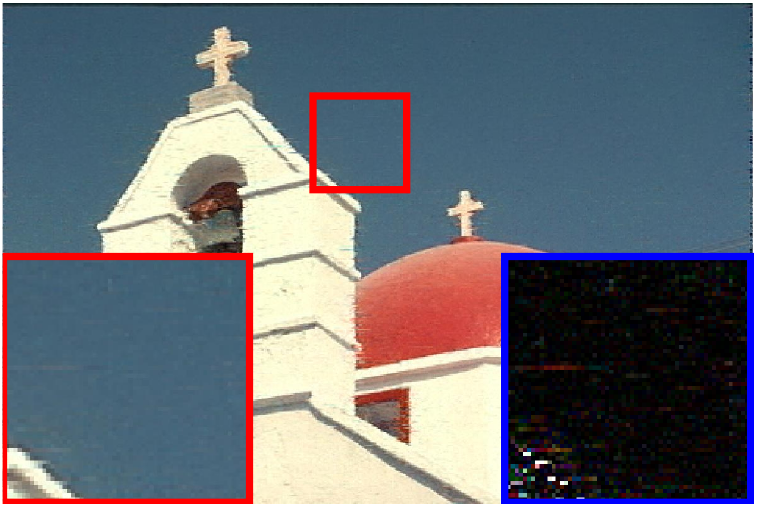} &
		\includegraphics[width=0.7in]{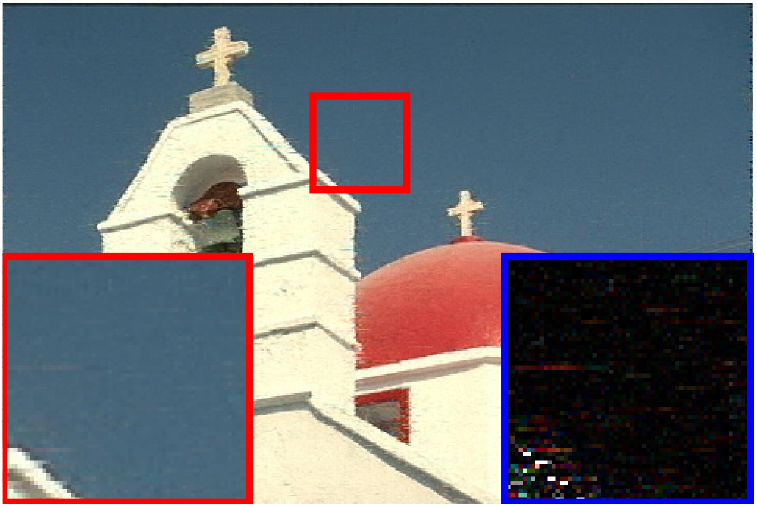} \\
		
		\includegraphics[width=0.7in]{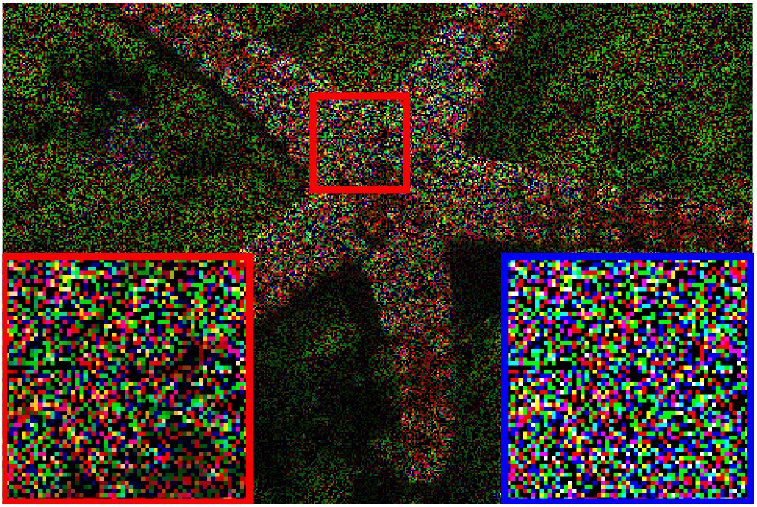} &
		\includegraphics[width=0.7in]{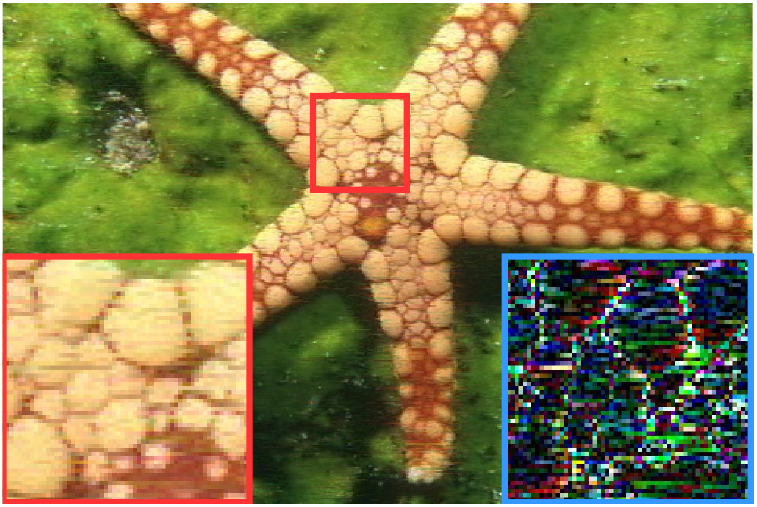} &
		\includegraphics[width=0.7in]{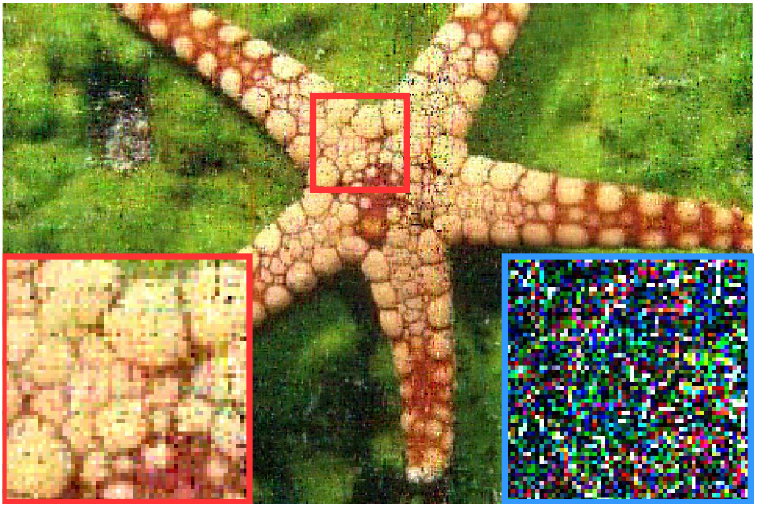} &
		\includegraphics[width=0.7in]{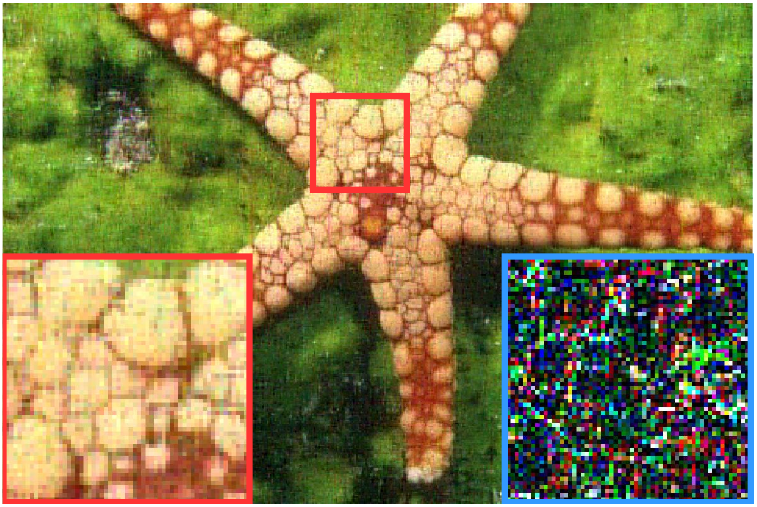} &
		\includegraphics[width=0.7in]{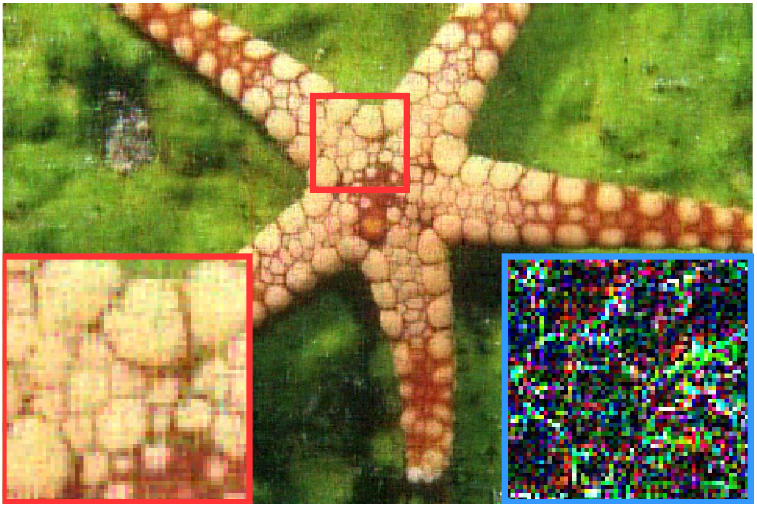} &
		\includegraphics[width=0.7in]{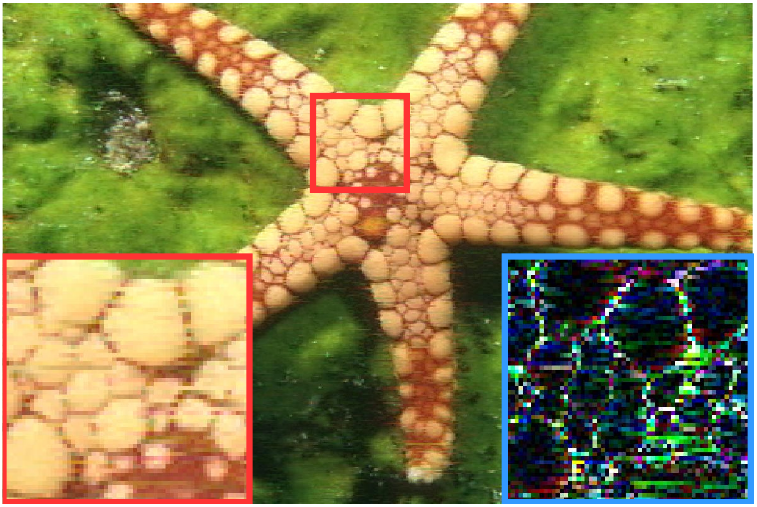} &
		\includegraphics[width=0.7in]{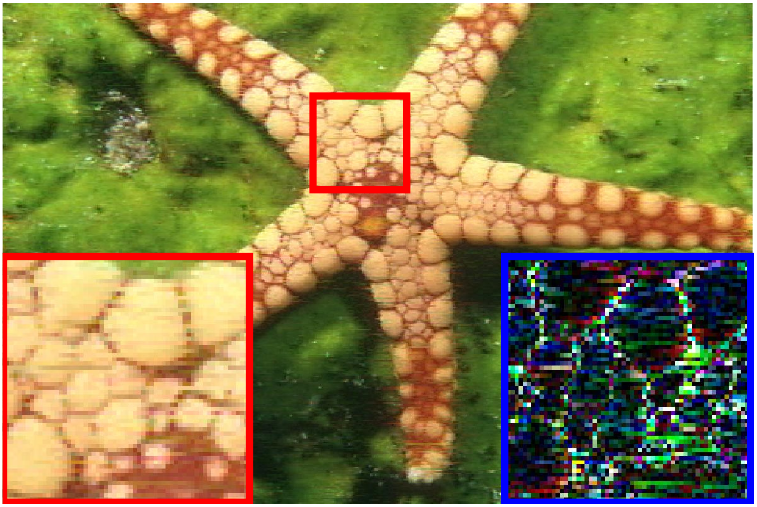} &
		\includegraphics[width=0.7in]{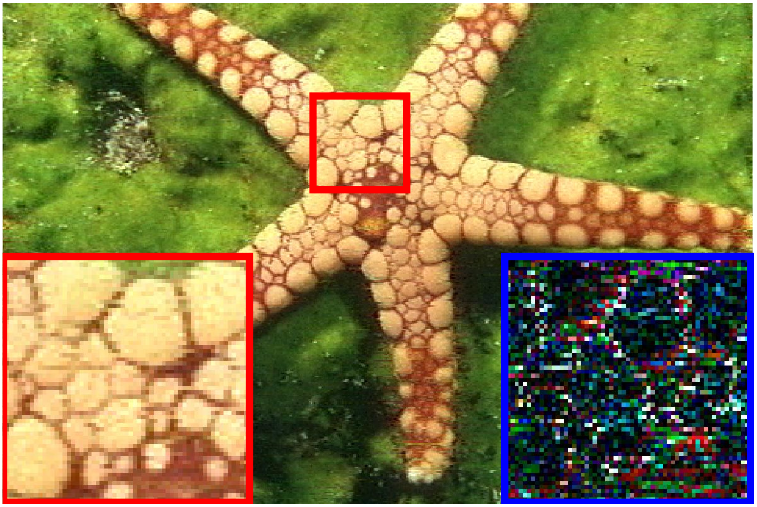} &
		\includegraphics[width=0.7in]{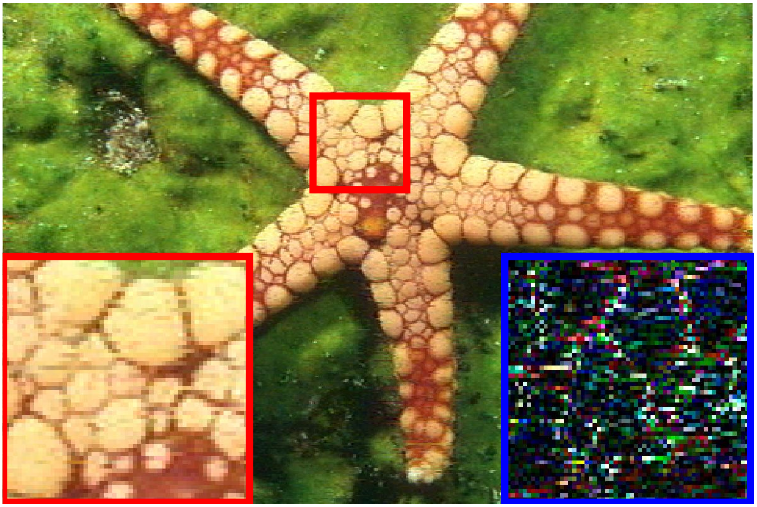} &
		\includegraphics[width=0.7in]{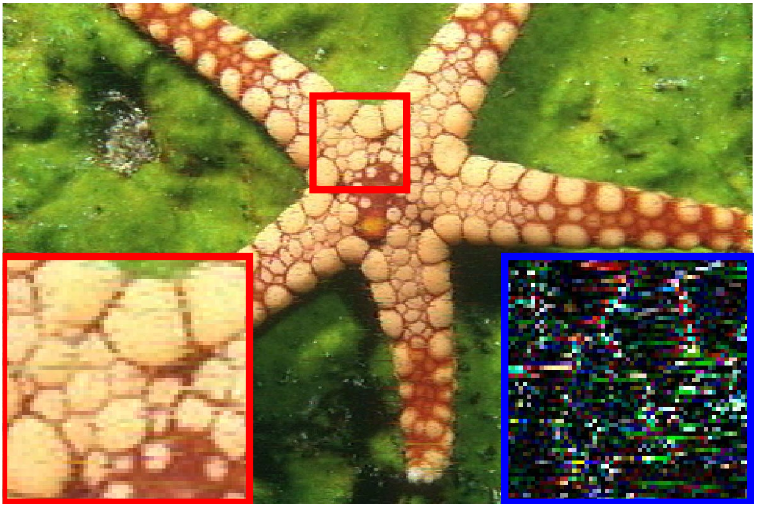} \\
		
		\includegraphics[width=0.7in]{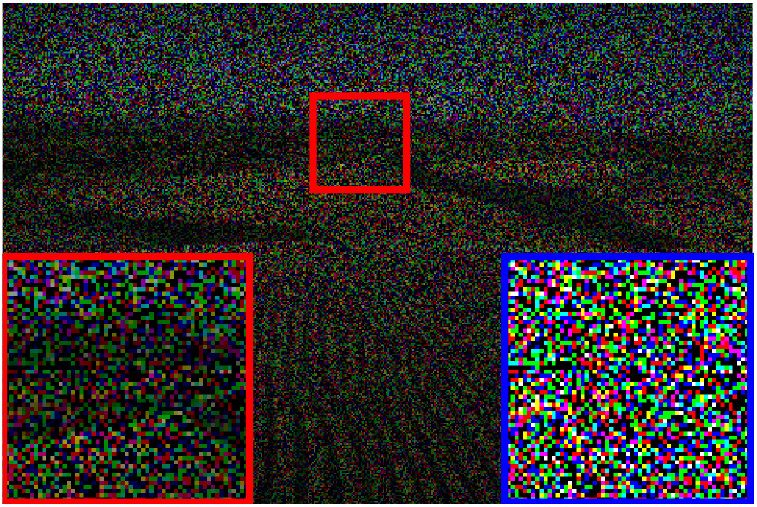} &
		\includegraphics[width=0.7in]{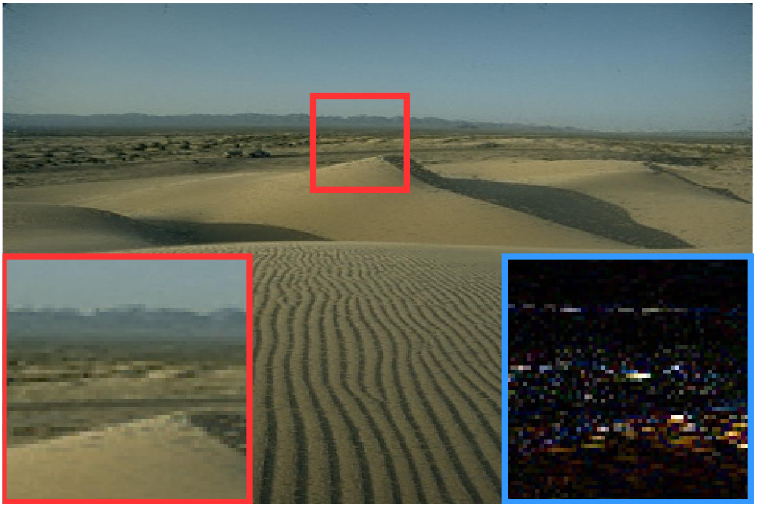} &
		\includegraphics[width=0.7in]{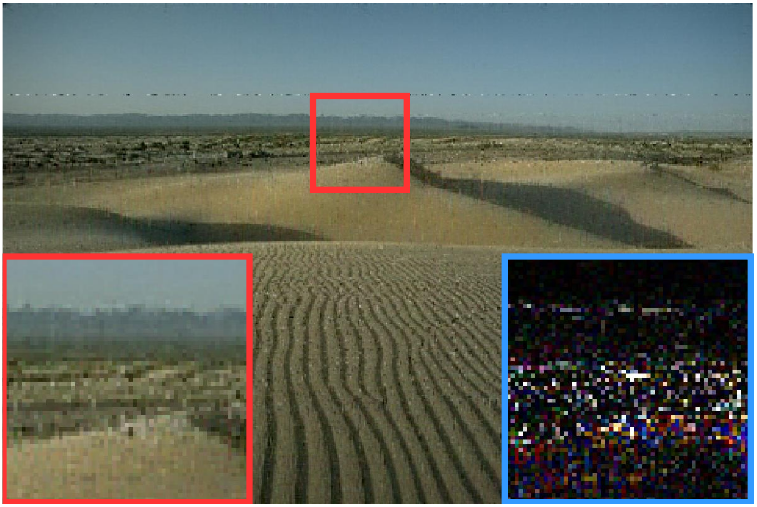} &
		\includegraphics[width=0.7in]{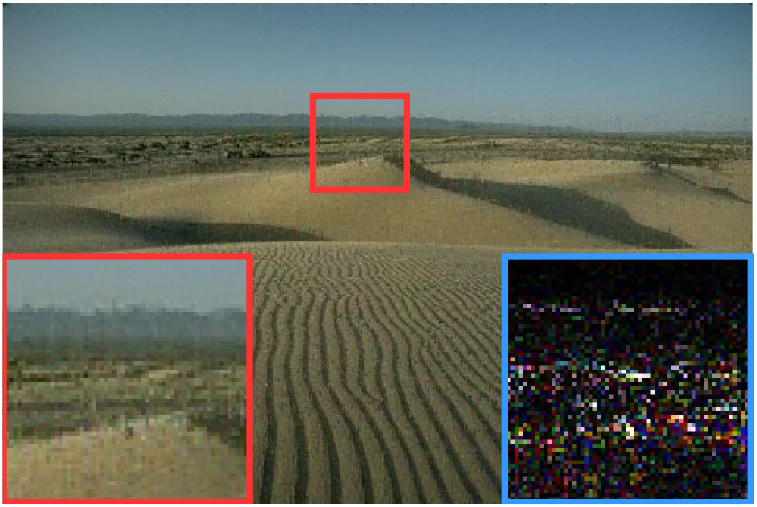} &
		\includegraphics[width=0.7in]{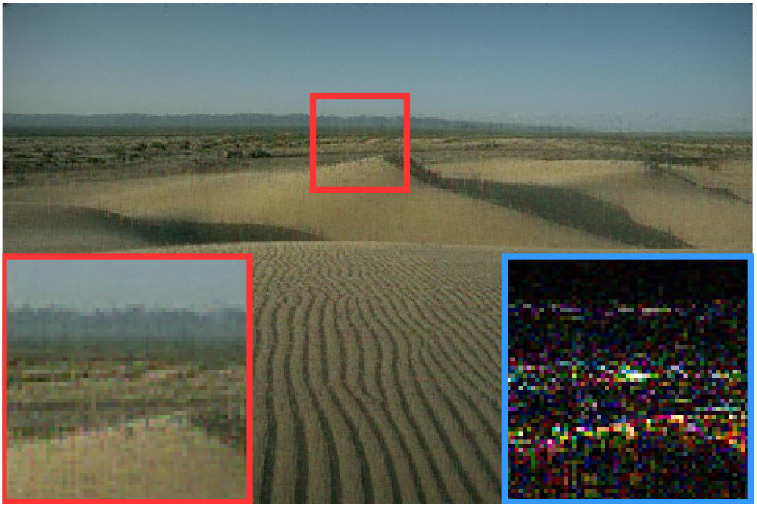} &
		\includegraphics[width=0.7in]{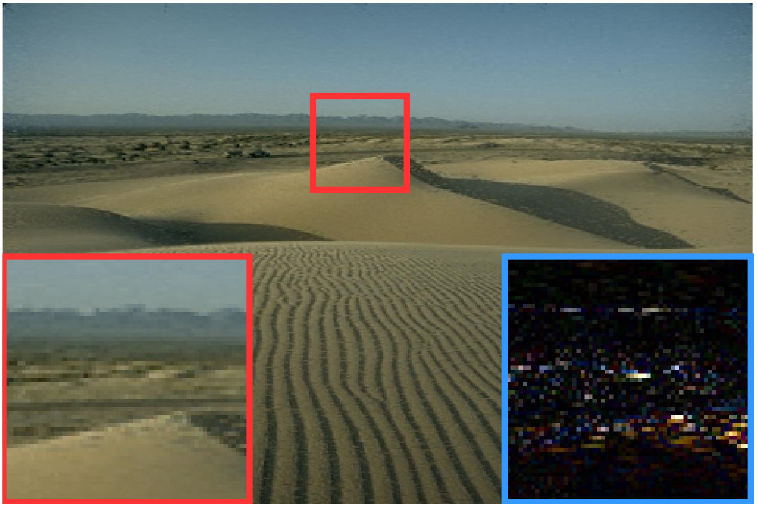} &
		\includegraphics[width=0.7in]{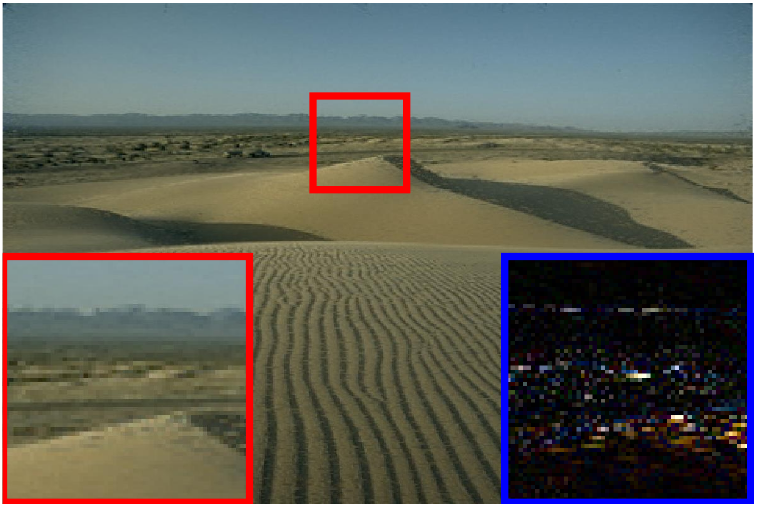} &
		\includegraphics[width=0.7in]{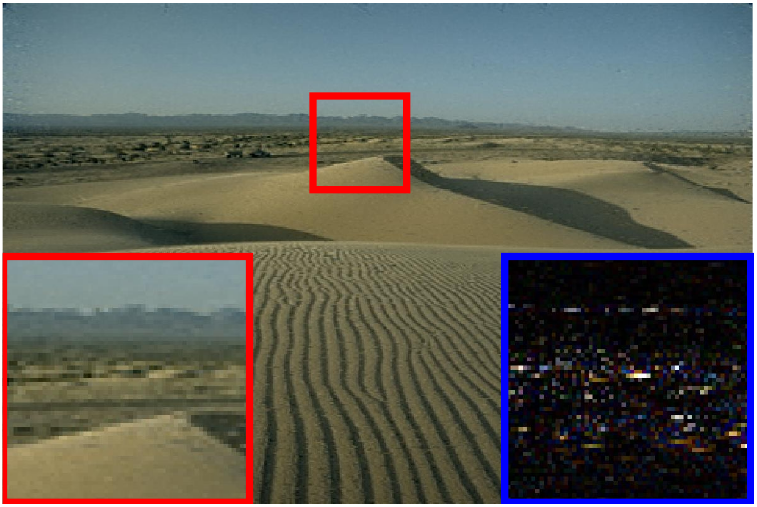} &
		\includegraphics[width=0.7in]{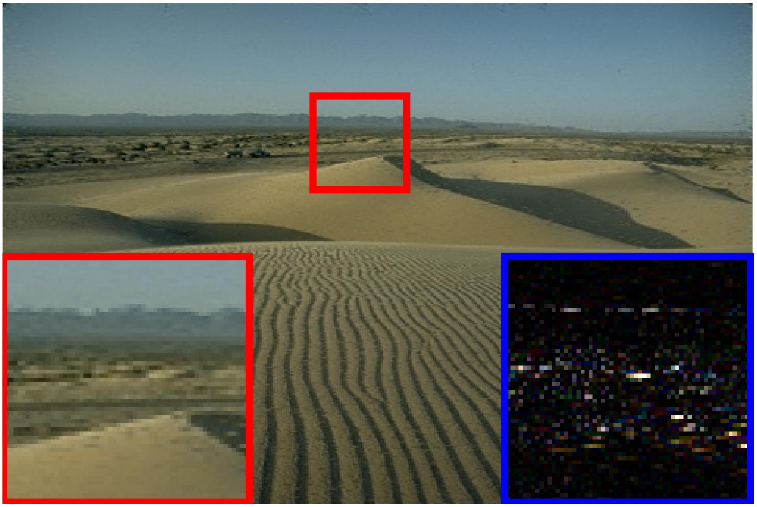} &
		\includegraphics[width=0.7in]{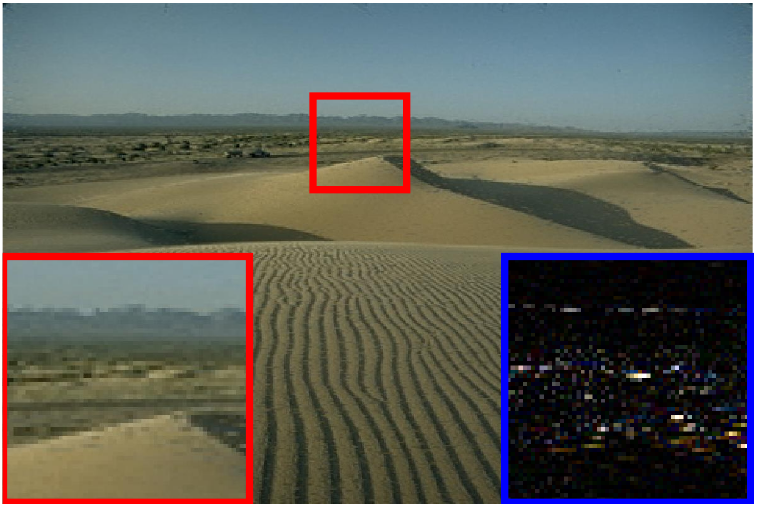} \\
		[+2mm]
		Observed   &
		TNN  & PSTNN &
		IR-t-TNN & W-t-TNN
		&TNF& TNF(ours) &
		TNK(k=1) &
		TNK(k=2)
		&  TNK(k=3)
	\end{tabular}
	\vspace{-0.15cm}
	\caption{
		Comparison of recovery performance for four color images with randomly missing pixels. From left to right: observed image, TNN recovery, PSTNN recovery, IR-t-TNN recovery, W-t-TNN recovery, TNF recovery, and our proposed TNF and TNK recovery methods.}
	\label{fig:color_images}
\end{figure*}
The results demonstrate that the TNK method consistently outperforms the other methods, while the proposed TNF (ours) achieves comparable performance to TNF, ranking just below TNK. 
For further analysis, four images were randomly selected, with detailed quantitative results provided in Table \ref{tab:color_images} and visual comparisons of the recovered images shown in Figure \ref{fig:color_images}.
From a visual perspective, all competing methods produce qualitatively similar reconstructions.
We further observed that, within the TNK model, the optimal performance is generally achieved when $k=3$, which aligns with the previous numerical findings. For a few images, setting $k=2$ yields slightly better results; however, the corresponding PSNR, SSIM, and FSIM values remain very close to those obtained with $k=3$.
Overall, these observations indicate that the choice of the parameter $k$ in the Ky Fan $k$ norm plays a crucial role in determining practical performance.
\begin{table}[]
	\renewcommand{\arraystretch}{1.08}
	\setlength\tabcolsep{0.6pt}
	\centering
	\footnotesize
	\caption{Comparison of PSNR, SSIM, and FSIM for color image inpainting under random missing entries.}
	\label{tab:color_images}
	\begin{tabular}{ccccccccccc}
		\hline
		Image                    & Metrics & TNN     & PSTNN   & \begin{tabular}[c]{@{}c@{}}IR-\\t-TNN\end{tabular} & \begin{tabular}[c]{@{}c@{}}W-\\t-TNN\end{tabular} & TNF                                    & \begin{tabular}[c]{@{}c@{}}TNF\\(ours)\end{tabular} & \begin{tabular}[c]{@{}c@{}}TNK\\(k=1)\end{tabular} & \begin{tabular}[c]{@{}c@{}}TNK\\(k=2)\end{tabular} & \begin{tabular}[c]{@{}c@{}}TNK\\(k=3)\end{tabular} \\ \hline \hline
		& PSNR    & 29.35 & 27.23 & 28.64                                             & 26.78                                            & 29.58                                 & 29.60                                             & {\color[HTML]{FE0000} \textbf{30.77}}             & {\color[HTML]{3531FF} \textbf{30.71}}             & 30.62                                             \\
		& SSIM    & 0.92  & 0.82  & 0.85                                              & 0.83                                             & 0.92                                 & 0.92                                              & 0.93                                              & {\color[HTML]{3531FF} \textbf{0.93}}              & {\color[HTML]{FE0000} \textbf{0.93}}              \\
		\multirow{-3}{*}{Image1} & FSIM    & 0.94  & 0.91   & 0.93                                              & 0.91                                             & 0.95                                 & 0.95                                              & {\color[HTML]{FE0000} \textbf{0.96}}              & {\color[HTML]{3531FF} \textbf{0.96}}              & 0.95                                              \\ \hline
		& PSNR    & 27.88 & 28.08 & 27.48                                             & 27.05                                            & 28.03                                & 28.01                                             & 28.29                                             & {\color[HTML]{3531FF} \textbf{29.18}}             & {\color[HTML]{FE0000} \textbf{29.19}}             \\
		& SSIM    & 0.90  & 0.81  & 0.76                                              & 0.80                                             & {\color[HTML]{3531FF} \textbf{0.91}} & {\color[HTML]{FE0000} \textbf{0.91}}              & 0.85                                              & 0.90                                              & 0.90                                              \\
		\multirow{-3}{*}{Image2} & FSIM    & 0.92  & 0.87  & 0.85                                               & 0.86                                             & 0.92                                 & 0.92                                              & 0.90                                              & {\color[HTML]{3531FF} \textbf{0.92}}              & {\color[HTML]{FE0000} \textbf{0.92}}              \\ \hline
		& PSNR    & 26.42 & 21.93 & 24.70                                             & 24.72                                            & 26.63                                & 26.67                                             & 26.07                                             & {\color[HTML]{3531FF} \textbf{27.06}}             & {\color[HTML]{FE0000} \textbf{27.86}}             \\
		& SSIM    & 0.85  & 0.60   & 0.71                                              & 0.74                                             & 0.85                                 & {\color[HTML]{3531FF} \textbf{0.85}}              & 0.81                                               & 0.85                                              & {\color[HTML]{FE0000} \textbf{0.88}}              \\
		\multirow{-3}{*}{Image3} & FSIM    & 0.91  & 0.82  & 0.87                                              & 0.88                                              & 0.92                                 & 0.92                                              & 0.91                                              & {\color[HTML]{3531FF} \textbf{0.92}}              & {\color[HTML]{FE0000} \textbf{0.93}}              \\ \hline
		& PSNR    & 33.18 & 28.50 & 29.60                                             & 29.18                                            & 33.21                                & 33.26                                             & 34.04                                             & {\color[HTML]{FE0000} \textbf{34.43}}             & {\color[HTML]{3531FF} \textbf{34.36}}             \\
		& SSIM    & 0.93   & 0.81  & 0.84                                              & 0.83                                             & 0.93                                 & 0.93                                              & 0.93                                              & {\color[HTML]{FE0000} \textbf{0.94}}              & {\color[HTML]{3531FF} \textbf{0.94}}              \\
		\multirow{-3}{*}{Image4} & FSIM    & 0.94  & 0.88  & 0.89                                              & 0.88                                             & 0.94                                 & 0.94                                              & 0.95                                              & {\color[HTML]{FE0000} \textbf{0.96}}              & {\color[HTML]{3531FF} \textbf{0.95}}              \\ \hline
	\end{tabular}
\begin{flushleft}
	\footnotesize
	The {\color[HTML]{FE0000} \textbf{best}} and {\color[HTML]{3531FF} \textbf{second-best}} values are highlighted, determined from the original unrounded values.
\end{flushleft}
\end{table}

%\begin{figure*}[!htbp]
%	\centering
%	\includegraphics[width=1\textwidth]{fig/2}
%	\caption{Comparison of recovery performance for five color images with randomly missing pixels. From left to right: observed image, TNN recovery, PSTNN recovery, IR-t-TNN recovery, W-t-TNN recovery, TNF recovery, and our proposed TNF and TNK recovery methods.}
%	\label{fig:color_images}
%\end{figure*}

To further validate the robustness of the proposed method under masked corruption, we designed and conducted supplementary experiments. Six color images were selected from the Berkeley Segmentation Dataset \cite{martin2001database} as the experimental data, with image sizes of $481 \times 321 \times 3$ and $321 \times 491 \times 3$. To simulate structural occlusions that may occur in real-world scenarios, the original images were degraded using grid masks and text masks to generate the observed images. Unlike random pixel loss, these two types of masks not only cover larger areas but also exhibit distinct structured distributions, providing a more rigorous test of the method’s ability to recover images under complex occlusions. In this experiment, only the best-performing DFT-based variants were considered for comparison. The recovery results under grid and text masks are summarized in Table \ref{tab:color_images2}, and the corresponding visual results are shown in Figure \ref{fig:color_images2}. The experimental results indicate that, compared with other competing methods, the proposed TNF(ours) and TNK methods demonstrate notable superiority in both subjective visual quality and objective quantitative metrics, confirming their robustness and effectiveness under structured masked corruption.
\begin{table}[]
	\renewcommand{\arraystretch}{1.08}
	\setlength\tabcolsep{0.5pt}
	\centering
	\footnotesize
	\caption{Comparison of PSNR, SSIM and FSIM for color image inpainting under grid and text mask corruption.}
	\label{tab:color_images2}
	\begin{tabular}{ccccccccccc}
		\hline
		Image                    & Metrics & \begin{tabular}[c]{@{}c@{}}TNN\end{tabular} & \begin{tabular}[c]{@{}c@{}}PSTNN\end{tabular} & \begin{tabular}[c]{@{}c@{}}IR-\\t-TNN\end{tabular} & \begin{tabular}[c]{@{}c@{}}W-\\t-TNN\end{tabular} & \begin{tabular}[c]{@{}c@{}}TNF\end{tabular} & \begin{tabular}[c]{@{}c@{}}TNF\\(ours)\end{tabular} & \begin{tabular}[c]{@{}c@{}}TNK\\(k=1)\end{tabular} & \begin{tabular}[c]{@{}c@{}}TNK\\(k=2)\end{tabular} & \begin{tabular}[c]{@{}c@{}}TNK\\(k=3)\end{tabular} \\ \hline\hline
		\multicolumn{11}{c}{\textbf{Grids}} \\ \hline\hline
		\multirow{3}{*}{Image5}  & PSNR    & 29.63 & 22.72 & 28.61  & 28.16 & 29.85 & {\color[HTML]{FE0000} \textbf{30.03}} & {\color[HTML]{3531FF} \textbf{29.97}} & 29.87  & 29.80  \\
		& SSIM    & 0.96  & 0.89  & 0.92   & 0.92  & 0.95  & 0.96   & 0.96   & {\color[HTML]{FE0000} \textbf{0.96}} & {\color[HTML]{3531FF} \textbf{0.96}} \\
		& FSIM    & 0.97  & 0.91  & 0.95   & 0.94  & 0.96  & 0.96   & {\color[HTML]{FE0000} \textbf{0.97}} & {\color[HTML]{3531FF} \textbf{0.97}} & 0.97   \\ \hline
		\multirow{3}{*}{Image6}  & PSNR    & 32.90 & 24.34 & 32.64  & 32.05 & 32.77 & 33.15  & {\color[HTML]{FE0000} \textbf{33.27}} & {\color[HTML]{3531FF} \textbf{33.19}} & 33.14  \\
		& SSIM    & 0.97  & 0.89  & 0.94   & 0.93  & 0.96  & 0.97   & 0.96   & {\color[HTML]{3531FF} \textbf{0.97}} & {\color[HTML]{FE0000} \textbf{0.97}} \\
		& FSIM    & 0.98  & 0.92  & 0.96   & 0.95  & 0.97  & 0.98   & 0.98   & {\color[HTML]{3531FF} \textbf{0.98}} & {\color[HTML]{FE0000} \textbf{0.98}} \\ \hline
		\multirow{3}{*}{Image7}  & PSNR    & 32.31 & 22.81 & 30.97  & 29.76 & 32.16 & 32.33  & 31.91  & {\color[HTML]{FE0000} \textbf{32.52}} & {\color[HTML]{3531FF} \textbf{32.46}} \\
		& SSIM    & 0.96  & 0.86  & 0.92   & 0.91  & 0.95  & 0.96   & 0.95   & {\color[HTML]{FE0000} \textbf{0.96}} & {\color[HTML]{3531FF} \textbf{0.96}} \\
		& FSIM    & 0.96  & 0.85  & 0.94   & 0.93  & 0.95  & 0.96   & 0.96   & {\color[HTML]{FE0000} \textbf{0.96}} & {\color[HTML]{3531FF} \textbf{0.96}} \\ \hline\hline
		\multicolumn{11}{c}{\textbf{Text}} \\ \hline\hline
		\multirow{3}{*}{Image8}  & PSNR    & 31.10 & 20.94 & 28.88  & 28.88 & 30.81 & {\color[HTML]{FE0000} \textbf{31.39}} & 31.04  & 31.12  & {\color[HTML]{3531FF} \textbf{31.28}} \\
		& SSIM    & 0.93  & 0.73  & 0.86   & 0.86  & 0.92  & {\color[HTML]{FE0000} \textbf{0.93}} & 0.92   & 0.93   & {\color[HTML]{3531FF} \textbf{0.93}} \\
		& FSIM    & 0.94  & 0.80  & 0.90   & 0.90  & 0.93  & 0.94   & 0.94   & {\color[HTML]{FE0000} \textbf{0.94}} & {\color[HTML]{3531FF} \textbf{0.94}} \\ \hline
		\multirow{3}{*}{Image9}  & PSNR    & 25.42 & 17.54 & 23.07  & 23.78 & 24.88 & 25.49  & 24.81  & {\color[HTML]{FE0000} \textbf{25.55}} & {\color[HTML]{3531FF} \textbf{25.57}} \\
		& SSIM    & {\color[HTML]{3531FF} \textbf{0.87}} & 0.74  & 0.81   & 0.80  & 0.85  & 0.86   & 0.83   & 0.87   & {\color[HTML]{FE0000} \textbf{0.88}} \\
		& FSIM    & 0.93  & 0.86  & 0.90   & 0.89  & 0.91  & 0.92   & 0.91   & {\color[HTML]{FE0000} \textbf{0.93}} & {\color[HTML]{3531FF} \textbf{0.93}} \\ \hline
		\multirow{3}{*}{Image10} & PSNR    & 26.62 & 18.74 & 23.46  & 25.96 & 26.28 & 26.64  & 26.71  & {\color[HTML]{FE0000} \textbf{26.73}} & {\color[HTML]{3531FF} \textbf{26.71}} \\
		& SSIM    & 0.90  & 0.80  & 0.86   & 0.87  & 0.90  & 0.90   & 0.90   & {\color[HTML]{3531FF} \textbf{0.90}} & {\color[HTML]{FE0000} \textbf{0.90}} \\
		& FSIM    & 0.93  & 0.84  & 0.91   & 0.92  & 0.93  & 0.93   & {\color[HTML]{FE0000} \textbf{0.93}} & {\color[HTML]{3531FF} \textbf{0.93}} & 0.93   \\ \hline
	\end{tabular}
\begin{flushleft}
	\footnotesize
	The {\color[HTML]{FE0000} \textbf{best}} and {\color[HTML]{3531FF} \textbf{second-best}} values are highlighted, determined from the original unrounded values.
\end{flushleft}
\end{table}
\begin{figure*}[!htbp]
	\renewcommand{\arraystretch}{0.45}
	\setlength\tabcolsep{0.43pt}
	\centering
	\begin{tabular}{ccc  ccc ccc c }%cc ccc  ccc c cc
		\centering
		\includegraphics[width=0.7in]{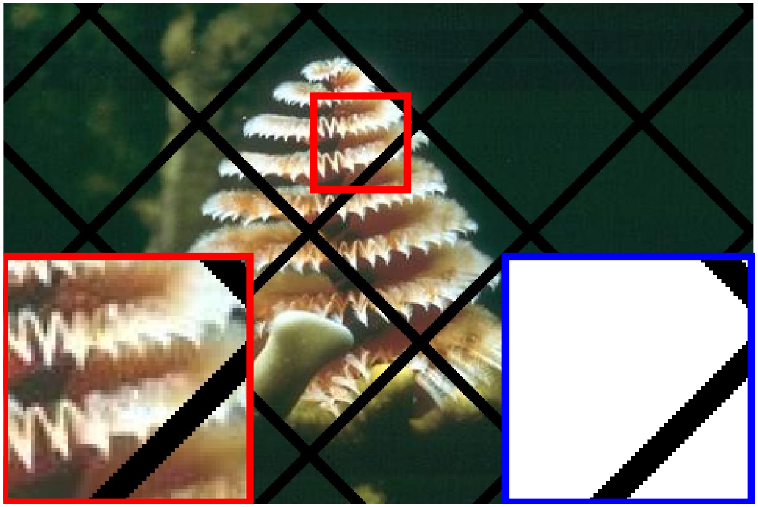} &
		\includegraphics[width=0.7in]{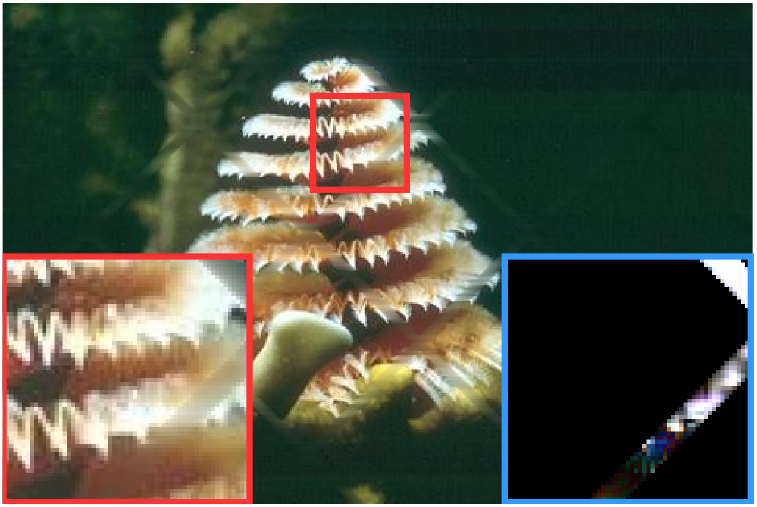} &
		\includegraphics[width=0.7in]{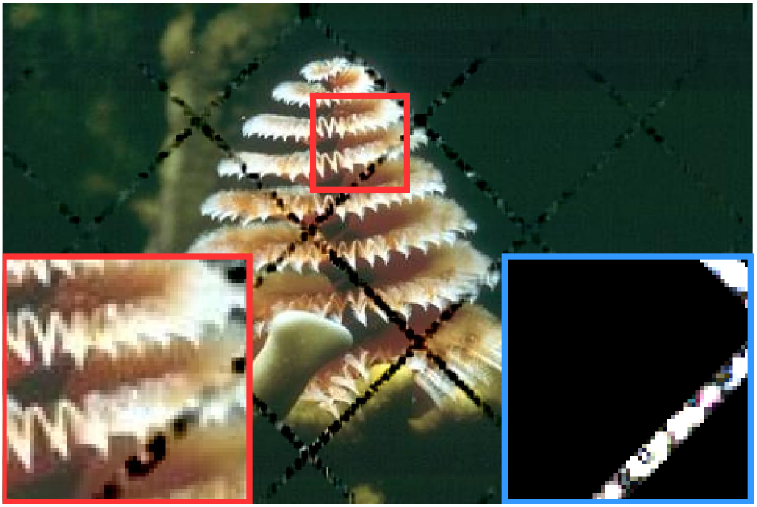} &
		\includegraphics[width=0.7in]{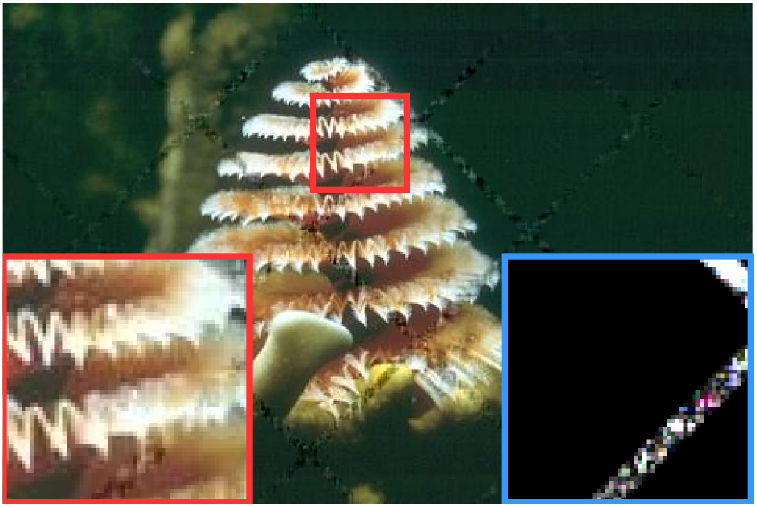} &
		\includegraphics[width=0.7in]{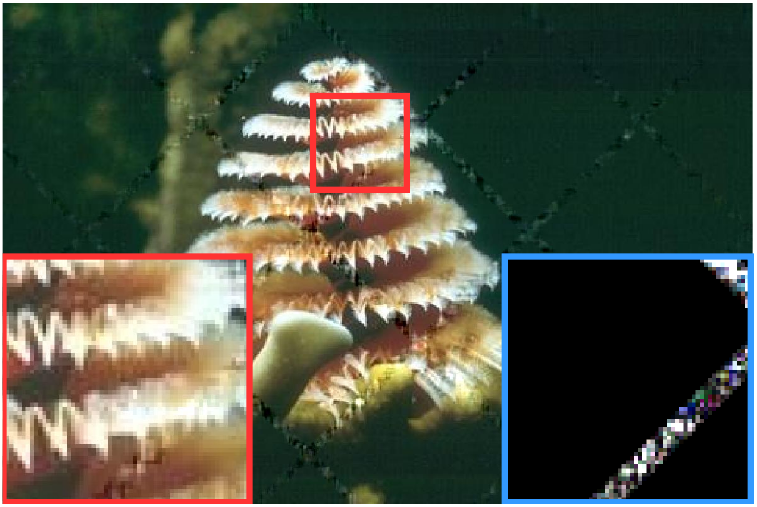} &
		\includegraphics[width=0.7in]{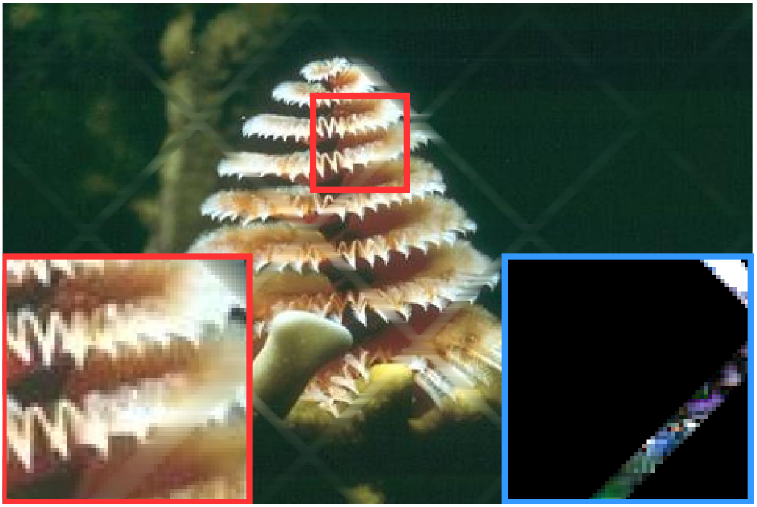} &
		\includegraphics[width=0.7in]{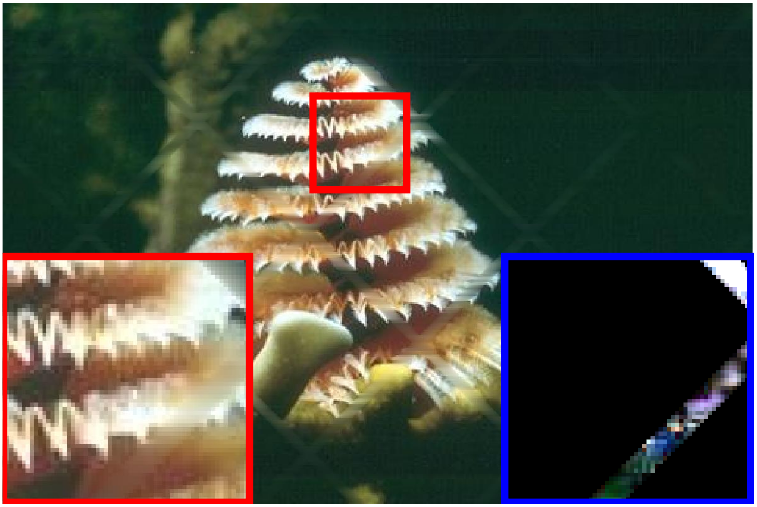} &
		\includegraphics[width=0.7in]{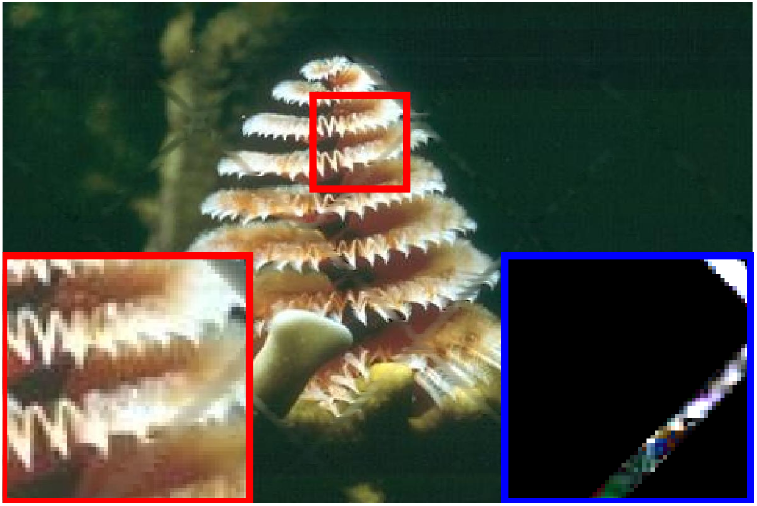} &
		\includegraphics[width=0.7in]{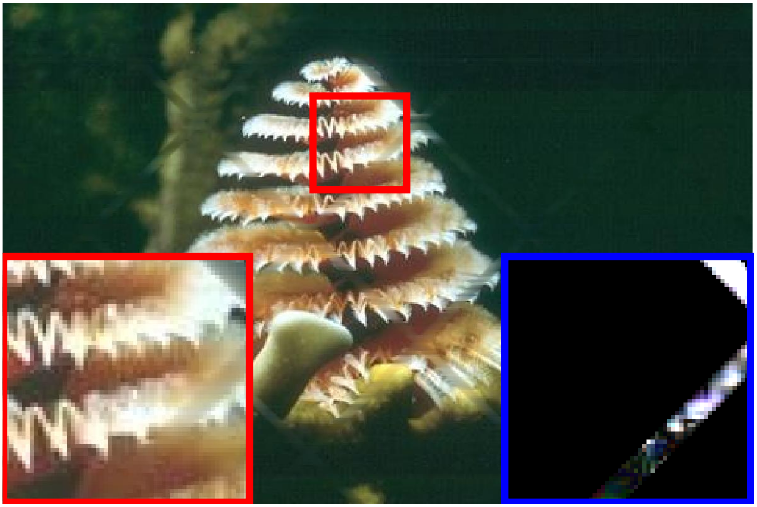} &
		\includegraphics[width=0.7in]{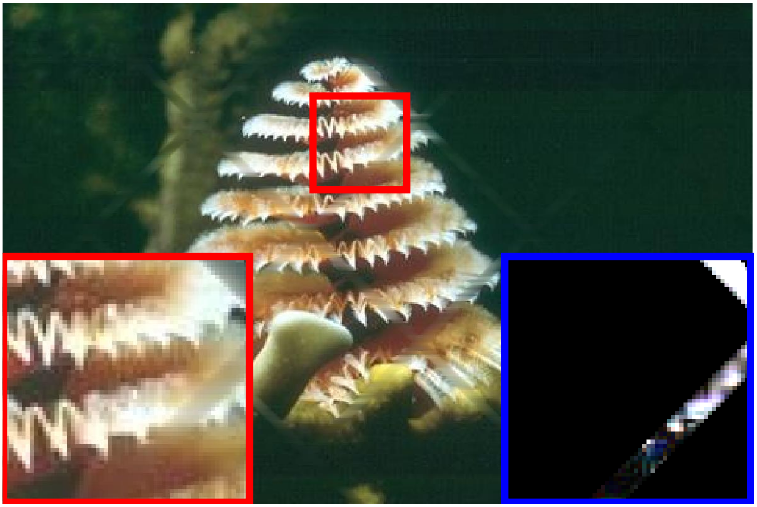} \\
		
		\includegraphics[width=0.7in]{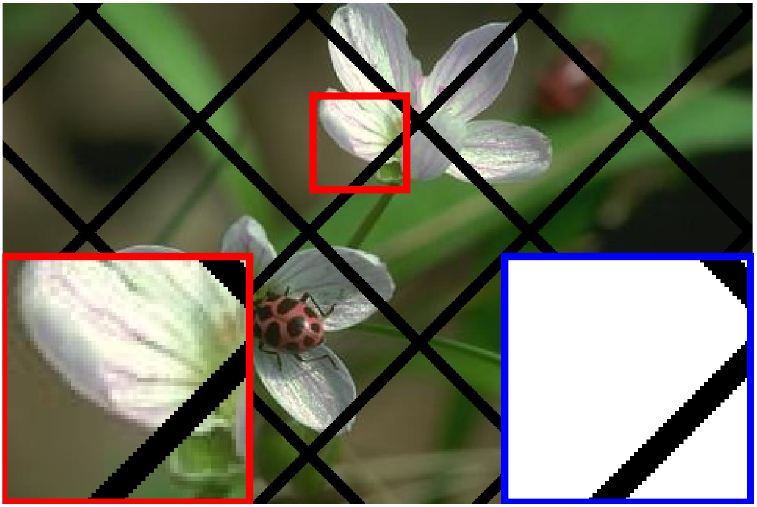} &
		\includegraphics[width=0.7in]{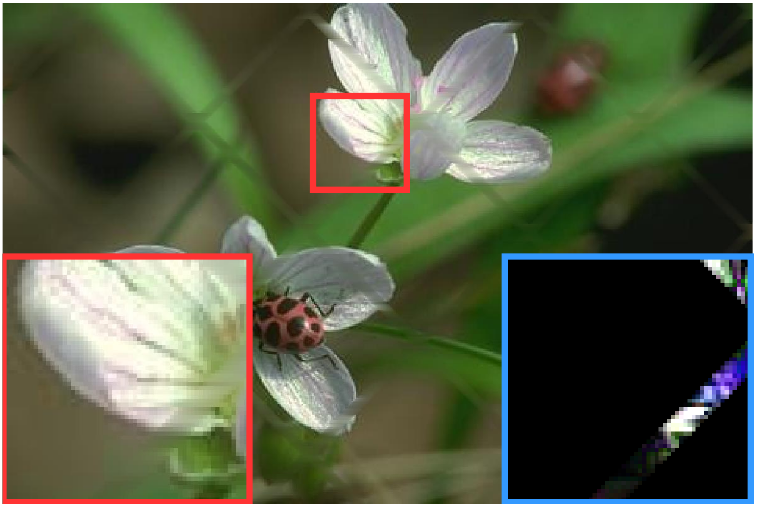} &
		\includegraphics[width=0.7in]{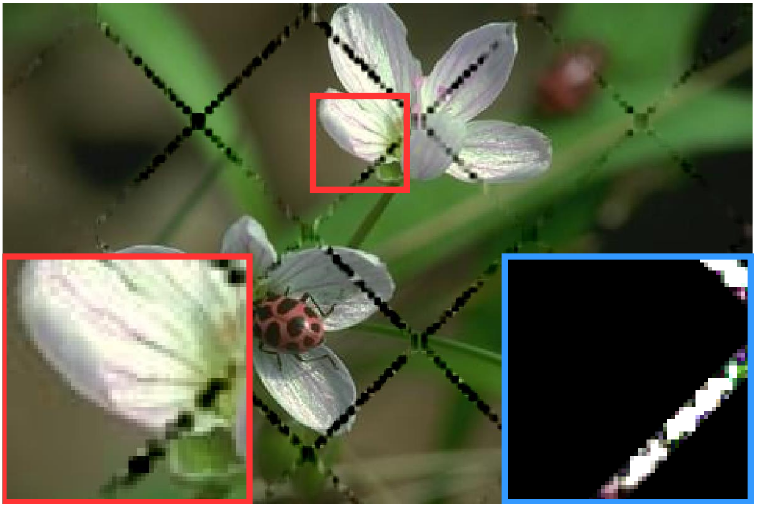} &
		\includegraphics[width=0.7in]{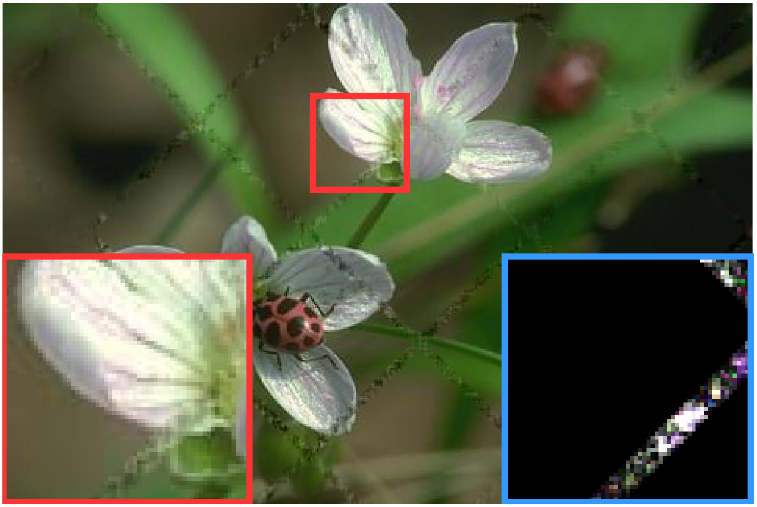} &
		\includegraphics[width=0.7in]{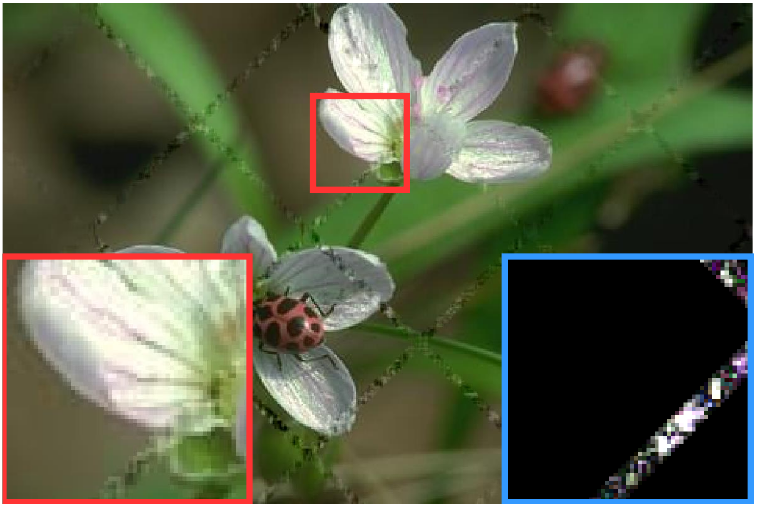} &
		\includegraphics[width=0.7in]{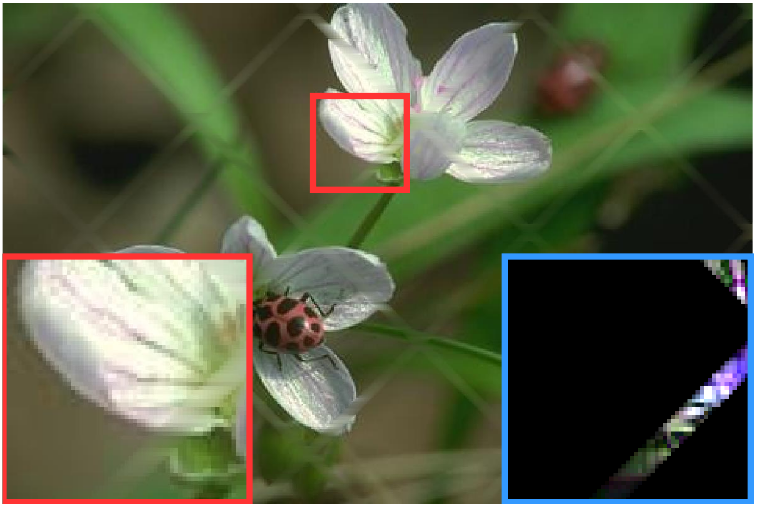} &
		\includegraphics[width=0.7in]{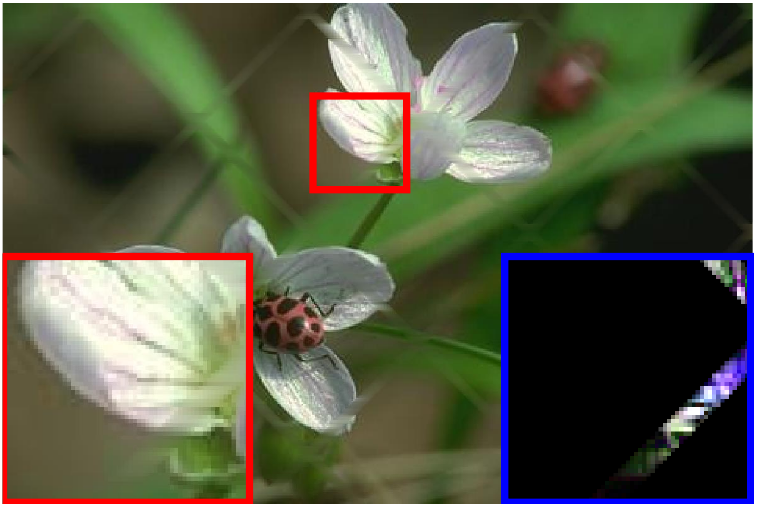} &
		\includegraphics[width=0.7in]{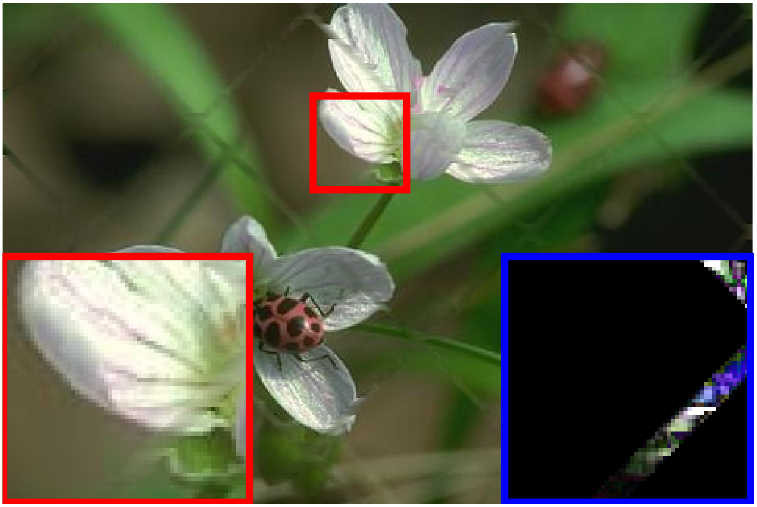} &
		\includegraphics[width=0.7in]{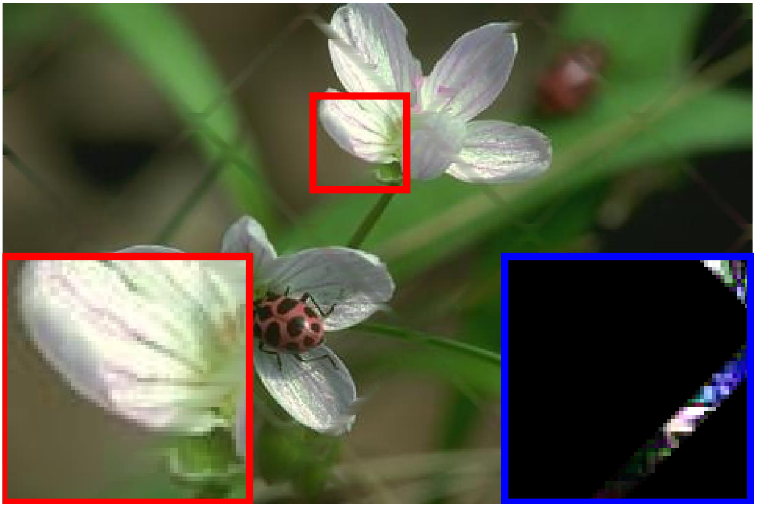} &
		\includegraphics[width=0.7in]{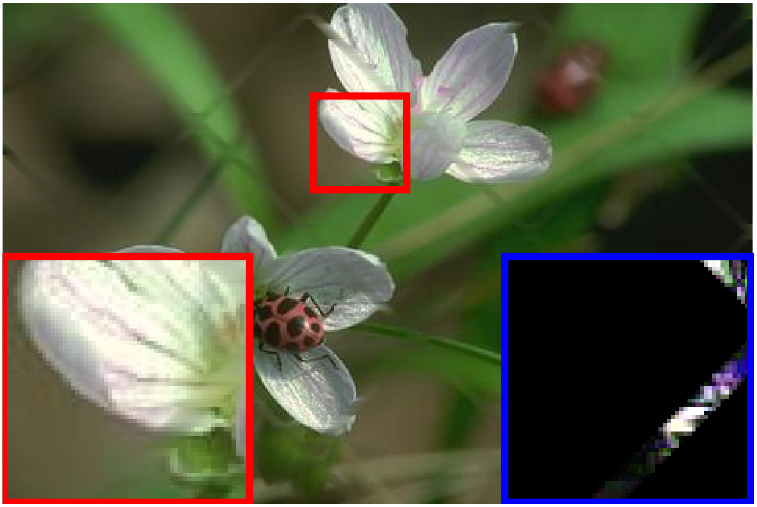} \\
		
		\includegraphics[width=0.7in]{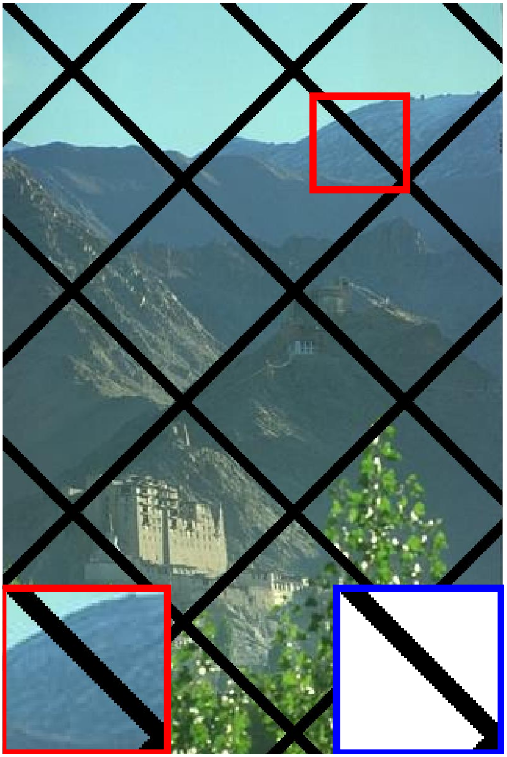} &
		\includegraphics[width=0.7in]{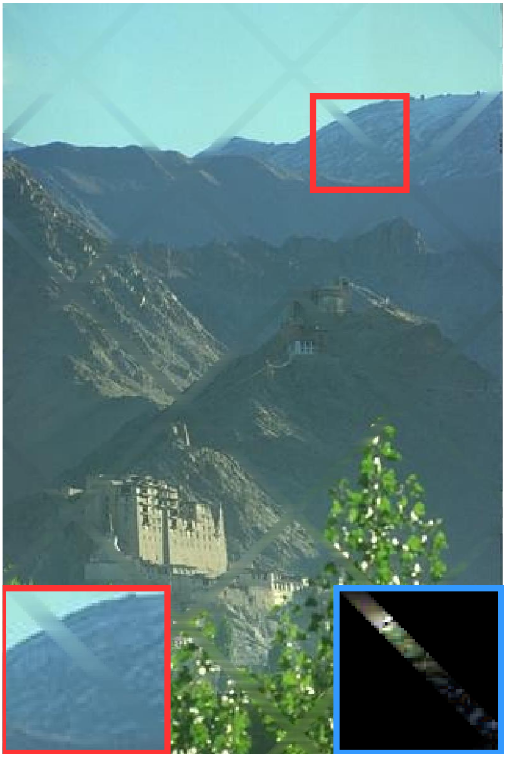} &
		\includegraphics[width=0.7in]{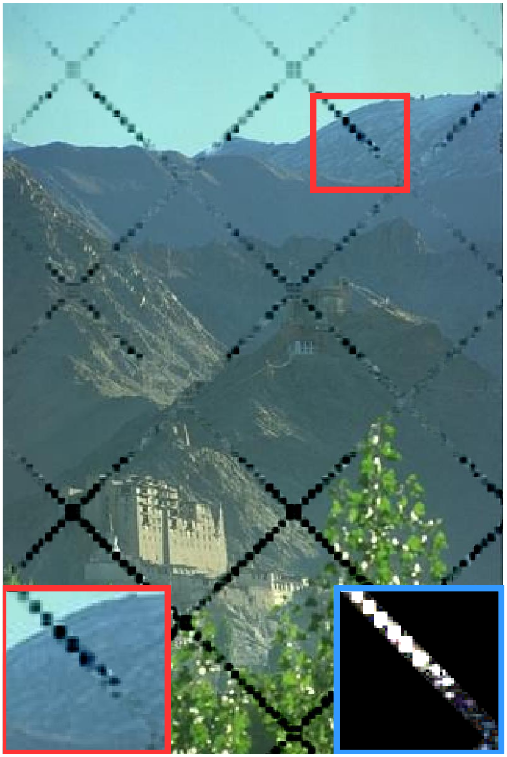} &
		\includegraphics[width=0.7in]{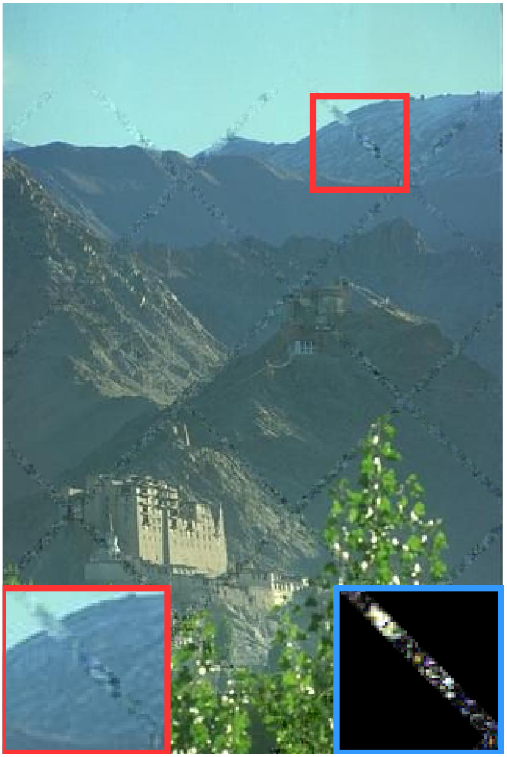} &
		\includegraphics[width=0.7in]{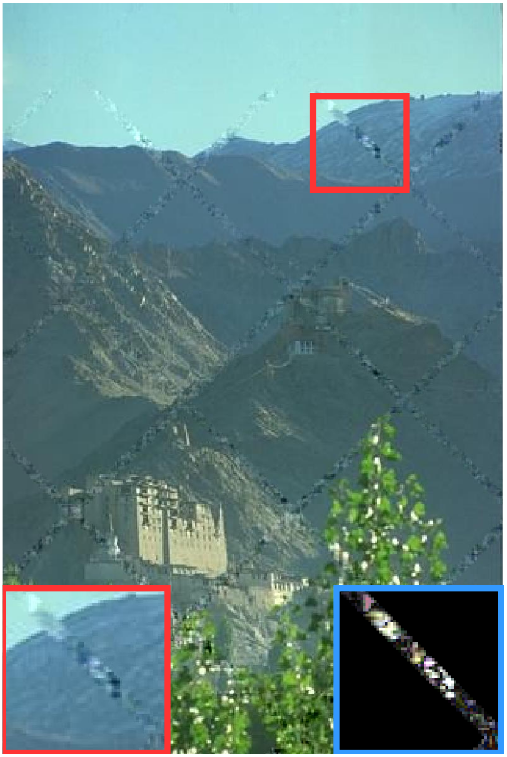} &
		\includegraphics[width=0.7in]{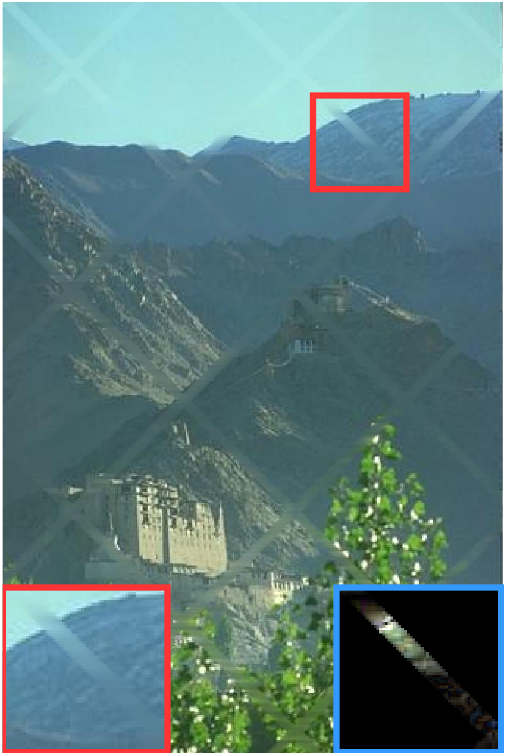} &
		\includegraphics[width=0.7in]{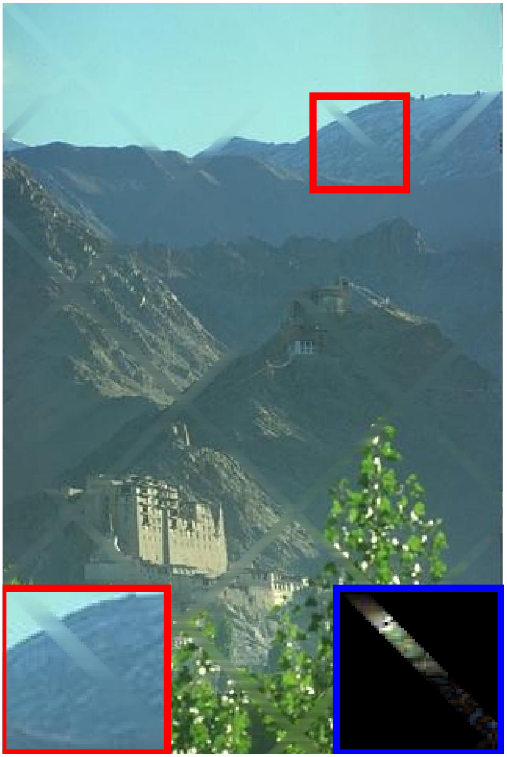} &
		\includegraphics[width=0.7in]{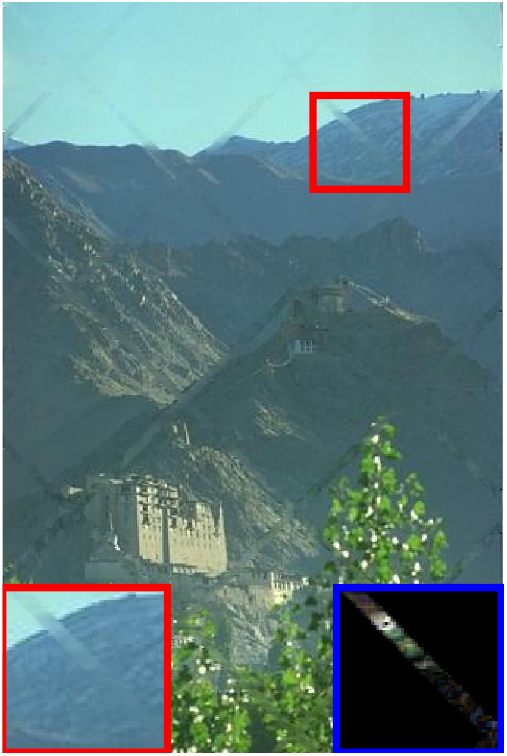} &
		\includegraphics[width=0.7in]{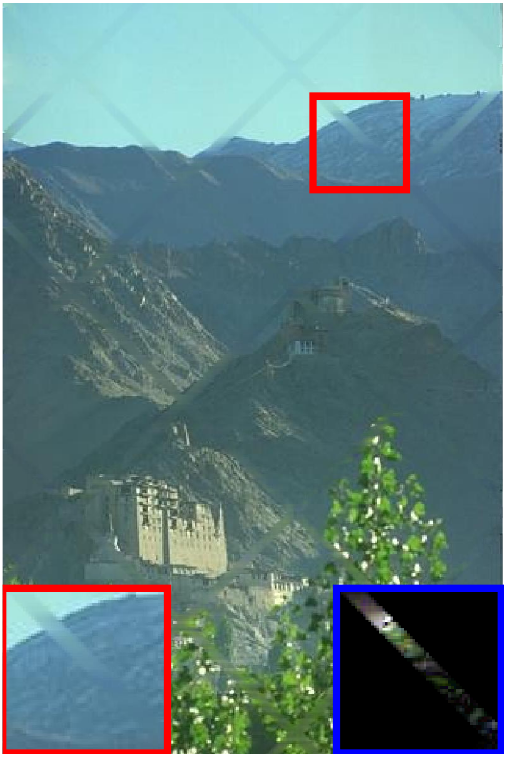} &
		\includegraphics[width=0.7in]{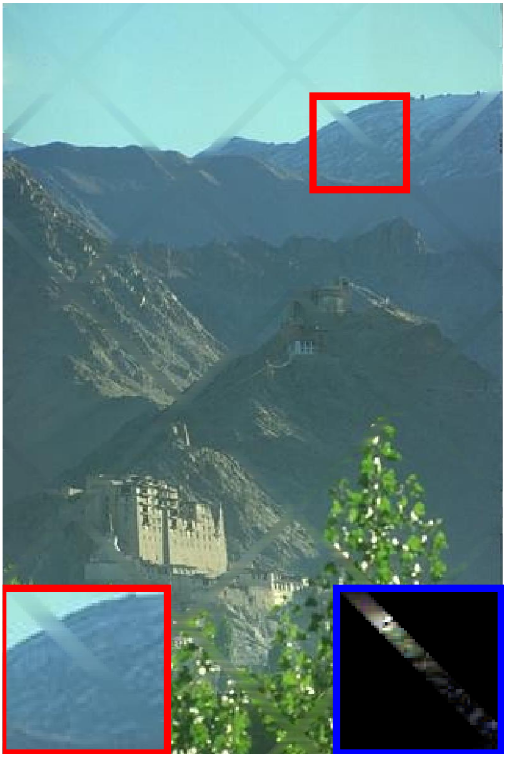} \\
		
		\includegraphics[width=0.7in]{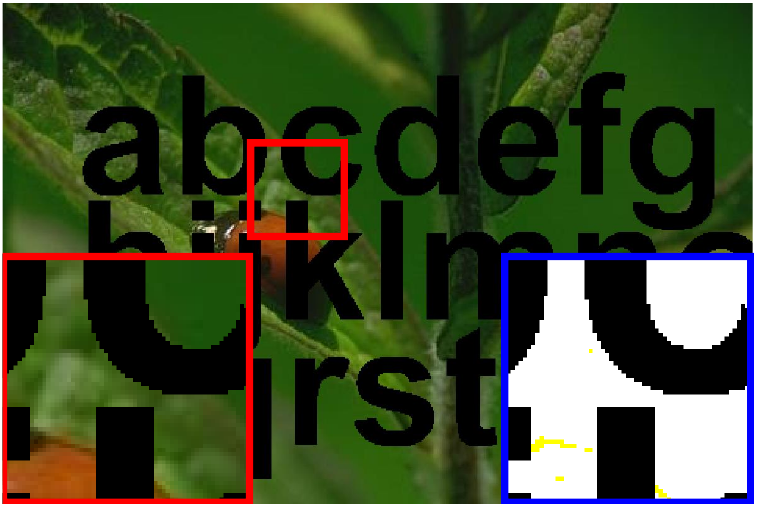} &
		\includegraphics[width=0.7in]{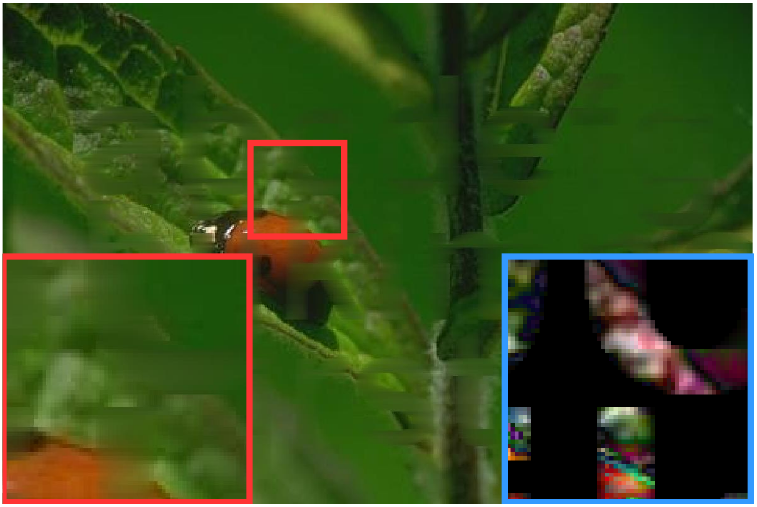} &
		\includegraphics[width=0.7in]{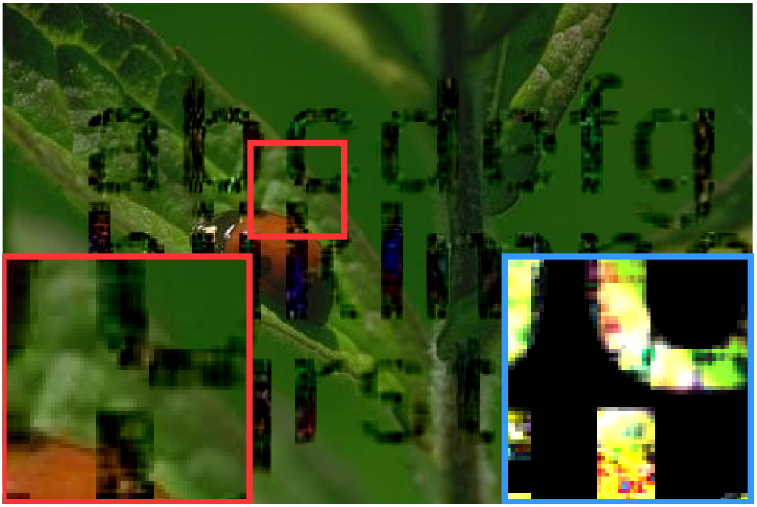} &
		\includegraphics[width=0.7in]{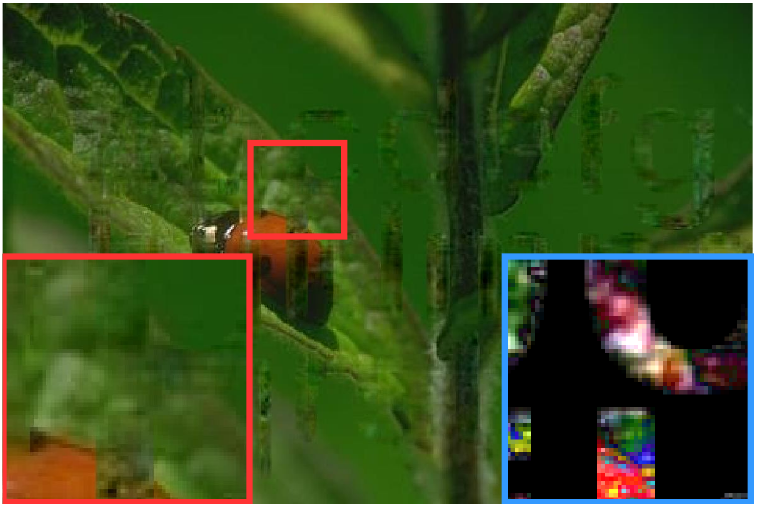} &
		\includegraphics[width=0.7in]{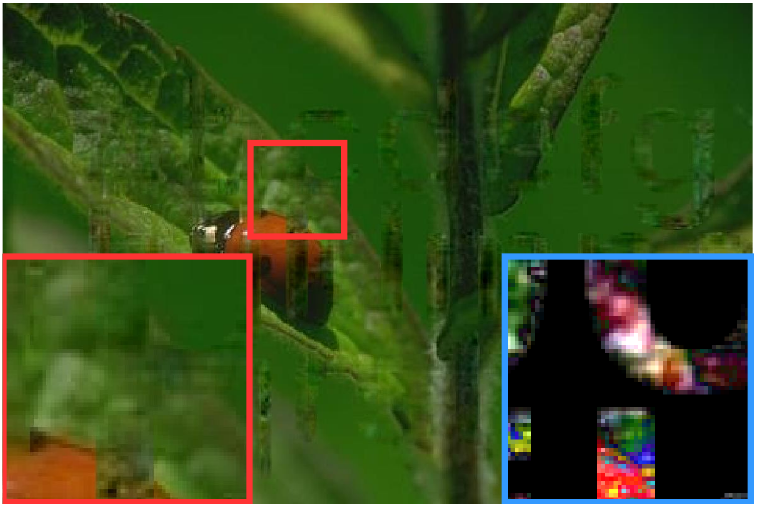} &
		\includegraphics[width=0.7in]{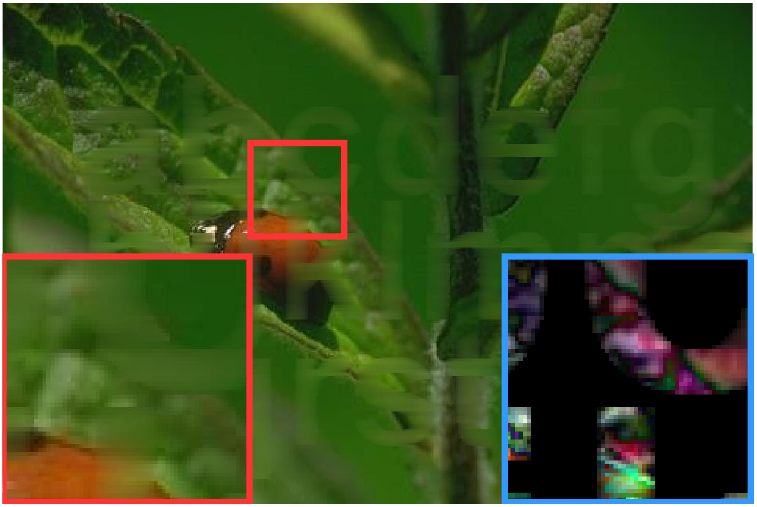} &
		\includegraphics[width=0.7in]{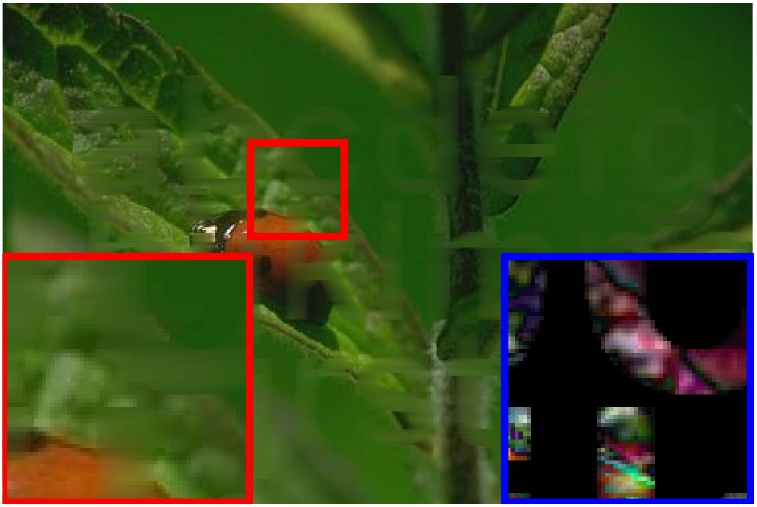} &
		\includegraphics[width=0.7in]{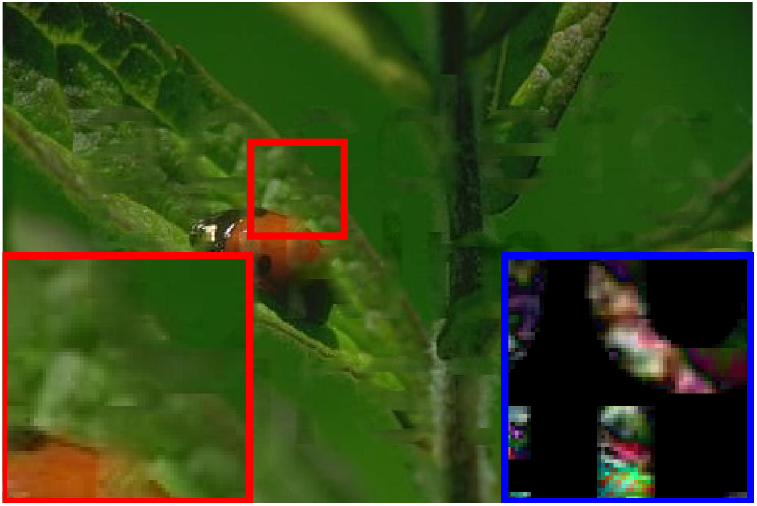} &
		\includegraphics[width=0.7in]{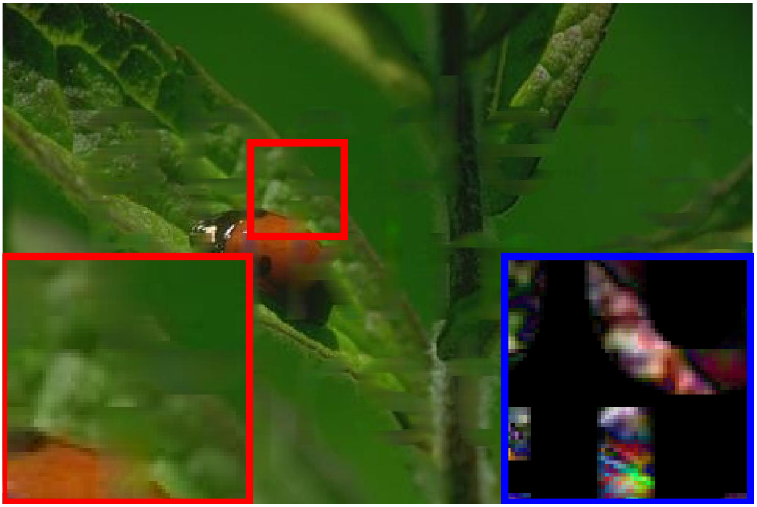} &
		\includegraphics[width=0.7in]{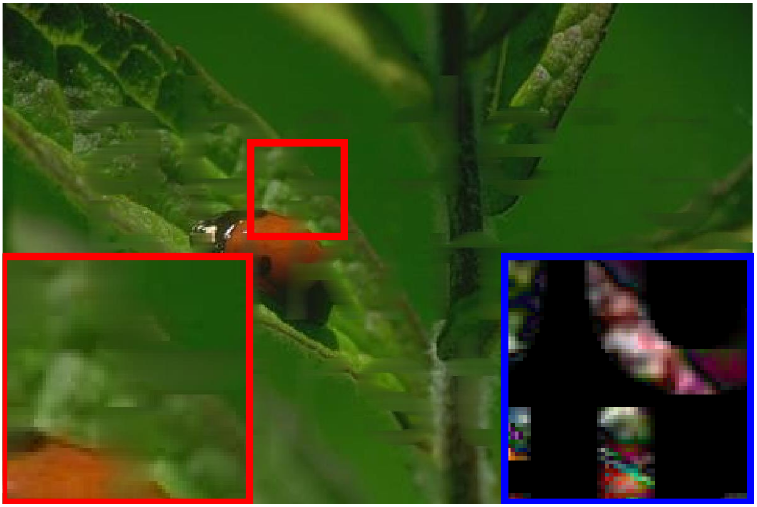} \\
		
		\includegraphics[width=0.7in]{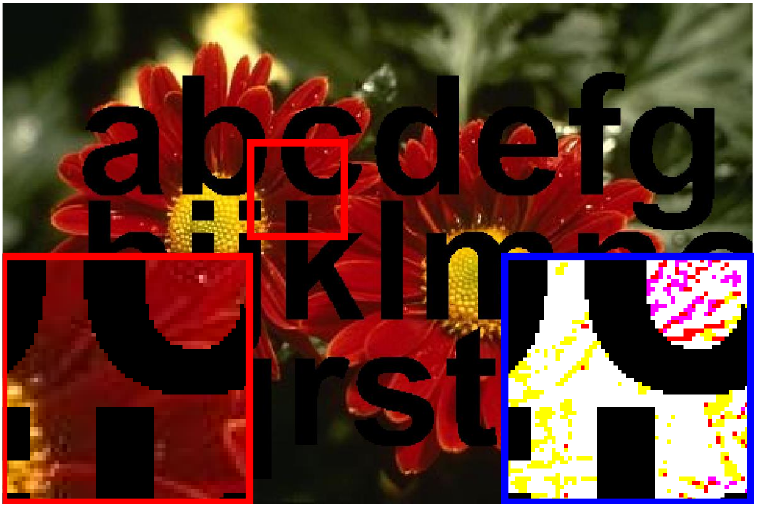} &
		\includegraphics[width=0.7in]{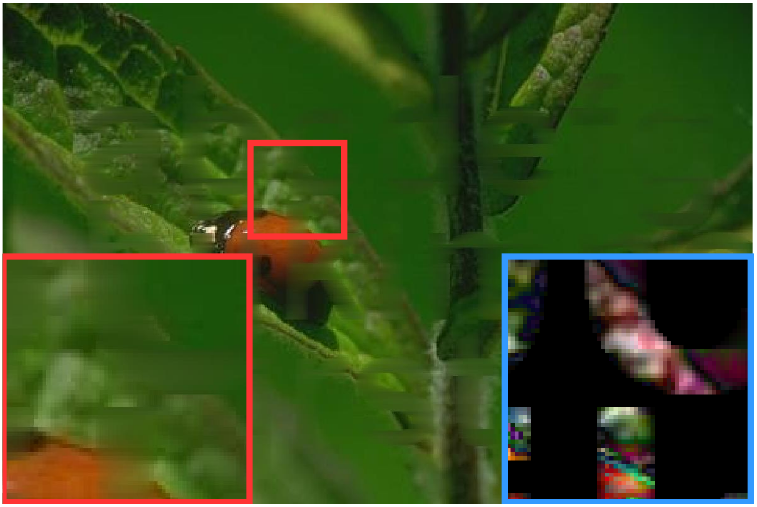} &
		\includegraphics[width=0.7in]{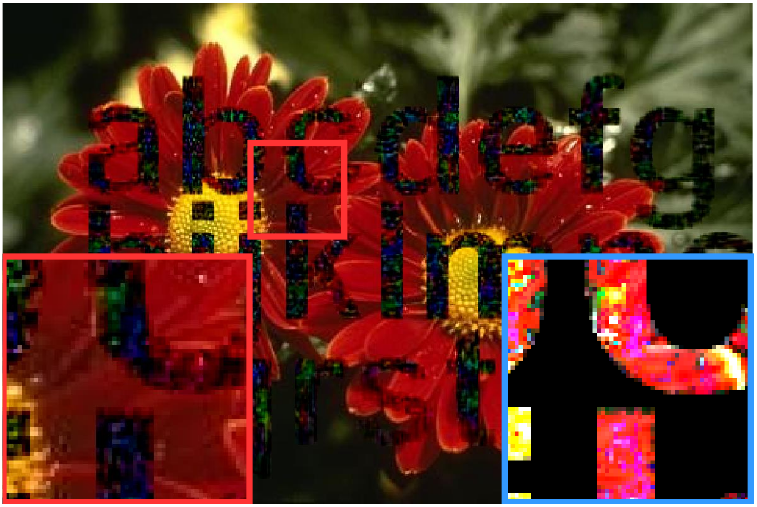} &
		\includegraphics[width=0.7in]{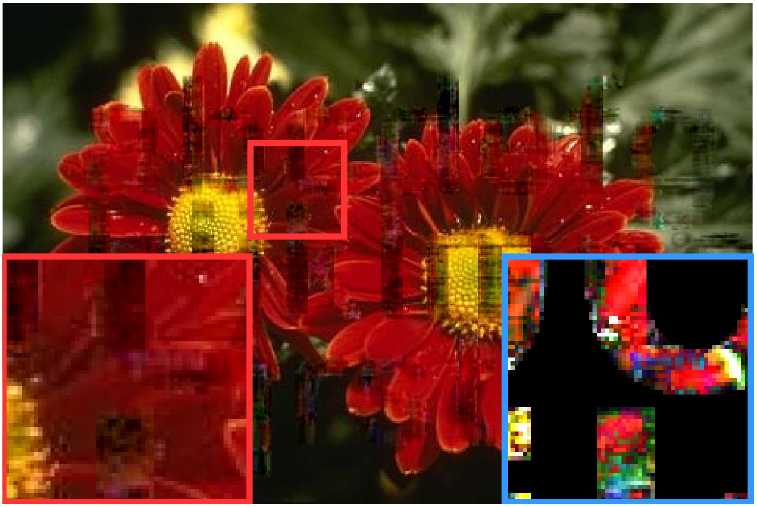} &
		\includegraphics[width=0.7in]{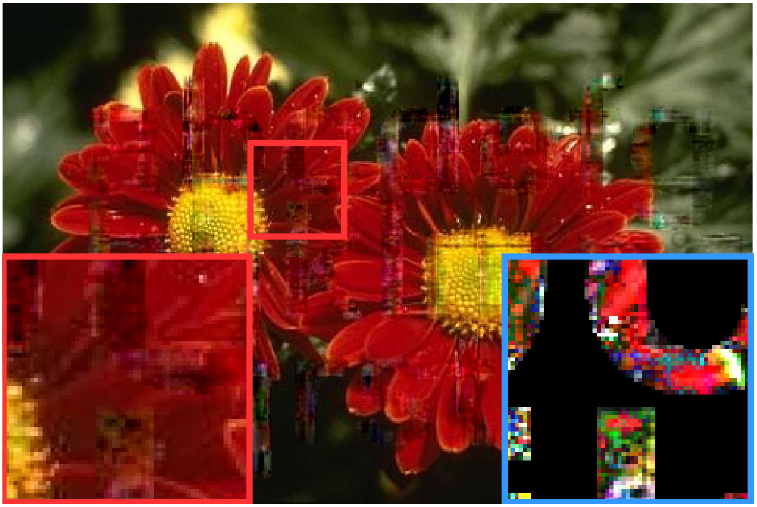} &
		\includegraphics[width=0.7in]{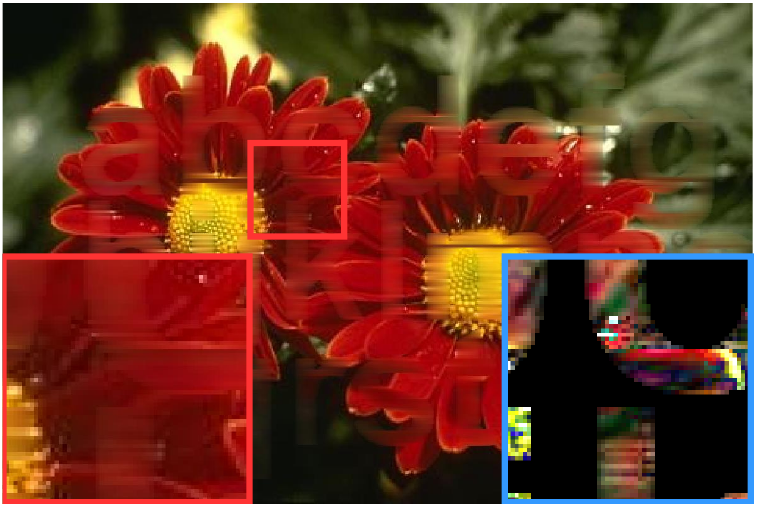} &
		\includegraphics[width=0.7in]{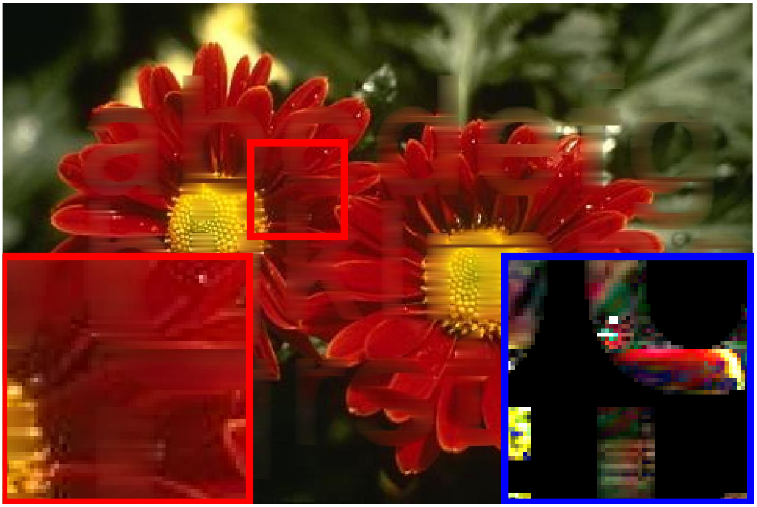} &
		\includegraphics[width=0.7in]{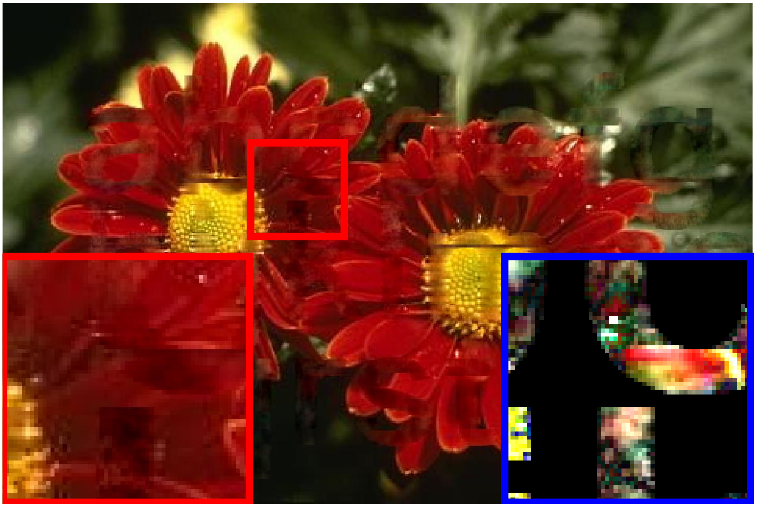} &
		\includegraphics[width=0.7in]{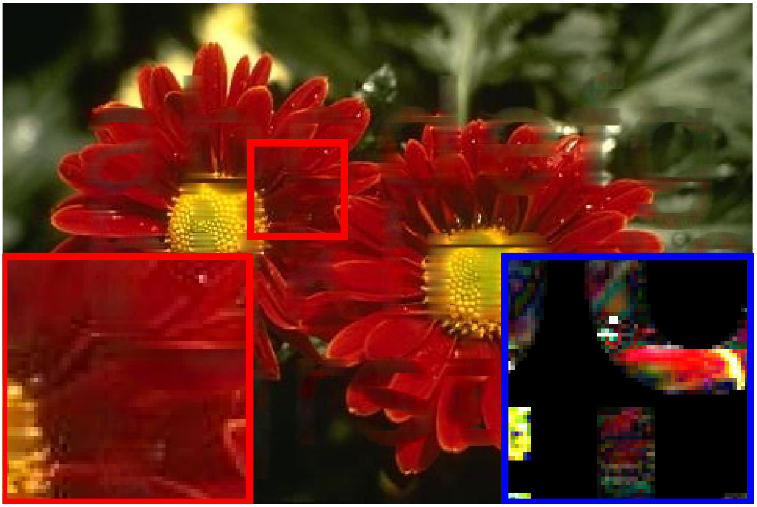} &
		\includegraphics[width=0.7in]{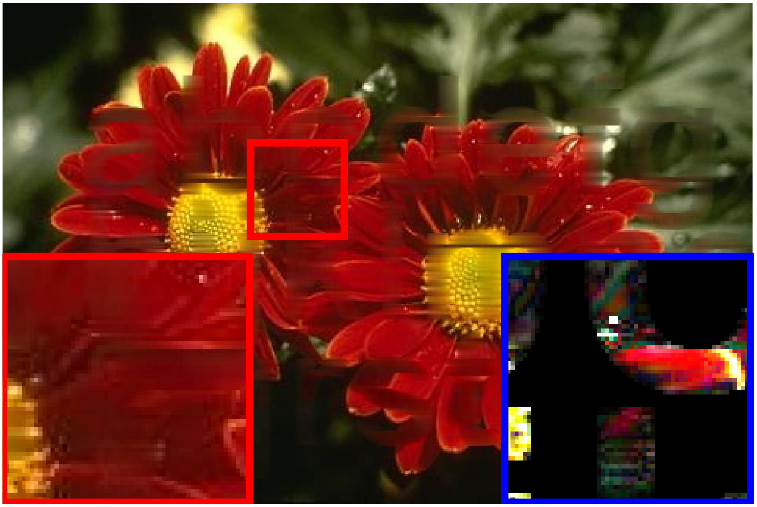} \\
		
		\includegraphics[width=0.7in]{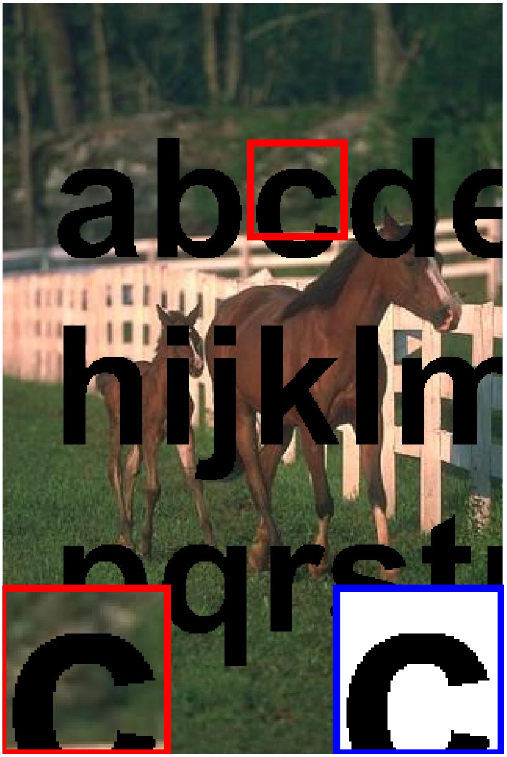} &
		\includegraphics[width=0.7in]{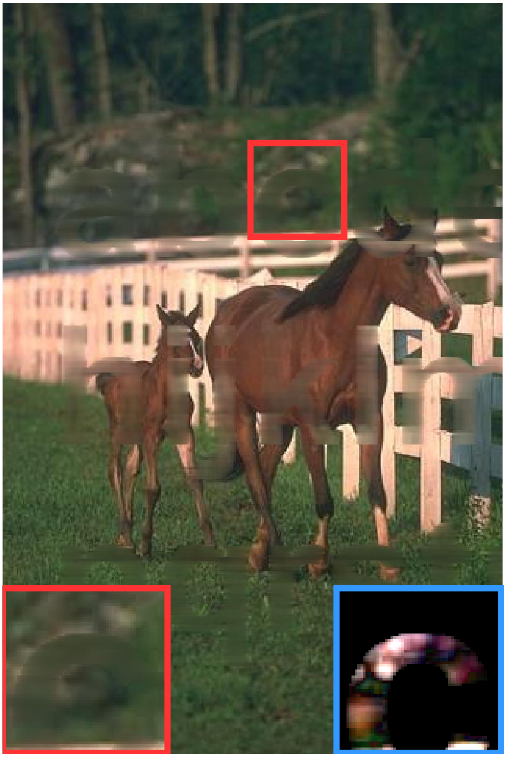} &
		\includegraphics[width=0.7in]{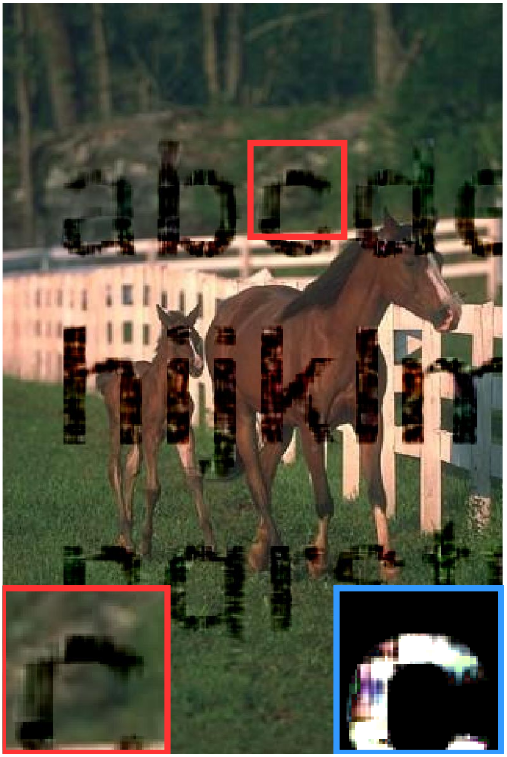} &
		\includegraphics[width=0.7in]{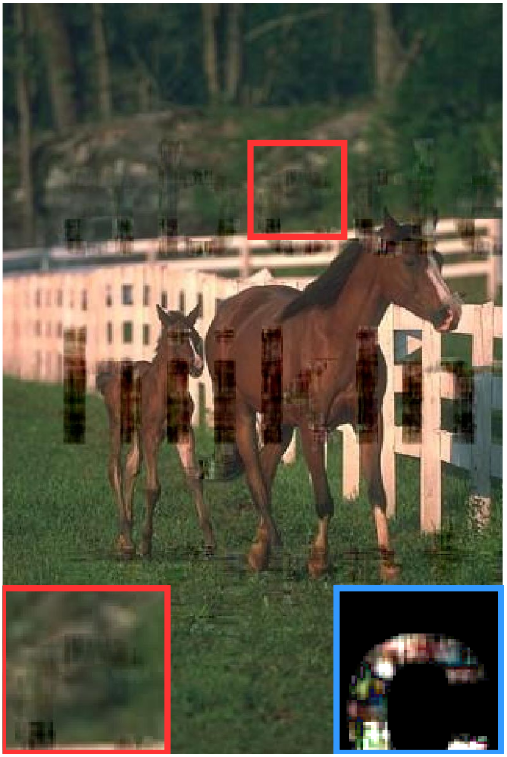} &
		\includegraphics[width=0.7in]{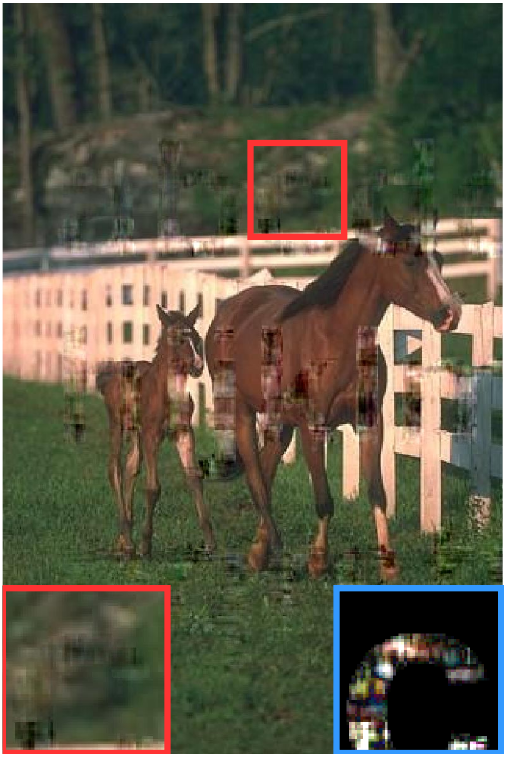} &
		\includegraphics[width=0.7in]{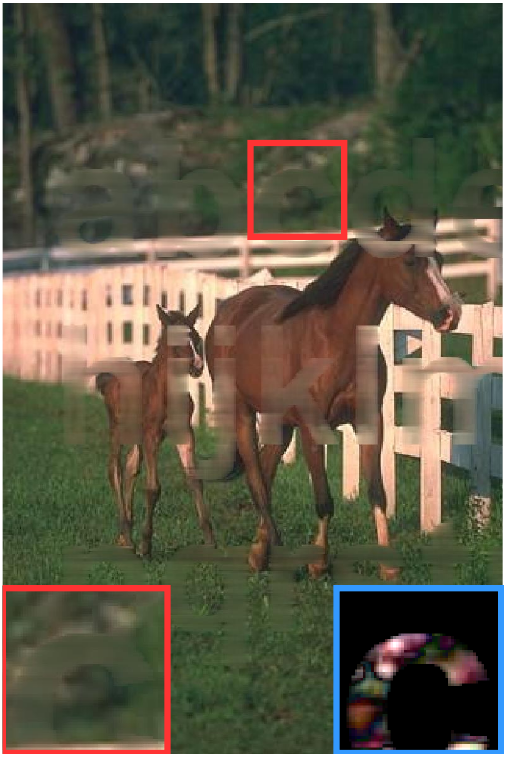} &
		\includegraphics[width=0.7in]{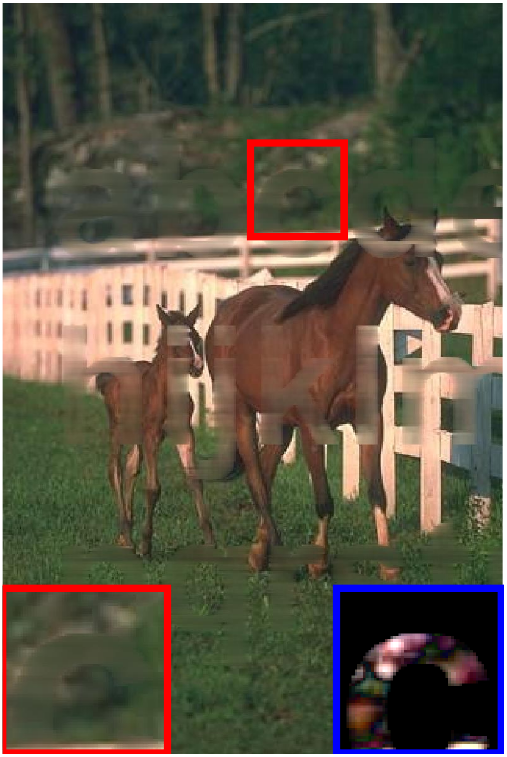} &
		\includegraphics[width=0.7in]{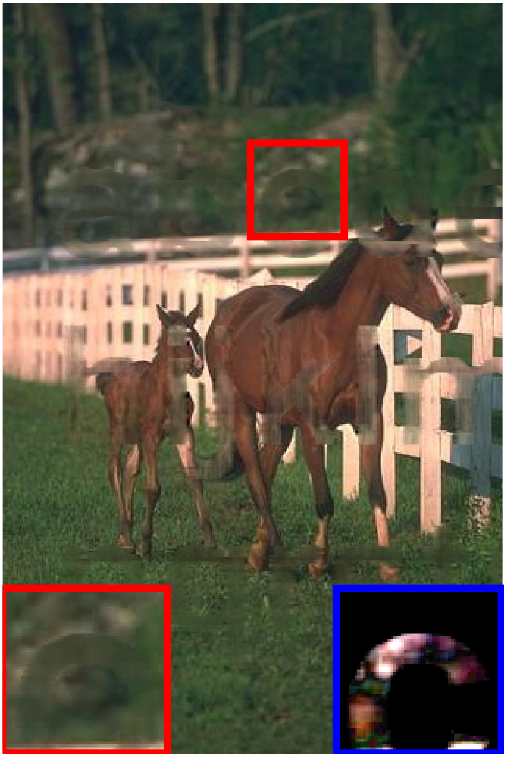} &
		\includegraphics[width=0.7in]{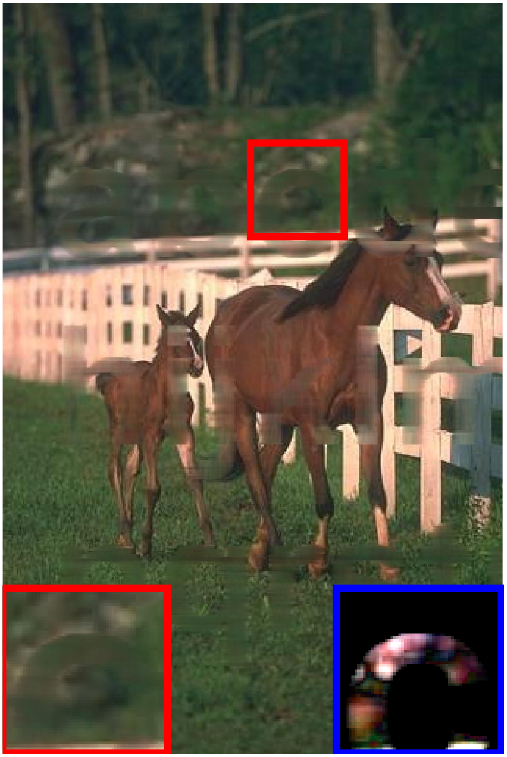} &
		\includegraphics[width=0.7in]{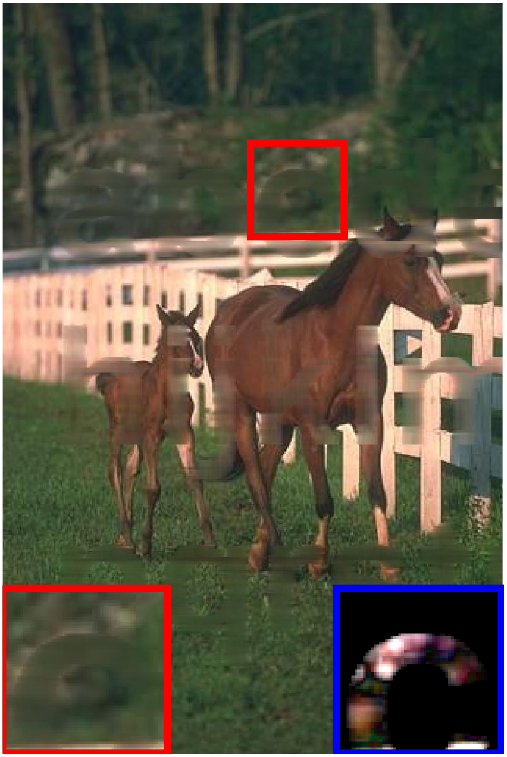} \\
		[+2mm]
		Observed   &
		TNN  & PSTNN &
		IR-t-TNN & W-t-TNN
		&TNF& TNF(ours) &
		TNK(k=1) &
		TNK(k=2)
		&  TNK(k=3)
	\end{tabular}
	\vspace{-0.15cm}
	\caption{Comparison of recovery performance for six color images under grid and text mask corruption. From left to right: observed image, TNN recovery, PSTNN recovery, IR-t-TNN recovery, W-t-TNN recovery, TNF recovery, and our proposed TNF and TNK recovery methods.}
	\label{fig:color_images2}
	\vspace{-0.15cm}
\end{figure*}

\subsubsection{\textbf{Multispectral Image Inpainting}}
We further evaluated the completion performance of the proposed method on multispectral images (MSIs). Since strong redundancy exists among spectral image slices, we assessed the proposed method under a low sampling rate (SR) of 10\%. Similarly, the DFT transform, which achieved the best performance, was adopted in all experiments. As the third dimension of multispectral images is significantly larger than that of color images, the range of parameter $k$ in TNK was correspondingly expanded. Therefore, we selected $k=20,25,31$ for comparative experiments. Table \ref{tab:MSI_images} reports the quantitative results of different methods for MSI recovery under different sampling rates. It can be observed that the proposed method achieves the highest-quality reconstruction results across various MSI datasets.
For visual comparison, Figure \ref{fig:MSI_images} presents one band image and one mode-3 tube structure of the reconstructed MSIs by different methods under SR = 10\%. From the visual comparisons, the recovered mode-3 tube structure obtained by the proposed method is noticeably closer to the original data than those produced by the competing methods.
\begin{table}[]
	\renewcommand{\arraystretch}{1.08}
	\setlength\tabcolsep{0.5pt}
	\centering
	\footnotesize
	\caption{Comparison of PSNR, SSIM and FSIM for MSI image inpainting under random missing entries.}
	\label{tab:MSI_images}
	\begin{tabular}{ccccccccccc}
		\hline
		\begin{tabular}[c]{@{}c@{}}MSI\\Data\end{tabular} & Metrics & TNN     & PSTNN   & \begin{tabular}[c]{@{}c@{}}IR-\\t-TNN\end{tabular} & \begin{tabular}[c]{@{}c@{}}W-\\t-TNN\end{tabular} & TNF     & \begin{tabular}[c]{@{}c@{}}TNF\\(ours)\end{tabular} & \begin{tabular}[c]{@{}c@{}}TNK\\(k=20)\end{tabular} & \begin{tabular}[c]{@{}c@{}}TNK\\(k=25)\end{tabular} & \begin{tabular}[c]{@{}c@{}}TNK\\(k=31)\end{tabular} \\ \hline \hline
		& PSNR    & 43.31  & 28.99 & 29.67                                             & 38.78                                            & 43.40 & 43.36                                             & {\color[HTML]{FE0000} \textbf{45.21}}              & {\color[HTML]{3531FF} \textbf{45.08}}              & 44.94                                              \\
		& SSIM    & 0.98  & 0.84  & 0.84                                              & 0.96                                             & 0.98  & 0.98                                              & {\color[HTML]{FE0000} \textbf{0.98}}               & {\color[HTML]{3531FF} \textbf{0.98}}               & 0.98                                               \\
		\multirow{-3}{*}{balloons}                         & FSIM    & 0.99  & 0.93  & 0.94                                              & 0.98                                             & 0.99  & 0.99                                              & {\color[HTML]{FE0000} \textbf{0.99}}               & {\color[HTML]{3531FF} \textbf{0.99}}               & {\color[HTML]{3531FF} \textbf{0.99}}               \\ \hline
		& PSNR    & 37.52 & 23.67 & 16.61                                             & 32.12                                            & 37.64 & 37.59                                             & {\color[HTML]{FE0000} \textbf{39.29}}              & {\color[HTML]{3531FF} \textbf{39.15}}              & 39.05                                              \\
		& SSIM    & 0.97  & 0.78  & 0.61                                              & 0.89                                             & 0.97  & 0.97                                              & {\color[HTML]{FE0000} \textbf{0.97}}               & {\color[HTML]{3531FF} \textbf{0.97}}               & 0.97                                                \\
		\multirow{-3}{*}{toy}                              & FSIM    & 0.99  & 0.90  & 0.83                                              & 0.96                                             & 0.99  & 0.99                                              & {\color[HTML]{FE0000} \textbf{0.99}}               & {\color[HTML]{3531FF} \textbf{0.99}}               & 0.99                                               \\ \hline
		& PSNR    & 33.82 & 25.37 & 23.74                                             & 27.97                                              & 34.02 & 33.96                                             & {\color[HTML]{FE0000} \textbf{34.74}}              & {\color[HTML]{3531FF} \textbf{34.63}}              & 34.53                                              \\
		& SSIM    & 0.93  & 0.64  & 0.53                                              & 0.73                                             & 0.93  & 0.93                                              & {\color[HTML]{FE0000} \textbf{0.93}}               & {\color[HTML]{3531FF} \textbf{0.93}}               & 0.93                                               \\
		\multirow{-3}{*}{cloth}                            & FSIM    & 0.97  & 0.91  & 0.88                                              & 0.92                                             & 0.98  & 0.98                                              & {\color[HTML]{FE0000} \textbf{0.98}}               & {\color[HTML]{3531FF} \textbf{0.98}}               & 0.98                                               \\ \hline
		& PSNR    & 35.96 & 23.19 & 22.16                                             & 32.69                                            & 36.26 & 36.19                                             & {\color[HTML]{FE0000} \textbf{37.27}}              & {\color[HTML]{3531FF} \textbf{37.10}}              & 37.01                                              \\
		& SSIM    & 0.95  & 0.74  & 0.66                                              & 0.88                                             & 0.95  & 0.95                                              & {\color[HTML]{FE0000} \textbf{0.95}}               & {\color[HTML]{3531FF} \textbf{0.94}}               & 0.94                                               \\
		\multirow{-3}{*}{feathers}                         & FSIM    & 0.98   & 0.90  & 0.89                                              & 0.96                                             & 0.98  & 0.98                                              & {\color[HTML]{FE0000} \textbf{0.98}}               & {\color[HTML]{3531FF} \textbf{0.98}}               & 0.98                                               \\ \hline
		& PSNR    & 40.96 & 25.53 & 29.31                                             & 34.40                                            & 41.20 & 41.16                                             & {\color[HTML]{FE0000} \textbf{42.80}}              & {\color[HTML]{3531FF} \textbf{42.67}}               & 42.52                                              \\
		& SSIM    & 0.97  & 0.74  & 0.76                                              & 0.87                                             & 0.97  & 0.97                                              & {\color[HTML]{FE0000} \textbf{0.98}}               & {\color[HTML]{3531FF} \textbf{0.98}}                & 0.98                                               \\
		\multirow{-3}{*}{flowers}                          & FSIM    & 0.99  & 0.89  & 0.92                                              & 0.95                                             & 0.99  & 0.99                                              & {\color[HTML]{FE0000} \textbf{0.99}}               & {\color[HTML]{3531FF} \textbf{0.99}}                & 0.99                                               \\ \hline
	\end{tabular}
\begin{flushleft}
	\footnotesize
	The {\color[HTML]{FE0000} \textbf{best}} and {\color[HTML]{3531FF} \textbf{second-best}} values are highlighted, determined from the original unrounded values.
\end{flushleft}
\end{table}
\begin{figure*}[!htbp]
	\centering
	\includegraphics[width=1\textwidth]{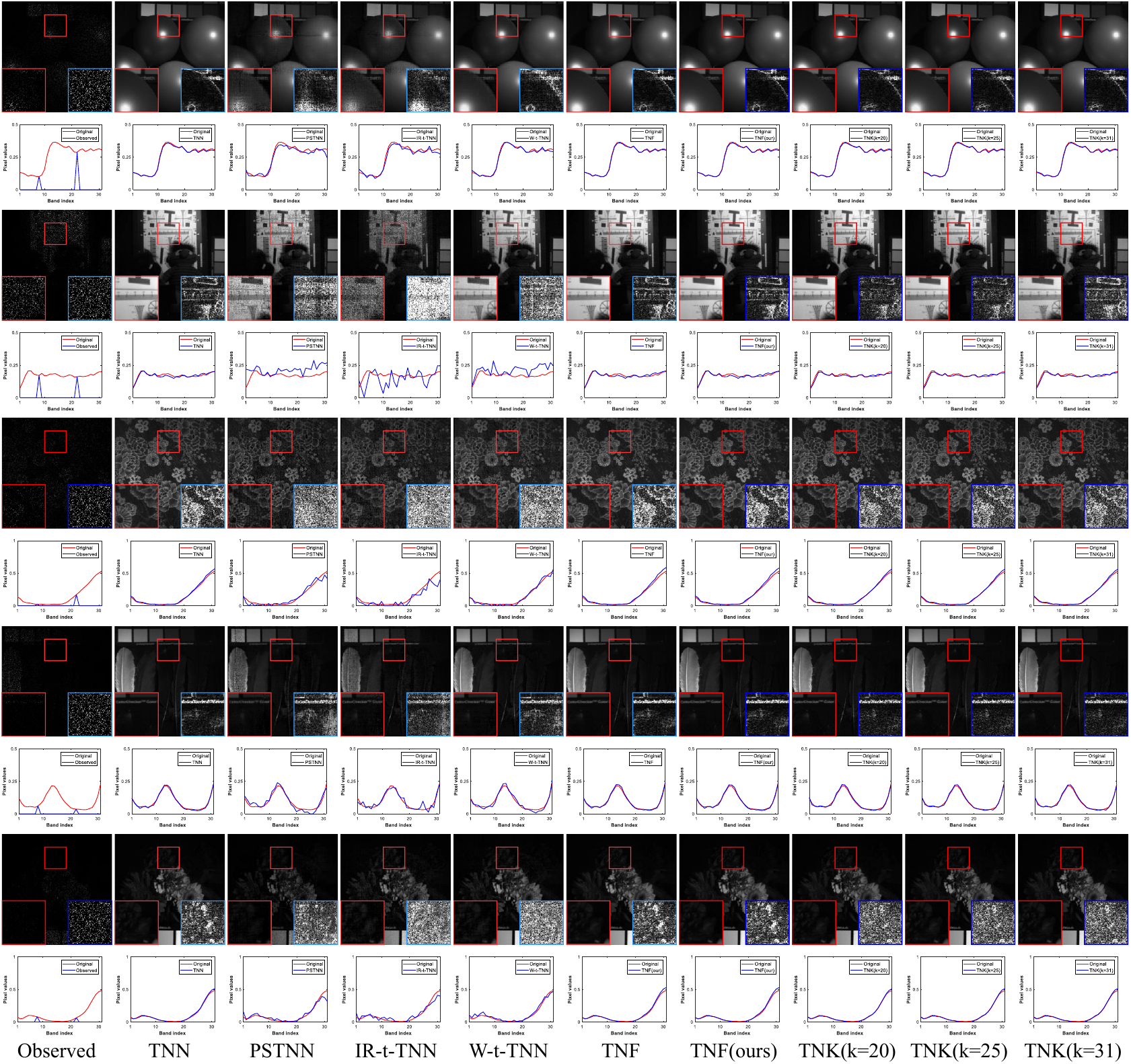}
	\caption{Comparison of recovery performance for five multispectral images under 10\% random sampling. From left to right: observed image, TNN recovery, PSTNN recovery, IR-t-TNN recovery, W-t-TNN recovery, TNF recovery, and our proposed TNF and TNK recovery methods.
	}
	\label{fig:MSI_images}
	\vspace{-0.15cm}
\end{figure*}

\subsubsection{\textbf{Grayscale Video Inpainting}}
We evaluated the performance of the proposed method on four grayscale video datasets. The sampling rate (SR) was set to 10\%. In the experiments, the parameter $k$ was chosen as 10, 20, and 30, respectively, while the DFT transform, which achieved the best performance, was adopted as the invertible linear transform. Table \ref{tab:video} reports the quantitative evaluation results of the recovered videos. The results demonstrate that the proposed method consistently outperforms the competing methods across all evaluation metrics, achieving significantly higher PSNR values.
For visual comparison, Figure \ref{fig:video} presents one frame image and one mode-3 tube structure reconstructed by different methods. It can be observed that the mode-3 tube structures recovered by the proposed method are the closest to the original data.
\begin{table}[]
	\renewcommand{\arraystretch}{1.08}
	\setlength\tabcolsep{0.5pt}
	\centering
	\footnotesize
	\caption{Comparison of PSNR, SSIM and SIM for video inpainting under random missing entries.}
	\label{tab:video}
	\begin{tabular}{ccccccccccc}
		\hline
		\begin{tabular}[c]{@{}c@{}}Video\\Data\end{tabular}                        & Metrics & TNN                                    & PSTNN   & \begin{tabular}[c]{@{}c@{}}IR-\\t-TNN\end{tabular} & \begin{tabular}[c]{@{}c@{}}W-\\t-TNN\end{tabular} & TNF                                     & \begin{tabular}[c]{@{}c@{}}TNF\\(ours)\end{tabular} & \begin{tabular}[c]{@{}c@{}}TNK\\(k=10)\end{tabular} & \begin{tabular}[c]{@{}c@{}}TNK\\(k=20)\end{tabular} & \begin{tabular}[c]{@{}c@{}}TNK\\(k=30)\end{tabular} \\ \hline \hline
		& PSNR    & 30.83                                & 26.59  & 15.63                                             & 28.71                                            & 30.91                                 & 30.93                                             & {\color[HTML]{FE0000} \textbf{32.06}}              & {\color[HTML]{3531FF} \textbf{31.83}}              & 31.56                                              \\
		& SSIM    & 0.90                                 & 0.81  & 0.47                                              & 0.88                                             & {\color[HTML]{3531FF} \textbf{0.90}}  & {\color[HTML]{FE0000} \textbf{0.90}}              & 0.90                                               & 0.89                                                & 0.89                                               \\
		\multirow{-3}{*}{akiyo}                                                     & FSIM    & {\color[HTML]{3531FF} \textbf{0.95}} & 0.91  & 0.74                                              & 0.93                                             & {\color[HTML]{FE0000} \textbf{0.95}}  & {\color[HTML]{FE0000} \textbf{0.95}}              & 0.95                                               & 0.95                                               & 0.95                                               \\ \hline
		& PSNR    & 26.46                                 & 21.64 & 13.32                                             & 24.72                                            & 26.62                                 & 26.60                                              & {\color[HTML]{FE0000} \textbf{27.10}}              & {\color[HTML]{3531FF} \textbf{26.76}}              & 26.51                                              \\
		& SSIM    & 0.80                                 & 0.62  & 0.30                                              & 0.69                                             & {\color[HTML]{FE0000} \textbf{0.81}}  & {\color[HTML]{3531FF} \textbf{0.80}}              & 0.76                                               & 0.75                                               & 0.74                                               \\
		\multirow{-3}{*}{carphone}                                                  & FSIM    & 0.89                                 & 0.80  & 0.61                                              & 0.84                                             & {\color[HTML]{FE0000} \textbf{0.90}}  & {\color[HTML]{3531FF} \textbf{0.90}}              & 0.87                                               & 0.87                                               & 0.87                                               \\ \hline
		& PSNR    & 24.08                                & 20.46 & 8.63                                              & 23.00                                            & {\color[HTML]{3531FF} \textbf{24.14}} & 24.13                                             & {\color[HTML]{FE0000} \textbf{24.33}}              & 24.10                                              & 23.97                                              \\
		& SSIM    & 0.66                                 & 0.41   & 0.07                                              & 0.54                                             & {\color[HTML]{FE0000} \textbf{0.66}}  & {\color[HTML]{3531FF} \textbf{0.66}}              & 0.59                                               & 0.59                                               & 0.59                                               \\
		\multirow{-3}{*}{foreman}                                                   & FSIM    & 0.82                                 & 0.69  & 0.46                                              & 0.76                                             & {\color[HTML]{FE0000} \textbf{0.82}}  & {\color[HTML]{3531FF} \textbf{0.82}}              & 0.80                                               & 0.80                                               & 0.80                                                  \\ \hline
		& PSNR    & 31.19                                & 27.26 & 12.15                                             & 28.63                                            & 31.26                                 & 31.25                                             & {\color[HTML]{FE0000} \textbf{32.66}}              & {\color[HTML]{3531FF} \textbf{32.41}}              & 32.15                                              \\
		& SSIM    & 0.86                                 & 0.75  & 0.16                                              & 0.79                                             & {\color[HTML]{3531FF} \textbf{0.86}}  & 0.86                                              & {\color[HTML]{FE0000} \textbf{0.86}}               & 0.86                                               & 0.85                                                \\
		\multirow{-3}{*}{\begin{tabular}[c]{@{}c@{}}mother\\ daughter\end{tabular}} & FSIM    & 0.91                                 & 0.87  & 0.58                                              & 0.89                                             & 0.91                                  & 0.91                                              & {\color[HTML]{FE0000} \textbf{0.92}}               & {\color[HTML]{3531FF} \textbf{0.92}}               & 0.92                                               \\ \hline
	\end{tabular}
\begin{flushleft}
	\footnotesize
	The {\color[HTML]{FE0000} \textbf{best}} and {\color[HTML]{3531FF} \textbf{second-best}} values are highlighted, determined from the original unrounded values.
\end{flushleft}
\end{table}

\begin{figure*}[!htbp]
	\centering
	\includegraphics[width=1\textwidth]{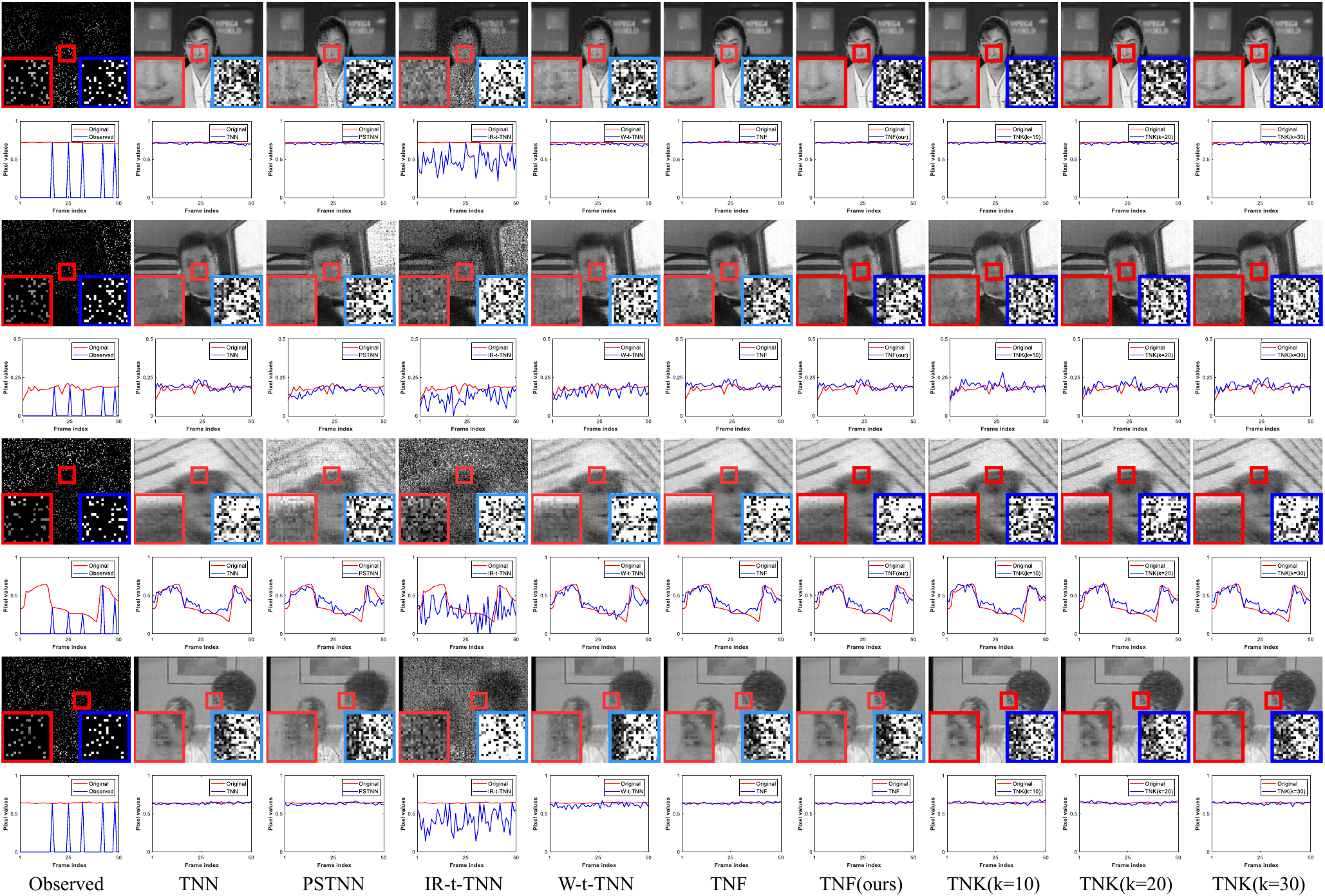}
	\caption{Comparison of recovery performance for four videos under 10\% random sampling. From left to right:  observed image, TNN recovery, PSTNN recovery, IR-t-TNN recovery, W-t-TNN recovery, TNF recovery, and our proposed TNF and TNK recovery methods.
	}
	\label{fig:video}
	\vspace{-0.15cm}
\end{figure*}

\subsubsection{\textbf{Convergence Analysis}}
We validate the convergence of the algorithm. The RSE curves for the color image recovery task with randomly missing entries are plotted in Figure \ref{fig:convergence}. It can be observed that, although our method exhibits some fluctuations during the first 20 iterations, the error eventually becomes stable after at most approximately 30 outer iterations, and the subsequent fluctuations are negligible. Therefore, the proposed algorithm demonstrates good convergence and stability.

\begin{figure}[!htbp]
	\renewcommand{\arraystretch}{0.5}
	\setlength\tabcolsep{3pt}
	\centering
	\begin{tabular}{cc }
		\centering
		\includegraphics[width=1.6in]{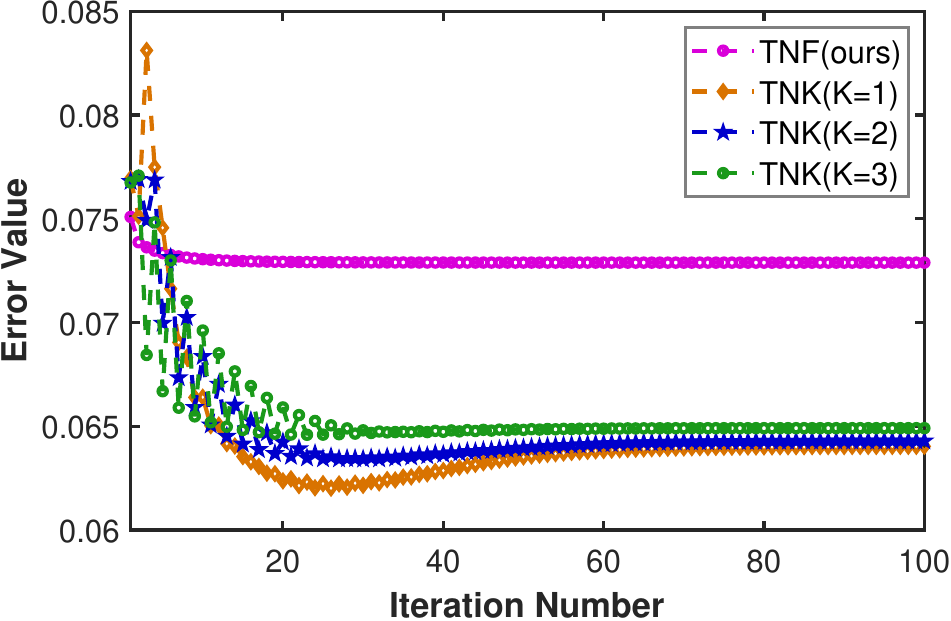}&
		\includegraphics[width=1.6in]{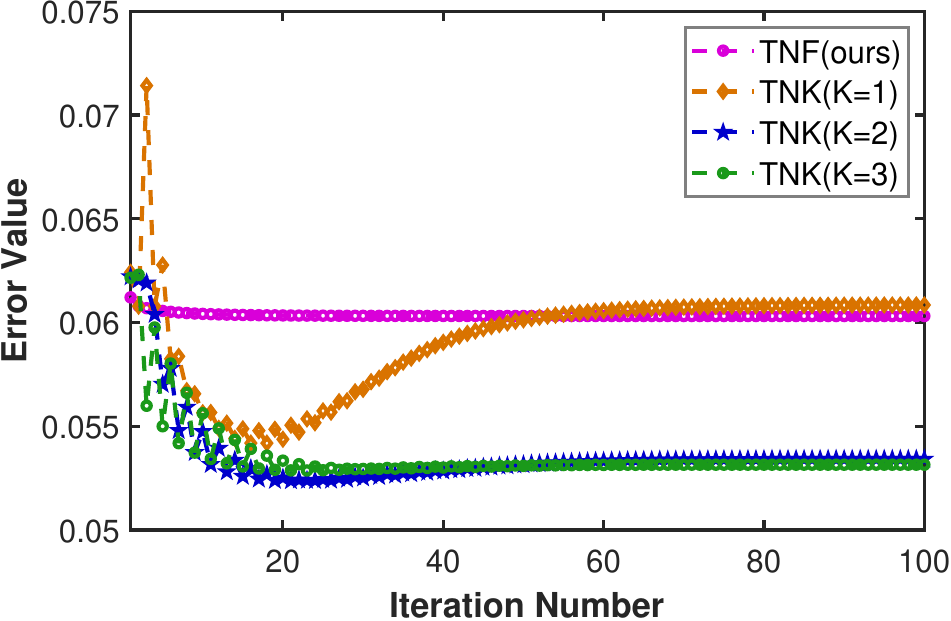}
		\\[+1mm]
		\footnotesize{(a)}  &
		\footnotesize{(b)}  
		\\[+1mm]
		\includegraphics[width=1.6in]{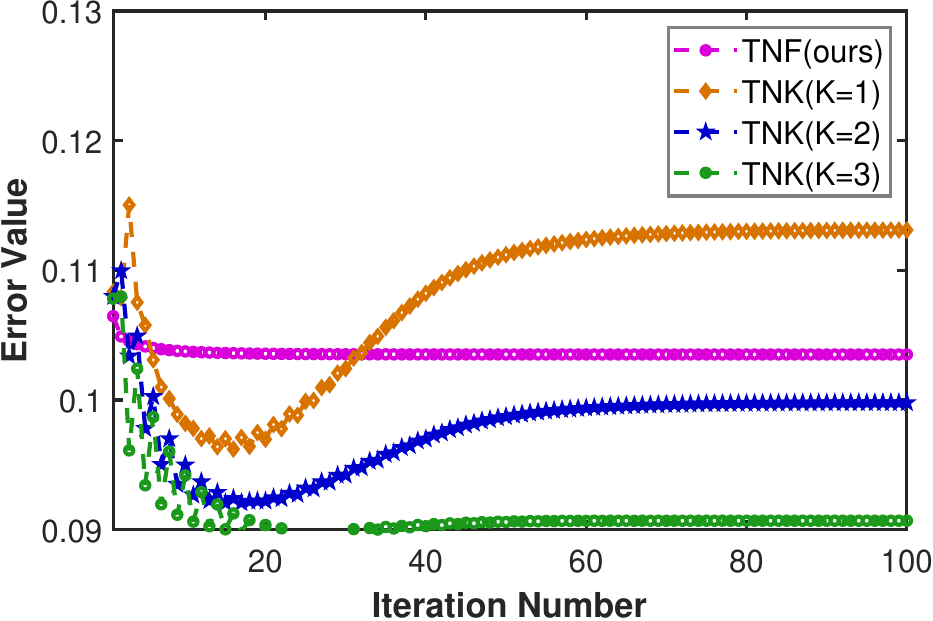}&
		\includegraphics[width=1.6in]{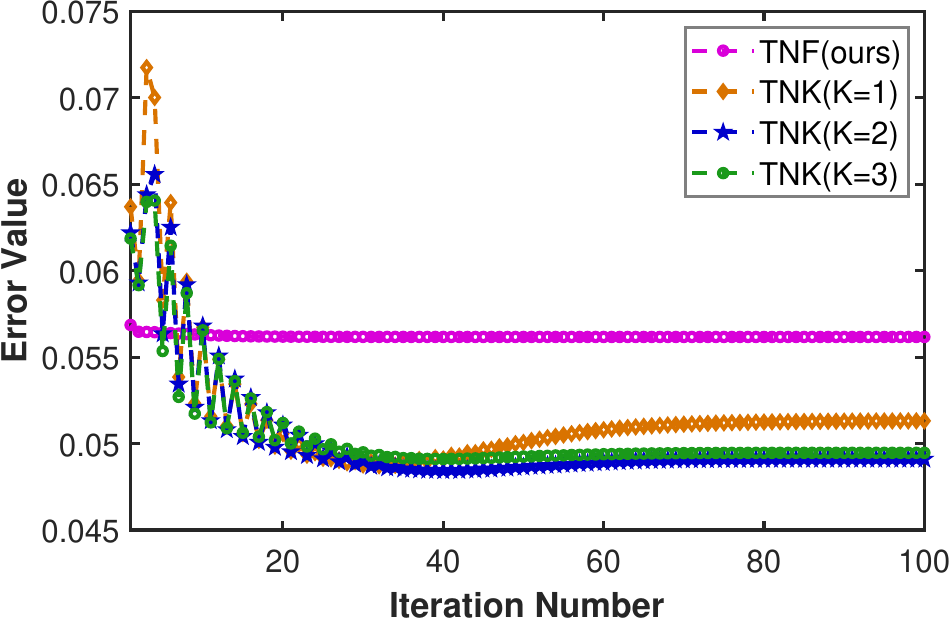}
		\\[+1mm]
		\footnotesize{(c) }  &
		\footnotesize{(d)}  
	\end{tabular}
	%\vspace{-0.15cm}
	\caption{Convergence behavior of the proposed TNPK regularization algorithm for color image completion: (a)-(d) show the convergence curves for four color image experiments with random missing entries, respectively.
	}
	\label{fig:convergence}
\end{figure}

\section{Conclusions}
\label{sec7}
This paper proposes a novel nonconvex surrogate, defined as the ratio of the tensor nuclear norm to the tensor Ky Fan $p $-$k$ norm, for characterizing the tensor tubal rank, and applies it to the LRTC problem. We analyze a series of properties of TNPK and prove, under an NSP-type condition, that a low-rank tensor is a local minimizer of the degenerated model. Extensive experiments on both synthetic and real-world data demonstrate that the proposed method achieves superior performance in terms of recovery accuracy and robustness. To ensure computational tractability, we degenerate the denominator Ky Fan $p $-$k$ norm into the Ky Fan $k$ norm or the Frobenius norm, thereby enabling efficient optimization of the TNPK model.
%Future work will focus on exploring more general settings, investigating the sample complexity and error bounds required for exact recovery, and further addressing the theoretical and computational challenges arising from the nonconvexity and non-separability of the model.
Future work will focus on deriving sharper error bounds for recovery guarantees and extending the proposed TNPK framework to other tensor recovery tasks, such as tensor robust principal component analysis.

\appendices

\section{Proofs of Theorem \ref{theorem:nsp}}
\label{appA}
We first establish two perturbation estimates. Throughout this section, let $\mathcal{X} \in \mathbb{R}^{n_1\times n_2\times n_3} \setminus{\mathcal{O}}$ be a tensor with $\operatorname{rank}_t(\mathcal{X})=r,1\leq k\leq r\leq n$, where $n:=\min\{n_1,n_2\}.$

\begin{Lemma}
	For every tensor $\mathcal{H}
	\in
	\mathbb{R}^{n_1\times n_2\times n_3},$ we have
	\[\|\mathcal{X}+\mathcal{H}\|_*\ge\|\mathcal{X}\|_*+|\mathcal{H}_{r^c}\|_*-|\mathcal{H}_r\|_*.\]
	\label{lemma:nuclear-perturbation}
\end{Lemma}

\begin{proof}
	\vspace{-0.35cm}
	By the tensor counterpart of Mirsky's inequality, we have
	\[\|\mathcal{X}+\mathcal{H}\|_*\ge \sum_{j=1}^{n} |\sigma_j (\mathcal{X} )-\sigma_j (\mathcal{H} )|.
	\]
	Since $\operatorname{rank}_t(\mathcal{X})=r$, it follows that $\sigma_j(\mathcal{X})=0, j=r+1,\ldots,n.$ Therefore,
	\[\begin{aligned}
		\|\mathcal{X}+\mathcal{H}\|_*&\ge \sum_{j=1}^{n} |\sigma_j (\mathcal{X} )-\sigma_j (\mathcal{H} )|+\sum_{j=r+1}^{n}\sigma_j(\mathcal{X})\\
		&\ge \sum_{j=1}^{n} (\sigma_j (\mathcal{X} )-\sigma_j (\mathcal{H} ))+\sum_{j=r+1}^{n}\sigma_j(\mathcal{X})\\
		&=\|\mathcal{X}\|_*-\|\mathcal{H}_r\|_*+\|\mathcal{H}_{r^c}\|_*.
	\end{aligned}\]
This completes the proof.
\end{proof}

\begin{Lemma}
	For every tensor $\mathcal{H}
	\in
	\mathbb{R}^{n_1\times n_2\times n_3},$ we have
	\begin{equation}
		\|\mathcal{X}+\mathcal{H}\|_{(k)}
		\le
		\|\mathcal{X}\|_{(k)}
		+
		\|\mathcal{H}_k\|_*.
		\label{eq:kyfan-perturbation}
	\end{equation}
\label{lemma:kyfan-perturbation}
\end{Lemma}

\begin{proof}
	\vspace{-0.35cm}
	Since the Ky Fan $k$ norm satisfies the triangle inequality, \[\|\mathcal{X}+\mathcal{H}\|_{(k)}
	\le
	\|\mathcal{X}\|_{(k)}
	+
	\|\mathcal{H}\|_{(k)}.\]
	Moreover, \[\|\mathcal{H}\|_{(k)}=\sum_{j=1}^{k}\sigma_j(\mathcal{H})=\|\mathcal{H}_k\|_*.\]
	Combining these relations yields \eqref{eq:kyfan-perturbation}.
\end{proof}

We are now ready to prove Theorem \ref{theorem:nsp}.
\begin{proof}
	Let $N
	:=
	\|\mathcal{X}\|_*,
	K
	:=
	\|\mathcal{X}\|_{(k)}.$ Since $\mathcal{X}\neq\mathcal{O}$ and $k\geq 1,$ we have $K>0.$ Take any tensor $\mathcal{H}
	\in
	\mathcal{N}(F)\setminus{\mathcal{O}}$, for convenience, define
	\[a :=\|\mathcal{H}_k\|_*, \quad b:=\|\mathcal{H}_{r^c}\|_*,\]
	and \[d
	:=
	\|\mathcal{H}_{r^c}\|_*-\|\mathcal{H}_{r}\|_*.\]
	Since $F$ satisfies the NSP-type condition of order $r$ with constant $s\in(0,1)$, we have \[\|\mathcal{H}_r\|_*
	\leq
	s\|\mathcal{H}_{r^c}\|_*.\]
	Because $\|\mathcal{H}_k\|_*
	\leq
	\|\mathcal{H}_r\|_*,$ it follows that 
	\begin{equation}
		a
		\leq
		bs.
		\label{eq:proof-a}
	\end{equation}
	Similarly,
	\begin{equation}
		d
		= \|\mathcal{H}_{r^c}\|_*-\|\mathcal{H}_r\|_*\geq
		b(1-s).
		\label{eq:proof-d}
	\end{equation}
	We also have $b>0.$ Indeed, if $b=0$, then the NSP-type condition implies $\|\mathcal{H}_r\|_*=0.$ Hence, 
	\[\mathcal{H}=\mathcal{H}_r+\mathcal{H}_{r^c}=\mathcal{O},\]
	which contradicts $\mathcal{H}
	\neq
	\mathcal{O}.$
	The assumption $s
	<
	\frac{K}{N+K}$ is equivalent to $(1-s)K-sN>0.$ Combining this inequality with \eqref{eq:proof-a} and \eqref{eq:proof-d}, we obtain
	\[\begin{aligned}
		dK-Na
		&\geq
		b(1-s)K-bsN
		\\
		&=
		b\bigl[(1-s)K-sN\bigr]>
		0.
	\end{aligned}\] 
Consequently, $\frac{N+d}{K+a}
\geq
\frac{N}{K}.$

Now, set $t_*=\frac{1}{2}
\|\mathcal{X}\|_F.$ For every tensor $\mathcal{H}
\in 
\mathcal{N}(F)$ satisfying $\|\mathcal{H}\|_F
\leq
t_*,$ we have
\[\|\mathcal{X}+\mathcal{H}\|_F
\geq
\|\mathcal{X}\|_F-\|\mathcal{H}\|_F
\geq
\frac{1}{2}
\|\mathcal{X}\|_F>0.\]
Thus, $\mathcal{X}+\mathcal{H}
\neq
\mathcal{O},$ and the TNK regularizer is well defined at $\mathcal{X}+\mathcal{H}.$ If $\mathcal{H}=\mathcal{O},$ the desired inequality holds trivially. For $\mathcal{H}
\neq
\mathcal{O},$ Lemmas \ref{lemma:nuclear-perturbation} and \ref{lemma:kyfan-perturbation} yield 
\[\begin{aligned}
	\|\mathcal{X}+\mathcal{H}\|_{\mathrm{TNK}}
	&=
	\frac{
		\|\mathcal{X}+\mathcal{H}\|_*
	}{
		\|\mathcal{X}+\mathcal{H}\|_{(k)}
	}
	\\
	&\geq
	\frac{
		N+d
	}{
		K+a
	}
	\\
	&\geq
	\frac{N}{K}
	\\
	&=
	\|\mathcal{X}\|_{\mathrm{TNK}}.
\end{aligned}\]
Moreover, since $\mathcal{H}
\in
\mathcal{N}(F),$ we have \[F(\mathcal{X}+\mathcal{H})= F(\mathcal{X})=\mathcal{T}.\] Therefore, $\mathcal{X}$ is a local minimizer of problem \eqref{eq:3.3}.
\end{proof}

\section{Proofs of convergence}
\label{appB}
Throughout the convergence analysis, we denote the singular values of a tensor by $\sigma_1, \sigma_2, \ldots, \sigma_r$, where $r$ represents the tubal rank of the tensor. To make this paper self-contained, we present in Lemmas \ref{lemmab1}-\ref{le333} several preliminary results from \cite{wang2021limited}.

\begin{Lemma}[See \cite{wang2021limited}]
	A function $f(\mathbf{x} )$ is called strongly convex with parameter $\mu$ if and only if one of the following conditions holds:
	
	(a) $g(\mathbf{x} )=f(\mathbf{x} )-\frac{\mu }{2} ||\mathbf{x} ||_2^2$ is convex;
	
	(b) $\left \langle \bigtriangledown f(\mathbf{x} )-\bigtriangledown f(\mathbf{y} ), \mathbf{x}-\mathbf{y}  \right \rangle
	\ge \mu ||\mathbf{x} -\mathbf{y} || _2^2 \quad \forall \mathbf{x},\mathbf{y}  $;
	
	(c) $f(\mathbf{y} )\ge f(\mathbf{x} )+\left \langle f(\mathbf{x} ),\mathbf{y}-\mathbf{x}   \right \rangle 
	+\frac{\mu }{2} ||\mathbf{y}-\mathbf{x}  ||_2^2 \quad \forall \mathbf{x},\mathbf{y}  $.
	\label{lemmab1}
\end{Lemma}

\begin{Lemma}[See \cite{wang2021limited}]
	The gradient of $f(\mathbf{x} )$ is Lipschitz continuous with parameter $L > 0$ if and only if one of the following conditions holds:
	
	(a) $||\bigtriangledown f(\mathbf{x} )-\bigtriangledown f(\mathbf{y} )||_2\le L||\mathbf{x}-\mathbf{y}||_2 \quad \forall \mathbf{x},\mathbf{y} $;
	
	(b) $g(\mathbf{x} )=\frac{L }{2} ||\mathbf{x} ||_2^2-f(\mathbf{x} )$ is convex;
	
	(C) $f(\mathbf{y} ) \le  f(\mathbf{x} )+\left \langle f(\mathbf{x} ),\mathbf{y}-\mathbf{x}   \right \rangle 
	+\frac{L}{2} ||\mathbf{y}-\mathbf{x}  ||_2^2 \quad \forall \mathbf{x},\mathbf{y}  $.
	\label{lemmab2}
\end{Lemma}

\begin{Lemma}[See \cite{wang2021limited}]
	Given a function $f(\mathbf{x} )=\frac{1}{||\mathbf{x} ||_2} $ and a set $\mathcal{M}_\epsilon :=\left \{ \mathbf{x}| \left \| \mathbf{x} \right \| _2\ge \epsilon    \right \}  $for a positive constant $\epsilon > 0$, we have 
	\[\left \| \bigtriangledown f(\mathbf{x} )-\bigtriangledown f(\mathbf{y} ) \right \| _2 \le 
	\frac{2}{\epsilon ^3} \left \| \mathbf{x}-\mathbf{y} \right \|_2 \forall \mathbf{x},\mathbf{y}\in \mathcal{M}_\epsilon .\]
	\label{le333}
\end{Lemma}

%%%%B1-B3是否需要

For the tensor Ky Fan $k$ norm, the function is convex but generally nondifferentiable. Therefore, in Lemma \ref{lemmab.4} we investigate its subdifferential properties.

\begin{Lemma}[Subdifferential properties of the Ky Fan $k$ norm]
%	\vspace{-0.35cm}
	Let $f(\mathcal{H} )=\frac{1}{\left \| \mathcal{H}  \right \|_{(k)} } $, and define the set $\mathcal{M}_{\theta ,k}:=\left \{ \mathcal{H}\in \mathbb{R}^{n_1\times n_2\times n_3}|\left \| \mathcal{H}  \right \|_{(k)} \ge \theta     \right \} $, where $\theta >0$ and $k$ is a positive integer. Then for any $\mathcal{H}_1 ,\mathcal{H}_2 \in \mathcal{M}_{\theta ,k} $, there exists a constant $C_k >0$ such that the following statements hold.
	
	(i) For any $\mathcal{G}_1\in  \partial \left \| \mathcal{H}_1  \right \|_{(k)} $, there exists $\mathcal{G}_2\in  \partial \left \| \mathcal{H}_2  \right \|_{(k)} $ such that
	\[\left \| \mathcal{G}_1  -\mathcal{G}_2 \right \|_F \le L_k\left \| \mathcal{H}_1 -\mathcal{H}_2  \right \|_F,  \]
	where $L_k=\frac{2\sqrt{2k}}{\gamma } +1$, and $\gamma$ denotes a positive lower bound of the singular value gap \[\min \left \{ \sigma _k(\mathcal{H}_1 )- \sigma _{k+1}(\mathcal{H}_1 ), \sigma _k(\mathcal{H}_2 )- \sigma _{k+1}(\mathcal{H}_2 )\right \} ,\]
	It is used to measure the gap between the $k$-th and $(k+1)$-th singular values.
	
	(ii) For any $\mathcal{F}_1\in \partial f(\mathcal{H}_1 ) $, there exists $\mathcal{F}_2\in \partial f(\mathcal{H}_2 ) $ such that
	\[\left \| \mathcal{F}_1-\mathcal{F}_2   \right \|_F \le \frac{C_k}{\theta ^3}\left \| \mathcal{H}_1-\mathcal{H}_2   \right \|_F  ,\]
	where $C_k=L_k\theta +2k$.
	\label{lemmab.4}
\end{Lemma}

In the convex but nonsmooth setting, we establish a convergence theory analogous to that under Lipschitz continuity of the gradient. Furthermore, starting from the subdifferential structure of the Ky Fan $k$ norm and incorporating singular vector perturbation analysis, we develop a gradient-Lipschitz–like stability framework and extend it to composite function models. Before proving Lemma \ref{lemmab.4}, we first present several preliminary lemmas and theorem.

Chr{\'e}tien and Wei \cite{chretien2015neumann} extends the von Neumann trace inequality from matrices \cite{rhea2011case} to tensors via the Tucker decomposition. In this work, we adopt a similar tensor extension and apply the von Neumann trace inequality to the tensor setting based on the t-SVD decomposition. %%是否有人在tsvd上用这个不等式？？

\begin{Lemma}[von Neumann trace inequality \cite{chretien2015neumann}]
	Let $\mathcal{X} ,\mathcal{H}\in \mathbb{R}^{n_1\times n_2\times n_3} $ be arbitrary tensors with t-SVD decompositions:
	\[\mathcal{X}=\mathcal{U}_x\ast _L\mathcal{S}_x\ast _L\mathcal{V}_x  ,  \mathcal{H}=\mathcal{U}_h\ast _L\mathcal{S}_h\ast _L\mathcal{V}_h   ,\]
	and let $\sigma _i (\mathcal{X} ),\sigma _i (\mathcal{H} )$ denote the corresponding singular value vectors. Then, the inner product satisfies the followingvon Neumann trace inequality for tensors:
	\[\left \langle \mathcal{X},\mathcal{H}   \right \rangle \le \sum_{i}^{} \sigma _i(\mathcal{X} )\sigma_i (\mathcal{H} ).\]
	\label{lemma:B.5}
\end{Lemma}

\begin{Theorem}[Subgradient of a Support Function \cite{rockafellar1997convex}] %我不确定在这本书里没，网上查到的定理23.4那几页缺失
	Let $C$ be a nonempty convex set, and let its support function be defined by
	\[
	h_C(x) := \max_{y \in C} \langle x, y \rangle.
	\]
	Then the subdifferential of $h_C$ at $x$ is given by
	\[
	\partial h_C(x) = conv \left \{  y^* \in C : \langle x, y^* \rangle = h_C(x) \right \} ,
	\]
	where conv denotes the convex hull.	In other words, the subgradients of a support function are precisely the convex combinations of all points in $C$ that attain the maximum in the definition of $h_C(x)$.
	\label{th:B.6}
\end{Theorem}

\begin{Lemma}[Subdifferential of the Ky Fan $k$ norm]
	Let $\mathcal{H} \in \mathbb{R}^{n_1 \times n_2 \times n_3}$ have the singular value decomposition $\mathcal{H}=\mathcal{U} *_L \Sigma  *_L \mathcal{V}^*$, and define
	\[\mathcal{U}_k=\mathcal{U}_{(:,1:k,:)},\mathcal{V}_k=\mathcal{V}_{(:,1:k)}.\]
	Then the subdifferential of the Ky Fan $k$ norm is given by
	\begin{equation}
		\partial \left \| \mathcal{H}  \right \|_{(k)}:=\begin{Bmatrix}
			\mathcal{U}_k\ast _L\mathcal{V}_k^\ast+\mathcal{W}  ,\quad \begin{matrix}
				\left \| \mathcal{W} \right \|\le1  \\
				\mathcal{U}_k\ast_L\mathcal{W}=\mathcal{O} \\
				\mathcal{W}\ast _L \mathcal{V}_k=\mathcal{O} .
			\end{matrix} 
		\end{Bmatrix}
		\label{eq:lemma:B.7}
	\end{equation}
	\label{lemma:B.7}
\end{Lemma}

\begin{proof}
	In Lemma \ref{lemma:B.5}, if $\left \| \mathcal{X}  \right \|  \le 1$ and $rank_t(\mathcal{X}) \le k$, then
	\[
	\langle  \mathcal{X}  , \mathcal{H} \rangle \le \sum_{i=1}^k \sigma_i(H) = \left \| \mathcal{H}  \right \|_{(k)}.
	\]
	The equality is attained at $\mathcal{X}^* = \mathcal{U}_k \ast_L \mathcal{V}_k^\ast.$ Hence, the Ky Fan $k$ norm can be expressed as the support function of the set
	\[
	C := \left \{ \mathcal{X} : \left \| \mathcal{X}  \right \| \le 1, \ rank_t(\mathcal{X}) \le k  \right \} ,
	\]
	that is,
	%\begin{equation}
	%	\left \| \mathcal{H}  \right \|_{(k)} = \max_{\mathcal{X} \in C} \langle \mathcal{X}, \mathcal{H} \rangle.
	%	\label{eq:support function}
	%\end{equation}
	\[\left \| \mathcal{H}  \right \|_{(k)} = \max_{\mathcal{X} \in C} \langle \mathcal{X}, \mathcal{H} \rangle.\]
	Consequently, by Theorem \ref{th:B.6}, the Ky Fan $k$ norm, regarded as a support function, possesses the following subdifferential property:
	\[\partial \left \| \mathcal{H}  \right \|_{(k)}=conv\left \{ \mathcal{X}^\ast \in C:\left \langle \mathcal{H}, \mathcal{X}^\ast\right \rangle=\left \| \mathcal{H}  \right \| _{(k)}    \right \}  .\]
	
	At this point, any maximizer $\mathcal{X}^\ast$ must satisfy the following conditions:
	(i) $\sigma (\mathcal{X}^\ast  )=\left\{\begin{matrix}
		1 & ,i\le k\\
		0&,i>k
	\end{matrix}\right.$;
	(ii) the left and right singular subspaces of $\mathcal{X}^\ast$ associated with its leading $k$ singular values coincide with the leading $k$ left and right singular subspaces of $\mathcal{H}$, respectively. Accordingly, $\mathcal{X}^\ast$ can be expressed as
	\[\mathcal{X}^\ast =\mathcal{U}_k*_L\mathcal{V}_k^\ast +\mathcal{W} ,\]
	where $\mathcal{W}$ is a perturbation term satisfying
	\[\left \| \mathcal{W} \right \|\le  1,\mathcal{U}_k^\ast *_L\mathcal{W}=0,\mathcal{W}*_L\mathcal{V}=0.\]
	The above orthogonality conditions ensure that
	\[\left \langle \mathcal{W},\mathcal{H}   \right \rangle =0,\]
	and hence the optimal inner product value remains unchanged.
	
	Observe that the set
	\[S:=\left \{ \mathcal{U}_k*_L\mathcal{V}_k^\ast +\mathcal{W} :\left \| \mathcal{W} \right \|\le  1,\mathcal{U}_k^\ast *_L\mathcal{W}=0,\mathcal{W}*_L\mathcal{V}=0\right \} ,\]
	is convex. Hence, its convex hull coincides with itself $conv(S)=S$. Therefore, we arrive at \eqref{eq:lemma:B.7}.
\end{proof}

Next, we extend the Davis–Kahan theorem, given in terms of the SVD in the matrix setting, to the tensor case. To ensure the consistency of the subsequent discussion, we first present the Davis–Kahan theorem in the matrix setting using the SVD formulation.
\begin{Lemma}[See \cite{wedin1972perturbation}]
	Let $\mathbf{A},\mathbf{B}  \in \mathbb{R} ^{n\times m}(n \ge m)$ be two matrices with singular value decompositions
	\[\mathbf{A}=\mathbf{U} \Sigma \mathbf{V} ^T,\mathbf{B} =\tilde{\mathbf{U} } \tilde{\Sigma }  \tilde{\mathbf{V} }^T ,\]
	where $\mathbf{U},\tilde{\mathbf{U}}\in \mathbb{R}^{n\times n} $ are orthogonal, $\Sigma,\tilde{\Sigma } \in \mathbb{R}^{n\times m}$ are diagonal (only the first $\min (n,m)$ diagonal entries are nonzero), and $\mathbf{V},\tilde{\mathbf{V} } \in \mathbb{R}^{m\times m}$ are orthogonal. Denote
	\[\mathbf{U}_k=\mathbf{U}(:,1:k)\in \mathbb{R}^{n\times k},\tilde{\mathbf{U}} _k=\tilde{\mathbf{U}} (:,1:k)\in \mathbb{R}^{n\times k},\]
	the matrices consisting of the first $k$ left singular vectors of $\mathbf{A}$ and $\mathbf{B}$, respectively.
	Assume there exists $\gamma >0$ such that
	\[\sigma _k(\mathbf{A})-\sigma _{k+1}(\mathbf{B})\ge \gamma ,\]
	where $\sigma _j(\cdot )$ denotes the $j$-th largest singular value. Then, there exists an orthogonal matrix $\mathbf{Q}\in\mathbb{R}^{k\times k}$ such that
	\[\left \| \mathbf{U} _k-\tilde{\mathbf{U} }_k \mathbf{Q}  \right \| _F\le \frac{2}{\gamma } \left \| \mathbf{A}-\mathbf{B}   \right \|_F .\]
\end{Lemma}

\begin{Lemma}[Davis–Kahan Perturbation Bound for Tensors]
	Given a tensor $\mathcal{H}_1\in \mathbb{R}^{n_1\times n_2\times n_3},\mathcal{H}_2\in \mathbb{R}^{n_1\times n_2\times n_3}$ and an invertible linear transform $L: \mathbb{R}^{n_1\times n_2\times n_3}\rightarrow \mathbb{R}^{n_1\times n_2\times n_3}$, Suppose that $\mathcal{H}_1,\mathcal{H}_2$ have a singular value gap of at least $\gamma$. Then there exists an orthogonal transformation $\mathcal{Q}$ such that inequality \eqref{lemmaB.3} holds.
	\begin{equation}
		\left \| \mathcal{U}_{1(:,1:k,:)} - \mathcal{U}_{2(:,1:k,:)} *_L\mathcal{Q}  \right \| _F \le \frac{c_0}{\gamma}  \left \| \mathcal{H}_1-\mathcal{H}_2   \right \| _F.
		\label{lemmaB.3}
	\end{equation}
	where $c_0=\frac{\sqrt{2} }{\gamma \sqrt{\ell }  }  $.
	\label{lemmab.9}
	
\end{Lemma}

\textbf{$\mathit{ Proof.}$}  Since
\[\left \| \mathcal{H}_1-\mathcal{H}_2 \right \| _F^2=\frac{1}{\ell }\sum_{i=1}^{n_3} \left \| \mathbf{\bar{H} }_1^{(i)} - \mathbf{\bar{H} }_2^{(i)} \right \|_F^2,\]
We have 
\[\begin{aligned}
	&\left \| \mathcal{U}_{1(:,1:k,:)} - \mathcal{U}_{2(:,1:k,:)} *_L\mathcal{Q}  \right \| _F^2 \\
	&
	=\frac{1}{\ell }
	\sum_{i=1}^{n_3}\left \| \mathbf{U}_1^{(i)} -\mathbf{U}_2^{(i)} \mathbf{Q}^{(i)}\right \|_F^2\\
	&\le \frac{1}{\ell }
	\sum_{i=1}^{n_3}(\frac{\sqrt{2} }{\gamma}\left \| \mathbf{\bar{H} }_1^{(i)} - \mathbf{\bar{H} }_2^{(i)} \right \|_F  )^2\\
	&=\frac{2}{\gamma ^2 \ell} \left \| \mathcal{H}_1-\mathcal{H}_2   \right \| _F^2.
\end{aligned} \]

\begin{Lemma}
	Given a tensor $\mathcal{X} \in \mathbb{R}^{n_1\times n_2\times n_3} $ with tubal rank $r$, then
	\[\left \| \mathcal{X}  \right \| _F\ge \frac{1}{\sqrt{k} } \left \| \mathcal{X}  \right \| _{(k)},\]
	\[\left \| \mathcal{X}  \right \| _F\le \sqrt{r}\left \| \mathcal{X}  \right \|  .\]
	\label{lemmab.10}
\end{Lemma}
%\textbf{$\mathit{ Proof.}$}....%写证明不？

\begin{Lemma}
	For any $\mathcal{H}_1\in \mathbb{R}^{n_1\times n_2\times n_3},\mathcal{H}_2\in \mathbb{R}^{n_1\times n_2\times n_3}$, the following estimate holds:
	\[|\left \| \mathcal{H}_1\right \|_{(k)}-\left \|\mathcal{H}_2\right \|_{(k)} | \le \sqrt{k}\left \| \mathcal{H}_1-\mathcal{H}_2\right \| _F .\] 
	\label{lemmab.11}
\end{Lemma}
The result follows directly from the Cauchy–Schwarz inequality and the definition of the Ky Fan $k$ norm, and we omit the details.

\subsection{Proof of Lemma \ref{lemmab.4}}
\textbf{$\mathit{ Proof.}$} Lemma \ref{lemma:B.7} implies the following bound
\[\begin{aligned}
	&\left \| \mathcal{G}_1-\mathcal{G}_2 \right \| _F\\
	&=\| (\mathcal{U}_{1(:,1:k,:)}*_L\mathcal{V}_{1(:,1:k,:)}^* +\mathcal{W}_1)\\
	&-(\mathcal{U}_{2(:,1:k,:)}*_L\mathcal{V}_{2(:,1:k,:)}^* +\mathcal{W}_2) \|_F\\
	&\le  \|\mathcal{U}_{1(:,1:k,:)}*_L\mathcal{V}_{1(:,1:k,:)}^* 
\end{aligned} \]
\[\begin{aligned}
	&-\mathcal{U}_{2(:,1:k,:)}*_L\mathcal{V}_{2(:,1:k,:)}^* \| _F+\|  \mathcal{W}_1-\mathcal{W}_2\| _F.
\end{aligned} \]
Applying Lemma \ref{lemmab.9}, we obtain the following estimate
\[\begin{aligned}
&\|\mathcal{U}_{1(:,1:k,:)}*_L\mathcal{V}_{1(:,1:k,:)}^* -\mathcal{U}_{2(:,1:k,:)}*_L\mathcal{V}_{2(:,1:k,:)}^* \| _F\\
&=\|\mathcal{U}_{1(:,1:k,:)}*_L\mathcal{V}_{1(:,1:k,:)}^* \\
&-\mathcal{U}_{2(:,1:k,:)}*_L\mathcal{Q}*_L\mathcal{Q}^* *_L\mathcal{V}_{2(:,1:k,:)}^* \| _F\\
&\le \left \| \mathcal{U}_{1(:,1:k,:)}*_L\mathcal{V}_{1(:,1:k,:)}^*-\mathcal{U}_{2(:,1:k,:)}*_L\mathcal{Q}^* *_L\mathcal{V}_{1(:,1:k,:)}^* \right \|_F\\   
&+\| \mathcal{U}_{2(:,1:k,:)}*_L\mathcal{Q}^* *_L\mathcal{V}_{1(:,1:k,:)}^* \\
&-\mathcal{U}_{2(:,1:k,:)}*_L\mathcal{Q}*_L\mathcal{Q}^* *_L\mathcal{V}_{2(:,1:k,:)}^*\|_F\\
&\le \left \| (\mathcal{U}_{1(:,1:k,:)}-\mathcal{U}_{2(:,1:k,:)}*_L\mathcal{Q} )*_L\mathcal{V}_{1(:,1:k,:)}^* \right \|_F   
\\
&+\left \| \mathcal{U}_{2(:,1:k,:)}*_L\mathcal{Q}*_L(\mathcal{V}_{1(:,1:k,:)}^*-\mathcal{Q}^* *_L\mathcal{V}_{2(:,1:k,:)}^*) \right \| _F\\
&\le \left \| \mathcal{U}_{1(:,1:k,:)}-\mathcal{U}_{2(:,1:k,:)}*_L\mathcal{Q} \right \|  \left \| \mathcal{V}_{1(:,1:k,:)} \right \| _F\\
&+\left \| \mathcal{U}_{2(:,1:k,:)}*_L\mathcal{Q} \right \| _F\left \| \mathcal{V}_{1(:,1:k,:)}^*-\mathcal{V}_{2(:,1:k,:)}*_L \mathcal{Q}\right \| \\
&=\left \| \mathcal{U}_{1(:,1:k,:)}-\mathcal{U}_{2(:,1:k,:)}*_L\mathcal{Q} \right \|  \left \| \mathcal{V}_{1(:,1:k,:)} \right \| _F\\
&+\left \| \mathcal{U}_{2(:,1:k,:)} \right \| _F\left \| \mathcal{V}_{1(:,1:k,:)}^*-\mathcal{V}_{2(:,1:k,:)}*_L \mathcal{Q}\right \| \\
&\le \left \| \mathcal{U}_{1(:,1:k,:)}-\mathcal{U}_{2(:,1:k,:)}*_L\mathcal{Q} \right \| _F \left \| \mathcal{V}_{1(:,1:k,:)} \right \| _F
\end{aligned}
\]
\begin{equation}
	\begin{aligned} 
		&+\left \| \mathcal{U}_{2(:,1:k,:)} \right \| _F\left \| \mathcal{V}_{1(:,1:k,:)}^*-\mathcal{V}_{2(:,1:k,:)}*_L \mathcal{Q}\right \| _F\\
		&\le \frac{c_0}{\gamma}  \left \| \mathcal{H}_1-\mathcal{H}_2   \right \| _F\sqrt{k}  +\sqrt{k} \cdot \frac{c_0}{\gamma}  \left \| \mathcal{H}_1-\mathcal{H}_2   \right \| _F\\
		&=\frac{2c_0\sqrt{k} }{\gamma}  \left \| \mathcal{H}_1-\mathcal{H}_2   \right \| _F.
	\end{aligned}
	\label{eq:lemmab.6}
\end{equation}
Since, $\mathcal{W}_2=\mathcal{P}_{\mathcal{U}_2^\bot \mathcal{V}_2^\bot}(\mathcal{W} _1) $, we have 
\[\mathcal{W}_1-\mathcal{W}_2=\mathcal{P}_{\mathcal{U}_2^\bot \mathcal{V}_2^\bot}(\mathcal{W}_1),\]
which yields the following inequality
\begin{equation}
	\begin{aligned} 
		\left \| \mathcal{P}_{\mathcal{U}_2^\bot \mathcal{V}_2^\bot}(\mathcal{W}_1) \right \|_F&\le \| \mathcal{W}_1   \| \| \mathcal{U}_{1(:,1:k,:)}*_L\mathcal{V}_{1(:,1:k,:)}^* \\
		&-\mathcal{U}_{2(:,1:k,:)}*_L\mathcal{V}_{2(:,1:k,:)}^* \|\\
		&\le \frac{2c_0\sqrt{k} }{\gamma}  \left \| \mathcal{H}_1-\mathcal{H}_2   \right \| _F  .
	\end{aligned}
	\label{eq:lemmab.7}
\end{equation}
Combining \eqref{eq:lemmab.6} and \eqref{eq:lemmab.7} gives 
\[\begin{aligned} 
	&\left \| \mathcal{G}_1-\mathcal{G}_2 \right \| _F\\
	&\le \frac{2c_0\sqrt{k} }{\gamma}  \left \| \mathcal{H}_1-\mathcal{H}_2   \right \| _F+\frac{2c_0\sqrt{k} }{\gamma}  \left \| \mathcal{H}_1-\mathcal{H}_2   \right \| _F\\
	&=\frac{4c_0\sqrt{k} }{\gamma}  \left \| \mathcal{H}_1-\mathcal{H}_2   \right \| _F\\
	&=\frac{C\sqrt{k} }{\gamma}  \left \| \mathcal{H}_1-\mathcal{H}_2   \right \| _F.
\end{aligned} \] 
Choosing $L_k=\frac{C\sqrt{k} }{\gamma}+1$ and adjusting by one for the boundary case completes the argument. Lemma \ref{lemmab.4} (i) is thus proved. We now turn to the proof of part (ii).

Let $f(\mathcal{H} )=\frac{1}{\left \| \mathcal{H}  \right \| _{(k)}} $, then 
\[\partial f(\mathcal{H} )=-\frac{1}{\left \| \mathcal{H}  \right \| _{(k)}^2} \partial \left \| \mathcal{H}  \right \| _{(k)}.\]
Given $\mathcal{F}_1\in \partial f(\mathcal{H}_1 )$, there exists $\mathcal{G}_1\in \partial \left \| \mathcal{H}_1  \right \|_{(k)}  $ satisfying
\[\left \| \mathcal{G}_1-\mathcal{G}_2   \right \| _F\le L_k\left \| \mathcal{H}_1-\mathcal{H}_2  \right \|_F .\]
Define $\mathcal{F}_2=-\frac{1}{\left \| \mathcal{H}_2  \right \| _{(k)}^2}\mathcal{G}_2 $, then $\mathcal{F}_2\in \partial f(\mathcal{H}_2 )$. We now estimate $\left \| \mathcal{F}_1-\mathcal{F}_2   \right \|_F$, let \[a=\left \| \mathcal{H}_1  \right \| _{(k)} \ge \theta ,b=\left \| \mathcal{H}_2  \right \| _{(k)} \ge \theta ,\]
then
\[\begin{aligned} 
	&\left \| \mathcal{F}_1-\mathcal{F}_2   \right \|_F\\
	&=\left \| -\frac{1}{a^2} \mathcal{G}_1+\frac{1}{b^2} \mathcal{G}_2   \right \|_F\\
	&=  \left \| -\frac{1}{a^2} \mathcal{G}_1+\frac{1}{a^2} \mathcal{G}_2-\frac{1}{a^2} \mathcal{G}_2+\frac{1}{b^2} \mathcal{G}_2   \right \|_F\\
	&\le  \frac{1}{a^2}\left \| \mathcal{G}_1-\mathcal{G}_2 \right \| _F+\left \| \mathcal{G}_2 \right \| _F\left \|\frac{1}{a^2}-\frac{1}{b^2}  \right \|_F\\
	&\le \frac{L_k}{\theta ^2}  \left \| \mathcal{H}_1-\mathcal{H}_2 \right \| _F+\frac{2k}{\theta ^3} \left \| \mathcal{H}_1-\mathcal{H}_2 \right \| _F\\
	&=\frac{L_k\theta +2k}{\theta ^3}  \left \| \mathcal{H}_1-\mathcal{H}_2 \right \| _F.
\end{aligned}\]
The above argument relies on Lemma \ref{lemmab.4} (i) and Lemma \ref{lemmab.10}.

The above argument relies on  Lemma \ref{lemmab.4} (i), Lemma \ref{lemmab.10} and Lemma \ref{lemmab.11}. This completes the proof of Lemma \ref{lemmab.4} (ii).

\subsection{Proof of Lemma \ref{4lemma1}}
\begin{proof}
	From the optimality condition of the $\mathcal{H}$-subproblem in \eqref{eq:10}, we have
	\[-\frac{\left \| \mathcal{X}^{(t+1)}  \right \|_* }{\left \| \mathcal{H}^{(t+1)}  \right \|_{(k)}^2 }
	\mathcal{G}^{(t+1)}-\mu _1( \mathcal{X}^{(t+1)} +\frac{ \mathcal{C}^{(t)}}{\mu _1}- \mathcal{H}^{(t+1)}  ) =0 ,\]
	where $\mathcal{G}^{(t+1)}=\partial \left \| \mathcal{H}  \right \| _{(k)}$ and $\mathcal{O}\in \mathbb{R}^{n_1\times n_2\times n_3}$ is the zero tensor. Using the dual variable update formula
	\[\mathcal{C}^{(t+1)}=\mathcal{C}^{(t)}+\mu _1(\mathcal{X}^{(t+1)}-\mathcal{H}^{(t+1)}  )  ,\]
	we obtain
	%\begin{equation}
	%	\mathcal{A}^{(t+1)}=-\frac{\left \| \mathcal{X}^{(t+1)}  \right \|_* }{\left \| \mathcal{H}^{(t+1)}  \right \|_{(k)}^2 }\mathcal{G}^{(t+1)}  .
	%	\label{eq:prooflemma4.3_1}
	%\end{equation}
	\[\mathcal{C}^{(t+1)}=-\frac{\left \| \mathcal{X}^{(t+1)}  \right \|_* }{\left \| \mathcal{H}^{(t+1)}  \right \|_{(k)}^2 }\mathcal{G}^{(t+1)}  .\]
	Hence, it directly follows that
	\begin{equation}
		\mathcal{C}^{(t)}=-\frac{\left \| \mathcal{X}^{(t)}  \right \|_* }{\left \| \mathcal{H}^{(t)}  \right \|_{(k)}^2 }\mathcal{G}^{(t)}  .
		\label{b.11}
	\end{equation}
	Consequently, we have
	\[
	\begin{aligned}
		&||\mathcal{C}^{(t+1)} - \mathcal{C}^{(t)}||_F\\
		&=\left \| \frac{\left \| \mathcal{X}^{(t+1)}  \right \|_* }{\left \| \mathcal{H}^{(t+1)}  \right \|_{(k)}^2 }\mathcal{G}^{(t+1)}-\frac{\left \| \mathcal{X}^{(t)}  \right \|_* }{\left \| \mathcal{H}^{(t)}  \right \|_{(k)}^2 }\mathcal{G}^{(t)} \right \|_F\\
		&= \| (\frac{\| \mathcal{X}^{(t+1)}  \|_* }{\| \mathcal{H}^{(t+1)}  \|_{(k)}^2 }\mathcal{G}^{(t+1)}-\frac{\| \mathcal{X}^{(t)}  \|_* }{\| \mathcal{H}^{(t+1)}  \|_{(k)}^2 }\mathcal{G}^{(t+1)})\\
		&+(\frac{\| \mathcal{X}^{(t)}  \|_* }{\| \mathcal{H}^{(t+1)}  \|_{(k)}^2 }\mathcal{G}^{(t+1)}-\frac{\| \mathcal{X}^{(t)}   \|_* }{ \| \mathcal{H}^{(t)}  \|_{(k)}^2 }\mathcal{G}^{(t)} )\|_F
	\end{aligned}\]
	\[\begin{aligned}
		&\le \left \| \frac{1}{\left \| \mathcal{H}^{(t+1)}  \right \|_{(k)}^2 } \mathcal{G}^{(t+1)} (\left \| \mathcal{X}^{(t+1)}  \right \|_*-\left \| \mathcal{X}^{(t)}  \right \|_*)\right \| _F
		\\
		&+\left \| \mathcal{X}^{(t)}  \right \|_* \left \|  \frac{\mathcal{G}^{(t+1)}}{\left \| \mathcal{H}^{(t+1)}  \right \|_{(k)}^2 } -\frac{\mathcal{G}^{(t)}}{\left \| \mathcal{H}^{(t)}  \right \|_{(t)}^2 }\right \| _F.
	\end{aligned}\]
	By Lemma \ref{lemmab.10} and condition C2, together with the relation
	\[\begin{aligned}
		\left | \left \| \mathcal{X}  ^{(t+1)}\right \|_* - \left \| \mathcal{X}  ^{(t)}\right \|_* \right |
		&\le \left \|\mathcal{X}  ^{(t+1)}-\mathcal{X}  ^{(t)}  \right \| _*\\
		&\le \sqrt{n}\left \| \mathcal{X}  ^{(t+1)}-\mathcal{X}  ^{(t)}  \right \|  _F ,
	\end{aligned}\]
	we obtain
	\begin{equation}
		\begin{aligned}
			&\left \| \frac{1}{\left \| \mathcal{H}^{(t+1)}  \right \|_{(k)}^2 } \mathcal{G}^{(t+1)} (\left \| \mathcal{X}^{(t+1)}  \right \|_*-\left \| \mathcal{X}^{(t)}  \right \|_*)\right \| _F\\
			&\le \frac{\sqrt{kn} }{\delta ^2} \left \| \mathcal{X}  ^{(t+1)}-\mathcal{X}  ^{(t)}  \right \|  _F.
		\end{aligned}
		\label{eq:prooflemma4.3_2}
	\end{equation}
	Let $\mathcal{F}(\mathcal{H} )= \frac{\mathcal{G} }{\left \| \mathcal{H}  \right \|_{(k)}^2 } $, then $\mathcal{F}(\mathcal{H})\in \partial (-\frac{1}{\left \| \mathcal{H}  \right \|_{(k)} } )  $. By Lemma \ref{lemmab.4}, there exists $C_k$ such that
	%\begin{equation}
	%	\left \|  \frac{\mathcal{G}^{(t+1)}}{\left \| \mathcal{H}^{(t+1)}  \right \|_{(k)}^2 } -\frac{\mathcal{G}^{(t)}}{\left \| \mathcal{H}^{(t)}  \right \|_{(k)}^2 }\right \| _F
	%	\le \frac{C_k}{\delta ^3} \left \| \mathcal{H}^{(t)}-\mathcal{H}^{(t+1)} \right \| _F,
	%	\label{eq:prooflemma4.3_3}
	%\end{equation}
	\[\left \|  \frac{\mathcal{G}^{(t+1)}}{\left \| \mathcal{H}^{(t+1)}  \right \|_{(k)}^2 } -\frac{\mathcal{G}^{(t)}}{\left \| \mathcal{H}^{(t)}  \right \|_{(k)}^2 }\right \| _F
	\le \frac{C_k}{\delta ^3} \left \| \mathcal{H}^{(t)}-\mathcal{H}^{(t+1)} \right \| _F,\]
	where $C_k=L_k\delta+2k.$ Then we have 
	\begin{equation}
		\begin{aligned}
			&\left \| \mathcal{X}^{(t)}  \right \|_* \left \|  \frac{\mathcal{G}^{(t+1)}}{\left \| \mathcal{H}^{(t+1)}  \right \|_{(k)}^2 } -\frac{\mathcal{G}^{(t)}}{\left \| \mathcal{H}^{(t)}  \right \|_{(k)}^2 }\right \| _F\\
			&\le \frac{MC_k}{\delta ^3}\left \| \mathcal{H}^{(t)}-\mathcal{H}^{(t+1)} \right \| _F. 
		\end{aligned}
		\label{eq:prooflemma4.3_4}
	\end{equation}
	Combining \eqref{eq:prooflemma4.3_2} and \eqref{eq:prooflemma4.3_4}, we obtain
	\[\begin{aligned}
		&||\mathcal{C}^{(t+1)} - \mathcal{C}^{(t)}||_F^2\\
		&\le \frac{2kn}{\delta ^4}||\mathcal{X}^{(t+1)} - \mathcal{X}^{(t)}||_F^2\\
		&+\frac{2M^2C_k^2}{\delta ^6}||\mathcal{H}^{(t+1)} - \mathcal{H}^{(t)}||_F^2.
	\end{aligned}\]
\end{proof}

\subsection{Proof of Lemma \ref{4lemma2}}
\begin{proof}
	Clearly, when $\mathcal{H}^{(t)} $ and $\mathcal{C}^{(t)} $ are fixed, the function $T_{TC}(\mathcal{X},\mathcal{H}^{(t)},\mathcal{C}^{(t)})-I_\Phi (\mathcal{X}) $ is strongly convex with constant $\mu_1$. By Lemma \ref{lemmab1}, we have
	\[\begin{aligned}
		&T_{TC}(\mathcal{X}^{(t+1)},\mathcal{H}^{(t)},\mathcal{C}^{(t)})-I_\Phi (\mathcal{X}^{(t+1)}) \\
		&\le T_{TC}(\mathcal{X}^{(t)},\mathcal{H}^{(k)},\mathcal{C}^{(t)})-I_\Phi (\mathcal{X}^{(t)})-\frac{\mu_1}{2} \left \| \mathcal{X}^{(t+1)}-\mathcal{X}^{(t)} \right \|  .
	\end{aligned}\]
	Since $\mathcal{X}^{(t+1)}$ and $\mathcal{X}^{(t)}$ are the optimal solutions of the $\mathcal{X}$-subproblem at iterations $t$ and $t-1$, respectively, they satisfy the constraint
	\[I_\Phi (\mathcal{X}^{(t+1)})=I_\Phi (\mathcal{X}^{(t)}),\]
	therefore we obtain
	\begin{equation}
		\begin{aligned}
			&T_{TC}(\mathcal{X}^{(t+1)},\mathcal{H}^{(t)},\mathcal{C}^{(t)}) \le\\ &T_{TC}(\mathcal{X}^{(t)},\mathcal{H}^{(t)},\mathcal{C}^{(t)})-\frac{\mu_1}{2} \left \| \mathcal{X}^{(t+1)}-\mathcal{X}^{(t)} \right \|_F^2.
		\end{aligned}
		\label{eq:prooflemma4.4_1}
	\end{equation}
	
	On the other hand, it holds that
	\begin{equation}
		\begin{aligned}
			&T_{TC}(\mathcal{X}^{(t+1)},\mathcal{H}^{(t+1)},\mathcal{C}^{(t)})-T_{TC}(\mathcal{X}^{(t+1)},\mathcal{H}^{(t)},\mathcal{C}^{(t)})\\
			&=\frac{\left \| \mathcal{X}^{(t+1)}  \right \|_* }{\left \| \mathcal{H} ^{(t+1)}\right \|_{(k)} } -\frac{\left \| \mathcal{X}^{(t+1)}  \right \|_* }{\left \| \mathcal{H} ^{(t)}\right \|_{(k)} } \\
			&+\frac{\mu_1}{2} \left \| \mathcal{X}^{(t+1)}-\mathcal{H}^{(t+1)} \right \|_F^2-\frac{\mu_1}{2} \left \| \mathcal{X}^{(t+1)}-\mathcal{H}^{(t)} \right \|_F^2\\
			&+\left \langle \mathcal{C}^{(t)} ,\mathcal{X}^{(t+1)}-\mathcal{H}^{(t+1)}  \right \rangle -\left \langle \mathcal{C}^{(t)} ,\mathcal{X}^{(t+1)}-\mathcal{H}^{(t)}  \right \rangle. 
		\end{aligned}
		\label{eqlemmab17}
	\end{equation}
	We first consider the first term. By Lemma \ref{lemmab2}, it holds that
	\[\begin{aligned}
		&\frac{1}{\left \| \mathcal{H}^{(t+1)}  \right \|_{(k)} } \\
		&\le \frac{1}{\left \| \mathcal{H}^{(t)}  \right \|_{(k)} }+\left \langle -\frac{\mathcal{G}^{(t)}}{\left \| \mathcal{H}^{(t)}  \right \|_{(k)} ^2},\mathcal{H}^{(t+1)}-\mathcal{H}^{(t)}   \right \rangle\\
		&+\frac{C_k}{\delta  ^3}\left \| \mathcal{H}^{(t+1)}-\mathcal{H}^{(t)}  \right \| _F^2  ,
	\end{aligned}\]
	which implies that
	\[\begin{aligned}
		&\frac{\left \| \mathcal{X}^{(t+1)}  \right \|_*}{\left \| \mathcal{H}^{(t+1)}  \right \|_{(k)} }\\
		&\le \frac{\left \| \mathcal{X}^{(t+1)}  \right \|_*}{\left \| \mathcal{H}^{(t)}  \right \|_{(k)} }-\left \langle \frac{\left \| \mathcal{X}^{(t+1)}  \right \|_*\mathcal{G}^{(t)}}{\left \| \mathcal{H}^{(t)}  \right \|_{(k)} ^2},\mathcal{H}^{(t+1)}-\mathcal{H}^{(t)}   \right \rangle
	\end{aligned}\]
	\begin{equation}
		\begin{aligned}
			&+\frac{\left \| \mathcal{X}^{(t+1)}  \right \|_*C_k}{\delta  ^3}\left \| \mathcal{H}^{(t+1)}-\mathcal{H}^{(t)}  \right \| _F^2.
		\end{aligned}
		\label{eqb19}
	\end{equation}
	A straightforward calculation of the third and fourth terms in \eqref{eqlemmab17} yields
	\begin{equation}
		\begin{aligned}
			&\frac{\mu_1}{2} \left \| \mathcal{X}^{(t+1)}-\mathcal{H}^{(t+1)} \right \|_F^2-\frac{\mu_1}{2} \left \| \mathcal{X}^{(t+1)}-\mathcal{H}^{(t)} \right \|_F^2\\
			&=-\frac{\mu_1}{2} \left \| \mathcal{H}^{(t+1)}-\mathcal{H}^{(t)} \right \|_F^2\\
			&-\left \langle \mathcal{C}^{(t+1)}-\mathcal{C}^{(t)},\mathcal{H}^{(t+1)}-\mathcal{H}^{(t)}  \right \rangle ,
		\end{aligned}
		\label{eqb20}
	\end{equation}
	\begin{equation}
		\begin{aligned}
			&\left \langle \mathcal{C}^{(t)} ,\mathcal{X}^{(t+1)}-\mathcal{H}^{(t+1)}  \right \rangle -\left \langle \mathcal{C}^{(t)} ,\mathcal{X}^{(t+1)}-\mathcal{H}^{(t)}  \right \rangle\\
			&=-\left \langle \mathcal{C}^{(t)} ,\mathcal{H}^{(t+1)}-\mathcal{H}^{(t)}  \right \rangle.
		\end{aligned}
		\label{eqb21}
	\end{equation}
	The derivation is analogous to that in \cite{zheng2024scale}, and hence is omitted here for brevity. Putting together \eqref{eqb19}, \eqref{eqb20}, \eqref{eqb21},
	we have
	\[\begin{aligned}
		&T_{TC}(\mathcal{X}^{(t+1)},\mathcal{H}^{(t+1)},\mathcal{C}^{(t)})-T_{TC}(\mathcal{X}^{(t+1)},\mathcal{H}^{(t)},\mathcal{C}^{(t)})\\
		&=-\left \langle \frac{\left \| \mathcal{X}^{(t+1)}  \right \|_*\mathcal{G}^{(t)}}{\left \| \mathcal{H}^{(t)}  \right \|_{(k)} ^2},\mathcal{H}^{(t+1)}-\mathcal{H}^{(t)}   \right \rangle\\
		&+\frac{\left \| \mathcal{X}^{(t+1)}  \right \|_*C_k}{\delta  ^3}\left \| \mathcal{H}^{(t+1)}-\mathcal{H}^{(t)}  \right \| _F^2-\frac{\mu_1}{2} \left \| \mathcal{H}^{(t+1)}-\mathcal{H}^{(t)} \right \|_F^2\\
		&-\left \langle \mathcal{C}^{(t+1)}-\mathcal{C}^{(t)},\mathcal{H}^{(t+1)}-\mathcal{H}^{(t)}  \right \rangle-\left \langle \mathcal{C}^{(t)} ,\mathcal{H}^{(t+1)}-\mathcal{H}^{(t)}  \right \rangle\\
		&=\left \langle \frac{\left \| \mathcal{X}^{(t+1)}  \right \|_*\mathcal{G}^{(t+1)}}{\left \| \mathcal{H}^{(t+1)}  \right \|_{(k)} ^2}-\frac{\left \| \mathcal{X}^{(t+1)}  \right \|_*\mathcal{G}^{(t)}}{\left \| \mathcal{H}^{(t)}  \right \|_{(k)} ^2},\mathcal{H}^{(t+1)}-\mathcal{H}^{(t)}   \right \rangle\\
			&+(\frac{\left \| \mathcal{X}^{(t+1)}  \right \|_*C_k}{\theta ^3}-\frac{\mu_1}{2}) \left \| \mathcal{H}^{(t+1)}-\mathcal{H}^{(t)}  \right \| _F^2\\
	\end{aligned}\]
	\begin{equation}
		\begin{aligned}
			&\le\left \| \frac{\left \| \mathcal{X}^{(t+1)}  \right \|_*\mathcal{G}^{(t+1)}}{\left \| \mathcal{H}^{(t+1)}  \right \|_{(k)} ^2}-\frac{\left \| \mathcal{X}^{(t+1)}  \right \|_*\mathcal{G}^{(t)}}{\left \| \mathcal{H}^{(t)}  \right \|_{(k)} ^2} \right \| _F\left \| \mathcal{H}^{(t+1)}-\mathcal{H}^{(t)}   \right \| _F\\
			&+(\frac{\left \| \mathcal{X}^{(t+1)}  \right \|_*C_k}{\delta  ^3}-\frac{\mu_1}{2}) \left \| \mathcal{H}^{(t+1)}-\mathcal{H}^{(t)}  \right \| _F^2\\
			&\le\frac{MC_k}{\delta  ^3} \left \| \mathcal{H}^{(t+1)}-\mathcal{H}^{(t)}  \right \| _F^2+(\frac{MC_k}{\delta  ^3}-\frac{\mu_1}{2}) \left \| \mathcal{H}^{(t+1)}-\mathcal{H}^{(t)}  \right \| _F^2\\
			&=(\frac{2MC_k}{\delta  ^3}-\frac{\mu_1}{2}) \left \| \mathcal{H}^{(t+1)}-\mathcal{H}^{(t)}  \right \| _F^2.
		\end{aligned}
		\label{eqb22}
	\end{equation}
	
	Lastly, the $\mathcal{C}$-update in  \eqref{eq:44} leads to
	\begin{equation}
		\begin{aligned}
			&T_{TC}(\mathcal{X}^{(t+1)},\mathcal{H}^{(t+1)},\mathcal{C}^{(t+1)})-T_{TC}(\mathcal{X}^{(t+1)},\mathcal{H}^{(t+1)},\mathcal{C}^{(t)})\\
			&=\frac{1}{\mu_1} \left \| \mathcal{C}^{(t+1)} - \mathcal{C}^{(t)}  \right \| _F^2\\
			&\le \frac{2kn}{\mu_1\delta ^4}||\mathcal{X}^{(t+1)} - \mathcal{X}^{(t)}||_F^2+\frac{2M^2C_k^2}{\mu_1\delta ^6}||\mathcal{H}^{(t+1)} - \mathcal{H}^{(t)}||_F^2.
		\end{aligned}
		\label{eqb23}
	\end{equation}
	
	\[\begin{aligned}
		&T_{TC}(\mathcal{X}^{(t+1)},\mathcal{H}^{(t+1)},\mathcal{C}^{(t+1)})-T_{TC}(\mathcal{X}^{(t)},\mathcal{H}^{(t)},\mathcal{C}^{(t)})\\
		&\le -\frac{\mu_1}{2} \left \| \mathcal{X}^{(t+1)}-\mathcal{X}^{(t)} \right \|_F^2+(\frac{2MC_k}{\delta  ^3}-\frac{\mu_1}{2}) \left \| \mathcal{H}^{(t+1)}-\mathcal{H}^{(t)}  \right \| _F^2\\
		&+\frac{2kn}{\mu_1\delta ^4}||\mathcal{X}^{(t+1)} - \mathcal{X}^{(t)}||_F^2\\
		&+\frac{2M^2C_k^2}{\mu_1\delta ^6}||\mathcal{H}^{(t+1)} - \mathcal{H}^{(t)}||_F^2\\
		&\le -(\frac{\mu_1}{2}-\frac{2kn}{\mu_1\delta ^4})||\mathcal{X}^{(t+1)} - \mathcal{X}^{(t)}||_F^2-(\frac{\mu_1}{2}
	\end{aligned}\]
	\[\begin{aligned}
		&-\frac{2MC_k}{\delta  ^3}-\frac{2M^2C_k^2}{\mu_1\delta ^6}) \left \| \mathcal{H}^{(t+1)}-\mathcal{H}^{(t)}  \right \| _F^2.
	\end{aligned}\]
	Let $c_1=\frac{\mu_1}{2}-\frac{2kn}{\mu_1\delta ^4}$ and $c_2=\frac{\mu_1}{2}-\frac{2MC_k}{\delta  ^3}-\frac{2M^2C_k^2}{\mu_1\delta ^6}$. If $\mu_1$ is sufficiently large, we can guarantee that $c_1, c_2$ are strictly positive.
\end{proof}

\subsection{Proof of Lemma \ref{4lemma3}}
\begin{proof}
	We construct and estimate $\mathcal{W}^{(t+1)}= (\mathcal{W}_1^{(t+1)},\mathcal{W}_2^{(t+1)},\mathcal{W}_3^{(t+1)})$ componentwise in the following order. We first consider the subgradient with respect to $\mathcal{X}$, denoted by $\mathcal{W}_1^{(t+1)}$. From the optimality condition of the $\mathcal{X}$-subproblem, there exist $\mathcal{P}^{(t+1)}\in \partial \left \| \mathcal{X}^{(t+1)} \right \|_*,\mathcal{T}^{(t+1)}\in \partial I_\Phi (\mathcal{X} ^{(t+1)})$, such that
	\begin{equation}
		\frac{\mathcal{P}^{(t+1)}}{\left \| \mathcal{H}^{(t)} \right \|_{(k)}} +\mathcal{T}^{(t+1)} +\mu _1(\mathcal{X}^{(t+1)}-\mathcal{H}^{(t)}  )+\mathcal{C}^{(t)} =0.
		\label{B4_1}
	\end{equation}
	Since the subgradient is defined at the point $(\mathcal{X}^{(t+1)},\mathcal{H}^{(t+1)},\mathcal{C}^{(t+1)} )$, we define
	\begin{equation}
		\mathcal{W}_1^{(t+1)}:=\frac{\mathcal{P}^{(t+1)}}{\left \| \mathcal{H}^{(t+1)} \right \|_{(k)}} +\mathcal{T}^{(t+1)} +\mu _1(\mathcal{X}^{(t+1)}-\mathcal{H}^{(t+1)}  )+\mathcal{C}^{(t+1)} .
		\label{B4_2}
	\end{equation}
	It is clear that
	\[\mathcal{W}_1^{(t+1)}\in \partial_\mathcal{X}  T_{TC}(\mathcal{X}^{(t+1)},\mathcal{H}^{(t+1)},\mathcal{C}^{(t+1)} ).\]
	
	It follows from \eqref{B4_1} and \eqref{B4_2} that
	\[\begin{aligned}
		\mathcal{W}^{(t+1)}&=(\frac{1}{\left \| \mathcal{H}^{(t+1)} \right \|_{(k)}} -\frac{1}{\left \| \mathcal{H}^{(t)} \right \|_{(k)}})\mathcal{P}^{(t+1)}\\
		&+\mu _1(\mathcal{H}^{(t)}  -\mathcal{H}^{(t+1)}  )+(\mathcal{C}^{(t+1)}-\mathcal{C}^{(t)}),
	\end{aligned}\]
	
	We estimate the norms of the three terms separately. By the characterization of the subdifferential of the nuclear norm, let $\mathcal{X}^{(t+1)}=\mathcal{U}^{(t+1)}*_L\Sigma ^{(t+1)}*_L (\mathcal{V}^{(t+1)} )^*$ be the t-SVD of $\mathcal{X}^{(t+1)}$. Then 
	\[\begin{aligned}
		&\mathcal{P}^{(t+1)}\\
		&\in \partial \left \| \mathcal{X}^{(t+1)}  \right \|  _*\\
		&= \{ \mathcal{U}^{(t+1)}*_L(\mathcal{V}^{(t+1)} )^*+\mathcal{J}:(\mathcal{U}^{(t+1)})^**_L\mathcal{J}=0,\\
		&\mathcal{J}*_L\mathcal{V}^{(t+1)}  =0 ,\left \| \mathcal{J} \right \|\le 1   \} .
	\end{aligned}
	\]
	Hence,
	\[	\left \| \mathcal{P}^{(t+1)}  \right \|_F=\left \|  \mathcal{U}^{(t+1)}*_L(\mathcal{V}^{(t+1)} )^*+\mathcal{J}^{(t+1)} \right \| _F .\]
	Using the same bounding technique as in \cite{zheng2024scale}, we obtain
	\[\left \| \mathcal{P}^{(t+1)}  \right \|_F\le 2\sqrt{n_1n_2} .\]
	Under Assumption C2, we have
	\[\begin{aligned}
		\left | \frac{1}{\left \| \mathcal{H}^{(t+1)}  \right \|_{(k)} } -\frac{1}{\left \| \mathcal{H}^{(t)}  \right \|_{(k)} }  \right |&=\left | \frac{\left \| \mathcal{H}^{(t)}  \right \|_{(k)} -\left \| \mathcal{H}^{(t+1)}  \right \|_{(k)} }{\left \| \mathcal{H}^{(t+1)}  \right \|_{(k)} \left \| \mathcal{H}^{(t)}  \right \|_{(k)} }  \right |\\
		&\le \frac{\left \| \mathcal{H}^{(t+1)}  \right \|_{(k)}-\left \| \mathcal{H}^{(t)}  \right \|_{(k)}}{\delta ^2}.  
	\end{aligned}\]
	Moreover, by Lemma \ref{lemmab.11}, it holds that
	\[\left | \frac{1}{\left \| \mathcal{H}^{(t+1)}  \right \|_{(k)} } -\frac{1}{\left \| \mathcal{H}^{(t)}  \right \|_{(k)} }  \right |\le \frac{\sqrt{k} }{\delta ^2}\left \| \mathcal{H}^{(t+1)} -\mathcal{H}^{(t)}  \right \| _F .\]
	Combining the above estimates yields
	\[\begin{aligned}
		&(\frac{1}{\left \| \mathcal{H}^{(t+1)} \right \|_{(k)}} -\frac{1}{\left \| \mathcal{H}^{(t)} \right \|_{(k)}})\mathcal{P}^{(t+1)}\\
		&\le \frac{2\sqrt{n_1n_2k} }{\delta ^2}\left \| \mathcal{H}^{(t+1)} -\mathcal{H}^{(t)}  \right \| _F  .
	\end{aligned}\]
	Then,
	\begin{equation}
		\begin{aligned}
			\left \| \mathcal{W}^{(t+1)}_1 \right \| _F&\le \frac{2\sqrt{n_1n_2k} }{\delta ^2}\left \| \mathcal{H}^{(t+1)} -\mathcal{H}^{(t)}  \right \| _F\\
			&+\mu_1\left \| \mathcal{H}^{(t+1)} -\mathcal{H}^{(t)}  \right \| _F+\left \| \mathcal{C}^{(t+1)} -\mathcal{C}^{(t)}  \right \| _F\\
			&=( \frac{2\sqrt{n_1n_2k} }{\delta ^2}+\mu_1)\left \| \mathcal{H}^{(t+1)} -\mathcal{H}^{(t)}  \right \| _F\\
			&+\left \| \mathcal{C}^{(t+1)} -\mathcal{C}^{(t)}  \right \| _F.
		\end{aligned}
		\label{eqb.33}
	\end{equation}
	This completes the estimate of the first term.
	
	Next, we consider the subgradient with respect to $\mathcal{H}$, denoted by $\mathcal{W}^{(t+1)}_2$. From the optimality condition of the $\mathcal{H}$-subproblem, there exists $\mathcal{G}^{(t+1)}\in \partial \left \| \mathcal{H}^{(t+1)}  \right \| _{(k)}$ such that
	\[	-\frac{\rho ^{(t+1)}}{\left \| \mathcal{H}^{(t+1)}  \right \|_{(k)}^2 } \mathcal{G}^{(t+1)}-\mu_1(\mathcal{X}^{(t+1)}-\mathcal{H}^{(t+1)} )-\mathcal{C}^{(t)}  =0,\]
	where $\rho ^{(t+1)}=\left \| \mathcal{X}^{(t+1)}  \right \| _*$.
	Combining this with the multiplier update formula $\mathcal{C}^{(t+1)}=\mathcal{C}^{(t)}+\mu _1(\mathcal{X}^{(t+1)}-\mathcal{H}^{(t+1)}  ) $, we obtain
	\begin{equation}
		\mathcal{C}^{(t+1)}=-\frac{\rho ^{(t+1)}}{\left \| \mathcal{H}^{(t+1)}  \right \|_{(k)}^2 } \mathcal{G}^{(t+1)} .
		\label{eqb35}
	\end{equation}
	We therefore define
	\[\begin{aligned}
		\mathcal{W}_2^{(t+1)}&:=-\frac{\rho ^{(t+1)}}{\left \| \mathcal{H}^{(t+1)}  \right \|_{(k)}^2 } \mathcal{G}^{(t+1)}\\
		&-\mu_1(\mathcal{X}^{(t+1)}-\mathcal{H}^{(t+1)} )-\mathcal{C}^{(t+1)}.
	\end{aligned}\]
	By \eqref{eqb35}, it follows that $\mathcal{W}_2^{(t+1)}=\mathcal{C}^{(t)}-\mathcal{C}^{(t+1)} $. Consequently,
	\[\left \| \mathcal{W}_2^{(t+1)} \right \| _F=\left \| \mathcal{C}^{(t)}-\mathcal{C}^{(t+1)} \right \|_F .\]
	This completes the estimate of the second term. Finally, we consider the subgradient with respect to $\mathcal{C}$, denoted by $\mathcal{W}_3^{(t+1)}$. We define
	\[\mathcal{W}_3^{(t+1)}=\mathcal{X}^{(t+1)}-\mathcal{H}^{(t+1)}  .\]
	By the multiplier update formula $\mathcal{C}^{(t+1)}=\mathcal{C}^{(t)}+\mu _1(\mathcal{X}^{(t+1)}-\mathcal{H}^{(t+1)}  ) $, it follows that
	\[	\mathcal{W}_3^{(t+1)}=\frac{1}{\mu _1} (\mathcal{C}^{(t+1)}-\mathcal{C}^{(t)}  ).\]
	Hence, $\left \| \mathcal{W}_3^{(t+1)} \right \|_F =\frac{1}{\mu _1} \left \| \mathcal{C}^{(t+1)}-\mathcal{C}^{(t)}  \right \|_F .$ This completes the estimate of the third component. 
	
	We next estimate $\left \| \mathcal{W}^{(t+1)} \right \|_F^2 $. By definition, $\left \| \mathcal{W}^{(t+1)} \right \|_F^2 =\left \| \mathcal{W}_1^{(t+1)} \right \|_F^2 +\left \| \mathcal{W}_2^{(t+1)} \right \|_F^2+\left \| \mathcal{W}_3^{(t+1)} \right \|_F^2  $. By \eqref{eqb.33} and the inequality $(a+b)^2\le 2(a^2+b^2)$, we obtain
	\[\begin{aligned}
		\left \| \mathcal{W}^{(t+1)}_1 \right \| _F^2&\le 2( \frac{2\sqrt{n_1n_2k} }{\delta ^2}+\mu_1)^2\left \| \mathcal{H}^{(t+1)} -\mathcal{H}^{(t)}  \right \| _F^2\\
		&+2\left \| \mathcal{C}^{(t+1)} -\mathcal{C}^{(t)}  \right \| _F^2\\
		&\le (\frac{16n_1n_2k }{\delta ^2}+4\mu_1^2)\left \| \mathcal{H}^{(t+1)} -\mathcal{H}^{(t)}  \right \| _F^2\\
		&+2\left \| \mathcal{C}^{(t+1)} -\mathcal{C}^{(t)}  \right \| _F^2.
	\end{aligned}\]
	herefore, the combined estimate is given by:
	\[\begin{aligned}
		\left \| \mathcal{W}^{(t+1)} \right \| _F^2&\le (\frac{16n_1n_2k }{\delta ^2}+4\mu_1^2)\left \| \mathcal{H}^{(t+1)} -\mathcal{H}^{(t)}  \right \| _F^2\\
		&+(3+\frac{1}{\mu _1^2} )\left \| \mathcal{C}^{(t+1)} -\mathcal{C}^{(t)}  \right \| _F^2.
	\end{aligned}\]
	According to Lemma \ref{4lemma1}, we can further simplify the above inequality and obtain
	\[\begin{aligned}
		\left \| \mathcal{W}^{(t+1)} \right \| _F^2&\le (\frac{16n_1n_2k }{\delta ^2}+4\mu_1^2)\left \| \mathcal{H}^{(t+1)} -\mathcal{H}^{(t)}  \right \| _F^2\\
		&+(3+\frac{1}{\mu _1^2} )(\frac{2kn}{\delta ^4}||\mathcal{X}^{(t+1)} - \mathcal{X}^{(t)}||_F^2
	\end{aligned}\]
	\[\begin{aligned}
		&+\frac{2M^2C_k^2}{\delta ^6}||\mathcal{H}^{(t+1)} - \mathcal{H}^{(t)}||_F^2)\\
		&=\kappa_1 ||\mathcal{X}^{(t+1)} - \mathcal{X}^{(t)}||_F^2+\kappa_2  ||\mathcal{H}^{(t+1)} - \mathcal{H}^{(t)}||_F^2,
	\end{aligned}\]
	where $\kappa _1=\frac{2n_1n_2(3\mu _1^2+1)}{\mu_1^2\delta ^4} ,\kappa _2=\frac{16n_1n_2k}{\delta ^4}+4\mu_1^2+\frac{2M^2C_k^2(3\mu _1^2+1)}{\mu _1^2\delta ^6}  .$
\end{proof}

\subsection{Proof of Theorem \ref{th1}}
\begin{proof}
	(i) We first prove that the sequence $\left \{ \mathcal{C}^{(t)}  \right \} $ is bounded. Under Assumptions C1–C2, and by \eqref{b.11}, we obtain
	\[\begin{aligned}
		\left \| \mathcal{C}^{(t)} \right \| _F&=\left \| -\frac{\left \| \mathcal{X}^{(t)}  \right \|_* }{\left \| \mathcal{H}^{(t)}  \right \|_{(t)}^2 }\mathcal{G}^{(t)} \right \|_F\\
		&\le\frac{M \sqrt{t}  }{\delta ^2}.  
	\end{aligned}\]
	Hence, $\left \{ \mathcal{C}^{(t)}  \right \} $ is a bounded sequence.
	Moreover, since the $\mathcal{H}$-update in \eqref{eq:44} is a continuous mapping depending only on the bounded sequences $\left \{ \mathcal{X}^{(t)}  \right \} $ and $\left \{ \mathcal{C}^{(t)}  \right \} $, it follows that $\left \{ \mathcal{H}^{(t)}  \right \} $ is also bounded.
	
	(ii) By Lemma \ref{4lemma2}, we obtain
	\begin{equation}
		\begin{aligned}
			&T_{\mathrm{TC}}\bigl(\mathcal{X}^{(t+1)}, \mathcal{H}^{(t+1)}, \mathcal{C}^{(t+1)}\bigr)\\
			&\le T_{\mathrm{TC}}\bigl(\mathcal{X}^{(0)}, \mathcal{H}^{(0)}, \mathcal{C}^{(0)}\bigr)-c_1\sum_{j=0}^{t-1} ||\mathcal{X}^{(j+1)} \\
			&- \mathcal{X}^{(j)}||_F^2-c_2\sum_{j=0}^{t-2} ||\mathcal{H}^{(j+1)} - \mathcal{H}^{(j)}|_F^2.
		\end{aligned}
		\label{b.45}
	\end{equation}
	On the other hand, it follows from \eqref{eq:4} that 
	\[\begin{aligned}
		&T_{\mathrm{TC}}\bigl(\mathcal{X}^{(t+1)}, \mathcal{H}^{(t+1)}, \mathcal{C}^{(t+1)}\bigr)\\
		&=\frac{\left \| \mathcal{X} ^{(t)} \right \| _*}{\left \| \mathcal{H}^{(t)}  \right \|_{(k)}}+I_\Phi (\mathcal{X} ^{(t)}) 
	\end{aligned}\]
	\[\begin{aligned}
		&+\frac{\mu _1}{2}\left \| \mathcal{X}^{(t)}-\mathcal{H}^{(t)} \right \|_F ^2 +\left \langle \mathcal{C} ^{(t)}, \mathcal{X}^{(t)}-\mathcal{H}^{(t)}\right \rangle\\
		&=\frac{\left \| \mathcal{X} ^{(t)} \right \| _*}{\left \| \mathcal{H}^{(t)}  \right \|_{(k)}}+I_\Phi (\mathcal{X} ^{(t)}) +\frac{\mu _1}{2}\left \| \mathcal{X}^{(t)}-\mathcal{H}^{(t)} +\frac{1}{\mu _1} \mathcal{C}^{(t)} \right \|_F ^2 \\
		&-\frac{1}{2\mu_1} \left \| \mathcal{C}^{(t)}  \right \|_F^2\\
		&\ge  -\frac{1}{2\mu_1} \left \| \mathcal{C}^{(t)}  \right \|_F^2,
	\end{aligned}\]
	Combining this with the boundedness of $\left \{ \mathcal{C}^{(t)}  \right \} $ established in part (i), we conclude that $T_{\mathrm{TC}}\bigl(\mathcal{X}^{(t+1)}, \mathcal{H}^{(t+1)}, \mathcal{C}^{(t+1)}\bigr)$ is bounded from below. Let $t\to \infty$, it follows from \eqref{b.45} that $\sum_{j=0}^{\infty } ||\mathcal{X}^{(j+1)} - \mathcal{X}^{(j)}||_F^2\le \infty ,\sum_{j=0}^{\infty } ||\mathcal{H}^{(j+1)} - \mathcal{H}^{(j)}||_F^2\le \infty$. Hence, $ ||\mathcal{X}^{(j+1)} - \mathcal{X}^{(j)}||_F^2\to 0, ||\mathcal{H}^{(j+1)} - \mathcal{H}^{(j)}||_F^2\to 0,$ as $t\to \infty$. Finally, Lemma \ref{4lemma1} implies that $||\mathcal{X}^{(t+1)} - \mathcal{X}^{(t)}||_F \to 0$.
	
	(iii) From part (i), we have shown that the sequence $\left \{ \mathcal{X}^{(t)}, \mathcal{H}^{(t)}, \mathcal{C}^{(t)} \right \} $ is bounded. By the Bolzano–Weierstrass theorem, there exists a subsequence $\left \{ t_i \right \} \in \mathbb{N} $ such that $(\mathcal{X}^{(t_i)}, \mathcal{H}^{(t_i)}, \mathcal{C}^{(t_i)})\to (\mathcal{X}^{*}, \mathcal{H}^{*}, \mathcal{C}^{*})$, as $i \to \infty$. From part (ii), we already know that $||\mathcal{X}^{(t+1)} - \mathcal{X}^{(t)}||_F \to 0, 
	||\mathcal{H}^{(t+1)} - \mathcal{H}^{(t)}||_F \to 0, 
	||\mathcal{C}^{(t+1)} - \mathcal{C}^{(t)}||_F \to 0 $. Therefore, the same holds along the subsequence: $||\mathcal{X}^{(t_i+1)} - \mathcal{X}^{(t_i)}||_F \to 0, 
	||\mathcal{H}^{(t_i+1)} - \mathcal{H}^{(t_i)}||_F \to 0, 
	||\mathcal{C}^{(t_i+1)} - \mathcal{C}^{(t_i)}||_F \to 0.$ Hence, Lemma \ref{4lemma3} guarantees that the zero tensor $\mathcal{O}$ satisfies $\mathcal{O}\in \partial T_{TC}(\mathcal{X}^{*}, \mathcal{H}^{*}, \mathcal{C}^{*})$. This completes the proof.
\end{proof}

\ifCLASSOPTIONcaptionsoff
  \newpage
\fi

\ifCLASSOPTIONcaptionsoff
  \newpage
\fi

%\newpage

\bibliographystyle{IEEEtran}
%\bibliography{reference_new1//hyperspectral,reference_new1//various_application,reference_new1//orderp_tensor}
%\bibliography{reference_newnew//hyperspectral,reference_newnew//various_application,reference_newnew//orderp_tensor}
%%%%%%%%%%%
\bibliography{rhtc}
%\bibliography{reference_new//hyperspectral,reference_new//various_application,reference_new//orderp_tensor,reference_newnew//rhtc}

\end{document}